%% file: main_arxiv.tex
\documentclass[11pt]{article}
\usepackage{xcolor} 
\input{preamble}

\usepackage{hyperref}
\usepackage{url}
\usepackage[T1]{fontenc}

\usepackage{natbib}
\usepackage[margin=0.8in]{geometry}

\title{Overfitting of Spectral Gradient Descent: How Matrix Geometry shapes Generalization and Implicit Bias}

\author{
  Guillaume Braun \thanks{Equal contribution.}\\
  RIKEN AIP
  \and
  Ichiro Hashimoto\footnotemark[1]\\
  University of Toronto
  \and
  Masaaki Imaizumi\\
  University of Tokyo, RIKEN AIP, Kyoto University
}

\date{}

\begin{document}

\maketitle
\begin{abstract}
We study the generalization of spectral gradient descent (SpecGD) in overparameterized matrix classification with corrupted labels. Each input combines a shared low-rank signal with a rank-one sample-specific perturbation, referred to as a shortcut, that enables memorization but does not generalize. We contrast collapsed shortcuts, which share a singular direction, with dispersed shortcuts, which occupy distinct singular directions. Changing only this geometry can reverse the relative generalization of GD and SpecGD: collapsed shortcuts can favor SpecGD, while dispersed shortcuts can favor GD. In the dispersed regime, exact shortcut orthogonality eliminates the signal from the late-stage SpecGD direction, while vanishing random correlations collectively generate a small but generalization-relevant signal through a second-order effect. To identify the direction selected by SpecGD, which the spectral max-margin problem alone does not determine, we combine a refined analysis of its dual with the exponentiated-gradient dynamics of normalized loss weights. Finally, we show that a single SpecGD step can already interpolate and generalize well, while continued training converges to a direction with substantially worse generalization.
\end{abstract}

\input{intro}

\input{framework}

\input{results}

\input{outline}

\section{Discussion}
Our results show that generalization is shaped jointly by shortcut geometry and optimizer geometry: reorganizing the same memorization features across singular directions can reverse the relative performance of GD and SpecGD. Figure~\ref{fig:klambda} illustrates this transition continuously as the shortcut dimension \(K\) and strength \(\lambda\) vary, while additional experiments in Appendix~\ref{app:xp} explore the effects of sample size, signal rank, and training time. In the dispersed setting, we further show that vanishing random correlations can qualitatively change the SpecGD limit relative to exact orthogonality. SpecGD may also interpolate and generalize after a single step while converging to a substantially worse late-stage classifier, making early stopping beneficial. On the technical side, our second-order dual analysis and dynamical selection argument may extend to other steepest-descent methods with analogous implicit-bias questions, such as SignGD. Natural extensions include higher-rank shortcuts, practical Muon variants, and nonlinear models.

\vspace{-2mm}

\begin{figure}[!htbp]
    \centering
    \includegraphics[width=\textwidth]{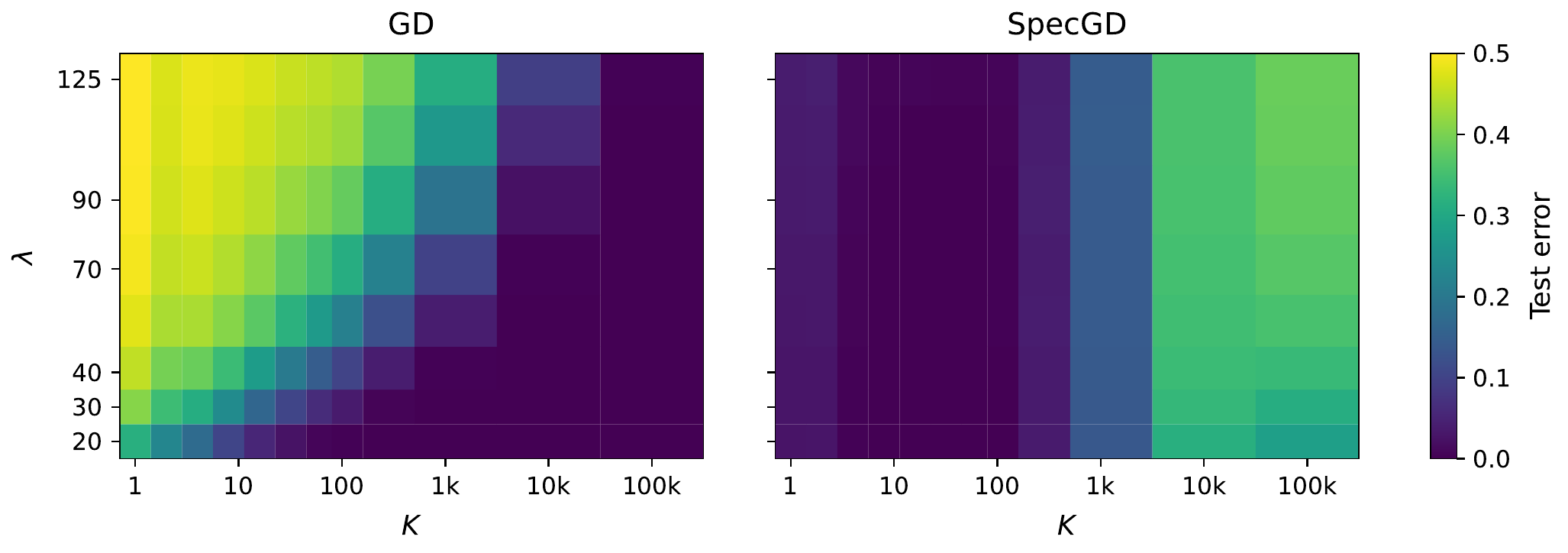}
    \vspace{-7mm}
    \caption{
Test error vs. shortcut dimension $K$ and strength $\lambda$ for GD and SpecGD. Parameters: $n=100$, $r=1$, $D=10^5$, $\rho=0.45$, $||S||_F=||S||_1=1$, $T=10^4$, $\eta=10^{-2}$.
}
    \label{fig:klambda}
\end{figure}

\section*{Acknowledgements} 

\subsection*{AI use statement}

We used generative AI tools to assist with mathematical derivations and proof checking, numerical code development and debugging, manuscript editing, and exploratory discussions. All AI-assisted outputs, including mathematical claims and experiments, were independently verified by the authors, who take full responsibility for the final content of the paper.

\newpage
\bibliographystyle{abbrvnat}
\bibliography{references}

\newpage

\appendix

\input{appendix}

\end{document}

%% file: preamble.tex
\usepackage[utf8]{inputenc}
\usepackage[english]{babel}
\usepackage{amsmath,amssymb,epsfig,amsthm}
\usepackage{graphicx}
\usepackage{comment}
\usepackage{array}
\usepackage{algorithm,algpseudocode}
\usepackage{enumerate}
\usepackage{enumitem}
\usepackage{algpseudocode}
\usepackage{xurl}
\usepackage{xargs}     
\usepackage{subcaption}
\usepackage{mwe}
\usepackage{booktabs}
\usepackage{multirow}

\usepackage{appendix}
\usepackage{hyperref}
\usepackage{textalpha}
\usepackage{thm-restate}

\newcommandx{\unsure}[2][1=]{\todo[linecolor=red,backgroundcolor=red!25,bordercolor=red,#1]{#2}}
\newcommandx{\change}[2][1=]{\todo[linecolor=blue,backgroundcolor=blue!25,bordercolor=blue,#1]{#2}}
\newcommandx{\info}[2][1=]{\todo[linecolor=OliveGreen,backgroundcolor=OliveGreen!25,bordercolor=OliveGreen,#1]{#2}}
\newcommandx{\improvement}[2][1=]{\todo[linecolor=Plum,backgroundcolor=Plum!25,bordercolor=Plum,#1]{#2}}
\newcommandx{\thiswillnotshow}[2][1=]{\todo[disable,#1]{#2}}

\allowdisplaybreaks[1]

\newtheorem{theorem}{Theorem}

\newtheorem{assumption}{Assumption}

\newtheorem{corollary}{Corollary}

\newtheorem{definition}{Definition}

\newtheorem{lemma}{Lemma}

\newtheorem{proposition}{Proposition}
\newtheorem{remark}{Remark}

\numberwithin{equation}{section}

\DeclareMathOperator{\diag}{diag}

\DeclareMathOperator{\spn}{span}

\DeclareMathOperator{\pol}{Polar}

\newcommand{\calD}{\ensuremath{\mathcal{D}}}

\newcommand{\calR}{\ensuremath{\mathcal{R}}}

\newcommand{\calU}{\ensuremath{\mathcal{U}}}

\newcommand{\calV}{\ensuremath{\mathcal{V}}}

\newcommand{\prob}{\ensuremath{\mathbb{P}}}

\newcommand{\E}{\ensuremath{\mathbf{E}}}
\newcommand{\Unif}{\operatorname{Unif}}

\definecolor{asparagus}{rgb}{0.53, 0.66, 0.42}
\definecolor{upforestgreen}{rgb}{0.0, 0.27, 0.13}
\definecolor{tealgreen}{rgb}{0.0, 0.51, 0.5}

\newcommand{\R}{\ensuremath{\mathbb{R}}}

\newcommand{\op}{\mathrm{op}}
\newcommand{\sgn}{\operatorname{sgn}}
\newcommand{\KL}{D_{\mathrm{KL}}}
\newcommand{\T}{\mathcal T}
\newcommand{\D}{\Delta}
\newcommand{\cG}{\mathcal G}
\newcommand{\cH}{\mathcal H}
\newcommand{\cQ}{\mathcal Q}

\newcommand{\ip}[2]{\left\langle #1,#2\right\rangle_F}

\newcommand{\dd}{\mathrm d}

%% file: intro.tex
\section{Introduction}\label{sec:intro}

The remarkable empirical success of Muon~\citep{jordan2024muon,liu2025muonscalablellmtraining,ai2025practicalefficiencymuonpretraining} has made it one of the most prominent alternatives to standard first-order optimizers for training large neural networks. Unlike stochastic gradient descent or Adam~\citep{KingmaB14}, which treat the parameters of a neural network as vectors, Muon explicitly exploits their matrix structure. At an idealized level, it replaces each gradient matrix by its polar factor, thereby equalizing its nonzero singular values rather than preserving their relative magnitudes.

By changing the geometry of the update, spectral gradient methods can substantially reshape training dynamics. In several tractable models, they mitigate the effects of anisotropy~\citep{braun2026spectral} or ill-conditioning~\citep{ma2026preconditioning}, and promote more uniform learning across spectral components~\citep{kang2026uniformspectralgrowthconvergence,vasudeva2025muonsspectraldesignbenefits}, which may in turn favor more robust and transferable feature learning~\citep{ruan2026muonlearnsrobusttransferable}. In a two-layer linear regression model, \citet{dragutinovic} identify a closely related phenomenon: spectral normalization suppresses the sequential simplicity bias of gradient descent by learning singular modes more uniformly. Across these works, the theoretical benefits of spectral normalization are primarily attributed to faster or more balanced learning of different components, rather than to the selection of a different predictor.

These results leave unresolved whether spectral normalization can affect generalization through the late-stage solution itself, rather than only through differences in the training trajectory. 
The closest work for this unresolved point is \citet{fan2025implicitbiasspectraldescent} for separable classification, which show that SpecGD asymptotically maximizes the margin under the spectral norm, revealing an implicit bias that can differ from that of GD. However, their characterization does not in general identify the particular direction selected by the dynamics, since the spectral max-margin problem need not have a unique solution, and therefore does not determine the generalization properties of the resulting classifier. 
This motivates our central question: when does SpecGD select a classifier with different generalization behavior from GD, and which properties of the matrix-valued data determine which optimizer generalizes better?

\textcolor{black}{
To address this question, we study GD and SpecGD in an overparameterized matrix classification model with corrupted labels. Each observation contains a common class signal and a sample-specific rank-one noise, which we refer to as a \textit{shortcut}. 
We vary how these shortcuts are organized, especially, mainly consider two setups: (i) shortcuts collapse along a common singular direction (CSM), and (ii) shortcuts disperse across many directions (DSM). %
This allows us to isolate the effect of matrix structure on generalization and to characterize regimes in which spectral normalization improves or degrades the population risk relative to GD.
}

\subsection{Contributions \& Technical novelty}

Our main contributions are as follows:
\begin{itemize}[leftmargin=*,topsep=2pt,itemsep=2pt,parsep=0pt]
    \item \textbf{%
    Shortcut geometry can reverse the late-stage generalization ordering of GD and SpecGD.} We show that the organization of the shortcuts across singular directions can reverse the relative generalization performance of GD and SpecGD, measured by the population $0$--$1$ risk in the late-stage training. In particular, while the risk of SpecGD $R_{\rm SpecGD}^\star$ is generally smaller than $R_{\rm GD}^\star$ in the collapsed shortcut model (CSM), the performance is reversed in a specific regime of shortcut strength with the dispersed shortcut model (DSM). Figure~\ref{tab:compare} shows the summary. This result implies that the spectral normalization can sometimes have a negative effect to its generalization.

    \vspace{-2mm}
    \begin{figure}[!htbp]
        \centering
        \includegraphics[width=0.98\linewidth]{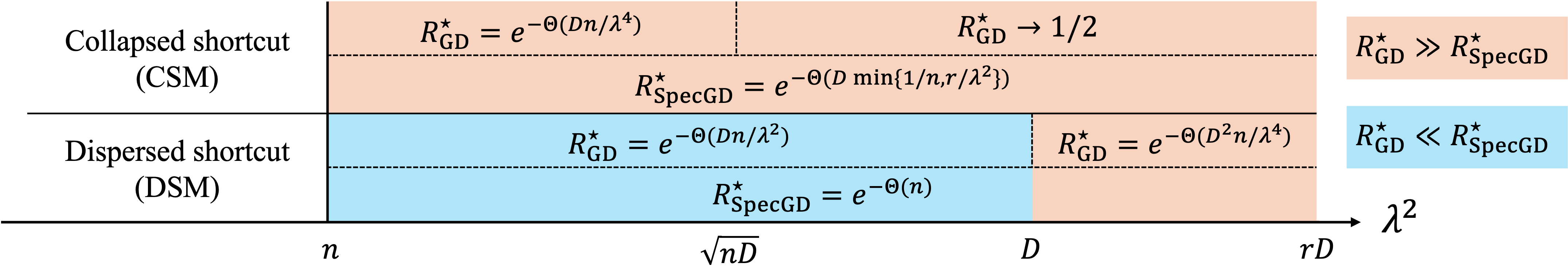}
        \vspace{-2mm}
        \caption{Late-stage test risk across shortcut geometries (CSM/DSM).
        Orange indicates $R_{\rm SpecGD}^\star \ll R_{\rm GD}^\star$, while blue indicates
        $R_{\rm GD}^\star\ll R_{\rm SpecGD}^\star$. Here $n,D,r,\lambda$ denote sample
        size, shortcut dimension, signal rank, and shortcut strength, respectively.}
        \label{tab:compare}
        \vspace{-2mm}
    \end{figure}
\item \textbf{Early-stopping avoids the deterioration of SpecGD in the late-stage.} In both shortcut geometries, we show that a single SpecGD step can already interpolate the corrupted training labels while generalizing well. The risk of SpecGD with early stopping can be smaller than $R_{\rm SpecGD}^\star$, despite already interpolating the corrupted training data. Thus, early stopping can avoid the deterioration induced by the late-stage implicit bias.
\item \textbf{Weak non-orthogonality of dispersed shortcuts improves SpecGD.}
In dispersed geometry, we prove that tiny correlations among dispersed shortcuts collectively generate a signal component to help generalization of SpecGD, whereas exact shortcut orthogonality eliminates the signal from the late-stage SpecGD direction. We identify that this difference is a genuinely second-order effect: a second-order expansion of the nuclear norm of the shortcut aggregate reveals curvature that is absent under exact orthogonality.  %
\end{itemize}

The technical novelty of our work is two-fold:
\begin{itemize}
[leftmargin=*,topsep=2pt,itemsep=2pt,parsep=0pt]
    \item \textbf{A dual-dynamical framework gives directional selection beyond static max-margin theory.}
    Recent work characterizes late-stage implicit of SpecGD in multiclass linear classification by spectral margin maximization. 
    However, a static spectral max-margin characterization is not sufficient to identify the selected solution since the primal max-margin problem is non-unique. Consequently, it does not determine the signal coefficient that governs generalization. We develop a dual-dynamical approach that establishes directional convergence of SpecGD without requiring uniqueness of the primal problem. The key observation is that the normalized exponential-loss weights follow exponentiated-gradient dynamics on the max-margin dual, allowing convergence of the dual iterates to determine the direction selected by SpecGD. 
    \item \textbf{Second-order geometry distinguishes exact from near orthogonality.} Under exact shortcut orthogonality, the dual objective is degenerate and fails to identify a unique dual optimizer. We show that the first nontrivial effect of random near-orthogonality appears at second order. A second-order expansion of the nuclear norm reveals a deterministic curvature term of order $D^{-1}$ together with random fluctuations and higher-order remainders. Sharp uniform concentration bounds show that this curvature dominates, leading to uniqueness and interiority of the dual minimizer, which are the key ingredients to establish directional convergence using the dual-dynamical framework.
    
\end{itemize}

\subsection{Related Work}

\paragraph{Implicit bias.}
On separable data, gradient descent converges in direction toward the Euclidean max-margin classifier~\citep{soudry2024implicitbiasgradientdescent}; more generally, the selected margin geometry can depend on the optimization norm~\citep{gunasekar18a}. Most closely related to our work, \citet{fan2025implicitbiasspectraldescent} show that SpecGD exhibits spectral-norm margin maximization in multiclass linear classification. However, the associated max-margin problem need not have a unique optimizer, so this characterization does not in general identify the direction selected by SpecGD or its generalization. Using a dual-dynamical perspective related to the primal-dual analysis of GD~\citep{ji2021characterizing}, we characterize the selected direction and its generalization despite non-uniqueness of the spectral max-margin solution.

\paragraph{Benign overfitting in classification.}
A broad literature studies when classifiers that interpolate corrupted labels nevertheless achieve small population error, particularly for high-dimensional linear max-margin classifiers~\citep{chatterji2021finitesample,cao21,wang2022binary,hashimoto2026universalitybenignoverfittingbinary}, with extensions to neural networks trained by gradient descent~\citep{frei2023benignoverfitting,hashimoto2026directional}. Related phenomena have also been studied in attention models and Transformers~\citep{magen2025benignattention}. We instead study matrix-linear predictors trained by SpecGD with intrinsically matrix-valued sample-specific features, showing that their organization across singular directions can determine whether interpolation is benign or catastrophic and can reverse the generalization ordering of GD and SpecGD.

\paragraph{Memorization features.}
Overparameterized neural networks can interpolate noisy or even randomly labeled training data \citep{zhang2017understanding,arpit2017closer}, motivating theoretical work on how statistically uninformative features can be used for memorization \citep{bombari2024spuriousfeatures}. Our model separates the shared predictive signal from sample-specific memorization features, allowing us to study how optimization exploits the latter to fit the training labels.%

%% file: framework.tex
\section{Statistical Framework}\label{sec:framework}

\paragraph{Data model.} We consider binary classification with matrix-valued covariates.  For each sample $i=1,\ldots, n $, the ground truth label $y_i \in \{ +1,-1\}$ is drawn independently and uniformly at random. The observed label is $\tilde y_i=\xi_i y_i$  where $\xi_i$ are i.i.d. random variables satisfying $\prob(\xi_i = -1)= \rho = 1 - \prob(\xi_i = 1)$ with $\rho \in (0,1/2)$. The matrix-valued covariates are generated as 
\begin{equation}
\label{eq:model}
X_i=y_iS+Z_i \in \mathbb{R}^{d\times d},
\end{equation}
where \(S\in\R^{d\times d}\) is a deterministic rank-\(r\) signal matrix, and \(Z_i\in\R^{d\times d}\) is a random rank-one sample-specific feature, which we refer to as a \emph{shortcut}. The component \(y_iS\) carries the class information shared across samples, whereas \(Z_i\) has no predictive value for unseen data but can nevertheless be exploited by an overparameterized classifier to memorize corrupted labels. We denote the observed training sample by $\mathcal D_n:=\{(X_i,\tilde y_i)\}_{i=1}^n$ and the uncorrupted and corrupted sample index sets by $I_+ = \{i\mid \xi_i = 1, i = 1, \ldots, n\}$ and $I_- = \{i\mid \xi_i = -1, i = 1, \ldots, n\}$. 

Our model can be viewed as a matrix-valued counterpart of a signal-plus-sample-specific-noise model $\boldsymbol{x}_i=y_i\boldsymbol{\mu}+\boldsymbol{z}_i$ for vectors $\boldsymbol{x}_i, \boldsymbol{\mu}, \boldsymbol{z}_i \in \mathbb{R}^d$.
This model for vectors has been widely studied in the benign-overfitting literature for binary classification \citep{chatterji2021finitesample,cao21,wang2022binary,shamir22,frei2023benignoverfitting,XuGu2023nonsmooth,hashimoto2026directional}. %

\paragraph{Shortcut geometry.}
\textcolor{black}{As a simple motivation, consider images with one spurious modified pixel per sample. If its location varies across the image, the shortcuts are dispersed; if it is restricted to a fixed column, they share a common right direction and collapse to a rank-one structure.}

Let $S=\sum_{l\leq r}\sigma_l\boldsymbol{p}_\ell \boldsymbol{q}_\ell^\top$ be the SVD of $S$. The polar factor of $S$ will be denoted by $P_S:=\operatorname{Polar}(S)
=\sum_{\ell=1}^r \boldsymbol{p}_\ell \boldsymbol{q}_\ell^\top$. We write $s:=\|S\|_1$ for the nuclear norm of $S$. Set $\calU=\spn \lbrace \boldsymbol{p}_1, \ldots, \boldsymbol{p}_r \rbrace^\perp $, $\calV=\spn \lbrace \boldsymbol{q}_1, \ldots, \boldsymbol{q}_r \rbrace^\perp $. Both subspaces have dimension $D=d-r$. For any subspace $\mathcal A$, let $\mathbb S(\mathcal A)$ denote its unit sphere. Let $K\in \lbrace 1, \ldots , D \rbrace $ and consider a subspace $\calV_K\subset \calV$ of dimension $K$. The $K$-dimensional shortcut model is
\begin{equation}
\label{eq:K-shortcut-model}
Z_i=\lambda \boldsymbol{u}_i \boldsymbol{v}_i^\top,
\qquad
\boldsymbol{u}_i\stackrel{\mathrm{i.i.d.}}{\sim}\Unif \bigl(\mathbb S(\calU)\bigr),
\qquad
\boldsymbol{v}_i\stackrel{\mathrm{i.i.d.}}{\sim}\Unif\bigl(\mathbb S(\calV_K)\bigr),
\end{equation}  
where the two families are mutually independent and $\lambda>0$ controls the shortcut strength. 

Our theoretical analysis focuses on the two extreme
cases $K=1$ and $K=D$, while simulations complement the theory by illustrating
the behavior in intermediate regimes $1<K<D$.  
We consider the following setups:
\begin{itemize}[leftmargin=*]
  \setlength{\parskip}{0cm}
  \setlength{\itemsep}{0cm}
    \item \textbf{Collapsed shortcut model} (CSM). We set $K=1$: all shortcuts share a common right singular direction, so any weighted sum of the shortcut matrices remains rank one.
    \item \textbf{Dispersed shortcut model} (DSM). We set \(K=D\): the right shortcut directions are sampled independently from the entire subspace \(\calV\), so it typically has rank \(n\).
    \item \textbf{Deterministic orthogonal benchmark} (DOB). Assume $D\geq n$, and fix orthonormal families $U=[\boldsymbol{u}_1,\ldots,\boldsymbol{u}_n]$ and $V=[\boldsymbol{v}_1,\ldots,\boldsymbol{v}_n]$. We then set $Z_i=\lambda \boldsymbol{u}_i\boldsymbol{v}_i^\top$.
 The shortcuts are now exactly orthogonal, contrast to shortcuts in DSM having small correlations.
\end{itemize}

\paragraph{Classifier and training.}
We consider the matrix linear classifier $f_W(X)=\langle W,X\rangle_F$ where $W\in\R^{d\times d},$
and minimize the empirical exponential loss
\[
\mathcal L(W)
=
\frac{1}{n}\sum_{i=1}^n
\exp\!\left(-\tilde y_i\langle W,X_i\rangle_F\right).
\]
Gradient descent (GD) and spectral gradient descent (SpecGD) iterates are defined by 
\begin{align}
W_{k+1}^{\mathrm{GD}}
&=
W_k^{\mathrm{GD}}
-\eta_k\nabla\mathcal L\!\left(W_k^{\mathrm{GD}}\right),
\\
W_{k+1}^{\mathrm{SpecGD}}
&=
W_k^{\mathrm{SpecGD}}
-\eta_k\operatorname{Polar}\!\Bigl(
\nabla\mathcal L\!\bigl(W_k^{\mathrm{SpecGD}}\bigr)
\Bigr)
\end{align}
for $k=0, 1, 2, \ldots$, with $W_0 = 0$ and stepsizes $(\eta_k)_{k\geq 0}$. Here, $\operatorname{Polar}(A)=A[(A^\top A)^\dagger]^{1/2}$, where $\dagger$ denotes the Moore-Penrose pseudoinverse. SpecGD is the idealized full-batch counterpart of Muon, in which the gradient is replaced by its polar factor.

 As in \citet{fan2025implicitbiasspectraldescent}, we impose the following stepsize condition.
\begin{assumption}
\label{ass:stepsize}
The stepsizes $(\eta_k)_{k\geq0}$ satisfy $\eta_k\downarrow0$ and
$\tau_k:=\sum_{t<k}\eta_t\to\infty$.
\end{assumption}

The following condition is used to establish the rate of directional convergence for \(W_k^{\mathrm{SpecGD}}\). It requires the stepsize to vary slowly across successive iterations, in the sense that \(\eta_{k+1}=\eta_k+o(\eta_k^2)\). In particular, polynomially decaying stepsizes \(\eta_k=(1+k)^{-a}\) with \(a\in(0,1)\) satisfy the condition.%

\begin{assumption}
\label{ass:stepsize1}
The stepsizes $(\eta_k)_{k\geq0}$ satisfy $\lim_{k\to \infty}\eta_k^{-2}(\eta_k - \eta_{k+1})=0$. %
\end{assumption}

\paragraph{Risk.} We measure interpolation using the empirical \(0\)-\(1\) error with respect to the observed labels,
\[
\widehat{\mathcal R}_n(W)
=
\frac{1}{n}\sum_{i=1}^n
\mathbf 1\left\{
\tilde y_i\langle W,X_i\rangle_F\leq 0
\right\}.
\]
We say that \(W\) interpolates the training data if $\widehat{\mathcal R}_n(W)=0$.
For an independent clean test sample \((X,y)\) generated according to the same data model, we define the population \(0\)-\(1\) error as
\[
\mathcal R(W)
=
\prob\left(
y\langle W,X\rangle_F\leq 0
\right).
\]
Thus, interpolation is measured with respect to the corrupted training labels, whereas generalization is measured with respect to the underlying clean labels.

\begin{remark}[Late-stage interpolation]
    For $D\ge n$, the training sample is linearly separable almost
    surely under the CSM and DSM (Lemma~\ref{lem:linear-separability}). The max-margin
    results of
    \citet{gunasekar18a} and \citet{fan2025implicitbiasspectraldescent} therefore imply eventual
    interpolation for GD and SpecGD, respectively.
\end{remark}

%% file: results.tex
\section{Main Results}\label{sec:main_res}
{\color{black}

We begin with an informal presentation of our main statistical finding: shortcut geometry can determine which of GD and SpecGD generalizes better. We then give a precise characterization of the late-stage direction selected by SpecGD, before comparing its generalization with GD. Finally, we investigate the first SpecGD iterate and show how early training can differ from the late-stage limit.

}

\subsection{Shortcut geometry can reverse the ordering of GD and SpecGD}
\label{sec:geometry-reversal}

We first illustrate the geometry-dependent reversal in Figure~\ref{fig:3plots}. We hold all model parameters fixed and change only the shortcut geometry. Both GD and SpecGD interpolate after the first update, so we focus on test error. With collapsed shortcuts, SpecGD generalizes nearly perfectly while GD remains close to chance; under random dispersed shortcuts, the ordering reverses, and GD generalizes substantially better. Thus, reorganizing the same sample-specific shortcuts across singular directions can switch which optimizer generalizes better after interpolation.

Within the dispersed setting, exact orthogonality and random near-orthogonality lead to sharply different late-stage behavior: SpecGD is close to chance in the former, while retaining nontrivial generalization in the latter.

\vspace{-2mm}
\begin{figure}[!htbp]
    \centering
    \includegraphics[width=\textwidth]{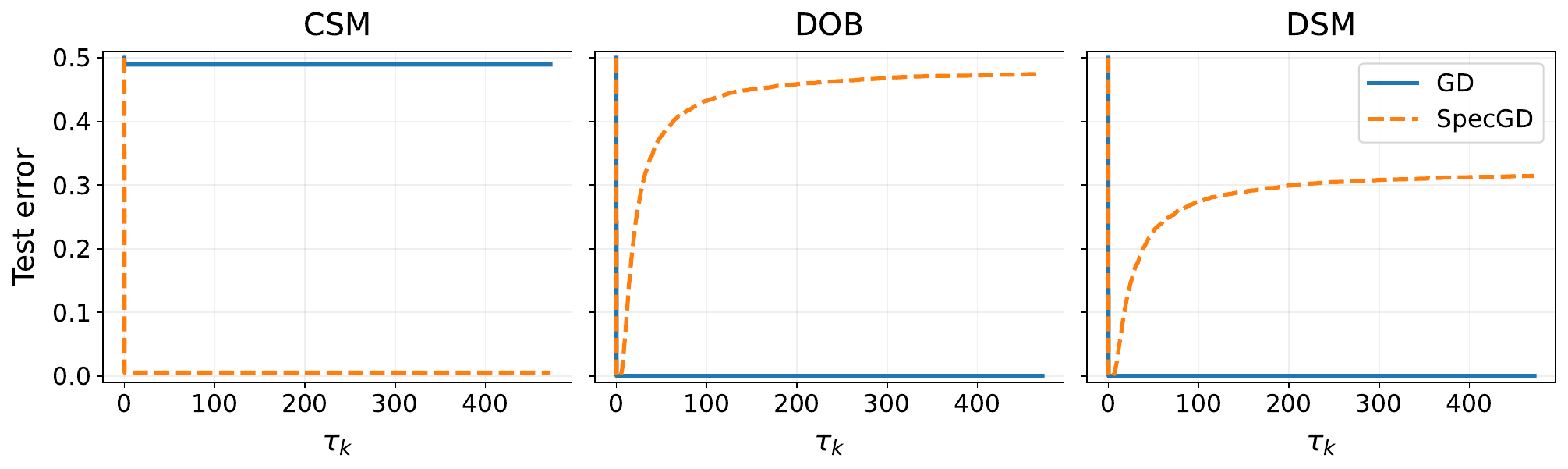}
    \vspace{-7mm}
    \caption{SpecGD generalizes for collapsed shortcuts; GD for dispersed shortcuts. Parameters: $n=100$, $r=1$, $D=10^5$, $\rho=0.45$,
$\lambda=125$, $\sigma_1=1$, $T=3D$ and $\eta_k=0.005(1+k)^{-0.1}$.}%
    \label{fig:3plots}
\vspace{-2mm}
\end{figure}

\subsection{Directional convergence and Generalization of SpecGD}

The two shortcut geometries lead to qualitatively different limiting directions as $k\to \infty$. We begin with the dispersed geometry (DSM and DOB), which exhibits a surprising phenomenon: exact orthogonality eliminates the signal from the convergent direction while tiny correlations under DSM generate a nonzero signal component. We then contrast this with the CSM.%

\begin{theorem}[Directional convergence under DSM]
\label{thm:main_dispersed} Assume the sample $\calD_n$ is generated from the DSM and that Assumption~\ref{ass:stepsize} holds.
Let \(0<\delta<e^{-1}\). 
There exist constants \(c,C>0\), depending only on $\rho$
such that, if $(\log \tfrac{n}{\delta})^5\leq cn, D\geq Cn(\log \tfrac{n}{\delta})^4,$ and $\frac{\lambda n}{sD}\leq c$
then, with probability at least $1-\delta$, the following statements hold.
\vspace{-1mm}
\begin{enumerate}
    \item \textbf{Directional convergence.}
    We have $W_k^{\mathrm{SpecGD}}/\tau_k
    \to
    W_\mathrm{DSM}^\star:=a_\mathrm{DSM}^\star P_S+H_\mathrm{DSM}^\star$, 
    where \(\operatorname{rank}(H_\mathrm{DSM}^\star)=n\) and all nonzero singular
    values of \(H_\mathrm{DSM}^\star\) are equal to \(1\). 
    
    \item \textbf{Signal strength.} Set $\Gamma_\rho := \frac18(1-2\rho)
    \left[
    1+4\rho(1-\rho)
    \right]$. $a_\mathrm{DSM}^\star$ in 1. satisfies
    \begin{equation}\label{eq:a-main-high-probability}
    \left|
    a_\mathrm{DSM}^\star
    -
    \Gamma_\rho\frac{\lambda n}{sD}
    \right|
    \leq
    C\frac{\lambda n}{sD}
    \left(
    \sqrt{\frac{\log(1/\delta)}{n}}
    +
    \frac{(\log \tfrac{n}{\delta})^3}{n}
    +
    \sqrt{\frac{n}{D}}
    \right).
    \end{equation}
    
\end{enumerate}
\end{theorem}

Although $a_\mathrm{DSM}^\star$ is only of order $\lambda n/(sD)$, the corresponding
test error is independent of $D$ at leading order and decays exponentially
with $n$ as shown by the following corollary.

{\color{black}
\begin{corollary}[Generalization under DSM]
\label{cor:generalization-dispersed}
Assume the hypotheses of Theorem~\ref{thm:main_dispersed}. There exist constants \(c,C>0\), depending only on $\rho$ such that, if $(\log \tfrac{n}{\delta})^5\leq cn, D\geq Cn(\log \tfrac{n}{\delta})^4,$ and $\frac{\lambda n}{sD}\leq c$,
then, with probability at least \(1-\delta\) over the training sample
\(\mathcal D_n\), we have
\begin{equation}
\label{eq:generalization-two-sided}
c_0 e^{-C_0 n}
\leq
\mathcal R(W_{\mathrm{DSM}}^\star)
\leq
2e^{-C_1 n}
\end{equation}
for a universal constant $c_0$ and constants $C_0, C_1 >0$ which only depend on $\rho$. Moreover, \(\mathcal R(W_k^{\mathrm{SpecGD}})\to \mathcal R(W^\star_\mathrm{DSM})\).
\end{corollary}

See Section~\ref{sec:outline} for a high-level description of the key proof techniques and Appendix~\ref{app:main_dispersed} for details.

The absence of $D$ from the test-error exponent is due to a matching scaling: while the signal contribution is of order $n/D$, the interaction of the rank-$n$ memorization component with a fresh shortcut is of order $\sqrt{n}/D$, yielding a signal-to-noise ratio of order $\sqrt{n}$.
The small signal in Theorem~\ref{thm:main_dispersed} is therefore statistically significant despite vanishing with $D$.

Residual shortcut correlations generate the small signal. Under DOB, exact orthogonality fixes the shortcut component of each update while the signal coefficient remains bounded. Indeed, we show

}

\begin{proposition}[Directional convergence under DOB]
\label{prop:implicitbias-orthogonal}
Assume $|I_+| > |I_-|>0$. Then, under DOB,
\(W_k^{\mathrm{SpecGD}}\) converges in direction to \(\sum_i \tilde y_i \boldsymbol{u}_i \boldsymbol{v}_i^\top\). 
\end{proposition}

The DSM and DOB spread shortcuts across $n$ right singular directions. We now turn to the opposite extreme, the CSM, where all shortcuts share a single right singular direction. 

\begin{theorem}[Directional convergence under CSM]
\label{thm:main_collapsed}
Assume that the sample \(\mathcal D_n\) is generated from the CSM and that Assumption~\ref{ass:stepsize} holds. Let
\(0<\delta<e^{-1}\).
There exist constants
\(c,C>0\), depending only on \(\rho\), such that, if $\log\frac{n}{\delta}\leq cn$ and $D\geq Cn\log\frac{n}{\delta}$ then, with probability at least \(1-\delta\) over %
\(\mathcal D_n\), the following statements hold.
\vspace{-1mm}
\begin{enumerate}
    \item \textbf{Directional convergence.}
   We have $W_k^{\mathrm{SpecGD}}/\tau_k
    \to
    W_\mathrm{CSM}^\star:=a_\mathrm{CSM}^\star P_S+H_\mathrm{CSM}^\star,$
    where $H_\mathrm{CSM}^\star=h^\star v_0^\top$ for some unit vector $h^\star$ explicitly characterized in equation~\eqref{eq:h_star_specgd}.

    \item \textbf{Signal strength.} Set $\Gamma'_\rho
    :=
    \frac{1-2\rho}{2\sqrt{\rho(1-\rho)}}.$
    $a_\mathrm{CSM}^\star$ in 1. satisfies
    \[
    a_{\mathrm{CSM}}^\star
    =
    \min\left\{
    1,\,
    \frac{\lambda}{s\sqrt n}
    \left(\Gamma_\rho'+\varepsilon_n\right)
    \right\},
    \qquad |\varepsilon_n|
    \leq
    C\left(
    \sqrt{\frac{\log(1/\delta)}{n}}
    +
    \sqrt{\frac nD}
    \right).
    \]
\end{enumerate}
\end{theorem}

In contrast to the DSM, the memorization component now remains rank one and the signal coefficient does not vanish with $D$. For $\lambda \gg s\sqrt n$, $a_\mathrm{CSM}^\star$ saturates to one. The test error nevertheless remains small when \(\lambda\ll s\sqrt D\), and approaches chance when \(\lambda\gg s\sqrt D\).

\begin{corollary}[Generalization under CSM]
\label{cor:generalization-collapsed}
Assume the hypotheses of
Theorem~\ref{thm:main_collapsed}. There exist constants
\(c,C>0\), depending only on \(\rho\), such that, if $\log\frac{n}{\delta}\leq cn$ and $D\geq Cn\log\frac{n}{\delta}$, then, with probability at least
\(1-\delta\) over \(\mathcal D_n\), we have
\[
c_o\exp\left(
-c'D
\min\left\{
\frac{1}{n},
\frac{s^2}{\lambda^2}
\right\}
\right) \leq \mathcal R(W_\mathrm{CSM}^\star)
\leq
2\exp\left(
-cD
\min\left\{
\frac{1}{n},
\frac{s^2}{\lambda^2}
\right\}
\right)
\] where $c_0, c, c'>0$ are constants. Moreover $\mathcal R\left(W_k^{\mathrm{SpecGD}}\right)
\longrightarrow
\mathcal R(W_\mathrm{CSM}^\star).$
\end{corollary}

See Section~\ref{app:main_collapsed} for a detailed proof.

\begin{remark}
For both Theorem~\ref{thm:main_dispersed} and \ref{thm:main_collapsed}, under Assumption~\ref{ass:stepsize1}, the following convergence rate holds 
\begin{equation}\label{eq:convergence_rate}
    \left\|
        \frac{W_k^{\mathrm{SpecGD}}}{\tau_k}
        -
        \left(
            a^\star P_S+H^\star
        \right)
    \right\|_F
    =
    O\left(
        \frac{1+\sum_{t<k}\eta_t^{3/2}}{\tau_k}
    \right).
\end{equation}
\end{remark}

The corollary reveals two regimes. For $\lambda^2\lesssim ns^2$, the late-stage error decays as $e^{-\Theta(D/n)}$; for stronger shortcuts, it degrades as $e^{-\Theta(Ds^2/\lambda^2)}$, and approaches chance once $\lambda^2\gg Ds^2$. Thus, unlike the DSM, late-stage generalization in the CSM depends also on $D,s,\lambda$.

We next compare with GD. The analysis is considerably simpler: the standard implicit-bias results \citep{soudry2024implicitbiasgradientdescent,gunasekar18a} reduce it to the unique Frobenius max-margin classifier. %

Comparison between GD and SpecGD in terms of generalization rates for representative regimes are summarized in Figure~\ref{tab:compare} and Table~\ref{tab:gd-specgd-one-step} when $\sigma_i = 1/\sqrt{r}$, i.e. $\|S\|_1 = \sqrt{r}$ while $\|S\|_F = 1$. Entry $1/2$ denotes convergence of the test error to chance, i.e. catastrophic overfitting. The most striking comparison arises in the parameter regime $n\ll \lambda^2 \ll D$, where the relative performance of SpecGD and GD is reversed by changing only the shortcut geometry.

\vspace{-1mm}
\begin{table}[!htbp]
\centering
\small
\caption{Representataive test risks of GD and SpecGD.
We set $\|S\|_F=1$ and $\|S\|_1=\sqrt r$.
The displayed regimes are chosen so that one-step SpecGD improves over
late-stage SpecGD.}
\label{tab:gd-specgd-one-step}
\begin{tabular}{@{}llccc@{}}
\hline
& Parameter regime
& GD (late)
& SpecGD (late)
& SpecGD (one step)
\\
\hline

\multirow{3}{*}{CSM}

&
$\displaystyle
n\ll\lambda^2\ll r\wedge\sqrt{nD}$
&
$\displaystyle e^{-\Theta(Dn/\lambda^4)}$
&
$\displaystyle e^{-\Theta(D/n)}$
&
$\displaystyle e^{-\Omega(cD)}$
\\[2mm]

&
$\displaystyle
n\vee r\ll\lambda^2
\ll rn\wedge\sqrt{nD}$
&
$\displaystyle e^{-\Theta(Dn/\lambda^4)}$
&
$\displaystyle e^{-\Theta(D/n)}$
&
$\displaystyle e^{-\Omega(rD/\lambda^2)}$
\\[2mm]

&
$\displaystyle
r\vee\sqrt{nD}\ll\lambda^2\ll rn$
&
$\displaystyle \frac12$
&
$\displaystyle e^{-\Theta(D/n)}$
&
$\displaystyle e^{-\Omega(rD/\lambda^2)}$
\\

\hline

\multirow{3}{*}{DSM$^\dagger$}

&
$\displaystyle
n\ll\lambda^2
\ll
D$
&
$\displaystyle e^{-\Theta(Dn/\lambda^2)}$
&
$\displaystyle e^{-\Theta(n)}$
&
$\displaystyle e^{-\Omega(\sqrt r\,D/\lambda)}$
\\[2mm]

&
$\displaystyle
D\ll\lambda^2
\ll
D\sqrt n$
&
$\displaystyle e^{-\Theta(D^2n/\lambda^4)}$
&
$\displaystyle e^{-\Theta(n)}$
&
$\displaystyle e^{-\Omega(\sqrt r\,D/\lambda)}$
\\[2mm]

&
$\displaystyle
D\sqrt n\ll\lambda^2
$
&
$\displaystyle \frac12$
&
$\displaystyle e^{-\Theta(n)}$
&
$\displaystyle e^{-\Omega(\sqrt r\,D/\lambda)}$
\\

\hline
\end{tabular}
\parbox{0.95\linewidth}{\footnotesize
$^\dagger$ DSM rates assume $\lambda^2 \ll \tfrac{rD^2}{n^2}$, which satisfies a condition required for Theorem~\ref{thm:main_dispersed}.}
\vspace{-3mm}
\end{table}

For a more detailed analysis of GD and comparison, see Appendix~\ref{app:GD_comparison} and~\ref{app:comparison}, respectively.

\subsection{One step interpolation and generalization}

We also evaluate SpecGD after a single iteration to complement our late-stage analysis. Since \(W_0=0\), the gradient has the same form in both CSM and DSM and the signal contribution is identical. 
\[
-\operatorname{Polar}\left(\nabla \mathcal L(W_0)\right)
=
P_S
+
\operatorname{Polar}\left(
\sum_{i=1}^n \tilde y_i Z_i\right).
\]
The two models differ only in the aggregate shortcut component. 
In the collapsed model, the aggregate shortcut component has rank one, whereas in the dispersed model it typically has rank \(n\). The following two results show that, despite arising from gradients of the same form, these different shortcut geometries lead to markedly different conditions for interpolation and generalization.

\begin{theorem}[One-step interpolation and generalization under DSM]
\label{thm:onestep_dispersed}
Assume the sample $\calD_n$ is generated from the DSM. There exist universal constants \(c,C>0\) such that the following holds.
Let \(\delta\in(0,1)\) and suppose that $n\geq
\frac{2}{(1-2\rho)^2}\log\frac{4}{\delta}.$
Then, with probability at least \(1-\delta\) over the training sample
\(\mathcal D_n\), the following statements hold.
\vspace{-1mm}
\begin{enumerate}
    \item \textbf{Generalization.} For an independent clean test sample \((X,y)\),
    \[
    \mathcal R(W_{ \mathrm{DSM},1}^{\mathrm{SpecGD}})
    \leq
    2\exp\left(
    -c\min\left\{
    \frac{D^2s^2}{n\lambda^2},
    \frac{Ds}{\lambda}
    \right\}
    \right).
    \]\vspace{-4mm}
    \item \textbf{Interpolation.} If, in addition, \(\lambda>2s\) and $D\geq
    Cn$,
    then $\widehat{\mathcal R}_n
    \left(W_1^{\mathrm{SpecGD}}\right)=0$.
\end{enumerate}
\end{theorem}

By contrast, under CSM, the shortcut contribution is confined to a single singular direction. This changes the one-step benign-overfitting window substantially: interpolation now requires \(\lambda\gtrsim s\sqrt n\), instead of \(\lambda\gtrsim s\), while generalization holds only up to \(\lambda\ll s\sqrt D\), instead of \(\lambda\ll sD/\sqrt n\).

\begin{theorem}[One-step interpolation and generalization under CSM] \label{thm:onestep_collapsed}Assume the sample $\calD_n$ is generated from the collapsed shortcut model. There exist universal constants \(c,C>0\) such that the following holds. Let \(\delta\in(0,1)\) and suppose that
$ n\geq \frac{2}{(1-2\rho)^2}\log\frac{2}{\delta}. $ Then, with probability at least \(1-\delta\) over the training sample \(\mathcal D_n\), the following statements hold. 
\vspace{-1mm}
\begin{enumerate}
    \item \textbf{Generalization.} For an independent clean test sample \((X,y)\), 
    \[ 
    \mathcal R(W_{ \mathrm{CSM},1}^{\mathrm{SpecGD}})
    \leq 2\exp\left( -cD\min\left\{ \frac{s^2}{\lambda^2}, 1 \right\} \right). 
    \] \vspace{-4mm}
    \item \textbf{Interpolation.} If, in addition, $\lambda>2s\sqrt n$ and $D\geq C%
    n\log \frac{4n}{\delta},
    $ then $\widehat{\mathcal R}_n \left(W_1^{\mathrm{SpecGD}}\right)=0$
\end{enumerate}
\end{theorem}

Under $\frac{\lambda n}{sD}\ll 1$, Theorem~\ref{thm:onestep_dispersed} and Corollary~\ref{cor:generalization-dispersed} imply $\calR(W_1^\mathrm{SpecGD}) \ll \calR(W^\star)$. In other words, the late-stage implicit bias of SpecGD provably fails to generalize better than early-stopped SpecGD. %

The proofs of Theorem~\ref{thm:onestep_dispersed} and \ref{thm:onestep_collapsed} are detailed in Section~\ref{app:one_step_dispersed} and \ref{app:one_step_collapsed}, respectively.

%% file: outline.tex
\section{Technical Contributions: Beyond Static Max Margin}\label{sec:outline}

We focus on SpecGD under the DSM, the technically most challenging case. The CSM is simpler because all shortcuts share a common right singular direction, keeping their aggregate rank one. GD is likewise simpler, as its late-stage behavior reduces to the unique Frobenius max-margin classifier.

\paragraph{From margin maximization to directional selection.}
Directional convergence of SpecGD is essential for our generalization analysis in Section~\ref{sec:main_res}.
While \citet{fan2025implicitbiasspectraldescent} guarantees asymptotic margin maximization, it does not identify the direction selected by SpecGD since the primal max-margin solution may be non-unique in our overparameterized setting. Hence the margin maximization characterization does not determine the signal component governing generalization. Even uniqueness of the dual minimizer is insufficient on its own. We instead analyze the dual dynamics directly, bypassing the primal problem.

\paragraph{Dual-dynamical selection.}
Proposition~\ref{prop:dual} gives the dual formulation
\[
    \gamma_\star:= \max_{\|W\|_{\op}\leq 1}\min_i\langle W, \tilde y_i X_i\rangle = 
    \min_{\boldsymbol{\alpha}\in\Delta_n}F(\boldsymbol{\alpha}),
    \qquad
    F(\boldsymbol{\alpha})
    :=
    \left\|
        \sum_{i=1}^n\alpha_i\tilde y_iX_i
    \right\|_1.
\]
For convenience, set $Z(\boldsymbol{\alpha}):=\sum_{i=1}^n\alpha_i\tilde y_i\boldsymbol u_i\boldsymbol v_i^\top$, the dual shortcut aggregate.
Then
\[
    \sum_{i=1}^n\alpha_i\tilde y_iX_i
    =
    \left(\sum_{i=1}^n\alpha_i\xi_i\right)S
    +\lambda Z(\boldsymbol{\alpha}).
\]
The key observation is that the normalized exponential-loss weights
\[
    \alpha_{ik}
    :=
    \frac{
        \exp\{-\tilde y_i
        \langle W_k^{\mathrm{SpecGD}},X_i\rangle_F\}
    }{
        \sum_{j=1}^n
        \exp\{-\tilde y_j
        \langle W_k^{\mathrm{SpecGD}},X_j\rangle_F\}
    }
\]
are exactly exponentiated-gradient iterates for the dual objective $F$; see Lemma~\ref{lem:eg}. Hence the training dynamics induce optimization dynamics directly on the max-margin dual.

Theorem~\ref{thm:dual-convergence} shows that if $F$ admits a unique
interior minimizer $\boldsymbol{\alpha}^\star$ satisfying the local relative growth
condition
\[
    F(\boldsymbol{\alpha})-F(\boldsymbol{\alpha}^\star)
    \ge
    \mu \KL(\boldsymbol{\alpha}^\star\|\boldsymbol{\alpha}),
\]
then $\boldsymbol{\alpha}_k\to\boldsymbol{\alpha}^\star$. Convergence of the associated polar
factors, together with a telescoping argument for the SpecGD iterates,
then yields directional convergence. In particular,
Proposition~\ref{prop:selection} shows
\[
    W_k^{\mathrm{SpecGD}}/\tau_k
    \longrightarrow
    a^\star P_S
    +
    \pol\bigl(Z(\boldsymbol{\alpha}^\star)\bigr).
\]
The dual-dynamical formulation also gives quantitative information that
is not provided by the static max-margin characterization. Under
Assumption~\ref{ass:stepsize1}, we further obtain the convergence rate \eqref{eq:convergence_rate}.

Thus, the dual dynamics identify both the direction selected by SpecGD
and the convergence rate without requiring uniqueness of the primal max-margin problem.

\paragraph{The DSM dual problem.}
Proposition~\ref{prop:selection} reduces the DSM analysis to verifying uniqueness, interiority, and local relative growth of the dual minimizer.

Using the orthogonality between the signal and shortcut blocks, the dual
problem can be written as
\begin{equation}\label{eq:dual}
    \min_{\boldsymbol{\alpha}\in\Delta_n}F(\boldsymbol{\alpha})
    =
    \min_{q\in[0,1]}
    \left\{
        2s\left|q-\frac12\right|
        +
        \lambda\psi(q)
    \right\},
    \qquad
    \psi(q)
    :=
    \inf_{\substack{\boldsymbol{\alpha}\in\Delta_n\\q(\boldsymbol{\alpha})=q}}
    \|Z(\boldsymbol{\alpha})\|_1,
\end{equation}
where 
$q(\boldsymbol{\alpha})
    :=
    \sum_{i\in I_+}\alpha_i
$
is the total dual mass assigned to uncorrupted observations. 
Since 
$\sum_{i=1}^n \alpha_i\xi_i = 2q(\boldsymbol{\alpha})-1$,
$q(\boldsymbol{\alpha})$ directly measures the imbalance between the two groups and hence controls the signal component of the dual aggregate.

By \eqref{eq:dual}, the dual optimization separates into a one-dimensional problem governing the signal balance $q$ and an $(n-2)$-dimensional problem about the dual mass distribution among shortcuts.

\paragraph{Second-order geometry of the DSM.}
We need to address another difficulty for the DSM. The difficulty is that the inner optimization is degenerate under exact shortcut orthogonality; in fact, $\|Z(\boldsymbol{\alpha})\|_1 \equiv 1$, so there is no unique dual minimizer. In contrast, the random DSM shortcuts are only nearly orthogonal, and the difference appears at second order.
Our expansion takes the form
\[
\begin{aligned}
\|Z(\boldsymbol{\alpha})\|_1
=
1
-
\underbrace{
\frac1D\sum_{i<j}h(\alpha_i,\alpha_j)
}_{\text{deterministic curvature}}
\,-\,
\underbrace{
\frac12\sum_{i<j}h(\alpha_i,\alpha_j)
\left(c_{ij}^2-\frac{2}{D}\right)
}_{\text{centered second-order fluctuation}}
\,+\,
R(\boldsymbol{\alpha}),
\end{aligned}
\]
where $c_{ij} = \langle \boldsymbol{u}_i, \boldsymbol{u}_j \rangle - \langle  \tilde y_i\boldsymbol{v}_i, \tilde y_j\boldsymbol{v}_j \rangle$ and $h(x,y):=\frac{xy}{x+y}$.
The expansion makes clear why random near-orthogonality is fundamentally
different from exact orthogonality. Under exact orthogonality,
$c_{ij}=0$, so the second-order fluctuation equals
$\frac1D\sum_{i<j}h(\alpha_i,\alpha_j)$,
and exactly cancels the deterministic curvature term, recovering the
flat objective $\|Z(\boldsymbol{\alpha})\|_1\equiv1$. In the random DSM, however,
$c_{ij}^2$ concentrates around its mean $2/D$. Hence the centered
second-order fluctuation is small, and the deterministic
$D^{-1}$-order curvature survives.

Thus, the crucial term is the deterministic curvature term. Proposition~\ref{prop:strong-concavity} quantifies the curvature of $\|Z(\boldsymbol{\alpha})\|_1$,
and Proposition~\ref{prop:inner} uses it to identify a unique interior
slice-wise optimizer $\boldsymbol{\alpha}^\star(q)$.

The main technical difficulty is to prove that this deterministic
curvature survives in the random problem. This requires sharp
\emph{uniform} concentration of the centered second-order fluctuations and simultaneous control of the higher-order remainder $R(\boldsymbol{\alpha})$ over a neighborhood of the symmetric path of dual weights. Toward this end, we combine concentration inequalities for degenerate matrix U-statistics with uniform bounds on the higher-order remainder and its derivatives; see Subsection~\ref{sec:uniform-local-random-geometry}.  

Finally, Proposition~\ref{prop:Vprime} controls the size of $\psi^\prime(1/2)$. Combined with the cusp of
$2s|q-1/2|$, Proposition~\ref{prop:balanced} shows that the global
minimizer satisfies $q(\boldsymbol{\alpha}^\star)=1/2$. The same second-order curvature yields
the local relative growth condition required by
Theorem~\ref{thm:dual-convergence}; see
Proposition~\ref{prop:local-growth}. These results verify the hypotheses
of Proposition~\ref{prop:selection} and complete the analysis for the DSM.

%% file: appendix.tex
\section{Additional related work}

\paragraph{Additional implicit-bias literature.}
Implicit bias has been studied as a possible mechanism for understanding generalization in overparameterized models~\citep{vardi2022implicitbiasdeeplearningalgorithms}, including margin-maximizing behavior in homogeneous and nearly homogeneous neural networks~\citep{Lyu2020Gradient,ji2020directionalconvergencealignmentdeep,cai2025implicitbiasgradientdescent,schechtman2025latestage}. Recent work also studies implicit bias and optimization geometry for Muon and related spectral methods~\citep{gronich2026the,li2026the,chen2025muon}.

\paragraph{Additional benign-overfitting literature.}
Further results on benign overfitting in classification include~\citep{shamir22,XuGu2023nonsmooth}, while recent work studies analogous phenomena in attention models and Transformers~\citep{magen2025benignattention,pmlr-v267-sakamoto25a,jiang2024unveil}.

\section{Directional convergence via dual framework}\label{app:dual_framework}
\input{dual_framework}

\section{Proof of Theorem~\ref{thm:main_dispersed}}\label{app:main_dispersed}
\input{proof_main_dispersed_WSpecGD_compressed}

\section{Proof of Theorem~\ref{thm:main_collapsed}}\label{app:main_collapsed}

\input{proof_main_collapsed}

\section{Analysis of GD}\label{app:GD_comparison}

\input{Proof_GD}

\section{Proof of Theorem~\ref{thm:onestep_dispersed}}\label{app:one_step_dispersed}
\input{Proof_onestep_dispersed}

\section{ Proof of Theorem~\ref{thm:onestep_collapsed}}\label{app:one_step_collapsed}

\input{one_step_collapsed}

\section{Auxiliary facts}
\label{sec:auxiliary}
\input{auxiliary}

\section{Additional numerical experiments}\label{app:xp}
\input{app_xp}

\section{Comparison between SpecGD and GD}\label{app:comparison}
The following table summarizes our theoretical findings.

\begin{table}%
\centering
\caption{Comparison of the late-stage generalization of GD and SpecGD.}
\label{tab:gd-specgd-comparison_complete}
\small
\begin{tabular}{@{}llccc@{}}
\toprule
& Parameter regime
& $R(W_{\rm GD}^\star)$
&
& $R(W^\star)$
\\
\midrule

\multirow{5}{*}{CSM}

&
$\lambda^2\ll n\|S\|_F^2$
&
$\displaystyle
e^{-\Theta\!\left(
\frac{D}{n}
\right)}$
&
&
$\displaystyle
e^{-\Theta\!\left(
\frac{D}{n}
\right)}$
\\[4mm]

&
$n\|S\|_F^2
\ll \lambda^2
\ll
\sqrt{nD}\|S\|_F^2
\wedge n\|S\|_1^2$
&
$\displaystyle
e^{-\Theta\!\left(
\frac{Dn\|S\|_F^4}{\lambda^4}
\right)}$
&
$\gg$
&
$\displaystyle
e^{-\Theta\!\left(
\frac{D}{n}
\right)}$
\\[4mm]

&
$n\|S\|_1^2
\ll \lambda^2
\ll
\sqrt{nD}\|S\|_F^2$
&
$\displaystyle
e^{-\Theta\!\left(
\frac{Dn\|S\|_F^4}{\lambda^4}
\right)}$
&
$\gg$
&
$\displaystyle
e^{-\Theta\!\left(
\frac{D\|S\|_1^2}{\lambda^2}
\right)}$
\\[4mm]

&
$\sqrt{nD}\|S\|_F^2
\ll \lambda^2
\ll D\|S\|_1^2$
&
$\frac12$
&
$\gg$
&
$\displaystyle
e^{-\Theta\!\left(
\frac{D}{n}
\wedge
\frac{D\|S\|_1^2}{\lambda^2}
\right)}$
\\[4mm]

&
$\lambda^2\gg D\|S\|_1^2$
&
$\frac12$
&
&
$\frac12$
\\

\midrule

\multirow{8}{*}{DSM$^1$}

&
$\displaystyle
\lambda^2
\ll
n\|S\|_F^2
\wedge
\frac{D^2\|S\|_1^2}{n^2}$
&
$\displaystyle
e^{-\Theta(D)}$
&
$\ll$
&
$\displaystyle
e^{-\Theta(n)}$
\\[4mm]

&
$\displaystyle
n\|S\|_F^2
\ll \lambda^2
\ll
D\|S\|_F^2
\wedge
\frac{D^2\|S\|_1^2}{n^2}$
&
$\displaystyle
e^{-\Theta\!\left(
\frac{Dn\|S\|_F^2}{\lambda^2}
\right)}$
&
$\ll$
&
$\displaystyle
e^{-\Theta(n)}$
\\[4mm]

&
$\displaystyle
D\|S\|_F^2
\ll \lambda^2
\ll
D\sqrt n\,\|S\|_F^2
\wedge
\frac{D^2\|S\|_1^2}{n^2}$
&
$\displaystyle
e^{-\Theta\!\left(
\frac{D^2n\|S\|_F^4}{\lambda^4}
\right)}$
&
$\gg$
&
$\displaystyle
e^{-\Theta(n)}$
\\[4mm]

&
$\displaystyle
D\sqrt n\,\|S\|_F^2
\ll \lambda^2
\ll
\frac{D^2\|S\|_1^2}{n^2}$
&
$\frac12$
&
$\gg$
&
$\displaystyle
e^{-\Theta(n)}$
\\[4mm]

&
$\displaystyle
\frac{D^2\|S\|_1^2}{n^2}
\ll \lambda^2
\ll
n\|S\|_F^2$
&
$\displaystyle
e^{-\Theta(D)}$
&
&
\emph{not characterized}
\\[4mm]

&
$\displaystyle
n\|S\|_F^2
\vee
\frac{D^2\|S\|_1^2}{n^2}
\ll \lambda^2
\ll
D\|S\|_F^2$
&
$\displaystyle
e^{-\Theta\!\left(
\frac{Dn\|S\|_F^2}{\lambda^2}
\right)}$
&
&
\emph{not characterized}
\\[4mm]

&
$\displaystyle
D\|S\|_F^2
\vee
\frac{D^2\|S\|_1^2}{n^2}
\ll \lambda^2
\ll
D\sqrt n\,\|S\|_F^2$
&
$\displaystyle
e^{-\Theta\!\left(
\frac{D^2n\|S\|_F^4}{\lambda^4}
\right)}$
&
&
\emph{not characterized}
\\[4mm]

&
$\displaystyle
\lambda^2
\gg
D\sqrt n\,\|S\|_F^2
\vee
\frac{D^2\|S\|_1^2}{n^2}$
&
$\frac12$
&
&
\emph{not characterized}
\\

\bottomrule
\\[-3mm]
\multicolumn{5}{l}{
\parbox{0.96\linewidth}{\footnotesize
$^1$ The relative magnitude of
$\frac{D^2\|S\|_1^2}{n^2}$
and the GD transition scales
$n\|S\|_F^2$, $D\|S\|_F^2$, and
$D\sqrt n\,\|S\|_F^2$
depends on $D,n,S$; consequently, some displayed parameter regimes may be empty.
}}
\end{tabular}
\end{table}

%% file: dual_framework.tex
In this section, we present our dual framework for establishing directional convergence, separating the general dual-dynamical mechanism from the model-specific arguments required to identify and control the limiting dual solution.

Throughout, set $N_i:=\tilde y_i \boldsymbol u_i\boldsymbol v_i^\top$ and
$A_i:=\tilde y_iX_i=\xi_iS+\lambda N_i$.

\subsection{Block geometry}

Choose orthonormal complements \(P_\perp\) and \(Q_\perp\).  Define \(\mathcal S := \{PMQ^\top:M\in\R^{r\times r}\}\)
and \(\mathcal N := \{P_\perp NQ_\perp^\top:N\in\R^{D\times D}\}\).
Then \(S\in\mathcal S\), every \(N_i\in\mathcal N\), and matrices in
\(\mathcal S\) and \(\mathcal N\) have mutually orthogonal row and
column spaces.

\begin{lemma}[Polar factor and Schatten 1-norm of orthogonal blocks]
\label{lem:block}
For every \(c\in\R\) and \(Z\in\mathcal N\), \(\pol(cS+\lambda Z) = \sgn(c)P_S+\pol(Z)\),
where the first term is zero when \(c=0\).  Moreover, \(\|cS+\lambda Z\|_1 = |c|s+\lambda\|Z\|_1\).
\end{lemma}

\begin{proof}
Relative to the orthogonal row and column bases
\([P\ P_\perp]\) and \([Q\ Q_\perp]\), the matrix is block diagonal:
\[
\begin{pmatrix}
c\Sigma&0\\
0&\lambda\overline Z
\end{pmatrix}.
\]
The singular values are the union of the singular values of the two
blocks, which proves Schatten-1-norm additivity.  Combining the compact
singular value decompositions of the two blocks gives the polar-factor
identity.
\end{proof}

\subsection{Primal--dual reduction}

The spectral max-margin value is \(\gamma_\star := \max_{\|W\|_{\op}\leq1} \min_{i\in[n]}\ip{W}{A_i}\).

\begin{proposition}[Exact dual and one-dimensional reduction]
\label{prop:dual}
We have
\[
\gamma_\star
=
\min_{\boldsymbol{\alpha}\in\D_n}
F(\boldsymbol{\alpha}),
\qquad
F(\boldsymbol{\alpha})
:=
\left\|
\sum_{i=1}^n\alpha_iA_i
\right\|_1.
\]
If \(q(\boldsymbol{\alpha}):=\sum_{i\in I_+}\alpha_i, \quad Z(\boldsymbol{\alpha}):=\sum_{i=1}^n\alpha_i N_i\),
then \(F(\boldsymbol{\alpha}) = 2s\left|q(\boldsymbol{\alpha})-\frac12\right| + \lambda\|Z(\boldsymbol{\alpha})\|_1\).
Consequently,
\begin{equation}
\label{eq:reduced-dual}
\gamma_\star
=
\min_{q\in[0,1]}
\left\{
2s\left|q-\frac12\right|
+\lambda\psi(q)
\right\},
\end{equation}
where
\[
\psi(q)
:=
\min_{\substack{\boldsymbol{\alpha}\in\D_n\\q(\boldsymbol{\alpha})=q}}
\|Z(\boldsymbol{\alpha})\|_1.
\]
The function \(\psi\) is convex.
\end{proposition}

\begin{proof}
For every \(W\), \(\min_i\ip{W}{A_i} = \min_{\boldsymbol{\alpha}\in\D_n} \sum_i\alpha_i\ip{W}{A_i}\).
Sion's minimax theorem gives \(\gamma_\star = \min_{\boldsymbol{\alpha}\in\D_n} \max_{\|W\|_{\op}\leq1} \ip{W}{\sum_i\alpha_iA_i}\).
Spectral--Schatten-1 duality yields \(\max_{\|W\|_{\op}\leq1}\ip{W}{M} = \|M\|_1\).
This proves the first identity.

Furthermore, \(\sum_i\boldsymbol{\alpha}_iA_i = (2q(\boldsymbol{\alpha})-1)S+\lambda Z(\boldsymbol{\alpha})\),
so Lemma~\ref{lem:block} proves the decomposition and the reduced
problem.

To prove convexity of \(\psi\), let \(q_0,q_1\in[0,1]\), choose
feasible \(\boldsymbol{\alpha}^{(0)},\boldsymbol{\alpha}^{(1)}\), and let \(t\in[0,1]\).  Then \(t\boldsymbol{\alpha}^{(1)}+(1-t)\boldsymbol{\alpha}^{(0)}\)
is feasible for \(tq_1+(1-t)q_0\), and convexity of the Schatten 1-norm
gives \(\psi(tq_1+(1-t)q_0) \leq t\|Z(\boldsymbol{\alpha}^{(1)})\|_1 +(1-t)\|Z(\boldsymbol{\alpha}^{(0)})\|_1\).
Taking infima proves convexity.
\end{proof}

\subsection{Exact dual dynamics}

Define
\begin{equation}
\label{eq:alpha-def}
\alpha_{i,k}
:=
\frac{
\exp(-\ip{W_k^{\mathrm{SpecGD}}}{A_i})
}{
\sum_{j=1}^n\exp(-\ip{W_k^{\mathrm{SpecGD}}}{A_j})
}.
\end{equation}
Then \(\boldsymbol{\alpha}_k\in\operatorname{int}(\D_n)\).

Set \(M(\boldsymbol{\alpha}):=\sum_i\alpha_iA_i\).
Since \(-\nabla\mathcal L(W_k^{\mathrm{SpecGD}}) = \mathcal L(W_k^{\mathrm{SpecGD}})M(\boldsymbol{\alpha}_k)\),
we have \(G_k := \pol(-\nabla\mathcal L(W_k^{\mathrm{SpecGD}})) = \pol(M(\boldsymbol{\alpha}_k))\).
Define \(g_{i,k}:=\ip{G_k}{A_i}\).

\begin{lemma}[Exponentiated-gradient recursion]
\label{lem:eg}
For every \(i,k\), \(\alpha_{i,k+1} = \frac{ \alpha_{i,k}\exp(-\eta_kg_{i,k}) }{ \sum_j\alpha_{j,k}\exp(-\eta_kg_{j,k}) }\).
Moreover, \(\boldsymbol{g}_k:=(g_{1,k},\ldots,g_{n,k}) \in\partial F(\boldsymbol{\alpha}_k), \quad \|g_k\|_\infty\leq s+\lambda\).
\end{lemma}

\begin{proof}
The recursion follows directly from \(\exp(-\ip{W_{k+1}^{\mathrm{SpecGD}}}{A_i}) = \exp(-\ip{W_k^{\mathrm{SpecGD}}}{A_i})e^{-\eta_kg_{i,k}}\).

Because \(G_k=\pol(M(\boldsymbol{\alpha}_k))\),
we have \(\|G_k\|_{\op}\leq1, \quad \ip{G_k}{M(\boldsymbol{\alpha}_k)} = \|M(\boldsymbol{\alpha}_k)\|_1\).
Thus \(G_k\in\partial\|\cdot\|_1(M(\boldsymbol{\alpha}_k))\).  For
\(\boldsymbol{\beta}\in\D_n\), \(F(\boldsymbol{\beta}) \geq F(\boldsymbol{\alpha}_k)+ \sum_i(\beta_i-\alpha_{i,k})\ip{G_k}{A_i}\),
which proves \(g_k\in\partial F(\boldsymbol{\alpha}_k)\).

Finally, $|g_{i,k}|\leq \|G_k\|_{\op}\|A_i\|_1\leq s+\lambda$
by Lemma~\ref{lem:block}.
\end{proof}

\subsection{Dual convergence under local growth}

For \(\boldsymbol{\alpha},\boldsymbol{\beta}\in\operatorname{int}(\D_n)\), define \(\KL(\boldsymbol{\alpha}\|\boldsymbol{\beta}) := \sum_i\alpha_i\log\frac{\alpha_i}{\beta_i}\).

\begin{lemma}[Local quadratic upper bound for relative entropy]
\label{lem:local-kl-quadratic}
Let \(\boldsymbol{\alpha}^\star\in\operatorname{int}(\D_n)\).  Then there exist constants
\(C_{\mathrm{KL}}>0\) and \(\varepsilon_{\mathrm{KL}}>0\) such that \(\KL(\boldsymbol{\alpha}^\star\|\boldsymbol{\alpha}) \leq C_{\mathrm{KL}}\|\boldsymbol{\alpha}-\boldsymbol{\alpha}^\star\|_2^2\)
for every \(\boldsymbol{\alpha}\in\D_n\) satisfying
\(\|\boldsymbol{\alpha}-\boldsymbol{\alpha}^\star\|_2\leq\varepsilon_{\mathrm{KL}}\).
\end{lemma}

\begin{proof}
Let \(a_\star:=\min_i\alpha_i^\star>0\), and choose
\(\varepsilon_{\mathrm{KL}}\leq a_\star/2\).  If
\(\|\boldsymbol{\alpha}-\boldsymbol{\alpha}^\star\|_2\leq\varepsilon_{\mathrm{KL}}\), then, for every
\(i\),
\begin{equation}\label{eq:local-kl-coordinate-lower}
\alpha_i
\geq
\alpha_i^\star-\|\boldsymbol{\alpha}-\boldsymbol{\alpha}^\star\|_\infty
\geq
\frac{\alpha_i^\star}{2}.
\end{equation}
For \(f_i(x):=\alpha_i^\star\log(\alpha_i^\star/x)\), we have
\(f_i(\alpha_i^\star)=0\), \(f_i'(\alpha_i^\star)=-1\), and
\(f_i''(x)=\alpha_i^\star/x^2\).  By Taylor's theorem, for some
\(\widetilde\alpha_i\) between \(\alpha_i\) and \(\alpha_i^\star\),
\[
f_i(\alpha_i)
=
-(\alpha_i-\alpha_i^\star)
+
\frac12
\frac{\alpha_i^\star}{\widetilde\alpha_i^2}
(\alpha_i-\alpha_i^\star)^2.
\]
By \eqref{eq:local-kl-coordinate-lower}, we have
\(\widetilde\alpha_i\geq \alpha_i^\star/2\), and hence \(\frac{\alpha_i^\star}{\widetilde\alpha_i^2} \leq \frac{4}{\alpha_i^\star} \leq \frac{4}{a_\star}\).
Summing over \(i\) and using
\(\sum_i(\alpha_i-\alpha_i^\star)=0\), we obtain
\[
\KL(\boldsymbol{\alpha}^\star\|\boldsymbol{\alpha})
=
\sum_i f_i(\alpha_i)
\leq
\frac{2}{a_\star}
\sum_i(\alpha_i-\alpha_i^\star)^2.
\]
Thus the claim holds with \(C_{\mathrm{KL}}=2/a_\star\).
\end{proof}

\begin{definition}[Local relative growth]\label{def:localrelativegrowth}
A convex function \(F\) has local relative growth at an interior
minimizer \(\boldsymbol{\alpha}^\star\) if there exist \(r,\mu>0\) such that \(F(\boldsymbol{\alpha})-F(\boldsymbol{\alpha}^\star) \geq \mu\KL(\boldsymbol{\alpha}^\star\|\boldsymbol{\alpha})\)
whenever \(\KL(\boldsymbol{\alpha}^\star\|\boldsymbol{\alpha})\leq r\).
\end{definition}

\begin{lemma}[One-step relative-entropy inequality]
\label{lem:kl-step}
If \(\boldsymbol{\alpha}^\star\) minimizes \(F\), then
\[
\KL(\boldsymbol{\alpha}^\star\|\boldsymbol{\alpha}_{k+1})
\leq
\KL(\boldsymbol{\alpha}^\star\|\boldsymbol{\alpha}_k)
-
\eta_k\bigl(F(\boldsymbol{\alpha}_k)-F(\boldsymbol{\alpha}^\star)\bigr)
+
\frac{(s+\lambda)^2}{2}\eta_k^2.
\]
\end{lemma}

\begin{proof}
Let \(Z_k^{\mathrm{EG}} := \sum_i\alpha_{i,k}e^{-\eta_kg_{i,k}}\).
The exponentiated-gradient recursion gives
\[
\KL(\boldsymbol{\alpha}^\star\|\boldsymbol{\alpha}_{k+1})
-\KL(\boldsymbol{\alpha}^\star\|\boldsymbol{\alpha}_k)
=
\eta_k\langle\boldsymbol{\alpha}^\star,\boldsymbol{g}_k\rangle
+\log Z_k^{\mathrm{EG}}.
\]
By Hoeffding's lemma, if a random variable \(X\) satisfies
\(X\in[a,b]\) almost surely, then \(\log \mathbb E e^{\theta(X-\mathbb EX)} \leq \frac{\theta^2(b-a)^2}{8}\).
We apply this to the finite random variable \(X\) defined by \(\mathbb P(X=g_{i,k})=\alpha_{i,k}\).
Since \(|g_{i,k}| = |\langle G_k,A_i\rangle_F| \leq \|G_k\|_{\op}\|A_i\|_1 \leq s+\lambda\),
we have \(X\in[-B,B]\), where \(B=s+\lambda\).  Taking
\(\theta=-\eta_k\), Hoeffding's lemma gives \(\log Z_k^{\mathrm{EG}} \leq -\eta_k\langle\boldsymbol{\alpha}_k,\boldsymbol{g}_k\rangle + \frac{(s+\lambda)^2}{2}\eta_k^2\).
Thus,
\[
\KL(\boldsymbol{\alpha}^\star\|\boldsymbol{\alpha}_{k+1})
-\KL(\boldsymbol{\alpha}^\star\|\boldsymbol{\alpha}_k)
 \leq \eta_k\langle\boldsymbol{\alpha}^\star - \boldsymbol{\alpha}_k,\boldsymbol{g}_k\rangle + \frac{(s+\lambda)^2}{2}\eta_k^2.
\]
The subgradient inequality gives \(F(\boldsymbol{\alpha}^\star) - F(\boldsymbol{\alpha}_k) \geq \langle \boldsymbol{g}_k,\boldsymbol{\alpha}^\star-\boldsymbol{\alpha}_k\rangle\),
and the conclusion follows.
\end{proof}

\begin{theorem}[Dual convergence]
\label{thm:dual-convergence}
Suppose \(F\) has a unique minimizer
\(\boldsymbol{\alpha}^\star\in\operatorname{int}(\D_n)\) and has local relative
growth $\mu$ at \(\boldsymbol{\alpha}^\star\). Then \(\boldsymbol{\alpha}_k\to\boldsymbol{\alpha}^\star\).

Moreover, if we further assume $\limsup_{k\to \infty}\frac{\eta_k - \eta_{k+1}}{\eta_k^2} < \mu$, then $\|\boldsymbol{\alpha}_k - \boldsymbol{\alpha}^\star\|_1 = O(\sqrt{\eta_k})$.
\end{theorem}

\begin{proof}
Write
\[
d_k:=\KL(\boldsymbol{\alpha}^\star\|\boldsymbol{\alpha}_k),
\qquad
\delta_k:=F(\boldsymbol{\alpha}_k)-F(\boldsymbol{\alpha}^\star), \qquad C = \frac{(s+\lambda)^2}{2}. 
\]
Lemma~\ref{lem:kl-step} gives \(d_{k+1} \leq d_k-\eta_k\delta_k+C\eta_k^2\).
Rearrangigng and summing from \(k=0\) to \(K-1\) gives
\[
\sum_{k=0}^{K-1}\eta_k\delta_k
\leq
d_0-d_K
+
C\sum_{k=0}^{K-1}\eta_k^2
\leq
d_0
+
C\sum_{k=0}^{K-1}\eta_k^2.
\]
Let \(\tau_K:=\sum_{k=0}^{K-1}\eta_k\).
Dividing by \(\tau_K\), we get
\[
\frac{
\sum_{k=0}^{K-1}\eta_k\delta_k
}{
\tau_K
}
\leq
\frac{d_0}{\tau_K}
+
C
\frac{
\sum_{k=0}^{K-1}\eta_k^2
}{
\tau_K
}.
\]
Since \(\sum_k\eta_k=\infty\), we have \(\tau_K\to\infty\).  Moreover,
because \(\eta_k\to0\), \(\frac{ \sum_{k=0}^{K-1}\eta_k^2 }{ \sum_{k=0}^{K-1}\eta_k } \to0\).
Indeed, for any \(\varepsilon>0\), choose \(K_0\) such that
\(\eta_k\leq\varepsilon\) for every \(k\geq K_0\).  Then \(\sum_{k=0}^{K-1}\eta_k^2 \leq \sum_{k=0}^{K_0-1}\eta_k^2 + \varepsilon \sum_{k=K_0}^{K-1}\eta_k\).
Dividing by \(\tau_K\) and letting \(K\to\infty\), we obtain \(\limsup_{K\to\infty} \frac{ \sum_{k=0}^{K-1}\eta_k^2 }{ \tau_K } \leq \varepsilon\).
Since \(\varepsilon>0\) is arbitrary, the ratio converges to zero.
Therefore, \(\frac{ \sum_{k=0}^{K-1}\eta_k\delta_k }{ \tau_K } \to0\).
Since \(\delta_k\geq0\), there exists a subsequence \(k_j\to\infty\)
such that \(\delta_{k_j}\to0\).
Compactness of
\(\D_n\), uniqueness of the minimizer, and continuity of $F$ yield \(\boldsymbol{\alpha}_{k_j}\to\boldsymbol{\alpha}^\star\).
Furthermore, continuity of $\KL(\boldsymbol{\alpha}^\star\| \cdot)$ at the interior point $\boldsymbol{\alpha}^\star$ yields \(d_{k_j}\to0\).

Let \(r,\mu\) be local-growth constants.  Choose a sufficiently late
\(k_j\) so that \(d_{k_j}<r/2\) and \(C\eta_k\leq\frac{\mu r}{4}\)
for every \(k\geq k_j\).  If \(r/2\leq d_k\leq r\), then \(d_{k+1} \leq d_k-\mu\eta_kd_k+C\eta_k^2 \leq d_k-\frac{\mu r}{4}\eta_k<d_k\).
If \(d_k<r/2\), then sufficiently small \(\eta_k\) implies
\(d_{k+1}<r\).  Thus the iterates eventually remain in the local
region, where \(d_{k+1} \leq (1-\mu\eta_k)d_k+C\eta_k^2\).

For any $\varepsilon >0$, there exists $K$ such that $k\geq K$ implies $C\eta_k \leq \mu\varepsilon/2$ and $C\eta_k^2 \leq \varepsilon$ since $\eta_k \to 0$. If $d_k \geq \varepsilon$, then 
\[
\begin{aligned}
d_{k+1}
& \leq (1 - \mu \eta_k)d_k + C \eta_k^2 \\
& \leq d_k - \mu \eta_k \varepsilon + \mu\varepsilon\eta_k/2 \\
& = d_k - \mu\varepsilon\eta_k/2.
\end{aligned}
\]
Since $\sum_k \eta_k = \infty$, there exists $K_1 > K$ such that $d_{K_1} < \varepsilon$. Then, for $k\geq K_1$, $d_k<\varepsilon$ implies $d_{k+1}\leq 2\varepsilon$ due to $C\eta_k^2 \leq \varepsilon$ while $d_k \geq \varepsilon$ implies $d_{k+1} < d_k$, which by induction implies that $d_k \leq 2 \varepsilon$ for $k\geq K_1$. 

Hence \(d_k\to0\).  Pinsker's inequality gives \(\|\boldsymbol{\alpha}_k-\boldsymbol{\alpha}^\star\|_1^2\leq2d_k\to0\).

Suppose $\limsup_{k\to \infty}\frac{\eta_k - \eta_{k+1}}{\eta_k^2} < \mu$. Fix $\mu_0$ so that $\limsup_{k\to \infty}\frac{\eta_k - \eta_{k+1}}{\eta_k^2} <\mu_0< \mu$. Let $K$ so that $k\geq K$ implies \(d_{k+1} \leq (1-\mu\eta_k)d_k+C\eta_k^2\) and $\frac{\eta_k - \eta_{k+1}}{\eta_k^2} \leq \mu_0$. Set 
\[
B := \max\left\{\frac{d_K}{\eta_K}, \frac{C}{\mu - \mu_0}\right\}.
\]
We now claim that $d_k \leq B \eta_k$ holds for all $k\geq K$. The claim holds at $k = K$ by the definition of $B$. Now suppose $d_k \leq B \eta_k$ holds. By noting that the definition of $B$ implies that $B\mu-C \geq B\mu_0$, we have
\[
\begin{aligned}
    d_{k+1} & \leq (1-\mu\eta_k)d_k+C\eta_k^2 \\
    & \leq B(1 - \mu \eta_k)\eta_k + C\eta_k^2 = B \eta_k - (B\mu - C)\eta_k^2 \\
    & \leq B\eta_k - B\mu_0 \eta_k^2 = B\eta_{k+1} + B(\eta_k - \eta_{k+1} - \mu_0 \eta_k^2) \leq B\eta_{k+1}. 
\end{aligned}
\]
Again, by Pinsker's inequality, we obtain $\|\boldsymbol{\alpha}_k - \boldsymbol{\alpha}^\star\|_1^2 = O(\eta_k)$.
\end{proof}

\subsection{Which primal optimizer is selected?}

Let \(c_k:=\sum_i\alpha_{i,k}\xi_i, \quad Z_k:=\sum_i\alpha_{i,k}N_i\).
By Lemma~\ref{lem:block}, \(G_k = \sigma_kP_S+H_k, \quad \sigma_k:=\sgn(c_k), \quad H_k:=\pol(Z_k)\).
Therefore,
\begin{equation}
\label{eq:W-sum}
W_k^{\mathrm{SpecGD}}
=
\left(\sum_{t<k}\eta_t\sigma_t\right)P_S
+
\sum_{t<k}\eta_tH_t.
\end{equation}

Let $\boldsymbol{\alpha}^\star$ be the unique interior minimizer.
Set \(Z^\star:=\sum_i\alpha_i^\star N_i\) and
\(H^\star:=\pol(Z^\star)\).
Under DSM, since \(D>n\) , \(U=[\boldsymbol{u}_1,\ldots,\boldsymbol{u}_n]\) and
\(\widehat V=[\widehat{\boldsymbol{v}}_1,\ldots,\widehat{\boldsymbol{v}}_n]\) have full column rank
almost surely.  Since all \(\alpha_i^\star>0\), \(Z^\star\) has rank
\(n\) on its row and column spans. Lemma~\ref{lem:polar-orthonormalization} and \ref{lem:polar-continuity} implies that the polar factor is continuous at
\(Z^\star\). Thus, \(\boldsymbol{\alpha}_k\to\boldsymbol{\alpha}^\star\) implies \(H_k\to H^\star\). Under CSM, the convergence $H_k\to H^\star$ easily follows by noting $\pol(Z(\boldsymbol{\alpha})) = \frac{\sum_i \alpha_i \boldsymbol{u_i}}{\|\sum_i \alpha_i \boldsymbol{u_i}\|_2}\boldsymbol{v}_0^\top$ and that $\sum_i \alpha_i \boldsymbol{u_i} \neq 0$ since $\boldsymbol{\alpha} \neq \boldsymbol{0}$.

Define \(\Lambda_k := \frac{1}{n_+}\sum_{i\in I_+}\log\alpha_{i,k} - \frac{1}{n_-}\sum_{i\in I_-}\log\alpha_{i,k}\),
and
\[
\overline h_{+,k}
:=
\frac{1}{n_+}\sum_{i\in I_+}\ip{H_k}{N_i},
\qquad
\overline h_{-,k}
:=
\frac{1}{n_-}\sum_{i\in I_-}\ip{H_k}{N_i}.
\]

\begin{lemma}[Telescoping identity]
\label{lem:telescoping}
For every \(k\),
\[
\Lambda_{k+1}-\Lambda_k
=
-\eta_k
\left[
2s\sigma_k
+
\lambda(\overline h_{+,k}-\overline h_{-,k})
\right].
\]
\end{lemma}

\begin{proof}
From Lemma~\ref{lem:eg}, \(\log\alpha_{i,k+1} = \log\alpha_{i,k} -\eta_kg_{i,k} -\log Z_k^{\mathrm{EG}}\).
The normalizing term cancels between the two group averages.  Moreover,
\(g_{i,k} = \ip{\sigma_kP_S+H_k}{\xi_iS+\lambda N_i} = s\sigma_k\xi_i+\lambda\ip{H_k}{N_i}\).
Substitution proves the identity.
\end{proof}

\begin{proposition}[Dynamical selection formula]
\label{prop:selection}
Suppose Assumption~\ref{ass:stepsize}. If the unique dual minimizer is interior and \(F\) has
local relative growth there, then \(\frac{W_k^{\mathrm{SpecGD}}}{\tau_k} \longrightarrow a^\star P_S+H^\star\),
where
\begin{equation}
\label{eq:a-h}
a^\star
=
-\frac{\lambda}{2s}
(h_+^\star-h_-^\star),
\end{equation}
with
\[
h_+^\star
:=
\frac{1}{n_+}\sum_{i\in I_+}\ip{H^\star}{N_i},
\qquad
h_-^\star
:=
\frac{1}{n_-}\sum_{i\in I_-}\ip{H^\star}{N_i}.
\]
If we further assume Assumption~\ref{ass:stepsize1}, then we have
\[
\left\|\frac{W_k^{\mathrm{SpecGD}}}{\tau_k} - (a^\star P_S+H^\star)\right\|_F = O\left(\frac{1 + \sum_{t<k}\eta_t^{3/2}}{\tau_k}\right)
\]
\end{proposition}
\begin{proof}
Theorem~\ref{thm:dual-convergence} gives
\(\boldsymbol{\alpha}_k\to\boldsymbol{\alpha}^\star\).  Hence \(\Lambda_k\) converges to a finite
limit and \(H_k\to H^\star\).  Summing
Lemma~\ref{lem:telescoping}, dividing by \(\tau_k\), and applying the
weighted Ces\`aro theorem gives \(\frac{\sum_{t<k}\eta_t\sigma_t}{\tau_k} \to -\frac{\lambda}{2s}(h_+^\star-h_-^\star)\).
Dividing \eqref{eq:W-sum} by \(\tau_k\) and applying weighted
Ces\`aro convergence to \(H_t\) proves the result.

Suppose $\limsup_{k\to \infty}\frac{\eta_k - \eta_{k+1}}{\eta_k^2} = 0$. Then, by Theorem~\ref{thm:dual-convergence} we have $\|\boldsymbol{\alpha}_k - \boldsymbol{\alpha}^\star\|_2 = O(\sqrt{\eta_k})$. 

Since $\boldsymbol{\alpha} \to \pol(Z(\boldsymbol{\alpha}))$ is Lipschitz continuous at $\boldsymbol{\alpha}^\star$, we have $\|H_k - H^\star\|_F = O(\sqrt{\eta_k})$. Then for the shortcut noise, we have
\[
\left\|\frac{\sum_{t<k}\eta_tH_t}{\tau_k} - H^\star\right\|_F = O\left(\frac{\sum_{t<k}\eta_t^{3/2}}{\tau_k}\right).
\]
Since $|\bar h_{+,k} - h_+^\star|, |\bar h_{+,k} - h_+^\star| = O(\|H_k - H^\star\|) = O(\sqrt{\eta_k})$, summing Lemma~\ref{lem:telescoping} and dividing by $\tau_k$ implies that
\[
\left|\frac{\sum_{t<k}\eta_t\sigma_t}{\tau_k} - \frac{\lambda}{2s}(h_-^\star-h_+^\star)\right| = O\left(\frac{1 + \sum_{t<k}\eta_t^{3/2}}{\tau_k}\right).
\]
This gives the convergence rate for the signal direction. Combining these estimates proves the desired convergence rate.
\end{proof}

%% file: proof_main_dispersed_WSpecGD_compressed.tex
By Proposition~\ref{prop:selection}, the proof of Theorem~\ref{thm:main_dispersed} is reduced to the uniqueness, interiority of the dual minimizer $\boldsymbol{\alpha}^\star$ and relative local growth at $\boldsymbol{\alpha}^\star$. Under DSM, the key ingredient is the second-order expansion of $\|Z(\boldsymbol{\alpha})\|_1$.

Throughout, set $\hat{\boldsymbol v}_i:=\tilde y_i \boldsymbol v_i$, so that $N_i=\boldsymbol u_i\hat{\boldsymbol v}_i^\top$.

\subsection{Second-order expansion}\label{app:second_order_exp}

This section derives the analytic expansion that turns the nuclear norm
of the noise combination into a tractable quadratic function of the
weights.  Recall that \(Z(\boldsymbol{\alpha})=\sum_i\alpha_i \boldsymbol u_i\widehat{\boldsymbol v}_i^\top\).
Since $\|\boldsymbol u_i\widehat{\boldsymbol v}_i^\top\|_1 = 1$, we have \(\|Z(\boldsymbol{\alpha})\|_1 \leq \sum_i\alpha_i=1\),
where the equality holds when the noise atoms are perfectly orthogonal. 

In the random model the atoms are only approximately
orthogonal. The goal of this section is to establish a second-order approximation of $\|Z(\boldsymbol{\alpha})\|_1$, which enables us to obtain a tight estimate of the noise norm gap 
\(1 - \|Z(\boldsymbol{\alpha})\|_1\).

Write
\begin{equation}\label{eq:noise-matrix-notation}
U=[\boldsymbol{u}_1,\ldots,\boldsymbol{u}_n],
\qquad
\widehat V=[\widehat{\boldsymbol v}_1,\ldots,\widehat{\boldsymbol v}_n],
\qquad
\Lambda_{\boldsymbol{\alpha}}:=\diag(\alpha_1,\ldots,\alpha_n).
\end{equation}
Then \(Z(\boldsymbol{\alpha})=U\Lambda_{\boldsymbol{\alpha}}\widehat V^\top\).  Also set
\begin{equation}\label{eq:noise-core-notation}
G_U:=U^\top U=I+\Delta_U,
\qquad
G_V:=\widehat V^\top\widehat V=I+\Delta_V,
\qquad
Y(\boldsymbol{\alpha}):=G_U^{1/2}\Lambda_{\boldsymbol{\alpha}} G_V^{1/2}.
\end{equation}
The matrices \(\Delta_U,\Delta_V\) are symmetric with zero diagonal.

Throughout this section, we write
\begin{equation}
\label{eq:normalized-schatten-functional}
\mathsf{N}(Y):=\frac1n\|Y\|_1
\end{equation}
for the normalized Schatten \(1\)-norm functional on \(n\times n\) matrices.
On nonsingular matrices we use the factorization
\begin{equation}
\label{eq:normalized-schatten-factorization}
\mathsf{N}=\Psi\circ G,
\qquad
G(Y):=Y^\top Y,
\qquad
\Psi(M):=\frac1n\operatorname{tr}(M^{1/2}).
\end{equation}

\begin{lemma}[Uniform derivatives of the matrix square-root map]
\label{lem:uniform-square-root-derivatives}
Fix constants \(0<\ell<u<\infty\) and an integer \(r\geq1\).  There
exists a constant \(C_r=C_r(\ell,u)<\infty\) such that the following
holds in every dimension \(n\).

Let \(A\in\mathbb R^{n\times n}\) be symmetric positive definite with \(\ell I\preceq A\preceq uI\).
Then the map \(A\mapsto A^{1/2}\)
is \(C^\infty\) in a neighborhood of \(A\), and for every \(1\leq q\leq r\),
\(\left\| \dd^q(A^{1/2})[H_1,\ldots,H_q] \right\|_{\op} \leq C_r \prod_{j=1}^q\|H_j\|_{\op}\).
The constant \(C_r\) depends only on \(\ell,u,r\), and not on \(n\).
\end{lemma}

\begin{proof}
Choose once and for all a positively oriented contour
\(\Gamma=\Gamma_{\ell,u}\) contained in the right half-plane, enclosing
\([\ell,u]\), and satisfying \(\operatorname{dist}(\Gamma,[\ell,u])\geq \delta_{\ell,u}>0\).
The length of \(\Gamma\), the quantity
\(\sup_{z\in\Gamma}|z^{1/2}|\), and \(\delta_{\ell,u}^{-1}\) depend only
on \(\ell,u\).  By the holomorphic functional calculus, \(A^{1/2} = \frac{1}{2\pi i} \int_\Gamma z^{1/2}(zI-A)^{-1}\,\dd z\).
Set \(R_z(A):=(zI-A)^{-1}\).
Since \(z\in\Gamma\) stays at distance at least \(\delta_{\ell,u}\) from
the spectrum of \(A\), \(\|R_z(A)\|_{\op}\leq \delta_{\ell,u}^{-1}\).
The resolvent derivative is \(\dd R_z(A)[H] = R_z(A)H R_z(A)\).
Iterating this identity gives
\[
\dd^q(A^{1/2})[H_1,\ldots,H_q]
=
\frac{1}{2\pi i}
\int_\Gamma
z^{1/2}
\sum_{\pi\in S_q}
R_z(A)H_{\pi(1)}R_z(A)
\cdots
H_{\pi(q)}R_z(A)
\,\dd z.
\]
Therefore,
\[
\begin{aligned}
\left\|
\dd^q(A^{1/2})[H_1,\ldots,H_q]
\right\|_{\op}
&\leq
\frac{\operatorname{length}(\Gamma)}{2\pi}
\sup_{z\in\Gamma}|z^{1/2}|\,
q!\,
\delta_{\ell,u}^{-(q+1)}
\prod_{j=1}^q\|H_j\|_{\op}.
\end{aligned}
\]
The right-hand side depends only on \(\ell,u,q\), not on \(n\).
Taking the maximum over \(1\leq q\leq r\) proves the claim.
\end{proof}

The following is the standard partition form of Faà di Bruno's formula
for Fréchet derivatives; see, for example, \cite{Levy2006FaaDiBruno}.

\begin{lemma}[Faà di Bruno formula for Fréchet derivatives]\label{lemm:faadibruno}
Let \(X,Y,Z\) be finite-dimensional normed vector spaces.  Let
\(U\subset X\) and \(V\subset Y\) be open sets, and let \(G:U\to V, \quad F:V\to Z\)
be \(C^q\).  Let \(\mathcal P_q\) denote the set of all partitions of
\(\{1,\ldots,q\}\).  For a block \(B\subset\{1,\ldots,q\}\), define \(G_B(x) := \dd^{|B|}G(x)[h_i:i\in B]\),
where the directions are taken in increasing order of their indices.
Then
\[
\dd^q(F\circ G)(x)[h_1,\ldots,h_q]
=
\sum_{\pi\in\mathcal P_q}
\dd^{|\pi|}F(G(x))
\left[
G_{B_1}(x),\ldots,G_{B_m}(x)
\right],
\]
where \(\pi=\{B_1,\ldots,B_m\}\) and \(|\pi|=m\).
Since the higher derivatives are symmetric, the order of the blocks and
the order inside each block are immaterial.

Equivalently, the \(q\)-th derivative of a composition is obtained by
summing over all ways of grouping the directions
\(\{h_1,\ldots,h_q\}\): each block \(B\) is first inserted into the
derivative \(\dd^{|B|}G(x)\), and the resulting outputs are then inserted
into the derivative of \(F\) whose order is the number of blocks.
\end{lemma}

\begin{lemma}[Composition of uniformly bounded derivatives]
\label{lem:uniform-derivative-composition}
Let \(X,Y,Z\) be normed vector spaces, and let \(G:U\subset X\to Y, \quad F:V\subset Y\to Z\)
be \(C^r\) maps with \(G(U)\subset V\).  Suppose that for every
\(1\leq m\leq r\), there exist constants \(A_m,B_m<\infty\) such that \(\|\dd^m F(y)[u_1,\ldots,u_m]\|_Z \leq A_m\prod_{j=1}^m\|u_j\|_Y\)
for all \(y\in V\), and \(\|\dd^m G(x)[h_1,\ldots,h_m]\|_Y \leq B_m\prod_{j=1}^m\|h_j\|_X\)
for all \(x\in U\).  Then, for every \(1\leq q\leq r\), \(\|\dd^q(F\circ G)(x)[h_1,\ldots,h_q]\|_Z \leq C_q\prod_{j=1}^q\|h_j\|_X\),
where one may take \(C_q := \sum_{\pi\in\mathcal P_q} A_{|\pi|} \prod_{B\in\pi}B_{|B|}\).
In particular, \(C_q\) depends only on \(q,\quad A_1,\ldots,A_q,\quad B_1,\ldots,B_q\).
\end{lemma}

\begin{proof}
By Faà di Bruno's formula,
\[
\dd^q(F\circ G)(x)[h_1,\ldots,h_q]
=
\sum_{\pi\in\mathcal P_q}
\dd^{|\pi|}F(G(x))
\left[
G_{B_1}(x),\ldots,G_{B_m}(x)
\right],
\]
where \(\pi=\{B_1,\ldots,B_m\}\), \(|\pi|=m\), and \(G_B(x) := \dd^{|B|}G(x)[h_i:i\in B]\).
For a fixed partition \(\pi\), the assumptions give
\[
\left\|
\dd^{|\pi|}F(G(x))
\left[
G_{B_1}(x),\ldots,G_{B_m}(x)
\right]
\right\|_Z
\leq
A_{|\pi|}
\prod_{B\in\pi}\|G_B(x)\|_Y.
\]
Moreover, \(\|G_B(x)\|_Y \leq B_{|B|} \prod_{i\in B}\|h_i\|_X\).
Hence the contribution of \(\pi\) is bounded by \(A_{|\pi|} \prod_{B\in\pi}B_{|B|} \prod_{j=1}^q\|h_j\|_X\).
Summing over all partitions gives the claim.
\end{proof}

\begin{lemma}[Dimension-uniform derivatives of the normalized Schatten \(1\)-norm]
\label{lem:dimension-uniform-schatten-derivatives}
Fix constants \(0<a<b<\infty\) and an integer \(r\geq1\).  There exists
a constant \(C_r<\infty\), depending only on \(a,b,r\), such that the
following holds for every dimension \(n\).

Let \(Y\in\mathbb R^{n\times n}\) satisfy \(a\leq \sigma_{\min}(Y)\leq \|Y\|_{\op}\leq b\).
For \(\mathsf{N}\) defined in \eqref{eq:normalized-schatten-functional}, equivalently \(\mathsf{N}(Y)=\frac1n\operatorname{tr}\bigl((Y^\top Y)^{1/2}\bigr)\),
\(\mathsf{N}\) is \(C^\infty\) in a neighborhood of \(Y\), and for every
\(1\leq q\leq r\), \(\left| \dd^q\mathsf{N}(Y)[E_1,\ldots,E_q] \right| \leq C_r \prod_{j=1}^q\|E_j\|_{\op}\).
In particular, the constant is independent of \(n\).
\end{lemma}

\begin{proof}
Recall the factorization \eqref{eq:normalized-schatten-factorization}.
On the set of symmetric positive-definite matrices whose spectra lie in
\([a^2,b^2]\), Lemma~\ref{lem:uniform-square-root-derivatives} implies that, for
each \(1\leq m\leq r\), there is a constant \(K_m(a,b)\), independent of
\(n\), such that \(\left\| \dd^m(M^{1/2})[H_1,\ldots,H_m] \right\|_{\op} \leq K_m(a,b) \prod_{j=1}^m\|H_j\|_{\op}\).
Since \(\frac1n|\operatorname{tr}(L)|\leq \|L\|_{\op}\),
it follows that \(\left| \dd^m\Psi(M)[H_1,\ldots,H_m] \right| \leq K_m(a,b) \prod_{j=1}^m\|H_j\|_{\op}\).

By \eqref{eq:normalized-schatten-factorization}, it remains to control
the derivatives of \(G(Y)=Y^\top Y\).  On the region \(\|Y\|_{\op}\leq b\), the map \(G\) satisfies
\(\dd G(Y)[E]=Y^\top E+E^\top Y, \quad \|\dd G(Y)[E]\|_{\op}\leq 2b\|E\|_{\op}\),
and
\[
\dd^2G(Y)[E_1,E_2]=E_1^\top E_2+E_2^\top E_1,
\qquad
\|\dd^2G(Y)[E_1,E_2]\|_{\op}
\leq 2\|E_1\|_{\op}\|E_2\|_{\op}.
\]
Moreover, \(\dd^mG=0\) for \(m\geq3\).  Therefore the derivatives of
\(G\) up to order \(r\) are bounded by constants depending only on
\(b\) and the derivative order.  Applying
Lemma~\ref{lem:uniform-derivative-composition} to the factorization \eqref{eq:normalized-schatten-factorization}
gives the desired bound.  All constants depend only on \(a,b,r\), and
not on \(n\).
\end{proof}

\begin{lemma}[Rank-sensitive derivatives of the normalized Schatten \(1\)-norm]
\label{lem:rank-sensitive-schatten-derivatives}
Fix constants \(0<a<b<\infty\) and an integer \(r\geq1\).  Consider \(\mathsf{N}\), defined in \eqref{eq:normalized-schatten-functional}, on the set of matrices satisfying
\(a\leq \sigma_{\min}(Y)\leq \|Y\|_{\op}\leq b\).
For every \(1\leq q\leq r\), there is a constant \(C_q=C_q(a,b)<\infty\),
independent of \(n\), such that the following bounds hold.

First, if one direction, say \(E_1\), is measured in Schatten \(1\)-norm
and the remaining directions are measured in operator norm, then
\[
\left|
\dd^q\mathsf{N}(Y)[E_1,E_2,\ldots,E_q]
\right|
\leq
\frac{C_q}{n}\,
\|E_1\|_1
\prod_{j=2}^q\|E_j\|_{\op}.
\]
Second, if two directions are measured in Frobenius norm and the
remaining directions are measured in operator norm, then
\[
\left|
\dd^q\mathsf{N}(Y)[E_1,E_2,E_3,\ldots,E_q]
\right|
\leq
\frac{C_q}{n}\,
\|E_1\|_F\|E_2\|_F
\prod_{j=3}^q\|E_j\|_{\op}.
\]
By symmetry of Fr\'echet derivatives, the distinguished directions may
appear in any slots.
\end{lemma}

\begin{proof}
We use the factorization \eqref{eq:normalized-schatten-factorization}.
For \(M\) whose spectrum lies in \([a^2,b^2]\), the holomorphic
functional calculus gives \(M^{1/2} = \frac{1}{2\pi i} \int_\Gamma z^{1/2}(zI-M)^{-1}\,\dd z\),
where the contour \(\Gamma\) and its distance from \([a^2,b^2]\) depend
only on \(a,b\).  Hence every fixed-order derivative of \(\Psi\) is a
finite sum of contour integrals of normalized traces of products of
uniformly bounded resolvents and inserted directions.  More precisely,
\[
\dd^m \Psi(M)[H_1, \ldots, H_m]
=\frac{1}{2\pi i}
\int_\Gamma
z^{1/2}
\sum_{\pi\in S_m} \frac{1}{n}\mathrm{tr} \left(
R_z H_{\pi(1)}R_z\cdots H_{\pi(m)}R_z
\right) \dd z.
\]
Recall that the resolvent factors $R_z = (zI-M)^{-1}$ are uniformly bounded in operator norm.

If one inserted direction, say \(H_1\), is controlled in Schatten
\(1\)-norm, then cyclicity of trace and \(|\operatorname{tr}(AB)|\leq \|A\|_1\|B\|_{\op}\)
give \(\left| \dd^m\Psi(M)[H_1,H_2,\ldots,H_m] \right| \leq \frac{K_m}{n}\, \|H_1\|_1\prod_{j=2}^m\|H_j\|_{\op}\).
Similarly, if two inserted directions \(H_1,H_2\) are controlled in
Frobenius norm, then cyclicity of trace allows us to split each cyclic
product at those two factors.  More explicitly, after cyclically rotating
the trace, each summand can be written as \(\operatorname{tr}(A H_1 B H_2)\),
where the factors in \(A\) and \(B\) are resolvents and the remaining
inserted directions.  The resolvents are uniformly bounded in operator
norm, and the remaining inserted directions are measured in operator
norm.  Hence
\[
\|A H_1\|_F
\leq
K_m\|H_1\|_F\prod_{j\in J_1}\|H_j\|_{\op},
\qquad
\|B H_2\|_F
\leq
K_m\|H_2\|_F\prod_{j\in J_2}\|H_j\|_{\op},
\]
where \(J_1\cup J_2=\{3,\ldots,m\}\).  Applying \(|\operatorname{tr}(AB)|\leq \|A\|_F\|B\|_F\)
to the product \((AH_1)(BH_2)\) gives
\[
\left|
\dd^m\Psi(M)[H_1,H_2,H_3,\ldots,H_m]
\right|
\leq
\frac{K_m}{n}\,
\|H_1\|_F\|H_2\|_F
\prod_{j=3}^m\|H_j\|_{\op}.
\]

It remains to pass these estimates from \(\Psi\) to \(\mathsf{N}\) using
the factorization \eqref{eq:normalized-schatten-factorization}.
The derivatives of \(G(Y)=Y^\top Y\) are \(\dd G(Y)[E]=Y^\top E+E^\top Y, \quad \dd^2G(Y)[E,F]=E^\top F+F^\top E\),
and all higher derivatives vanish. Multiplication by matrices with
operator norm bounded by \(b\) preserves the above norm classes in the
following sense: \(\|AEB\|_1\leq \|A\|_{\op}\|E\|_1\|B\|_{\op}, \quad \|AEB\|_F\leq \|A\|_{\op}\|E\|_F\|B\|_{\op}\),
and
\[
\|E^\top F\|_1\leq \|E\|_F\|F\|_F,
\qquad
\|E^\top F\|_1\leq \|E\|_1\|F\|_{\op}, \qquad \|E^\top F\|_F \leq \|E\|_{\op}\|F\|_F.
\]

More precisely, if one direction $E_1$ is measured in Schatten 1-norm and the remaining directions are measured in operator norm, then the direction \(E_1\) enters exactly one argument passed
to a derivative of \(\Psi\) since each input direction belongs to exactly one block in the Faà di
Bruno expansion. Then, every Faà di Bruno term can be bounded in the desired form by noticing
\(\|\dd G(Y)[E_1]\|_1 \leq 2b\|E_1\|_1, \quad \|\dd^2 G(Y)[E_1, E_i]\|_1 \leq 2\|E_i\|_{\op}\|E_1\|_1\).

Similarly, if two directions $E_1, E_2$ are measured in Frobenius norm, and the remaining directions are measured in operator norm, every Faà di Bruno term can be bounded in the desired form by noticing
\[
\|\dd G(Y)[E_a]\|_F \leq 2b\|E_a\|_F, \qquad \|\dd^2 G(Y)[E_a, E_i]\|_F \leq 2\|E_i\|_{\op}\|E_a\|_F, \qquad a\in\{1, 2\},
\]
\(\|\dd^2 G(Y)[E_1, E_2]\|_1 \leq 2\|E_1\|_F\|E_2\|_F\).

Thus, in every nonzero Faà di Bruno term for \(\Psi\circ G\), a
Schatten \(1\) controlled input direction produces an argument of \(\Psi\) with
Schatten \(1\)-norm control, and two Frobenius-controlled input directions
either produce two Frobenius-controlled arguments of \(\Psi\) or one
Schatten-\(1\)-controlled argument through \(\dd^2G\).  All other
arguments are controlled in operator norm.  Applying the preceding
bounds for \(\Psi\) to each partition term and summing over the finitely
many partitions proves the claim.
\end{proof}

\begin{lemma}[Diagonal perturbation estimates for the normalized remainder]
\label{lem:diagonal-remainder-derivatives}
Fix constants \(0<c<C<\infty\).  There exist constants
\(\varrho_0=\varrho_0(c,C)>0\) and \(K=K(c,C)<\infty\), independent of
\(n\), such that the following holds.

Let \(\widetilde\Lambda\) be a positive diagonal \(n\times n\) matrix whose
entries lie in \([c,C]\).  Let \(L_t=(I+t\Delta_U)^{1/2}, \quad R_t=(I+t\Delta_V)^{1/2}, \quad 0\leq t\leq1\),
where \(\|\Delta_U\|_{\op}\leq \varrho, \quad \|\Delta_V\|_{\op}\leq \varrho, \quad 0<\varrho\leq\varrho_0\).
For diagonal matrices \(W_1,W_2\), define \(\varphi_{W_1,W_2}(t,s,r) := \mathsf N\bigl(L_t(\widetilde\Lambda+sW_1+rW_2)R_t\bigr)\)
and \(\mathcal R_{W_1,W_2}(s,r) := \frac12 \int_0^1 (1-t)^2 \partial_t^3\varphi_{W_1,W_2}(t,s,r) \,\dd t\).
Then \(|\mathcal R_{0,0}(0,0)|\leq K\varrho^3\).
Moreover, for every diagonal matrix \(W\), \(\left| \partial_s\mathcal R_{W,0}(0,0) \right| \leq K\varrho^3\frac{\|W\|_1}{n}\),
and, for every pair of diagonal matrices \(W_1,W_2\), \(\left| \partial_s\partial_r\mathcal R_{W_1,W_2}(0,0) \right| \leq K\varrho^3\frac{\|W_1\|_F\|W_2\|_F}{n}\).
Finally, the mixed Schatten-\(1\)/operator estimate
\begin{equation}\label{eq:diagonal-remainder-mixed-second-bound}
\left|
\partial_s\partial_r\mathcal R_{W_1,W_2}(0,0)
\right|
\leq
K\varrho^3\frac{\|W_1\|_1\|W_2\|_{\op}}{n}
\end{equation}
also holds.  By symmetry, the same estimate holds with \(W_1\) and \(W_2\)
interchanged.
\end{lemma}

\begin{proof}
Choose \(\varrho_0\in(0,1)\) sufficiently small.  Since the entries of
\(\widetilde\Lambda\) lie in \([c,C]\), the singular values of \(L_t\widetilde\Lambda R_t\)
lie in a fixed interval \([a,b]\subset(0,\infty)\), depending only on
\(c,C\), uniformly over \(0\leq t\leq1\).  Thus the derivative estimates
from Lemmas~\ref{lem:dimension-uniform-schatten-derivatives} and
\ref{lem:rank-sensitive-schatten-derivatives} apply uniformly along this
path.

We first record the size of the matrix directions inserted into the
Fr\'echet derivatives of \(\mathsf N\).  By
Lemma~\ref{lem:uniform-square-root-derivatives}, for every fixed
\(m\geq0\), \(\|\partial_t^mL_t\|_{\op}\leq C_m\varrho^m, \quad \|\partial_t^mR_t\|_{\op}\leq C_m\varrho^m\).
Hence, by the product rule,
\[
\partial_t^m(L_t\widetilde\Lambda R_t)
=
\sum_{a+b=m}{m\choose a}
(\partial_t^aL_t)\widetilde\Lambda(\partial_t^bR_t),
\]
and therefore, for \(m=1,2,3\),
\begin{equation}
\label{eq:pure-t-direction-bound}
\|\partial_t^m(L_t\widetilde\Lambda R_t)\|_{\op}
\leq C_m\varrho^m.
\end{equation}
Similarly, if \(A_W(t,s):=L_t(\widetilde\Lambda+sW)R_t\),
then \(\partial_sA_W(t,s)=L_tWR_t\).
Thus, for every fixed \(a\geq0\), \(\partial_t^a\partial_sA_W(t,0) = \sum_{p=0}^a {a\choose p} (\partial_t^pL_t)W(\partial_t^{a-p}R_t)\).
Consequently,
\begin{equation}\label{eq:one-diagonal-direction-nuclear-bound}
\|\partial_t^a\partial_sA_W(t,0)\|_1
\leq C_a\varrho^a\|W\|_1,
\end{equation}
\begin{equation}\label{eq:one-diagonal-direction-op-bound}
\|\partial_t^a\partial_sA_W(t,0)\|_{\op}
\leq C_a\varrho^a\|W\|_{\op},
\end{equation}
and
\begin{equation}\label{eq:one-diagonal-direction-frobenius-bound}
\|\partial_t^a\partial_sA_W(t,0)\|_F
\leq C_a\varrho^a\|W\|_F.
\end{equation}

We now estimate the remainder itself.  Set \(A_0(t):=L_t\widetilde\Lambda R_t\).
Then \(\varphi_{0,0}(t,0,0)=\mathsf N(A_0(t))\).
The one-variable chain rule gives
\[
\begin{aligned}
\partial_t^3\varphi_{0,0}(t,0,0)
&=
\dd\mathsf N(A_0(t))[A_0^{(3)}(t)] \\
&\quad
+3\dd^2\mathsf N(A_0(t))[A_0''(t),A_0'(t)] \\
&\quad
+\dd^3\mathsf N(A_0(t))[A_0'(t),A_0'(t),A_0'(t)].
\end{aligned}
\]
By Lemma~\ref{lem:dimension-uniform-schatten-derivatives} and
\eqref{eq:pure-t-direction-bound}, \(|\partial_t^3\varphi_{0,0}(t,0,0)| \leq C\varrho^3\).
Therefore \(|\mathcal R_{0,0}(0,0)| \leq \frac12\int_0^1(1-t)^2C\varrho^3\,\dd t \leq K\varrho^3\).

Next we estimate the first derivative of the remainder in a diagonal
matrix direction \(W\).  Let \(e_t,e_s\) denote the coordinate directions
corresponding to \(t\) and \(s\).  By the coordinate-direction
interpretation of scalar derivatives, \(\partial_s\partial_t^3\varphi_{W,0}(t,0,0) = \dd^4(\mathsf N\circ A_W)(t,0)[e_t,e_t,e_t,e_s]\).
We apply the Fa\`a di Bruno formula from Lemma~\ref{lemm:faadibruno} to
\(\mathsf N\circ A_W\).  Each term is indexed by a partition
\(\pi\in\mathcal P_4\).  Since exactly one inserted coordinate direction
is \(e_s\), exactly one block of \(\pi\) contains \(e_s\).  Suppose this
block contains \(a\) copies of \(e_t\), where \(0\leq a\leq3\).  The
corresponding inserted direction is \(\partial_t^a\partial_sA_W(t,0)\),
and by \eqref{eq:one-diagonal-direction-nuclear-bound} it is controlled
in Schatten \(1\)-norm by \(C\varrho^a\|W\|_1\).
All other blocks contain only copies of \(e_t\).  If their sizes are
\(b_1,\ldots,b_q\), then their inserted directions are pure
\(t\)-derivatives and are controlled in operator norm, by
\eqref{eq:pure-t-direction-bound}, as \(C\varrho^{b_1},\ldots,C\varrho^{b_q}\).
Since there are three copies of \(e_t\) in total, \(a+b_1+\cdots+b_q=3\).
Applying the one-Schatten-
\(1\)-input estimate from
Lemma~\ref{lem:rank-sensitive-schatten-derivatives}, this partition term
is bounded by \(\frac{C}{n} \varrho^a\|W\|_1 \prod_{j=1}^q\varrho^{b_j} = \frac{C}{n}\varrho^3\|W\|_1\).
There are only finitely many partitions in \(\mathcal P_4\), so
\begin{equation}
\label{eq:first-diagonal-remainder-integrand-bound}
|\partial_s\partial_t^3\varphi_{W,0}(t,0,0)|
\leq
K\varrho^3\frac{\|W\|_1}{n}.
\end{equation}
Integrating in \(t\) gives \(\left| \partial_s\mathcal R_{W,0}(0,0) \right| \leq K\varrho^3\frac{\|W\|_1}{n}\).

Finally, we estimate the second derivative of the remainder in two
diagonal matrix directions.  Put \(A_{W_1,W_2}(t,s,r) := L_t(\widetilde\Lambda+sW_1+rW_2)R_t\).
Let \(e_t,e_s,e_r\) denote the coordinate directions corresponding to
\(t,s,r\).  Then
\[
\partial_s\partial_r\partial_t^3\varphi_{W_1,W_2}(t,0,0)
=
\dd^5(\mathsf N\circ A_{W_1,W_2})(t,0,0)
[e_t,e_t,e_t,e_s,e_r].
\]
Again we use Lemma~\ref{lemm:faadibruno}.  Since the dependence of
\(A_{W_1,W_2}\) on \(s\) and \(r\) is affine and separate,
\(\partial_t^a\partial_s\partial_rA_{W_1,W_2}(t,s,r)=0\) for every
\(a\geq0\).
Therefore every Fa\`a di Bruno term whose partition has a block
containing both \(e_s\) and \(e_r\) vanishes.

In every nonzero term, the block containing \(e_s\) and the block
containing \(e_r\) are distinct.  Suppose these two blocks contain
\(a\) and \(b\) copies of \(e_t\), respectively.  Their inserted
directions are \(\partial_t^a\partial_sA_{W_1,W_2}(t,0,0)\) and
\(\partial_t^b\partial_rA_{W_1,W_2}(t,0,0)\).  By
\eqref{eq:one-diagonal-direction-frobenius-bound}, these are controlled in
Frobenius norm by \(C\varrho^a\|W_1\|_F\) and \(C\varrho^b\|W_2\|_F\).
The remaining blocks contain only copies of \(e_t\).  If their sizes are
\(c_1,\ldots,c_q\), then the corresponding inserted directions are pure
\(t\)-derivatives and are controlled in operator norm by \(C\varrho^{c_1},\ldots,C\varrho^{c_q}\).
Since \(a+b+c_1+\cdots+c_q=3\),
the two-Frobenius-input estimate from
Lemma~\ref{lem:rank-sensitive-schatten-derivatives} gives, for this
partition term,
\[
\frac{C}{n}
\varrho^a\|W_1\|_F
\varrho^b\|W_2\|_F
\prod_{j=1}^q\varrho^{c_j}
=
\frac{C}{n}\varrho^3\|W_1\|_F\|W_2\|_F.
\]
Summing over the finitely many partitions in \(\mathcal P_5\), we obtain
\begin{equation}
\label{eq:second-diagonal-remainder-integrand-bound}
\left|
\partial_s\partial_r\partial_t^3
\varphi_{W_1,W_2}(t,0,0)
\right|
\leq
K\varrho^3\frac{\|W_1\|_F\|W_2\|_F}{n}.
\end{equation}
Integrating in \(t\) gives \(\left| \partial_s\partial_r\mathcal R_{W_1,W_2}(0,0) \right| \leq K\varrho^3\frac{\|W_1\|_F\|W_2\|_F}{n}\).

We also prove the mixed Schatten-\(1\)/operator estimate.  In every nonzero
Fa\`a di Bruno term, the block containing \(e_s\) and the block containing
\(e_r\) are distinct.  Suppose these two blocks contain \(a\) and \(b\)
copies of \(e_t\), respectively.  The inserted direction corresponding to
the \(e_s\)-block is \(\partial_t^a\partial_sA_{W_1,W_2}(t,0,0)\),
and by \eqref{eq:one-diagonal-direction-nuclear-bound} it is bounded in
Schatten \(1\)-norm by \(C\varrho^a\|W_1\|_1\).
The inserted direction corresponding to the \(e_r\)-block is \(\partial_t^b\partial_rA_{W_1,W_2}(t,0,0)\),
and by \eqref{eq:one-diagonal-direction-op-bound} it is bounded in operator
norm by \(C\varrho^b\|W_2\|_{\op}\).
All remaining blocks contain only copies of \(e_t\), and their inserted
directions are controlled in operator norm by \(C\varrho^{c_1},\ldots,C\varrho^{c_q}\).
Since \(a+b+c_1+\cdots+c_q=3\),
the one-Schatten-\(1\)-input estimate from
Lemma~\ref{lem:rank-sensitive-schatten-derivatives} gives, for this
partition term,
\[
\frac{C}{n}
\varrho^a\|W_1\|_1
\varrho^b\|W_2\|_{\op}
\prod_{\ell=1}^q\varrho^{c_\ell}
=
\frac{C}{n}\varrho^3\|W_1\|_1\|W_2\|_{\op}.
\]
There are only finitely many partitions in \(\mathcal P_5\), so
\[
\left|
\partial_s\partial_r\partial_t^3
\varphi_{W_1,W_2}(t,0,0)
\right|
\leq
K\varrho^3\frac{\|W_1\|_1\|W_2\|_{\op}}{n}.
\]
Integrating in \(t\) gives
\eqref{eq:diagonal-remainder-mixed-second-bound}.  Interchanging \(W_1\)
and \(W_2\) gives the symmetric version.
\end{proof}

Throughout the sequel, we use the pairwise weight
\begin{equation}\label{eq:h-definition}
h(x,y):=\frac{xy}{x+y},
\qquad x,y\geq 0,\ x+y>0,
\end{equation}
with the convention \(h(0,0)=0\).

\begin{lemma}[Dimension-uniform analytic remainder]
\label{lem:analytic-remainder}
Fix \(0<c<C<\infty\).  There exist constants
\(\varrho_0=\varrho_0(c,C)>0\) and \(K=K(c,C)<\infty\), independent of \(n\), such
that the following holds.  Suppose that \(\boldsymbol{\alpha}\in\mathbb R^n\) satisfies
\(\sum_i\alpha_i=1, \quad \frac cn\leq \alpha_i\leq \frac Cn \quad\text{for all }i\),
and that \(\Delta_U,\Delta_V\) are symmetric, have zero diagonal, and satisfy
\(\|\Delta_U\|_{\op}\leq \varrho, \quad \|\Delta_V\|_{\op}\leq \varrho, \quad 0<\varrho\leq \varrho_0\).
Then, with \(h\) as in \eqref{eq:h-definition},
\begin{equation}\label{eq:analytic-remainder-expansion}
\|Z(\boldsymbol{\alpha})\|_1
=
1
-
\frac12
\sum_{i<j}
h(\alpha_i,\alpha_j)
\bigl((\Delta_U)_{ij}-(\Delta_V)_{ij}\bigr)^2
+
R(\boldsymbol{\alpha}).
\end{equation}
The remainder satisfies
\begin{equation}\label{eq:analytic-remainder-bounds}
|R(\boldsymbol{\alpha})|\leq K\varrho^3,
\qquad
\|\nabla R(\boldsymbol{\alpha})\|_\infty\leq K\varrho^3,
\qquad
\|\nabla^2R(\boldsymbol{\alpha})\|_{\op}\leq Kn\varrho^3,
\qquad
\|\nabla^2R(\boldsymbol{\alpha})\|_{\infty\to\infty}\leq Kn\varrho^3 .
\end{equation}
Moreover, if \(q\mapsto\boldsymbol{\alpha}(q)\) is affine, \(\xi:=\boldsymbol{\alpha}'(q)\), and
\(W_\xi:=n\operatorname{diag}(\xi)\), then
\begin{equation}\label{eq:analytic-remainder-path-bound}
\left|
\frac{d}{dq}R(\boldsymbol{\alpha}(q))
\right|
\leq
K\|W_\xi\|_{\op}\varrho^3 .
\end{equation}
\end{lemma}

\begin{proof}
By Lemma~\ref{lem:polar-orthonormalization}, \(\|Z(\boldsymbol{\alpha})\|_1 = \|(I+\Delta_U)^{1/2}\Lambda_{\boldsymbol{\alpha}}(I+\Delta_V)^{1/2}\|_1\).
Set \(\widetilde{\Lambda}_{\boldsymbol{\alpha}}:=n\Lambda_{\boldsymbol{\alpha}}\).  Then
\(\Lambda_{\boldsymbol{\alpha}}=n^{-1}\widetilde{\Lambda}_{\boldsymbol{\alpha}}\), and the entries of
\(\widetilde{\Lambda}_{\boldsymbol{\alpha}}\) lie in \([c,C]\).  We use the normalized Schatten
functional \(\mathsf N\) from \eqref{eq:normalized-schatten-functional}.

For \(0\leq t\leq1\), write
\[
L_t:=(I+t\Delta_U)^{1/2},
\qquad
R_t:=(I+t\Delta_V)^{1/2},
\qquad
Y_{\boldsymbol{\alpha}}(t):=L_t\Lambda_{\boldsymbol{\alpha}} R_t,
\qquad
g_{\boldsymbol{\alpha}}(t):=\|Y_{\boldsymbol{\alpha}}(t)\|_1 .
\]
Then \(g_{\boldsymbol{\alpha}}(1)=\|Z(\boldsymbol{\alpha})\|_1\) and
\(g_{\boldsymbol{\alpha}}(0)=\|\Lambda_{\boldsymbol{\alpha}}\|_1\).  Taylor expansion of the square-root map gives
\[
L_t=I+\frac t2\Delta_U-\frac{t^2}{8}\Delta_U^2+S_U(t),
\qquad
R_t=I+\frac t2\Delta_V-\frac{t^2}{8}\Delta_V^2+S_V(t),
\]
with \(\|S_U(t)\|_{\op},\|S_V(t)\|_{\op}\leq Kt^3\varrho^3\).  Since
\(\|\Lambda_{\boldsymbol{\alpha}}\|_{\op}\leq C/n\), multiplication gives
\[
Y_{\boldsymbol{\alpha}}(t)=\Lambda_{\boldsymbol{\alpha}}+tE_1+t^2E_2+\mathcal R_Y(t),
\qquad
\|\mathcal R_Y(t)\|_{\op}\leq K\frac{t^3\varrho^3}{n},
\]
where
\[
E_1=\frac12(\Delta_U\Lambda_{\boldsymbol{\alpha}}+\Lambda_{\boldsymbol{\alpha}}\Delta_V),
\qquad
E_2=
-\frac18\Delta_U^2\Lambda_{\boldsymbol{\alpha}}
-\frac18\Lambda_{\boldsymbol{\alpha}}\Delta_V^2
+\frac14\Delta_U\Lambda_{\boldsymbol{\alpha}}\Delta_V .
\]
Thus \(Y_{\boldsymbol{\alpha}}'(0)=E_1\) and \(Y_{\boldsymbol{\alpha}}''(0)=2E_2\).

Since \(\Lambda_{\boldsymbol{\alpha}}\) is positive diagonal, Lemma~\ref{lem:gradient-schatten-one}
gives \(\dd\|\cdot\|_1(\Lambda_{\boldsymbol{\alpha}})[H]=\langle I,H\rangle_F\).  Hence
\(g_{\boldsymbol{\alpha}}'(0)=\langle I,E_1\rangle_F\), and the second-order chain rule gives
\(g_{\boldsymbol{\alpha}}''(0) = \dd^2\|\cdot\|_1(\Lambda_{\boldsymbol{\alpha}})[E_1,E_1] + \dd\|\cdot\|_1(\Lambda_{\boldsymbol{\alpha}})[2E_2]\).
Therefore
\begin{equation}\label{eq:analytic-remainder-taylor-main}
\|Z(\boldsymbol{\alpha})\|_1
=
\|\Lambda_{\boldsymbol{\alpha}}\|_1
+
\langle I,E_1+E_2\rangle_F
+
\frac12
\dd^2\|\cdot\|_1(\Lambda_{\boldsymbol{\alpha}})[E_1,E_1]
+
R(\boldsymbol{\alpha}),
\end{equation}
where \(R(\boldsymbol{\alpha})\) is the third-order Taylor remainder.  Since
\(\|\Lambda_{\boldsymbol{\alpha}}\|_1=\sum_i\alpha_i=1\), it remains to compute the two
second-order contributions in \eqref{eq:analytic-remainder-taylor-main}.

First,
\[
\langle I,E_1\rangle_F
=
\frac12\operatorname{tr}(\Delta_U\Lambda_{\boldsymbol{\alpha}})
+
\frac12\operatorname{tr}(\Lambda_{\boldsymbol{\alpha}}\Delta_V)
=
0
\]
because \(\Delta_U\) and \(\Delta_V\) have zero diagonal.  Also,
\[
\operatorname{tr}(E_2)
=
-\frac18\operatorname{tr}(\Delta_U^2\Lambda_{\boldsymbol{\alpha}})
-\frac18\operatorname{tr}(\Lambda_{\boldsymbol{\alpha}}\Delta_V^2)
+\frac14\operatorname{tr}(\Delta_U\Lambda_{\boldsymbol{\alpha}}\Delta_V),
\]
so the contribution of each pair \(i<j\) is
\[
-\frac{\alpha_i+\alpha_j}{8}(\Delta_U)_{ij}^2
-\frac{\alpha_i+\alpha_j}{8}(\Delta_V)_{ij}^2
+
\frac{\alpha_i+\alpha_j}{4}(\Delta_U)_{ij}(\Delta_V)_{ij}.
\]
Thus
\begin{equation}\label{eq:analytic-remainder-trace-term}
\langle I,E_2\rangle_F
=
-\sum_{i<j}
\frac{\alpha_i+\alpha_j}{8}
\bigl((\Delta_U)_{ij}-(\Delta_V)_{ij}\bigr)^2 .
\end{equation}

Next, for \(i\neq j\), \((E_1)_{ij}-(E_1)_{ji} = \frac12(\alpha_j-\alpha_i) \bigl((\Delta_U)_{ij}-(\Delta_V)_{ij}\bigr)\).
Therefore Lemma~\ref{lem:schatten-one-hessian-diagonal} gives
\begin{equation}\label{eq:analytic-remainder-hessian-term}
\frac12\dd^2\|\cdot\|_1(\Lambda_{\boldsymbol{\alpha}})[E_1,E_1]
=
\sum_{i<j}
\frac{(\alpha_i-\alpha_j)^2}{8(\alpha_i+\alpha_j)}
\bigl((\Delta_U)_{ij}-(\Delta_V)_{ij}\bigr)^2 .
\end{equation}
Combining \eqref{eq:analytic-remainder-trace-term} and
\eqref{eq:analytic-remainder-hessian-term}, the coefficient of
\(\bigl((\Delta_U)_{ij}-(\Delta_V)_{ij}\bigr)^2\) is
\[
-\frac{\alpha_i+\alpha_j}{8}
+
\frac{(\alpha_i-\alpha_j)^2}{8(\alpha_i+\alpha_j)}
=
-\frac12\frac{\alpha_i\alpha_j}{\alpha_i+\alpha_j}.
\]
Using \eqref{eq:h-definition}, this proves
\eqref{eq:analytic-remainder-expansion}.

It remains to prove \eqref{eq:analytic-remainder-bounds} and
\eqref{eq:analytic-remainder-path-bound}.  For diagonal \(W_1,W_2\), define
\(\varphi_{W_1,W_2}(t,s,r) := \mathsf N\bigl(L_t(\widetilde{\Lambda}_{\boldsymbol{\alpha}}+sW_1+rW_2)R_t\bigr)\),
and \(\mathcal R_{W_1,W_2}(s,r) := \frac12 \int_0^1(1-t)^2 \partial_t^3\varphi_{W_1,W_2}(t,s,r)\,\dd t\).
Then \(R(\boldsymbol{\alpha})=\mathcal R_{0,0}(0,0)\), so
Lemma~\ref{lem:diagonal-remainder-derivatives} gives \(|R(\boldsymbol{\alpha})|\leq K\varrho^3\).

We now translate the diagonal-matrix derivative bounds to derivatives in
\(\boldsymbol{\alpha}\).  A perturbation of \(\boldsymbol{\alpha}\) in direction \(\boldsymbol{v}\) corresponds to
\(W_{\boldsymbol{v}}:=n\operatorname{diag}(\boldsymbol{v})\), since \(\widetilde{\Lambda}_{\boldsymbol{\alpha}+s\boldsymbol{v}} = \widetilde{\Lambda}_{\boldsymbol{\alpha}}+sW_{\boldsymbol{v}}\).
Taking \(\boldsymbol{v}=\boldsymbol{e}_i\), \(W_i:=n\boldsymbol{e}_i\boldsymbol{e}_i^\top\), we have \(\partial_{\alpha_i}R(\boldsymbol{\alpha}) = \left.\partial_s\mathcal R_{W_i,0}(s,0)\right|_{s=0}\).
Since \(\|W_i\|_1=n\), Lemma~\ref{lem:diagonal-remainder-derivatives} gives
\(|\partial_{\alpha_i}R(\boldsymbol{\alpha})|\leq K\varrho^3\).  Taking the maximum over \(i\)
proves \(\|\nabla R(\boldsymbol{\alpha})\|_\infty\leq K\varrho^3\).

For the operator-norm Hessian bound, let \(\boldsymbol v,\boldsymbol w\in\mathbb R^n\) and set
\(W_{\boldsymbol{v}}:=n\operatorname{diag}(\boldsymbol{v})\), \(W_{\boldsymbol{w}}:=n\operatorname{diag}(\boldsymbol w)\).  Since
\(R(\boldsymbol \alpha+s \boldsymbol v+r \boldsymbol w)=\mathcal R_{W_{\boldsymbol{v}},W_{\boldsymbol{w}}}(s,r)\), Lemma~\ref{lem:diagonal-remainder-derivatives} gives
\(|\nabla^2R(\boldsymbol{\alpha})[\boldsymbol v,\boldsymbol w]| \leq K\varrho^3\frac{\|W_{\boldsymbol{v}}\|_F\|W_{\boldsymbol{w}}\|_F}{n} = Kn\varrho^3\|\boldsymbol v\|_2\|\boldsymbol w\|_2\).
Taking the supremum over unit \(\boldsymbol v,\boldsymbol w\) proves \(\|\nabla^2R(\boldsymbol{\alpha})\|_{\op}\leq Kn\varrho^3\).

For the row-sum bound, write \(H_R(\boldsymbol{\alpha}):=\nabla^2R(\boldsymbol{\alpha})\).  For each
\(i\), \(\sum_{j=1}^n |(H_R(\boldsymbol{\alpha}))_{ij}| = \sup_{\|\boldsymbol w\|_\infty\leq1}|\boldsymbol{e}_i^\top H_R(\boldsymbol{\alpha})\boldsymbol w|\).
Fix \(\|\boldsymbol w\|_\infty\leq1\).  With \(W_i:=n\boldsymbol{e}_i\boldsymbol{e}_i^\top\) and
\(W_{\boldsymbol{w}}:=n\operatorname{diag}(\boldsymbol w)\), we have \(\|W_i\|_1=n\) and
\(\|W_{\boldsymbol{w}}\|_{\op}\leq n\).  Therefore, by
\eqref{eq:diagonal-remainder-mixed-second-bound},
\[
|\boldsymbol{e}_i^\top H_R(\boldsymbol{\alpha})\boldsymbol vw|
=
|\nabla^2R(\boldsymbol{\alpha})[\boldsymbol{e}_i,\boldsymbol w]|
\leq
K\varrho^3\frac{\|W_i\|_1\|W_{\boldsymbol{w}}\|_{\op}}{n}
\leq
Kn\varrho^3.
\]
Taking the supremum over \(w\) and then \(i\) proves the
\(\infty\to\infty\) bound in \eqref{eq:analytic-remainder-bounds}.

Finally, let \(q\mapsto\boldsymbol{\alpha}(q)\) be affine, set \(\boldsymbol{\beta}:=\boldsymbol{\alpha}(q)\),
\(\boldsymbol \xi:=\boldsymbol \alpha'(q)\), and \(W_\xi:=n\operatorname{diag}(\boldsymbol{\xi})\).  Since
\(R(\boldsymbol \beta+s \boldsymbol\xi)=\mathcal R_{W_{\boldsymbol\xi},0}(s,0)\), \(\frac{d}{dq}(R\circ \boldsymbol \alpha)(q) = \left.\partial_s\mathcal R_{W_{\boldsymbol\xi},0}(s,0)\right|_{s=0}\).
Applying Lemma~\ref{lem:diagonal-remainder-derivatives} gives
\[
\left|
\frac{d}{dq}(R\circ\boldsymbol \alpha)(q)
\right|
\leq
K\varrho^3\frac{\|W_{\boldsymbol\xi}\|_1}{n}
\leq
K\|W_{\boldsymbol\xi}\|_{\op}\varrho^3,
\]
which proves \eqref{eq:analytic-remainder-path-bound}.
\end{proof}

\subsection{Deterministic averaged geometry}\label{app:deterministic_geometry}

In this section we isolate the deterministic geometry of the averaged
second-order term. 

For \(q\in[0,1]\), define the fixed group-mass affine set
\[
\mathcal A_q
:=
\left\{
\boldsymbol\alpha\in\R_+^n:
\sum_{i\in I_+}\alpha_i=q,\ 
\sum_{i\in I_-}\alpha_i=1-q
\right\},
\]
and the groupwise-uniform curve
\[
\overline{\boldsymbol\alpha}_i(q)
=
\begin{cases}
q/n_+,&i\in I_+,\\[1mm]
(1-q)/n_-,&i\in I_-.
\end{cases}
\]

With \(h\) as in \eqref{eq:h-definition}, define
\begin{equation}\label{eq:H-definition}
\cH(\boldsymbol \alpha):=\sum_{i<j}h(\alpha_i,\alpha_j),
\qquad
\cH_{n_+,n_-}(q):=\cH(\overline{\boldsymbol \alpha}(q)).
\end{equation}
The main point of this section is not merely concavity of
\(\cH\), but the quantitative fact that, near this groupwise-uniform curve,
\(\cH\) has curvature of order \(n^2\) in all within-group zero-sum directions.
This curvature will later provide the deterministic concavity controlling
within-group fluctuations.

The first lemma records the exact value of \(\cH\) along the
groupwise-uniform curve and its derivatives with respect to \(q\).

\begin{lemma}[Exact groupwise-uniform functional]
\label{lem:H-formula}
For \(q\in[0,1]\),
\begin{equation}\label{eq:H-formula}
\cH_{n_+,n_-}(q)
=
\frac{n_+-1}{4}q
+
\frac{n_--1}{4}(1-q)
+
\frac{n_+n_- q(1-q)}
{n_+(1-q)+n_- q}.
\end{equation}
Moreover,
\begin{equation}\label{eq:theta-definition}
\theta_{n_+,n_-}
:=
\cH_{n_+,n_-}'\left(\frac12\right)
=
\frac{
(n_+ - n_-)(n_+^2+6n_+n_-+n_-^2)
}{
4(n_+ + n_-)^2
}.
\end{equation}
Moreover,
\begin{equation}\label{eq:Hsecond}
\cH_{n_+,n_-}''(q)
=
-\frac{2n_+^2n_-^2}
{(n_+(1-q)+n_- q)^3}<0.
\end{equation}
\end{lemma}

\begin{proof}
By direct computation.
\end{proof}

We next study the curvature of \(\cH\) in directions that keep the group
masses fixed.  These are the directions relevant after \(q\) has been
chosen.

Define the tangent subspace
\[
\T
:=
\left\{
\boldsymbol\beta\in\R^n:
\sum_{i\in I_+}\beta_i=0,\ 
\sum_{i\in I_-}\beta_i=0
\right\}.
\]
Then \(\mathcal A_q=\overline{\boldsymbol \alpha}(q)+\T\).
Let \(P_{\T}\) denote the Euclidean orthogonal projection onto \(\T\).
Equivalently, for \(\boldsymbol x\in\R^n\), \((P_{\T}\boldsymbol x)_i = x_i-\frac1{n_+}\sum_{\ell\in I_+}x_\ell, \quad i\in I_+\),
and \((P_{\T}\boldsymbol x)_i = x_i-\frac1{n_-}\sum_{\ell\in I_-}x_\ell, \quad i\in I_-\).
We write \(I_{\T}\) for the identity operator on \(\T\).  Thus an
operator inequality involving \(I_{\T}\) is understood as an inequality
after restricting both sides to the subspace \(\T\).

Fix a small constant \(r_0\in(0,1/4)\) and define the local \(q\)-window
\begin{equation}\label{eq:q-window}
\mathcal I
:=
\left[\frac12-r_0,\frac12+r_0\right].
\end{equation}

The following lemma says that \(\overline{\boldsymbol \alpha}(q)\) is a critical point
of \(\cH\) after restricting to the fixed group-mass affine space
\(\mathcal A_q\), and that the restricted Hessian is negative definite
with curvature scale \(n^2\).  Equivalently, the restricted gradient is
\(P_{\T}\nabla\cH\), and the restricted Hessian is
\(P_{\T}\nabla^2\cH P_{\T}\) acting on \(\T\).

\begin{lemma}[Exact mean curvature]
\label{lem:mean-curvature}
Assume \(n_+/n,n_-/n\in[\kappa,1-\kappa]\) for a fixed
\(\kappa>0\).  Uniformly over \(q\in\mathcal I\), \(P_{\T}\nabla\cH(\overline{\boldsymbol \alpha}(q))=0\),
and \(cn^2I_{\T} \preceq - P_{\T}\nabla^2\cH(\overline{\boldsymbol \alpha}(q))P_{\T} \preceq Cn^2I_{\T}\).
Moreover, the restriction is a scalar multiple of the identity on each
of the two within-group zero-sum subspaces.  Consequently, for \(L_q:= -\frac1D P_{\T}\nabla^2\cH(\overline{\boldsymbol \alpha}(q))P_{\T}\),
we have \(\|L_q^{-1}\|_{\op}\leq C\frac{D}{n^2}\)
and \(\|L_q^{-1}\boldsymbol x\|_\infty \leq C\frac{D}{n^2}\|\boldsymbol x\|_\infty\).
Here and below, operator norms on \(\T\) are taken with respect to the
Euclidean norm inherited from \(\R^n\).
\end{lemma}

\begin{proof}
Fix \(q\in\mathcal I\), and write \(a:=\frac{q}{n_+}, \quad b:=\frac{1-q}{n_-}\).
Since \(q\in\mathcal I\subset(1/4,3/4)\) and
\(n_+/n,n_-/n\in[\kappa,1-\kappa]\), we have \(a\asymp \frac1n, \quad b\asymp \frac1n, \quad a+b\asymp \frac1n\),
uniformly over \(q\in\mathcal I\).

We first prove the projected-gradient identity.  For \(x,y>0\), \(\partial_x h(x,y)=\frac{y^2}{(x+y)^2}, \quad \partial_y h(x,y)=\frac{x^2}{(x+y)^2}\).
At \(\overline{\boldsymbol \alpha}(q)\), all coordinates inside the same group are
equal.  Therefore the \(i\)-th coordinate of
\(\nabla\cH(\overline{\boldsymbol \alpha}(q))\) is the same for all \(i\in I_+\), and
is also the same for all \(i\in I_-\).  More explicitly, for
\(i\in I_+\), 
\[
\begin{aligned}
\partial_{\alpha_i}\cH(\overline{\boldsymbol \alpha}(q))
& = \sum_{j\in I_+}\partial_{\alpha_i}h(\alpha_i, \alpha_j)\big|_{\boldsymbol \alpha = \overline{\boldsymbol \alpha} (q)} + \sum_{j\in I_-}\partial_{\alpha_i}h(\alpha_i, \alpha_j)\big|_{\boldsymbol \alpha = \overline{\boldsymbol \alpha} (q)} \\
& = \sum_{j\in I_+}\frac{a^2}{(2a)^2} + \sum_{j\in I_-}\frac{b^2}{(a+b)^2} \\
&  = \frac{n_+-1}{4} + n_-\frac{b^2}{(a+b)^2},    
\end{aligned}
\]
while for \(i\in I_-\),
\[
\begin{aligned}
\partial_{\alpha_i}\cH(\overline{\boldsymbol \alpha}(q)) 
& = \sum_{j\in I_-}\partial_{\alpha_i}h(\alpha_i, \alpha_j)\big|_{\boldsymbol \alpha = \overline{\boldsymbol \alpha} (q)} + \sum_{j\in I_+}\partial_{\alpha_i}h(\alpha_i, \alpha_j)\big|_{\boldsymbol \alpha = \overline{\boldsymbol \alpha} (q)} \\
& = \sum_{j\in I_-}\frac{b^2}{(2b)^2} + \sum_{j\in I_+}\frac{a^2}{(b+a)^2} \\
& = \frac{n_--1}{4} + n_+\frac{a^2}{(a+b)^2}.    
\end{aligned}
\]
Thus \(\nabla\cH(\overline{\boldsymbol \alpha}(q))\) is constant within each group.
Since \(P_{\T}\) subtracts the within-group average from each group, it
annihilates every vector that is constant within each group.  Hence \(P_{\T}\nabla\cH(\overline{\boldsymbol \alpha}(q))=0\).

We now turn to the Hessian.  For one pair, the second directional
derivative satisfies \(-\nabla^2 h(x,y)[(u,v),(u,v)] = \frac{2(yu-xv)^2}{(x+y)^3}\).
Therefore, for \(\boldsymbol \beta\in\T\), \(-\boldsymbol \beta^\top \nabla^2\cH(\overline{\boldsymbol \alpha}(q)) \boldsymbol \beta = S_{++}(\boldsymbol \beta)+S_{--}(\boldsymbol \beta)+S_{+-}(\boldsymbol \beta)\),
where the three terms correspond to pairs inside \(I_+\), pairs inside
\(I_-\), and cross-group pairs.

For the \(I_+\)-within pairs, both coordinates equal \(a\).  Hence
\[
S_{++}(\boldsymbol \beta)
=
\frac1{4a}
\sum_{\substack{i<j\\ i,j\in I_+}}
(\beta_i-\beta_j)^2.
\]
Since \(\boldsymbol \beta\in\T\), we have \(\sum_{i\in I_+}\beta_i=0\), and therefore
\[
\sum_{\substack{i<j\\ i,j\in I_+}}
(\beta_i-\beta_j)^2
=
n_+\sum_{i\in I_+}\beta_i^2.
\]
Thus \(S_{++}(\boldsymbol \beta) = \frac{n_+}{4a} \sum_{i\in I_+}\beta_i^2\).
Similarly, \(S_{--}(\boldsymbol \beta) = \frac{n_-}{4b} \sum_{i\in I_-}\beta_i^2\).

For the cross-group pairs, we have
\[
S_{+-}(\boldsymbol \beta)
=
\frac{2}{(a+b)^3}
\sum_{\substack{i\in I_+\\ j\in I_-}}
(b\beta_i-a\beta_j)^2.
\]
Expanding the square gives
\[
\begin{aligned}
S_{+-}(\boldsymbol \beta)
&=
\frac{2}{(a+b)^3}
\left[
n_-b^2\sum_{i\in I_+}\beta_i^2
+
n_+a^2\sum_{j\in I_-}\beta_j^2
-
2ab
\left(\sum_{i\in I_+}\beta_i\right)
\left(\sum_{j\in I_-}\beta_j\right)
\right].
\end{aligned}
\]
The last term vanishes because \(\boldsymbol \beta\in\T\).  Hence
\[
S_{+-}(\boldsymbol \beta)
=
\frac{2n_-b^2}{(a+b)^3}
\sum_{i\in I_+}\beta_i^2
+
\frac{2n_+a^2}{(a+b)^3}
\sum_{j\in I_-}\beta_j^2.
\]

Combining the three contributions, we obtain
\[
-\boldsymbol \beta^\top
\nabla^2\cH(\overline{\boldsymbol \alpha}(q))
\boldsymbol \beta
=
\lambda_+(q)
\sum_{i\in I_+}\beta_i^2
+
\lambda_-(q)
\sum_{i\in I_-}\beta_i^2,
\]
where \(\lambda_+(q) := \frac{n_+}{4a} + \frac{2n_-b^2}{(a+b)^3}\),
and \(\lambda_-(q) := \frac{n_-}{4b} + \frac{2n_+a^2}{(a+b)^3}\).
Using
\[
a\asymp \frac1n,
\qquad
b\asymp \frac1n,
\qquad
a+b\asymp \frac1n,
\qquad
n_+\asymp n,
\qquad
n_-\asymp n,
\]
we get \(cn^2\leq \lambda_+(q),\lambda_-(q)\leq Cn^2\)
uniformly over \(q\in\mathcal I\).  Therefore, for every \(\boldsymbol \beta\in\T\),
\(cn^2\|\boldsymbol \beta\|_2^2 \leq -\boldsymbol \beta^\top \nabla^2\cH(\overline{\boldsymbol \alpha}(q))\boldsymbol  \beta \leq Cn^2\|\boldsymbol \beta\|_2^2\).
Equivalently, \(cn^2I_{\T} \preceq - P_{\T}\nabla^2\cH(\overline{\boldsymbol \alpha}(q))P_{\T} \preceq Cn^2I_{\T}\).

The displayed formula also shows the scalar-multiple structure.  On the
zero-sum subspace supported on \(I_+\), the restricted operator \(- P_{\T}\nabla^2\cH(\overline{\boldsymbol \alpha}(q))P_{\T}\)
acts as multiplication by \(\lambda_+(q)\).  On the zero-sum subspace
supported on \(I_-\), it acts as multiplication by \(\lambda_-(q)\).
Thus its inverse on \(\T\) acts by multiplication by
\(\lambda_+(q)^{-1}\) on the first subspace and by
\(\lambda_-(q)^{-1}\) on the second.  Since \(\lambda_+(q),\lambda_-(q)\geq cn^2\),
we obtain \(\|L_q^{-1}\|_{\op} \leq C\frac{D}{n^2}\).
Moreover, if \(\boldsymbol x\in\T\), then \(L_q^{-1}\boldsymbol x\) is obtained by multiplying
the \(I_+\)-component of \(\boldsymbol x\) by \(D/\lambda_+(q)\) and the
\(I_-\)-component of \(x\) by \(D/\lambda_-(q)\).  Therefore \(\|L_q^{-1}\boldsymbol x\|_\infty \leq C\frac{D}{n^2}\|\boldsymbol x\|_\infty\).
\end{proof}

\subsection{Canonical spherical kernel}

For
\[
(\boldsymbol u,\widehat{\boldsymbol v}),(\boldsymbol u',\widehat{\boldsymbol v}')
\in
\mathbb S^{D-1}\times\mathbb S^{D-1},
\]
define
\[
\zeta\bigl((\boldsymbol u,\widehat{\boldsymbol v}),(\boldsymbol u',\widehat{\boldsymbol v}')\bigr)
:=
\left(\boldsymbol u^\top \boldsymbol u'-\widehat{\boldsymbol v}^\top \widehat{\boldsymbol v}'\right)^2-\frac2D.
\]
Then, for \(i\neq j\),
\begin{equation}\label{eq:zeta-ij-definition}
\zeta_{ij}
=
\zeta\bigl(
(\boldsymbol u_i,\widehat{\boldsymbol v}_i),
(\boldsymbol u_j,\widehat{\boldsymbol v}_j)
\bigr)
=
\left(
(\Delta_U)_{ij}-(\Delta_V)_{ij}
\right)^2-\frac2D .
\end{equation}
For convenience, write \(\omega_i:=(\boldsymbol u_i,\widehat{\boldsymbol v}_i)\).

\begin{lemma}[Fourth moments on the sphere]
\label{lem:spherical-fourth-moments}
Let \(\boldsymbol u\sim\mathrm{Unif}(\mathbb S^{D-1})\).  Then \(\E u_1^4=\frac{3}{D(D+2)}, \quad \E u_1^2u_2^2=\frac{1}{D(D+2)}, \quad \E u_1^3u_2=0\),
More generally, for any deterministic unit vectors
\(\boldsymbol x,\boldsymbol y\in\mathbb S^{D-1}\),
\[
\E
\left[
(\boldsymbol x^\top\boldsymbol u)^2
(\boldsymbol y^\top\boldsymbol u)^2
\right]
=
\frac{1+2(\boldsymbol x^\top\boldsymbol y)^2}{D(D+2)}.
\]
\end{lemma}

\begin{proof}
Using the representation
\(\boldsymbol{u}= \boldsymbol g/\|\boldsymbol g\|_2\), with \(\boldsymbol g\sim N(0,I_D)\), we have $u_1 = g_1/\|\boldsymbol g\|_2$
Recall that if \(X\sim\Gamma(\alpha,\theta)\) and
\(Y\sim\Gamma(\beta,\theta)\) are independent gamma random variables with
the same scale parameter, then \(\frac{X}{X+Y}\sim\mathrm{Beta}(\alpha,\beta)\)
(see \cite[Exercise~4.24]{casella2002statistical}). Now \(g_1^2\sim\chi_1^2=\Gamma(1/2,2)\), while
\(g_2^2+\cdots+g_D^2\sim\chi_{D-1}^2 = \Gamma\left(\frac{D-1}{2},2\right)\),
and these two random variables are independent.  Therefore \(u_1^2 = \frac{g_1^2}{g_1^2+\cdots+g_D^2} \sim \mathrm{Beta}\left(\frac12,\frac{D-1}{2}\right)\).
Therefore, we have \(\E u_1^4=\frac{3}{D(D+2)}\).
Next, since \(\sum_{k=1}^D u_k^2=1\),
we have \(1 = \left(\sum_{k=1}^D u_k^2\right)^2 = \sum_{k=1}^D u_k^4 + 2\sum_{1\leq k<\ell\leq D}u_k^2u_\ell^2\).
Taking expectations and using symmetry gives \(1 = D\E u_1^4 + D(D-1)\E u_1^2u_2^2\).
Using the formula for \(\E u_1^4\), we obtain \(\E u_1^2u_2^2=\frac{1}{D(D+2)}\).

Now let \(\boldsymbol x,\boldsymbol y\in\mathbb S^{D-1}\), and set \(a:=\boldsymbol x^\top\boldsymbol y\).
By rotational invariance, we may assume \(\boldsymbol x=\boldsymbol e_1, \quad \boldsymbol y=a\boldsymbol e_1+\sqrt{1-a^2}\boldsymbol e_2\).
Then \(\boldsymbol x^\top\boldsymbol u=u_1, \quad \boldsymbol y^\top\boldsymbol u = a u_1+\sqrt{1-a^2}u_2\).
Therefore
\[
\begin{aligned}
\E
\left[
(\boldsymbol x^\top\boldsymbol u)^2
(\boldsymbol y^\top\boldsymbol u)^2
\right]
&=
\E\left[
u_1^2
\left(a u_1+\sqrt{1-a^2}u_2\right)^2
\right] \\
&=
a^2\E u_1^4
+
(1-a^2)\E u_1^2u_2^2,
\end{aligned}
\]
because \(\E u_1^3u_2=0\) by symmetry.  Substituting the two coordinate
moment formulas gives
\[
\E
\left[
(\boldsymbol x^\top\boldsymbol u)^2
(\boldsymbol y^\top\boldsymbol u)^2
\right]
=
\frac{3a^2}{D(D+2)}
+
\frac{1-a^2}{D(D+2)}
=
\frac{1+2a^2}{D(D+2)}.
\]
This proves the claim.
\end{proof}

\begin{lemma}[Conditional moments and degeneracy]
\label{lem:zeta}
Let \(\omega_0=(\boldsymbol u_0,\widehat{\boldsymbol v}_0) \in \mathbb S^{D-1}\times\mathbb S^{D-1}\)
be fixed, and let \(\omega=(\boldsymbol u,\widehat{\boldsymbol v})\)
be random, where \(\boldsymbol u\) and \(\widehat{\boldsymbol v}\) are
independent and uniformly distributed on \(\mathbb S^{D-1}\).  Then \(\E_\omega \zeta(\omega_0,\omega)=0\).
For every \(p\geq1\), \(\|\zeta(\omega_0,\omega)\|_{L_p(\omega)} \leq C_\zeta\frac pD\)
uniformly over \(\omega_0\).  Moreover, \(\E_\omega \zeta(\omega_0,\omega)^2 = \frac{6}{D(D+2)}+\frac{2}{D^2}\),
and hence \(\frac{4}{D^2} \leq \E_\omega \zeta(\omega_0,\omega)^2 \leq \frac{8}{D^2}\).

Consequently, for every \(i\neq j\),
\[
\E[\zeta_{ij}\mid \omega_i]=0,
\qquad
\|\zeta_{ij}\|_{L_p}\leq C_\zeta\frac pD,
\qquad
\frac{4}{D^2}\leq\E\zeta_{ij}^2\leq\frac{8}{D^2}.
\]
Moreover, the conditional second-moment bound \(\E[\zeta_{ij}^2\mid \omega_i]\leq \frac{8}{D^2}\)
holds almost surely.
\end{lemma}

\begin{proof}
Write
\[
\omega_0=(\boldsymbol u_0,\widehat{\boldsymbol v}_0),
\qquad
\omega=(\boldsymbol u,\widehat{\boldsymbol v}),
\]
and set \(r:=\boldsymbol u_0^\top\boldsymbol u, \quad t:=\widehat{\boldsymbol v}_0^\top\widehat{\boldsymbol v}\).
Since \(\omega_0\) is fixed, and by rotational invariance, the random
variables \(r\) and \(t\) are independent, and each has the same
distribution as \(u_1=\langle \boldsymbol e_1,\boldsymbol u\rangle, \quad \boldsymbol u\sim\mathrm{Unif}(\mathbb S^{D-1})\).
Therefore \(\E_\omega r=\E_\omega t=0, \quad \E_\omega r^2=\E_\omega t^2=\frac1D\).
Since \(\zeta(\omega_0,\omega) = (r-t)^2-\frac2D\),
we obtain \(\E_\omega\zeta(\omega_0,\omega) = \E_\omega(r-t)^2-\frac2D = \frac1D+\frac1D-\frac2D = 0\).

Next, by the sub-Gaussianity of one-dimensional
marginals of the uniform distribution on the sphere
\cite[Theorem~3.4.5]{vershynin2026high}, we have, for every \(p\geq1\), \(\|r\|_{L_p(\omega)} = \|t\|_{L_p(\omega)}\leq C_s\sqrt{\frac pD}\),
where $C_s (>1)$ is a universal constant.
Hence
\[
\begin{aligned}
\|(r-t)^2\|_{L_p(\omega)}
&=
\|r-t\|_{L_{2p}(\omega)}^2 \\
&\leq
2\|r\|_{L_{2p}(\omega)}^2
+
2\|t\|_{L_{2p}(\omega)}^2 \\
&\leq
8C_s^2\frac pD.
\end{aligned}
\]
Since \(2/D\leq 2p/D\), this gives \(\|\zeta(\omega_0,\omega)\|_{L_p(\omega)} \leq C_\zeta\frac pD\),
where $C_\zeta = 8C_s^2 + 2$.

It remains to compute the second moment.  By
Lemma~\ref{lem:spherical-fourth-moments}, \(\E_\omega r^4=\E_\omega t^4=\frac{3}{D(D+2)}\).
Using independence of \(r\) and \(t\),
\[
\E_\omega(r-t)^4
=
\E_\omega r^4
+
6\E_\omega r^2\E_\omega t^2
+
\E_\omega t^4
=
\frac{6}{D(D+2)}+\frac{6}{D^2}.
\]
Therefore
\[
\begin{aligned}
\E_\omega\zeta(\omega_0,\omega)^2
&=
\E_\omega\left[
\left((r-t)^2-\frac2D\right)^2
\right] \\
&=
\E_\omega(r-t)^4
-
\frac4D\E_\omega(r-t)^2
+
\frac4{D^2} \\
&=
\frac{6}{D(D+2)}
+
\frac{6}{D^2}
-
\frac8{D^2}
+
\frac4{D^2} \\
&=
\frac{6}{D(D+2)}+\frac{2}{D^2}.
\end{aligned}
\]
This lies between \(4/D^2\) and \(8/D^2\).

The final statements follow by applying the fixed-\(\omega_0\) identities
with \(\omega_0=\omega_i\) and \(\omega=\omega_j\), and then conditioning
on \(\omega_i\).
\end{proof}

\begin{lemma}[Spherical coherence moment bound]
\label{lem:spherical-coherence-moment}
Let \(\boldsymbol u_1,\ldots,\boldsymbol u_n, \widehat{\boldsymbol v}_1,\ldots,\widehat{\boldsymbol v}_n\)
be independent random vectors, each uniformly distributed on
\(\mathbb S^{D-1}\).  Let \(L=\log n\).  Then, for every \(p\geq1\),
\[
\left\|
\max_{i\neq j}
\left\{
|\boldsymbol u_i^\top\boldsymbol u_j|,
|\widehat{\boldsymbol v}_i^\top\widehat{\boldsymbol v}_j|
\right\}
\right\|_{L_{2p}}
\leq
2C_s n^{1/p}\sqrt{\frac{p}{D}}.
\]
\end{lemma}

\begin{proof}
For each \(i\neq j\), by rotational invariance, \(\boldsymbol u_i^\top\boldsymbol u_j\)
has the distribution of a spherical inner product.  Hence, by the
spherical inner-product moment bound used in Lemma~\ref{lem:zeta}, for
every \(\ell\geq1\), \(\|\boldsymbol u_i^\top\boldsymbol u_j\|_{L_\ell} \leq C_s\sqrt{\frac{\ell}{D}}\).
The same bound holds for
\(\widehat{\boldsymbol v}_i^\top\widehat{\boldsymbol v}_j\).

Let
\[
M:=
\max_{i\neq j}
\left\{
|\boldsymbol u_i^\top\boldsymbol u_j|,
|\widehat{\boldsymbol v}_i^\top\widehat{\boldsymbol v}_j|
\right\}.
\]
Using the elementary maximum inequality
\[
\left\|\max_{j\in\mathcal J}|X_j|\right\|_{L_\ell}
\leq
\left(
\sum_{j\in\mathcal J}\E |X_j|^\ell
\right)^{1/\ell},
\]
with \(\ell=2p\), we get
\[
\begin{aligned}
\|M\|_{L_{2p}}
&\leq
\left(
\sum_{i\neq j}
\E|\boldsymbol u_i^\top\boldsymbol u_j|^{2p}
+
\sum_{i\neq j}
\E|\widehat{\boldsymbol v}_i^\top\widehat{\boldsymbol v}_j|^{2p}
\right)^{1/(2p)} \\
&\leq
(2n(n-1))^{1/(2p)}
C_s\sqrt{\frac{2p}{D}} \\
&\leq
2C_s n^{1/p}\sqrt{\frac{p}{D}}.
\end{aligned}
\]
\end{proof}

The following shared-index covariance formula is useful for verifying
the block-variance parameter in the matrix \(U\)-statistic inequality.

\begin{lemma}[Shared-index covariance]
\label{lem:shared-cov}
Let \((\boldsymbol u,\widehat{\boldsymbol v})\) be an independent copy of
\((\boldsymbol u_1,\widehat{\boldsymbol v}_1)\).  Define
\[
\zeta_i(\boldsymbol u,\widehat{\boldsymbol v})
:=
\zeta\bigl(
(\boldsymbol u_i,\widehat{\boldsymbol v}_i),
(\boldsymbol u,\widehat{\boldsymbol v})
\bigr).
\]
For \(i\neq j\), put
\[
a:=\boldsymbol{u}_i^\top \boldsymbol{u}_j,
\qquad
b:=\widehat{\boldsymbol{v}}_i^\top\widehat{\boldsymbol{v}}_j.
\]
Then
\[
\E_{\boldsymbol u,\widehat{\boldsymbol v}}
\left[
\zeta_i(\boldsymbol u,\widehat{\boldsymbol v})
\zeta_j(\boldsymbol u,\widehat{\boldsymbol v})
\right]
=
\frac{2(Da^2-1)}{D^2(D+2)}
+
\frac{2(Db^2-1)}{D^2(D+2)}
+
\frac{4ab}{D^2}.
\]
In particular, on the coherence event \(\max_{i\neq j}\{|(\Delta_U)_{ij}|,|(\Delta_V)_{ij}|\} \leq\mu\),
the absolute value is at most \(8\left(\frac{\mu^2}{D^2}+\frac1{D^3}\right)\).
\end{lemma}

\begin{proof}
Write
\[
r_i=\boldsymbol{u}_i^\top \boldsymbol{u},
\qquad
t_i=\widehat{\boldsymbol{v}}_i^\top\widehat{\boldsymbol{v}}.
\]
Then \(\zeta_i(\boldsymbol u,\widehat{\boldsymbol v}) = (r_i^2-1/D)+(t_i^2-1/D)-2r_it_i\).
Since \(\boldsymbol{u}\) and \(\widehat{\boldsymbol{v}}\) are independent so as $r_i$ and $t_i$ conditioning on $\boldsymbol u_i, \widehat{\boldsymbol v}_i$. 

By Lemma~\ref{lem:spherical-fourth-moments}, we obtain \(\E_{\boldsymbol u} r_i^2r_j^2 = \frac{1+2a^2}{D(D+2)}\).

Thus, \(\E_{\boldsymbol u}[(r_i^2-1/D)(r_j^2-1/D)] = \frac{2(Da^2-1)}{D^2(D+2)}\).
The corresponding \(t\)-term has the same form with \(b\).  The only mixed term that does not vanish is
\[
\E_{\boldsymbol u,\widehat{\boldsymbol v}}
[
(-2r_it_i)(-2r_jt_j)
]
=
4\E_{\boldsymbol u}[r_ir_j]\,
\E_{\widehat{\boldsymbol v}}[t_it_j].
\]
Since
\[
\E_{\boldsymbol u}[r_ir_j]
=
\boldsymbol u_i^\top
\E[\boldsymbol u\boldsymbol u^\top]
\boldsymbol u_j
=
\frac{a}{D},
\qquad
\E_{\widehat{\boldsymbol v}}[t_it_j]
=
\frac{b}{D},
\]
we obtain \(4\E_{\boldsymbol u}[r_ir_j]\, \E_{\widehat{\boldsymbol v}}[t_it_j] = \frac{4ab}{D^2}\).
\end{proof}

\subsection{A specialized matrix \(U\)-statistic bound}

We now record the only non-elementary concentration input.  The result we
need is a direct specialization of Theorem~3.1 of \cite{MinskerWei} to a scalar completely degenerate kernel multiplied by
deterministic sparse self-adjoint coefficient matrices.

We first state the Minsker--Wei inequality in the form used below.  Let
\(\omega_1,\ldots,\omega_n\) be i.i.d. random variables taking values in
a measurable space \(\mathcal S\), $\mathbb S^d \in \mathbb R^{d\times d}$ be the set of all symmetric matrices, and let
\(H_{ij}:\mathcal S\times\mathcal S\to \mathbb S^d, \quad i\neq j\)
be matrix-valued kernels satisfying \(H_{ij}(\omega,\omega')=H_{ji}(\omega',\omega)\).
Assume that the kernels are completely degenerate, namely \(\E_{\omega'}H_{ij}(\omega,\omega')=0 \quad \text{for every fixed }\omega\).
Let \(\{\omega_i^{(1)}\}_{i=1}^n\) and
\(\{\omega_i^{(2)}\}_{i=1}^n\) be two independent copies of the original
sample.  Define the \(U\)-statistic \(\mathsf U_H := \sum_{(i,j):\,i\neq j}H_{ij}(\omega_i,\omega_j)\),
and let \(\widetilde G\in\mathbb H^{nd}\) be the block matrix whose
\((i,j)\)-block is
\[
\widetilde G_{ij}
=
\begin{cases}
H_{ij}(\omega_i^{(1)},\omega_j^{(2)}),& i\neq j,\\
0,& i=j.
\end{cases}
\]

\begin{theorem}[Theorem~3.1, \cite{MinskerWei}]\label{thm:MinskerWei}
For every \(p\geq1\) with \(r:=\max\{p,\log(ed)\}\),
we have
\[
\left(\E\|\mathsf U_H\|_{\op}^{2p}\right)^{1/(2p)}
\leq
C_{\scriptscriptstyle{MW}}
\left[
r^{3/2} A_p + r B + r C_p
\right],
\]
where
\[
\begin{aligned}
A_p &:=
\left(
\E \max_i
\left\|
\sum_{j:\,j\neq i}
H_{ij}(\omega_i^{(1)},\omega_j^{(2)})^2
\right\|_{\op}^{p}
\right)^{1/(2p)}, \\
B &:= \left\| \sum_{(i,j):\,i\neq j} \E H_{ij}(\omega_i^{(1)},\omega_j^{(2)})^2 \right\|_{\op}^{1/2}, \\
C_p &:= \left( \E \left\| \E_2\widetilde G\widetilde G^\top \right\|_{\op}^{p} \right)^{1/(2p)},
\end{aligned}
\]
and $C_{\scriptscriptstyle{MW}}$ is a universal constant.
Here \(\E_2\) denotes expectation with respect to
\(\{\omega_i^{(2)}\}_{i=1}^n\), conditionally on
\(\{\omega_i^{(1)}\}_{i=1}^n\). 
\end{theorem}

We apply this inequality to the canonical spherical kernel.  For
\(\omega=(\boldsymbol u,\widehat{\boldsymbol v}), \quad \omega'=(\boldsymbol u',\widehat{\boldsymbol v}')\)
in \(\mathbb S^{D-1}\times\mathbb S^{D-1}\), recall that
\[
\zeta(\omega,\omega')
=
\left(
\boldsymbol u^\top\boldsymbol u'
-
\widehat{\boldsymbol v}^{\top}\widehat{\boldsymbol v}'
\right)^2-\frac2D.
\]

We also use the following auxiliary lemma:

\begin{lemma}[Block row-sum bound]
\label{lem:block-row-sum}
Let \(\mathcal A=(A_{ij})_{i,j=1}^n\) be a symmetric block matrix,
where each block \(A_{ij}\) is a square matrix of the same size. Then \(\|\mathcal A\|_{\op} \leq \max_i\sum_{j=1}^n \|A_{ij}\|_{\op}\).
In particular, if \(\max_i\|A_{ii}\|_{\op}\leq D, \quad \max_{i\neq j}\|A_{ij}\|_{\op}\leq R\),
then \(\|\mathcal A\|_{\op}\leq D+(n-1)R\).
\end{lemma}

\begin{proof}
Let \(\boldsymbol x=(\boldsymbol x_1,\ldots,\boldsymbol x_n)\), where the \(\boldsymbol x_i\)'s are block vectors. Then
\[
\begin{aligned}
\|\mathcal A\|_{\op} 
& = \sup_{\|\boldsymbol x\|_2=1}|\boldsymbol x^\top A \boldsymbol x|  = \sup_{\|x\|_2=1} \left|\sum_{i, j=1}^n \boldsymbol x_i^\top A_{ij}\boldsymbol x_j\right| \\
& \leq \sup_{\|\boldsymbol x\|_2=1}\sum_{i,j=1}^n \|A_{ij}\|_{\op}\|\boldsymbol x_i\|_2 \|\boldsymbol x_j\|_2 \\
& \leq \sup_{\|\boldsymbol x\|_2=1} \frac{1}{2}\left(
\sum_{i=1}^n \|\boldsymbol x_i\|_2^2 \sum_{j=1}^n\|A_{ij}\|_{\op}
+ \sum_{j=1}^n \|\boldsymbol x_j\|_2^2 \sum_{i=1}^n\|A_{ij}\|_{\op}
\right) \\
& \leq \sup_{\|\boldsymbol x\|_2=1} \frac{1}{2}\left(
\sum_{i=1}^n \|\boldsymbol x_i\|_2^2 \sum_{j=1}^n\|A_{ij}\|_{\op}
+ \sum_{j=1}^n \|\boldsymbol x_j\|_2^2 \sum_{j=1}^n\|A_{ij}\|_{\op}
\right) \\
& = \sup_{\|\boldsymbol x\|_2=1}\frac{1}{2} \max_i \sum_{j=1}^n \|A_{ij}\|_{\op} \left(\sum_{i=1}^n\|\boldsymbol x_i\|_2^2 + \sum_{i=j}^n\|\boldsymbol x_j\|_2^2\right) \\
& \leq \max_i \sum_{j=1}^n \|A_{ij}\|_{\op},
\end{aligned}
\]
where the second inequality is due to $2\|\boldsymbol x_i\|_2\|\boldsymbol x_j\|_2 \leq \|\boldsymbol x_i\|_2^2 + \|\boldsymbol x_j\|_2^2$ and the third equality is due to the symmetry of $\mathcal A$.

\end{proof}

\begin{lemma}[Sparse coefficient matrix \(U\)-statistic]
\label{lem:matrix-U}
Let \(0<\delta<e^{-1}\), and set \(L_\delta:=\log (n/\delta)\).
Let \(C_{ij}=C_{ji}\in\R^{n\times n}\) be deterministic symmetric
matrices, each supported on coordinates \(\{i,j\}\), and suppose \(\|C_{ij}\|_{\op}\leq K n\)
and \(\left\|\sum_{i<j}C_{ij}^2\right\|_{\op} \leq K^2n^3\).
Let \(\mathsf U:=\sum_{i<j}\zeta_{ij}C_{ij}\).
If \(D \geq nL_\delta\),
then, with probability at least \(1-\delta\), \(\|\mathsf U\|_{\op} \leq C_{\mathsf U} K\frac{n^{3/2}L_\delta^{5/2}}{D}\),
where $C_{\mathsf U}$ is a universal constant.

Moreover, suppose that \(C_{ij}(q)=C_{ji}(q)\), \(q\in\mathcal I\), is a
continuously differentiable family of deterministic self-adjoint
matrices supported on coordinates \(\{i,j\}\).  Assume that, for
\(\ell=0,1\), \(\sup_{q\in\mathcal I}\sup_{i<j} \|\partial_q^\ell C_{ij}(q)\|_{\op} \leq K_\ell n\)
and
\[
\sup_{q\in\mathcal I}
\left\|
\sum_{i<j}
\bigl(\partial_q^\ell C_{ij}(q)\bigr)^2
\right\|_{\op}
\leq
K_\ell^2 n^3.
\]
Then, with probability at least \(1-\delta\),
\[
\sup_{q\in\mathcal I}
\left\|
\sum_{i<j}\zeta_{ij}C_{ij}(q)
\right\|_{\op}
\leq
C_{\mathsf U}(K_0+K_1)
\frac{n^{3/2}L_\delta^{5/2}}{D}.
\]
\end{lemma}

\begin{proof}
We first prove the fixed-\(q\) statement.  Set \(p:=\lceil L_\delta\rceil, \quad r:=\max\{p,\log(en)\}\).
Since $p\geq \log(n/\delta)$ and $\delta < e^{-1}$, we have \(p\geq \log n, \quad p\geq \log(1/\delta), \quad p \geq \log (en)\).
Thus, we have \(1\leq n^{1/p} = \exp\left(\frac{\log n}{p}\right) \leq e\),
and \(1\leq \delta^{-1/(2p)} = \exp\left(\frac{\log(1/\delta)}{2p}\right) \leq \sqrt e\).
Moreover \(r = p\).  Finally, since \(p\leq 2L_\delta\) for
\(L_\delta\geq1\), the assumption \(D \geq nL_\delta\)
implies \(\frac{np}{D}\leq 2\).

We apply the Minsker--Wei inequality.  For \(\omega=(\boldsymbol u,\widehat{\boldsymbol v}), \quad \omega'=(\boldsymbol u',\widehat{\boldsymbol v}')\),
define, for \(i\neq j\), \(H_{ij}(\omega,\omega') := \frac12\zeta(\omega,\omega')C_{ij}\).
The factor \(1/2\) converts the full off-diagonal sum in the
Minsker--Wei theorem into our upper-triangular sum.  Indeed, since
\(C_{ij}=C_{ji}\) and \(\zeta\) is symmetric, \(\sum_{i\neq j} H_{ij}(\omega_i,\omega_j) = \sum_{i<j} \zeta(\omega_i,\omega_j)C_{ij} = \mathsf U\).
By Lemma~\ref{lem:zeta}, the scalar kernel \(\zeta\) is centered in
each argument.  Hence \(\E_{\omega'}H_{ij}(\omega,\omega')=0\),
so the matrix-valued kernel \(H_{ij}\) is completely degenerate.

We now bound the three parameters in the Minsker--Wei inequality.

First, since Lemma~\ref{lem:zeta} gives \(\E \zeta(\omega_i,\omega_j)^2\leq \frac{8}{D^2}\),
we have \(\E H_{ij}(\omega_i^{(1)},\omega_j^{(2)})^2 \preceq \frac{2}{D^2}C_{ij}^2\).
Therefore
\[
B
:=
\left\|
\sum_{i\neq j}
\E H_{ij}(\omega_i^{(1)},\omega_j^{(2)})^2
\right\|_{\op}^{1/2} \leq
\frac{2}{D}
\left\|
\sum_{i<j}C_{ij}^2
\right\|_{\op}^{1/2} \leq
2K\frac{n^{3/2}}{D}.
\]

Next we estimate the row-square parameter \(A_p\).  Using
\(\|C_{ij}\|_{\op}\leq Kn\), for each fixed \(i\),
\[
\begin{aligned}
\left\|
\sum_{j:\,j\neq i}
H_{ij}(\omega_i^{(1)},\omega_j^{(2)})^2
\right\|_{\op}
&\leq
\sum_{j:\,j\neq i}
\left\|
H_{ij}(\omega_i^{(1)},\omega_j^{(2)})^2
\right\|_{\op} \\
&\leq
\frac14
\sum_{j:\,j\neq i}
|\zeta(\omega_i^{(1)},\omega_j^{(2)})|^2
\|C_{ij}\|_{\op}^2 \\
&\leq
\frac{K^2n^3}{4}
\max_{j\neq i}
|\zeta(\omega_i^{(1)},\omega_j^{(2)})|^2 .
\end{aligned}
\]
Consequently,
\[
A_p
\leq
\frac{K n^{3/2}}{2}
\left\|
\max_{i\neq j}
|\zeta(\omega_i^{(1)},\omega_j^{(2)})|
\right\|_{L_{2p}}.
\]
By Lemma~\ref{lem:zeta}, \(\|\zeta(\omega_i,\omega_j)\|_{L_{2p}} \leq C_\zeta\frac{p}{D}\).
Therefore
\[
\left\|
\max_{i\neq j}
|\zeta(\omega_i^{(1)},\omega_j^{(2)})|
\right\|_{L_{2p}}
\leq
\left(
\sum_{i\neq j}
\E|\zeta(\omega_i^{(1)},\omega_j^{(2)})|^{2p}
\right)^{1/(2p)} \leq
n^{1/p}C_\zeta\frac{p}{D}.
\]
Since \(n^{1/p}\leq e\), we obtain \(A_p \leq eC_\zeta K\frac{p n^{3/2}}{D}\).

It remains to control the block parameter \(C_p\).  Let
\(\widetilde G\) be the \(n\times n\) block matrix whose \((i,j)\)-block
is \(H_{ij}(\omega_i^{(1)},\omega_j^{(2)})\) for \(i\neq j\), and zero
on the diagonal.  The diagonal \((i,i)\)-block of
\(\E_2\widetilde G\widetilde G^\top\) is
\[
\sum_{k:\,k\neq i}
\E_\omega
\left[
H_{ik}(\omega_i^{(1)},\omega)^2
\right]
=
\frac14
\sum_{k:\,k\neq i}
\E_\omega\zeta(\omega_i^{(1)},\omega)^2 C_{ik}^2.
\]
By Lemma~\ref{lem:zeta}, \(\E_\omega\zeta(\omega_i^{(1)},\omega)^2\leq \frac{8}{D^2}\).
Together with \(\|C_{ik}\|_{\op}\leq Kn\), this shows that each diagonal
block has operator norm at most \(2K^2\frac{n^3}{D^2}\).

For \(i\neq j\), write
\[
\Gamma_{ij}
:=
\E_\omega
\left[
\zeta(\omega_i^{(1)},\omega)
\zeta(\omega_j^{(1)},\omega)
\right].
\]
Then the off-diagonal \((i,j)\)-block of
\(\E_2\widetilde G\widetilde G^\top\) is \(\mathcal B_{ij} := \frac14 \sum_{k\notin\{i,j\}} \Gamma_{ij} C_{ik}C_{jk}\).
By Lemma~\ref{lem:shared-cov},
\[
|\Gamma_{ij}|
\leq
8\left(
\frac{
(\boldsymbol u_i^\top\boldsymbol u_j)^2
+
(\widehat{\boldsymbol v}_i^\top\widehat{\boldsymbol v}_j)^2
}{D^2}
+
\frac1{D^3}
\right).
\]
Moreover, since \(\|C_{ik}C_{jk}\|_{\op}\leq K^2n^2\),
we have the pointwise bound
\[
\|\mathcal B_{ij}\|_{\op}
\leq
\frac14
\sum_{k\notin\{i,j\}}
|\Gamma_{ij}|\,
\|C_{ik}C_{jk}\|_{\op}
\leq
\frac14 K^2 n^3 |\Gamma_{ij}|.
\]
Consequently, \(\max_{i\neq j}\|\mathcal B_{ij}\|_{\op} \leq \frac14 K^2 n^3 \max_{i\neq j}|\Gamma_{ij}|\).

We now bound the right-hand side in \(L_p\).  Let
\[
M
:=
\max_{i\neq j}
\left\{
|\boldsymbol u_i^\top\boldsymbol u_j|,
|\widehat{\boldsymbol v}_i^\top\widehat{\boldsymbol v}_j|
\right\}.
\]
The preceding bound on \(\Gamma_{ij}\) gives \(\max_{i\neq j}|\Gamma_{ij}| \leq 8\left( \frac{M^2}{D^2} + \frac1{D^3} \right)\).
Therefore
\[
\left\|
\max_{i\neq j}|\Gamma_{ij}|
\right\|_{L_p}
\leq
8\left(
\frac{\|M^2\|_{L_p}}{D^2}
+
\frac1{D^3}
\right) 
=
8\left(
\frac{\|M\|_{L_{2p}}^2}{D^2}
+
\frac1{D^3}
\right).
\]
By Lemma~\ref{lem:spherical-coherence-moment}, \(\|M\|_{L_{2p}} \leq 2C_s n^{1/p}\sqrt{\frac{p}{D}}\).
Since \(n^{1/p}\leq e\), this implies \(\|M\|_{L_{2p}} \leq 2eC_s\sqrt{\frac{p}{D}}\).
Since $C_\zeta = 8C_s^2 +2$, 
\[
\left\|
\max_{i\neq j}|\Gamma_{ij}|
\right\|_{L_p}
\leq
8\left(
(2eC_s)^2\frac{p}{D^3}
+
\frac1{D^3}
\right)
\leq
(2e)^2C_\zeta\frac{p}{D^3}.
\]
Combining this with the pointwise block bound gives \(\left\| \max_{i\neq j}\|\mathcal B_{ij}\|_{\op} \right\|_{L_p} \leq e^2C_\zeta K^2\frac{n^3p}{D^3}\).

Combining the diagonal block bound with the off-diagonal block bound and
using Lemma~\ref{lem:block-row-sum}, we obtain
\[
\left\|
\left\|
\E_2\widetilde G\widetilde G^\ast
\right\|_{\op}
\right\|_{L_p}
\leq
K^2
\left(
2\frac{n^3}{D^2}
+
e^2 C_\zeta\frac{n^4p}{D^3}
\right) \leq e^2 C_\zeta K^2\left(\frac{n^3}{D^2}
+
\frac{n^4p}{D^3}\right).
\]
Since $2D\geq n2L_\delta \geq np$, we have
\[
\left\|
\left\|
\E_2\widetilde G\widetilde G^\ast
\right\|_{\op}
\right\|_{L_p}
\leq
3e^2 C_\zeta K^2\frac{n^3}{D^2}.
\]
Therefore
\[
C_p
:=
\left(
\E
\left\|
\E_2\widetilde G\widetilde G^\ast
\right\|_{\op}^{p}
\right)^{1/(2p)}
\leq
\sqrt{3}e C_\zeta^{1/2} K\frac{n^{3/2}}{D}.
\]

Substituting the bounds
\[
A_p\leq eC_\zeta K\frac{p n^{3/2}}{D},
\qquad
B\leq 2 K\frac{n^{3/2}}{D},
\qquad
C_p\leq \sqrt{3}e C_\zeta^{1/2} K\frac{n^{3/2}}{D}
\]
into the Minsker--Wei inequality gives
\begin{equation}\label{eq:UL2pnorm}
\begin{aligned}
\left(\E\|\mathsf U\|_{\op}^{2p}\right)^{1/(2p)}
&\leq
C_{\scriptscriptstyle{MW}}\left(r^{3/2}A_p+rB+rC_p\right) \\
&\leq
\sqrt{3}eC_\zeta C_{\scriptscriptstyle{MW}} K\frac{n^{3/2}p^{5/2}}{D},
\end{aligned}
\end{equation}
where we used \(r = p\).  By Markov's inequality,
\[
\prob\left(
\|\mathsf U\|_{\op}
>
\sqrt{3}eC_\zeta C_{\scriptscriptstyle{MW}} K\frac{n^{3/2}p^{5/2}}{D}\,\delta^{-1/(2p)}
\right)
\leq
\delta.
\]
Since \(\delta^{-1/(2p)}\leq\sqrt e\), increasing \(C\) gives, with
probability at least \(1-\delta\), \(\|\mathsf U\|_{\op} \leq \sqrt{3}e^{3/2}C_\zeta C_{\scriptscriptstyle{MW}} K\frac{n^{3/2}p^{5/2}}{D}\).
Finally, since \(p=\lceil L_\delta\rceil\leq 2L_\delta\) for
\(L_\delta\geq1\), we obtain \(\|\mathsf U\|_{\op} \leq C_{\mathsf U} K\frac{n^{3/2}L_\delta^{5/2}}{D}\)
with probability at least \(1-\delta\), where $C_{\mathsf U} = 2^{5/2}\sqrt{3}e^{3/2}C_\zeta C_{\scriptscriptstyle{MW}}$.

We now prove the uniform statement.  Let \(\mathsf U(q):=\sum_{i<j}\zeta_{ij}C_{ij}(q)\).
Fix \(q_\star\in\mathcal I\).  By the fundamental theorem of calculus,
for every \(q\in\mathcal I\), \(\mathsf U(q) = \mathsf U(q_\star) + \int_{q_\star}^q \mathsf U'(t)\,dt\),
where \(\mathsf U'(t) = \sum_{i<j}\zeta_{ij}\partial_q C_{ij}(t)\).
Therefore
\[
\sup_{q\in\mathcal I}\|\mathsf U(q)\|_{\op}
\leq
\|\mathsf U(q_\star)\|_{\op}
+
\int_{\mathcal I}\|\mathsf U'(t)\|_{\op}\,dt.
\]
Taking \(L_{2p}\)-norms and using Minkowski's inequality gives
\[
\left(
\E\sup_{q\in\mathcal I}\|\mathsf U(q)\|_{\op}^{2p}
\right)^{1/(2p)}
\leq \left(
\E\|\mathsf U(q_\star)\|_{\op}^{2p}
\right)^{1/(2p)} 
+ \int_{\mathcal I}
\left(
\E\|\mathsf U'(t)\|_{\op}^{2p}
\right)^{1/(2p)}
\,dt.
\]
The fixed-\(q\) moment estimate \eqref{eq:UL2pnorm} just proved applies to
\(\mathsf U(q_\star)\) with coefficient parameter \(K_0\), and to
\(\mathsf U'(t)\) with coefficient parameter \(K_1\).  Hence
\[
\left(
\E\sup_{q\in\mathcal I}\|\mathsf U(q)\|_{\op}^{2p}
\right)^{1/(2p)}
\leq
\sqrt{3}eC_\zeta C_{\scriptscriptstyle{MW}}(K_0+K_1)
\frac{n^{3/2}p^{5/2}}{D}.
\]
Applying Markov's inequality,
\[
\prob\left(
\sup_{q\in\mathcal I}\|\mathsf U(q)\|_{\op}
>
\sqrt{3}eC_\zeta C_{\scriptscriptstyle{MW}}(K_0+K_1)
\frac{n^{3/2}p^{5/2}}{D}\,
\delta^{-1/(2p)}
\right)
\leq
\delta.
\]
Since \(\delta^{-1/(2p)}\leq\sqrt e\), we have, with
probability at least \(1-\delta\),
\[
\sup_{q\in\mathcal I}\|\mathsf U(q)\|_{\op}
\leq
\sqrt{3}e^{3/2}C_\zeta C_{\scriptscriptstyle{MW}}(K_0+K_1)
\frac{n^{3/2}p^{5/2}}{D}.
\]
Using \(p\leq 2L_\delta\), we conclude that, with probability at least
\(1-\delta\), \(\sup_{q\in\mathcal I}\|\mathsf U(q)\|_{\op} \leq C_{\mathsf U}(K_0+K_1) \frac{n^{3/2}L_\delta^{5/2}}{D}\).
\end{proof}

\begin{remark}
    In the case when $\delta$ depends on $n$, e.g. $\delta = \frac 1n$ the lemma still holds by taking $D\gtrsim \log n + \log \frac 1\delta$
\end{remark}

\subsection{Uniform local random geometry}\label{sec:uniform-local-random-geometry}

The purpose of this section is to establish a high-probability local
description of the random objective near the deterministic symmetric path
\(q\mapsto\overline{\boldsymbol \alpha}(q)\), \(q\in\mathcal I\), where \(\mathcal I\) is the
local window defined in \eqref{eq:q-window}.
The deterministic part of the objective has curvature of order \(n^2/D\)
in the tangent directions, as shown in Lemma~\ref{lem:mean-curvature}.
We show that, on a suitable local box around this path, all random and
higher-order error terms are smaller than this curvature scale.

The argument has three components.  First, we control the centered
quadratic fluctuation \(\cQ\) and its derivatives along the symmetric
path.  Second, we show that these bounds remain stable when
\(\boldsymbol{\alpha}=\overline{\boldsymbol \alpha}(q)\) is perturbed by a small tangent vector
\(\beta\).  Third, we combine these estimates with the deterministic
mean curvature to obtain uniform strong concavity of the noise norm gap
\(\cG\) on the local box.

Define the noise norm gap
\begin{equation}\label{eq:noise-gap-definition}
\cG(\boldsymbol{\alpha}):=1-\|Z(\boldsymbol{\alpha})\|_1 .
\end{equation}
With \(\cH\) as in \eqref{eq:H-definition}, define
\begin{equation}\label{eq:Gbar-Q-definitions}
\overline{\cG}(\boldsymbol{\alpha}):=\frac1D\cH(\boldsymbol \alpha),
\qquad
\cQ(\boldsymbol{\alpha}):=\frac12\sum_{i<j}h(\alpha_i,\alpha_j)\zeta_{ij}.
\end{equation}
Then Lemma~\ref{lem:analytic-remainder} and \eqref{eq:zeta-ij-definition} give
\begin{equation}\label{eq:noise-gap-decomposition}
\cG(\boldsymbol{\alpha})
=
\overline{\cG}(\boldsymbol{\alpha})+\cQ(\boldsymbol{\alpha})-R(\boldsymbol{\alpha}).
\end{equation}

We first record deterministic coefficient bounds along the symmetric
path.  These bounds are used below to control \(\cQ\) and its derivatives
by scalar and matrix \(U\)-statistic estimates.

\begin{lemma}[Geometry of the symmetric path]
\label{lem:path-geometry}
Assume that \(n_+/n,n_-/n\in[\kappa,1-\kappa]\), and let \(\mathcal I\) be as in
\eqref{eq:q-window}.  Define
\[
c_\alpha:=\frac{1/2-r_0}{1-\kappa},
\qquad
C_\alpha:=\frac{1/2+r_0}{\kappa}.
\]
Then, uniformly over \(q\in\mathcal I\) and \(i\leq n\),
\[
\frac{c_\alpha}{n}\leq\overline\alpha_i(q)\leq\frac{C_\alpha}{n},
\qquad
|\overline\alpha_i'(q)|\leq\frac1{\kappa n}.
\]
Moreover, for \(h(x,y)=xy/(x+y)\), whenever
\(c_\alpha/n\leq x,y\leq C_\alpha/n\), we have
\(|h_x|,|h_y|\leq M_1\),
\(|h_{xx}|,|h_{xy}|,|h_{yy}|\leq M_2n\), and
\(|h_{abc}|\leq M_3n^2\) for \(a,b,c\in\{x,y\}\), where
\[
M_1:=\frac{C_\alpha^2}{4c_\alpha^2},\qquad
M_2:=\frac{C_\alpha^2}{4c_\alpha^3},
\qquad
M_3:=\frac{C_\alpha}{2c_\alpha^3}
+
\frac{3C_\alpha^2}{8c_\alpha^4}.
\]
\end{lemma}

\begin{proof}
For \(q\in\mathcal I\),
\[
\overline\alpha_i(q)
=
\begin{cases}
q/n_+,& i\in I_+,\\[1mm]
(1-q)/n_-,& i\in I_-.
\end{cases}
\]
Since \(q\in\mathcal I\) and \(n_+/n,n_-/n\in[\kappa,1-\kappa]\), the claimed upper
and lower bounds on \(\overline\alpha_i(q)\) follow immediately. Also
\(|\overline\alpha_i'(q)|\leq(\kappa n)^{-1}\).

Next, \(h_x(x,y)=\frac{y^2}{(x+y)^2}, \quad h_y(x,y)=\frac{x^2}{(x+y)^2}\),
so \(|h_x|,|h_y|\leq M_1\) on the stated region.  Also,
\[
\nabla^2h(x,y)
=
-\frac{2}{(x+y)^3}
\begin{pmatrix}
y^2&-xy\\
-xy&x^2
\end{pmatrix},
\]
which gives \(|h_{xx}|,|h_{xy}|,|h_{yy}|\leq M_2n\).

For the third derivatives, for example, \(|h_{xxx}(x,y)| = \frac{6y^2}{(x+y)^4} \leq \frac{3C_\alpha^2}{8c_\alpha^4}n^2\),
and
\[
|h_{xxy}(x,y)|
\leq
\frac{4y}{(x+y)^3}
+
\frac{6y^2}{(x+y)^4}
\leq
\left(
\frac{C_\alpha}{2c_\alpha^3}
+
\frac{3C_\alpha^2}{8c_\alpha^4}
\right)n^2.
\]
The other third derivatives are treated identically by symmetry.
\end{proof}

\begin{lemma}[First-derivative coefficients along the symmetric path]
\label{lem:path-first-coefficients}
Let \(M_1,M_2\) be as in Lemma~\ref{lem:path-geometry}.  For
\(i\neq j\), define \(c_{ij}(q):=h_x(\overline\alpha_i(q),\overline\alpha_j(q))\).
Then \(\partial_i\cQ(\overline{\boldsymbol \alpha}(q)) = \frac12\sum_{j:\,j\neq i}c_{ij}(q)\zeta_{ij}\),
and
\[
\sup_{q\in\mathcal I}\sup_{i\neq j}|c_{ij}(q)|\leq M_1,
\qquad
\sup_{q\in\mathcal I}\sup_{i\neq j}|\partial_qc_{ij}(q)|
\leq \frac{2M_2}{\kappa}.
\]

For \(i<j\), define \(b_{ij}(q):= \frac12\frac{d}{dq} h(\overline\alpha_i(q),\overline\alpha_j(q))\).
Then
\[
\partial_q\left[\cQ(\overline{\boldsymbol \alpha}(q))\right]
=
\sum_{i<j}b_{ij}(q)\zeta_{ij},
\]
and
\[
\sup_{q\in\mathcal I}\sup_{i<j}|b_{ij}(q)|
\leq \frac{M_1}{\kappa n},
\qquad
\sup_{q\in\mathcal I}\sup_{i<j}|\partial_qb_{ij}(q)|
\leq \frac{2M_2}{\kappa^2 n}.
\]
\end{lemma}

\begin{proof}
The representation of \(\partial_i\cQ\) follows by differentiating
\(\cQ\) as defined in \eqref{eq:Gbar-Q-definitions}.  The symmetry of \(h\)
gives the same coefficient \(c_{ij}(q)\) whether \(j>i\) or \(j<i\).

By Lemma~\ref{lem:path-geometry}, \(|c_{ij}(q)|\leq M_1\).  Moreover,
by the chain rule,
\[
|\partial_qc_{ij}(q)|
\leq
M_2n
\left(
|\overline\alpha_i'(q)|+|\overline\alpha_j'(q)|
\right)
\leq
\frac{2M_2}{\kappa}.
\]

For \(b_{ij}\), we have
\[
b_{ij}(q)
=
\frac12
\left[
h_x(\overline\alpha_i(q),\overline\alpha_j(q))\overline\alpha_i'(q)
+
h_y(\overline\alpha_i(q),\overline\alpha_j(q))\overline\alpha_j'(q)
\right].
\]
Thus \(|b_{ij}(q)|\leq M_1/(\kappa n)\).  Differentiating once more in
\(q\), and using \(\overline\alpha_i''(q)=0\), gives \(|\partial_qb_{ij}(q)| \leq \frac12\cdot 4M_2n\frac1{\kappa^2n^2} = \frac{2M_2}{\kappa^2n}\).
\end{proof}

\begin{lemma}[Second-derivative coefficients along the symmetric path]
\label{lem:path-second-coefficients}
\label{lem:hessian-coefficients}
Let \(M_2,M_3\) be as in Lemma~\ref{lem:path-geometry}.  For \(i<j\),
define \(C_{ij}^{(2)}(q)\in\mathbb R^{n\times n}\) by
\[
C_{ij}^{(2)}(q)
:=
\frac12
\left[
h_{xx}\boldsymbol e_i\boldsymbol e_i^\top
+
h_{xy}
\left(
\boldsymbol e_i\boldsymbol e_j^\top
+
\boldsymbol e_j\boldsymbol e_i^\top
\right)
+
h_{yy}\boldsymbol e_j\boldsymbol e_j^\top
\right],
\]
where the derivatives of \(h\) are evaluated at
\((\overline\alpha_i(q),\overline\alpha_j(q))\).  Then \(\nabla_\alpha^2\cQ(\overline{\boldsymbol \alpha}(q)) = \sum_{i<j}\zeta_{ij}C_{ij}^{(2)}(q)\),
and
\[
\sup_{q\in\mathcal I}\sup_{i<j}\|C_{ij}^{(2)}(q)\|_{\op}
\leq M_2n,
\qquad
\sup_{q\in\mathcal I}\sup_{i<j}\|\partial_q C_{ij}^{(2)}(q)\|_{\op}
\leq \frac{2M_3}{\kappa}n.
\]
Moreover,
\[
\sup_{q\in\mathcal I}
\left\|\sum_{i<j}\left(C_{ij}^{(2)}(q)\right)^2\right\|_{\op}
\leq M_2^2n^3, \qquad
\sup_{q\in\mathcal I}
\left\|\sum_{i<j}\left(\partial_q C_{ij}^{(2)}(q)\right)^2\right\|_{\op}
\leq
\left(\frac{2M_3}{\kappa}\right)^2n^3.
\]

Finally, let \(\boldsymbol \xi_0:=\frac{d}{dq}\overline{\boldsymbol \alpha}(q), \quad \boldsymbol d_{ij}(q):=C_{ij}^{(2)}(q)\boldsymbol \xi_0\).
Then
\[
\left.
\nabla_\beta\partial_q
\left[
\cQ(\overline{\boldsymbol \alpha}(q)+\boldsymbol \beta)
\right]
\right|_{\boldsymbol \beta=0}
=
\sum_{i<j}\zeta_{ij}\boldsymbol d_{ij}(q).
\]
Furthermore,
\[
\sup_{q\in\mathcal I}\sup_{i<j}\|\boldsymbol d_{ij}(q)\|_2
\leq
\frac{\sqrt2M_2}{\kappa},
\qquad
\sup_{q\in\mathcal I}\sup_{i<j}\|\partial_q\boldsymbol d_{ij}(q)\|_2
\leq
\frac{2\sqrt2M_3}{\kappa^2},
\]
and
\[
\sup_{q\in\mathcal I}
\left\|
\sum_{i<j}
\boldsymbol d_{ij}(q)\boldsymbol d_{ij}(q)^\top
\right\|_{\op}
\leq
\frac{M_2^2}{\kappa^2}n^2,
\qquad
\sup_{q\in\mathcal I}
\left\|
\sum_{i<j}
\partial_q\boldsymbol d_{ij}(q)
\partial_q\boldsymbol d_{ij}(q)^\top
\right\|_{\op}
\leq
\frac{4M_3^2}{\kappa^4}n^2.
\]
\end{lemma}

\begin{proof}
The identity for the Hessian follows by differentiating
\eqref{eq:Gbar-Q-definitions} twice in the \(\alpha\)-coordinates.

By Lemma~\ref{lem:path-geometry}, the entries of \(C_{ij}^{(2)}(q)\) are
bounded in absolute value by \((M_2/2)n\).  Since
\(C_{ij}^{(2)}(q)\) is supported on the two coordinates \(\{i,j\}\), its
row sums are bounded by \(M_2n\).  Hence \(\|C_{ij}^{(2)}(q)\|_{\op}\leq M_2n\).

Similarly, the third-derivative bound from Lemma~\ref{lem:path-geometry}
and the chain rule give, for \(a,b\in\{x,y\}\),
\[
\left|
\frac{d}{dq}
h_{ab}(\overline\alpha_i(q),\overline\alpha_j(q))
\right|
\leq
M_3n^2
\left(
|\overline\alpha_i'(q)|+|\overline\alpha_j'(q)|
\right)
\leq
\frac{2M_3}{\kappa}n.
\]
Since \(C_{ij}^{(2)}(q)\) contains an additional factor \(1/2\), each
entry of \(\partial_qC_{ij}^{(2)}(q)\) is bounded by
\((M_3/\kappa)n\).  Therefore its row sums are bounded by \(\frac{2M_3}{\kappa}n\),
and hence \(\|\partial_qC_{ij}^{(2)}(q)\|_{\op} \leq \frac{2M_3}{\kappa}n\).

We next prove the square-sum bounds.  Let \(M_{ij}\), \(i<j\), be
self-adjoint matrices supported on coordinates \(\{i,j\}\), and suppose
\(\|M_{ij}\|_{\op}\leq Kn\).  Then, for every \(\boldsymbol x\in\mathbb R^n\),
\[
\begin{aligned}
\boldsymbol x^\top\left(\sum_{i<j}M_{ij}^2\right)\boldsymbol x
&=
\sum_{i<j}\|M_{ij}\boldsymbol x\|_2^2 \\
&\leq
K^2n^2\sum_{i<j}(x_i^2+x_j^2) \\
&=
K^2n^2(n-1)\|\boldsymbol x\|_2^2 \\
&\leq
K^2n^3\|\boldsymbol x\|_2^2.
\end{aligned}
\]
Applying this with \(M_{ij}=C_{ij}^{(2)}(q)\) and \(K=M_2\), and then
with \(M_{ij}=\partial_qC_{ij}^{(2)}(q)\) and
\(K=2M_3/\kappa\), proves the two square-sum bounds.

It remains to identify and bound the mixed-derivative coefficients.
For fixed \(\beta\), differentiating in \(q\) gives
\[
\partial_q
\left[
\cQ(\overline{\boldsymbol \alpha}(q)+\boldsymbol \beta)
\right]
=
\left\langle
\nabla_\alpha\cQ(\overline{\boldsymbol \alpha}(q)+\boldsymbol \beta),
\boldsymbol \xi_0
\right\rangle.
\]

Taking the usual gradient with respect to \(\beta\) and setting \(\beta=0\), we obtain
\[
\left.
\nabla_\beta\partial_q
\left[
\cQ(\overline{\boldsymbol \alpha}(q)+\boldsymbol \beta)
\right]
\right|_{\boldsymbol \beta=0}
=
\nabla_\alpha^2\cQ(\overline{\boldsymbol \alpha}(q))\boldsymbol \xi_0.
\]
Using the Hessian representation already proved, this becomes
\[
\left.
\nabla_\beta\partial_q
\left[
\cQ(\overline{\boldsymbol \alpha}(q)+\boldsymbol \beta)
\right]
\right|_{\boldsymbol \beta=0}
=
\sum_{i<j}\zeta_{ij}C_{ij}^{(2)}(q)\boldsymbol \xi_0
=
\sum_{i<j}\zeta_{ij}\boldsymbol d_{ij}(q).
\]

Since \(\|(\xi_0)_{\{i,j\}}\|_2\leq \frac{\sqrt2}{\kappa n}\),
we have
\[
\|\boldsymbol d_{ij}(q)\|_2
\leq
\|C_{ij}^{(2)}(q)\|_{\op}\|(\boldsymbol \xi_0)_{\{i,j\}}\|_2
\leq
\frac{\sqrt2M_2}{\kappa}.
\]
Because \(\boldsymbol \xi_0\) is independent of \(q\), \(\partial_q\boldsymbol d_{ij}(q) = \partial_qC_{ij}^{(2)}(q)\boldsymbol \xi_0\),
and therefore
\[
\|\partial_q\boldsymbol d_{ij}(q)\|_2
\leq
\frac{2M_3}{\kappa}n\frac{\sqrt2}{\kappa n}
=
\frac{2\sqrt2M_3}{\kappa^2}.
\]

Finally, using the preceding bounds and \(\binom n2\leq n^2/2\),
\[
\begin{aligned}
\left\|
\sum_{i<j}
\boldsymbol d_{ij}(q)\boldsymbol d_{ij}(q)^\top
\right\|_{\op}
&\leq
\sum_{i<j}\|\boldsymbol d_{ij}(q)\|_2^2 \\
&\leq
\binom n2\frac{2M_2^2}{\kappa^2}
\leq
\frac{M_2^2}{\kappa^2}n^2.
\end{aligned}
\]
The same argument gives
\[
\begin{aligned}
\left\|
\sum_{i<j}
\partial_q\boldsymbol d_{ij}(q)
\partial_q\boldsymbol d_{ij}(q)^\top
\right\|_{\op}
&\leq
\sum_{i<j}\|\partial_q\boldsymbol d_{ij}(q)\|_2^2 \\
&\leq
\binom n2\frac{8M_3^2}{\kappa^4}
\leq
\frac{4M_3^2}{\kappa^4}n^2.
\end{aligned}
\]
This completes the proof.
\end{proof}

We next collect derivative bounds for the centered random fluctuation
\(\cQ\) along the symmetric path
\(q\mapsto\overline{\boldsymbol \alpha}(q)\).
By the second-order decomposition \eqref{eq:noise-gap-decomposition},
the non-deterministic perturbation of the deterministic averaged objective
is described by \(\cQ\) and the analytic remainder \(R\).  Moreover,
\eqref{eq:Gbar-Q-definitions} shows that \(\cQ\) is entirely encoded by the
pairwise fluctuations \(\zeta_{ij}\).  The deterministic curvature estimate
in Lemma~\ref{lem:mean-curvature} shows that the averaged objective has
tangent-direction curvature of order \(n^2/D\).  Thus the relevant
comparison is between this curvature scale and the derivatives of the
perturbation terms.  These estimates will later be combined with
deterministic strong concavity to localize the maximizer near the
symmetric path.

For a matrix \(M\in\mathbb R^{n\times n}\), write
\(\|M\|_{2\to\infty}:=\sup_{\|\boldsymbol x\|_2\leq1}
\|M\boldsymbol x\|_\infty\), equivalently
\(\|M\|_{2\to\infty}
=\max_{1\leq i\leq n}\|\boldsymbol e_i^\top M\|_2\).

Motivated by \eqref{eq:noise-gap-decomposition}--\eqref{eq:Gbar-Q-definitions},
we package the path derivative bounds together with a uniform bound on
\(\zeta_{ij}\) and the Gram operator-norm bounds needed to control \(R\)
into a single event.

\begin{definition}[Basic event]
\label{def:basic-event}
For
$\boldsymbol b
:=
\bigl(
b_q,\,
b_\nabla,\,
b_{H,\op},\,
b_{H,2\to\infty},\,
b_{\nabla q}
\bigr)$
and $\varepsilon,\nu>0$, we define
$\mathcal E(\boldsymbol b,\varepsilon,\nu)$ as the event on which,
uniformly over $q\in\mathcal I$,
\begin{align}
&\left|\partial_q[\cQ(\overline{\boldsymbol\alpha}(q))]\right|
\leq b_q,
\label{eq:path-event-q}\\
&\|P_{\T}\nabla_\alpha\cQ(\overline{\boldsymbol\alpha}(q))\|_\infty
\leq b_\nabla,
\qquad
\|P_{\T}\nabla_\alpha\cQ(\overline{\boldsymbol\alpha}(q))\|_2
\leq \sqrt{n}b_\nabla,
\label{eq:path-event-gradient}\\
&\|P_{\T}\nabla_\alpha^2\cQ(\overline{\boldsymbol\alpha}(q))P_{\T}\|_{\op}
\leq b_{H,\op},
\qquad
\|P_{\T}\nabla_\alpha^2\cQ(\overline{\boldsymbol\alpha}(q))P_{\T}\|_{2\to\infty}
\leq b_{H,2\to\infty},
\label{eq:path-event-hessian}\\
&\left\|
P_{\T}
\left.
\nabla_\beta\partial_q[\cQ(\overline{\boldsymbol\alpha}(q)+\boldsymbol\beta)]
\right|_{\boldsymbol\beta=0}
\right\|_2
\leq b_{\nabla q},
\label{eq:path-event-mixed}
\end{align}
and
\begin{align}
&\|\Delta_U\|_{\op}\vee\|\Delta_V\|_{\op}
\leq\varepsilon,
\label{eq:gram-event-gram}\\
&\max_{i<j}|\zeta_{ij}|
\leq\nu.
\label{eq:gram-event-zeta}
\end{align}
\end{definition}

\begin{lemma}[High-probability basic event]
\label{lem:path-derivatives}
There exist constants
$c_{\mathrm{path}},K_q,K_\nabla,K_H,K_{H,\infty},K_{\nabla q}>0$,
depending only on $\kappa,r_0$, and universal constants $K_G,K_\zeta>0$,
such that the following holds. Let $0<\delta<e^{-1}$ and set
$L:=\log(n/\delta)$. Let $\boldsymbol b_{n,D,\delta}$ be the vector in
Definition~\ref{def:basic-event} whose components are
\begin{equation}
\label{eq:path-event-parameters}
\begin{aligned}
b_q
:=K_q\frac{L^{5/2}}{D},
\qquad
b_{\nabla q}
:=K_{\nabla q}\frac{nL^{5/2}}{D},
\qquad
b_\nabla
:=K_\nabla\frac{\sqrt{nL}+L}{D},
\\
b_{H,\op}
:=K_H\frac{n^{3/2}L^{5/2}}{D},
\qquad
b_{H,2\to\infty}
:=K_{H,\infty}\frac{n^{3/2}L}{D}.
\end{aligned}
\end{equation}
Define also
\begin{equation}
\label{eq:gram-event-parameters}
\varepsilon_{n,D,\delta}
:=
K_G\sqrt{\frac{n+L}{D}},
\qquad
\nu_{n,D,\delta}
:=
K_\zeta\frac{L}{D}.
\end{equation}
If $D\geq c_{\mathrm{path}}nL$, then
\[
\prob\left(
\mathcal E(\boldsymbol b_{n,D,\delta},
\varepsilon_{n,D,\delta},\nu_{n,D,\delta})
\right)
\geq1-\delta.
\]
\end{lemma}

\begin{proof}
Let \(\mathcal N\subset\mathcal I\) be a grid of mesh \(n^{-1}\). Since
\(|\mathcal I|<1\), we have \(|\mathcal N|\leq n\).
We shall use six events, each with failure probability at most \(\delta/6\).
We repeatedly use the elementary bounds
\begin{equation}\label{eq:path-log-nongrid}
\log\frac{6n}{\delta}
=
L+\log 6
\leq
3L,
\end{equation}
and
\begin{equation}\label{eq:path-log-grid}
\log\frac{12n|\mathcal N|}{\delta}
\leq
\log\frac{12n^2}{\delta}
=
L+\log 12+\log n
\leq
5L.
\end{equation}
Here we used \(0<\delta<e^{-1}\). We also use
\begin{equation}\label{eq:path-log-coh}
\log\frac{24n^2}{\delta}
=
L+\log(24n)
\leq
6L.
\end{equation}

Let \(c_G\) be the constant required in
Lemma~\ref{lem:gram-concentration-2} when the failure probability is
\(\delta/6\). Let \(c_q,c_H,c_{\nabla q}\) be the constants required in the
moving-coefficient estimates below, again with failure probability
\(\delta/6\). We choose
\begin{equation}\label{eq:path-c-choice}
c_{\mathrm{path}}
\geq
\max\{6c_G,3c_q,3c_H,3c_{\nabla q}\}.
\end{equation}
Then \(D\geq c_{\mathrm{path}}nL\), together with
\eqref{eq:path-log-nongrid} and \eqref{eq:path-log-coh}, allows all
auxiliary estimates below to be applied.

First consider the gradient. By Lemma~\ref{lem:path-first-coefficients},
\[
\partial_i\cQ(\overline{\boldsymbol\alpha}(q))
=
\frac12\sum_{j:\,j\neq i}c_{ij}(q)\zeta_{ij},
\qquad
|c_{ij}(q)|\leq M_1,
\qquad
|\partial_qc_{ij}(q)|\leq\frac{2M_2}{\kappa}.
\]
By Lemma~\ref{lem:zeta}, conditional on
\((\boldsymbol u_i,\widehat{\boldsymbol v}_i)\), the variables
\(\zeta_{ij}\), \(j\neq i\), are independent and centered. The same lemma
and Bernstein's inequality give
\begin{equation}\label{eq:path-grad-fixed}
\mathbb P\left(
\left|
\partial_i\cQ(\overline{\boldsymbol\alpha}(q))
\right|
>
C_1\frac{\sqrt{ns}+s}{D}
\right)
\leq
2e^{-s},
\end{equation}
where
\(C_1\geq M_1\max\{c_1^{-1},c_1^{-1/2}\}\).
Here \(c_1>0\) is the absolute Bernstein constant.

Applying \eqref{eq:path-grad-fixed} with
\(s=\log\frac{12n|\mathcal N|}{\delta}\) and taking a union bound over
\(i\leq n\) and \(q\in\mathcal N\), we obtain an event
\(\mathcal E_{\nabla,\mathrm{grid}}\) satisfying
\[
\mathbb P(\mathcal E_{\nabla,\mathrm{grid}}^c)
\leq
\frac{\delta}{6}.
\]
On this event, by \eqref{eq:path-log-grid},
\begin{equation}\label{eq:path-grad-grid}
\max_{q\in\mathcal N}\max_{i\leq n}
\left|
\partial_i\cQ(\overline{\boldsymbol\alpha}(q))
\right|
\leq
5C_1\frac{\sqrt{nL}+L}{D}.
\end{equation}

Next define the spherical-geometry event. By
Lemma~\ref{lem:gram-concentration-2}, applied with failure probability
\(\delta/6\), there is an event \(\mathcal E_G\) satisfying
\(\mathbb P(\mathcal E_G^c)\leq\delta/6\) on which
\[
\|\Delta_U\|_{\op}\vee\|\Delta_V\|_{\op}
\leq
C_G\sqrt{\frac{n+\log(6n/\delta)}{D}}
\]
and
\[
\max_{i\neq j}
\left\{
|\boldsymbol u_i^\top\boldsymbol u_j|,
|\widehat{\boldsymbol v}_i^\top\widehat{\boldsymbol v}_j|
\right\}
\leq
C_G\sqrt{\frac{\log(24n^2/\delta)}{D}},
\]
where \(C_G\) is universal. By \eqref{eq:path-log-nongrid},
\eqref{eq:path-log-coh}, and \(\sqrt6\leq3\), we may take
\begin{equation}\label{eq:path-gram-final}
\|\Delta_U\|_{\op}\vee\|\Delta_V\|_{\op}
\leq
K_G\sqrt{\frac{n+L}{D}},
\qquad
K_G:=3C_G,
\end{equation}
and
\begin{equation}\label{eq:path-coherence}
\max_{i\neq j}
\left\{
|\boldsymbol u_i^\top\boldsymbol u_j|,
|\widehat{\boldsymbol v}_i^\top\widehat{\boldsymbol v}_j|
\right\}
\leq
3C_G\sqrt{\frac{L}{D}}.
\end{equation}
Consequently, on \(\mathcal E_G\),
\begin{equation}\label{eq:path-zeta-max}
\max_{i<j}|\zeta_{ij}|
\leq
K_\zeta\frac{L}{D},
\qquad
K_\zeta:=36C_G^2+2.
\end{equation}

We now pass the gradient estimate from the grid to the whole interval.
Fix \(i\leq n\) and set
\(f_i(q):=\partial_i\cQ(\overline{\boldsymbol\alpha}(q))\).
Using Lemma~\ref{lem:path-first-coefficients} and
\eqref{eq:path-zeta-max}, on \(\mathcal E_G\),
\begin{equation}\label{eq:path-gradient-lipschitz}
\sup_{q\in\mathcal I}|f_i'(q)|
\leq
\frac{K_\zeta M_2}{\kappa}\frac{nL}{D}.
\end{equation}
For arbitrary \(q\in\mathcal I\), choose \(q_0\in\mathcal N\) with
\(|q-q_0|\leq n^{-1}\). Combining \eqref{eq:path-grad-grid} and
\eqref{eq:path-gradient-lipschitz} gives, uniformly over
\(q\in\mathcal I\),
\[
\|\nabla_\alpha\cQ(\overline{\boldsymbol\alpha}(q))\|_\infty
\leq
\left(
5C_1+\frac{K_\zeta M_2}{\kappa}
\right)
\frac{\sqrt{nL}+L}{D}.
\]
Since
\(\|P_{\T}\boldsymbol x\|_\infty\leq2\|\boldsymbol x\|_\infty\), we obtain
\begin{equation}\label{eq:path-gradient-final}
\left\|
P_{\T}\nabla_\alpha\cQ(\overline{\boldsymbol\alpha}(q))
\right\|_\infty
\leq
K_\nabla\frac{\sqrt{nL}+L}{D},
\qquad
K_\nabla
:=
2\left(
5C_1+\frac{K_\zeta M_2}{\kappa}
\right).
\end{equation}
The Euclidean bound follows from \eqref{eq:path-gradient-final} and
\(\|P_{\T}\boldsymbol x\|_2\leq\|\boldsymbol x\|_2
\leq\sqrt n\|\boldsymbol x\|_\infty\):
\[
\left\|
P_{\T}\nabla_\alpha\cQ(\overline{\boldsymbol\alpha}(q))
\right\|_2
\leq
\sqrt{n}K_\nabla\frac{\sqrt{nL}+L}{D}.
\]

Next consider the scalar \(q\)-derivative. By
Lemma~\ref{lem:path-first-coefficients},
\[
\partial_q[\cQ(\overline{\boldsymbol\alpha}(q))]
=
\sum_{i<j}b_{ij}(q)\zeta_{ij},
\qquad
|b_{ij}(q)|\leq\frac{B_0}{n},
\qquad
|\partial_qb_{ij}(q)|\leq\frac{B_1}{n},
\]
where
\[
B_0:=\frac{M_1}{\kappa},
\qquad
B_1:=\frac{2M_2}{\kappa^2}.
\]
Moreover,
\[
\sup_{q\in\mathcal I}\sum_{i<j}b_{ij}(q)^2
\leq
B_0^2,
\qquad
\sup_{q\in\mathcal I}
\sum_{i<j}(\partial_qb_{ij}(q))^2
\leq
B_1^2.
\]
Applying the moving-coefficient estimate in
Lemma~\ref{lem:matrix-U} with failure probability \(\delta/6\) and
coefficient scales
\[
K_0:=\frac{B_0}{n^{3/2}},
\qquad
K_1:=\frac{B_1}{n^{3/2}},
\]
gives an event \(\mathcal E_q\) satisfying
\(\mathbb P(\mathcal E_q^c)\leq\delta/6\) on which
\[
\sup_{q\in\mathcal I}
\left|
\partial_q[\cQ(\overline{\boldsymbol\alpha}(q))]
\right|
\leq
C_U(B_0+B_1)
\frac{\log^{5/2}(6n/\delta)}{D}.
\]
By \eqref{eq:path-log-nongrid}, this yields
\begin{equation}\label{eq:path-q-derivative-final}
\sup_{q\in\mathcal I}
\left|
\partial_q[\cQ(\overline{\boldsymbol\alpha}(q))]
\right|
\leq
K_q\frac{L^{5/2}}{D},
\qquad
K_q:=3^{5/2}C_U(B_0+B_1).
\end{equation}

Next consider the Hessian operator norm. By
Lemma~\ref{lem:path-second-coefficients},
\[
\nabla_\alpha^2\cQ(\overline{\boldsymbol\alpha}(q))
=
\sum_{i<j}\zeta_{ij}C_{ij}^{(2)}(q),
\]
where
\[
\sup_{q\in\mathcal I}\sup_{i<j}
\|C_{ij}^{(2)}(q)\|_{\op}
\leq
H_0n,
\qquad
\sup_{q\in\mathcal I}\sup_{i<j}
\|\partial_qC_{ij}^{(2)}(q)\|_{\op}
\leq
H_1n,
\]
with
\[
H_0:=M_2,
\qquad
H_1:=\frac{2M_3}{\kappa}.
\]
Moreover,
\[
\sup_{q\in\mathcal I}
\left\|
\sum_{i<j}(C_{ij}^{(2)}(q))^2
\right\|_{\op}
\leq
H_0^2n^3,
\qquad
\sup_{q\in\mathcal I}
\left\|
\sum_{i<j}(\partial_qC_{ij}^{(2)}(q))^2
\right\|_{\op}
\leq
H_1^2n^3.
\]
Applying the moving-coefficient matrix estimate in
Lemma~\ref{lem:matrix-U} with failure probability \(\delta/6\) and
coefficient scales
\[
K_0:=H_0,
\qquad
K_1:=H_1,
\]
gives an event \(\mathcal E_{H,\op}\) satisfying
\(\mathbb P(\mathcal E_{H,\op}^c)\leq\delta/6\) on which
\[
\sup_{q\in\mathcal I}
\left\|
\nabla_\alpha^2
\cQ(\overline{\boldsymbol\alpha}(q))
\right\|_{\op}
\leq
C_U(H_0+H_1)
\frac{n^{3/2}\log^{5/2}(6n/\delta)}{D}.
\]
Since \(P_{\T}\) is an orthogonal projection, the same bound holds after
multiplication by \(P_{\T}\) on both sides. Hence, by
\eqref{eq:path-log-nongrid},
\begin{equation}\label{eq:path-hessian-op-final}
\sup_{q\in\mathcal I}
\left\|
P_{\T}\nabla_\alpha^2
\cQ(\overline{\boldsymbol\alpha}(q))P_{\T}
\right\|_{\op}
\leq
K_H\frac{n^{3/2}L^{5/2}}{D},
\qquad
K_H:=3^{5/2}C_U(H_0+H_1).
\end{equation}

We now prove the \(2\to\infty\) Hessian bound. Write
\[
H(q):=
\nabla_\alpha^2\cQ(\overline{\boldsymbol\alpha}(q)).
\]
Recall that the off-diagonal entries 
\(C_{ab}^{(2)}(q)\) is supported only on the coordinates \(\{a,b\}\). Hence, for
\(i\neq j\),
\[
H_{ij}(q)
=
\zeta_{ij}\bigl(C_{ij}^{(2)}(q)\bigr)_{ij}.
\]
Combined with \(\|C_{ij}^{(2)}(q)\|_{\op}\leq H_0n\) and
\eqref{eq:path-zeta-max}, on \(\mathcal E_G\),
\[
|H_{ij}(q)|
\leq
K_\zeta H_0\frac{nL}{D}.
\]
Therefore, uniformly over \(q\in\mathcal I\) and \(i\leq n\),
\begin{equation}\label{eq:path-hessian-offdiag-row}
\left(
\sum_{j\neq i}|H_{ij}(q)|^2
\right)^{1/2}
\leq
K_\zeta H_0\frac{n^{3/2}L}{D}.
\end{equation}

It remains to control the diagonal entries. For fixed \(i\), we can write
\[
H_{ii}(q)
=
\sum_{j:\,j\neq i}a_{ij}(q)\zeta_{ij},
\qquad
|a_{ij}(q)|\leq H_0n,
\qquad
|\partial_qa_{ij}(q)|\leq H_1n.
\]
The same conditional Bernstein argument as in
\eqref{eq:path-grad-fixed} gives, for every fixed \(i,q\),
\begin{equation}\label{eq:path-diag-fixed}
\mathbb P\left(
|H_{ii}(q)|
>
C_3\frac{n(\sqrt{ns}+s)}{D}
\right)
\leq
2e^{-s},
\end{equation}
where
\(C_3\geq H_0\max\{c_1^{-1},c_1^{-1/2}\}\). Applying \eqref{eq:path-diag-fixed} with
\(s=\log\frac{12n|\mathcal N|}{\delta}\) and taking a union bound over \(i\leq n\) and
\(q\in\mathcal N\), we obtain an event
\(\mathcal E_{\mathrm{diag},\mathrm{grid}}\) satisfying
\[
\mathbb P(\mathcal E_{\mathrm{diag},\mathrm{grid}}^c)
\leq
\frac{\delta}{6}.
\]
On this event, by \eqref{eq:path-log-grid},
\begin{equation}\label{eq:path-diag-grid}
\max_{q\in\mathcal N}\max_{i\leq n}
|H_{ii}(q)|
\leq
C_3\frac{n(3\sqrt{nL}+5L)}{D}.
\end{equation}

We next pass \eqref{eq:path-diag-grid} from the grid to the whole interval.
Set \(g_i(q):=H_{ii}(q)\). By the derivative bound on \(a_{ij}\) and
\eqref{eq:path-zeta-max}, on \(\mathcal E_G\),
\begin{equation}\label{eq:path-diag-lipschitz}
\sup_{q\in\mathcal I}|g_i'(q)|
\leq
K_\zeta H_1\frac{n^2L}{D}.
\end{equation}
For arbitrary \(q\in\mathcal I\), choose \(q_0\in\mathcal N\) with
\(|q-q_0|\leq n^{-1}\). Combining \eqref{eq:path-diag-grid} and
\eqref{eq:path-diag-lipschitz} gives
\[
|H_{ii}(q)|
\leq
C_3\frac{n(3\sqrt{nL}+5L)}{D}
+
K_\zeta H_1\frac{nL}{D}.
\]
Since
\[
3\sqrt{nL}+5L\leq10\sqrt n\,L
\qquad\text{and}\qquad
nL\leq n^{3/2}L,
\]
we have, uniformly over \(q\in\mathcal I\) and \(i\leq n\),
\begin{equation}\label{eq:path-hessian-diagonal}
|H_{ii}(q)|
\leq
(10C_3+K_\zeta H_1)
\frac{n^{3/2}L}{D}.
\end{equation}

Combining \eqref{eq:path-hessian-offdiag-row} and
\eqref{eq:path-hessian-diagonal}, and using
\(\sqrt{a^2+b^2}\leq a+b\), we obtain
\begin{equation}\label{eq:path-hessian-2inf-raw}
\sup_{q\in\mathcal I}
\|H(q)\|_{2\to\infty}
\leq
K_{H,\infty}^{(0)}
\frac{n^{3/2}L}{D},
\end{equation}
where
\[
K_{H,\infty}^{(0)}
:=
10C_3+K_\zeta H_1+K_\zeta H_0.
\]
Finally,
\[
\|P_{\T}\boldsymbol x\|_\infty
\leq
2\|\boldsymbol x\|_\infty,
\qquad
\|P_{\T}\boldsymbol x\|_2
\leq
\|\boldsymbol x\|_2.
\]
Therefore \eqref{eq:path-hessian-2inf-raw} implies
\begin{equation}\label{eq:path-hessian-2inf-final}
\sup_{q\in\mathcal I}
\left\|
P_{\T}
\nabla_\alpha^2\cQ(\overline{\boldsymbol\alpha}(q))
P_{\T}
\right\|_{2\to\infty}
\leq
K_{H,\infty}\frac{n^{3/2}L}{D},
\end{equation}
where
\[
K_{H,\infty}
:=
2K_{H,\infty}^{(0)}
=
2(10C_3+K_\zeta H_1+K_\zeta H_0).
\]

Finally consider the mixed derivative. By
Lemma~\ref{lem:path-second-coefficients},
\[
\left.
\nabla_\beta\partial_q
\left[
\cQ(\overline{\boldsymbol\alpha}(q)+\boldsymbol\beta)
\right]
\right|_{\boldsymbol\beta=0}
=
\sum_{i<j}\zeta_{ij}\boldsymbol d_{ij}(q).
\]
Set
\[
D_0:=\frac{\sqrt2M_2}{\kappa},
\qquad
D_1:=\frac{2\sqrt2M_3}{\kappa^2}.
\]
Then
\[
\sup_{q\in\mathcal I}\sup_{i<j}
\|\boldsymbol d_{ij}(q)\|_2
\leq
D_0,
\qquad
\sup_{q\in\mathcal I}\sup_{i<j}
\|\partial_q\boldsymbol d_{ij}(q)\|_2
\leq
D_1,
\]
and
\[
\sup_{q\in\mathcal I}
\left\|
\sum_{i<j}
\boldsymbol d_{ij}(q)\boldsymbol d_{ij}(q)^\top
\right\|_{\op}
\leq
D_0^2n^2,
\qquad
\sup_{q\in\mathcal I}
\left\|
\sum_{i<j}
\partial_q\boldsymbol d_{ij}(q)
\partial_q\boldsymbol d_{ij}(q)^\top
\right\|_{\op}
\leq
D_1^2n^2.
\]
For each \(i<j\), introduce the self-adjoint dilation
\[
\mathsf M_{ij}(q)
:=
\begin{pmatrix}
0&\boldsymbol d_{ij}(q)^\top\\
\boldsymbol d_{ij}(q)&0
\end{pmatrix}.
\]
Since
\[
\|\mathsf M_{ij}(q)\|_{\op}
=
\|\boldsymbol d_{ij}(q)\|_2
\leq D_0
\]
and
\[
\left\|
\sum_{i<j}\mathsf M_{ij}(q)^2
\right\|_{\op}
\leq
D_0^2n^2,
\]
with the analogous bounds for \(\partial_q\mathsf M_{ij}(q)\) with
\(D_1\), the moving-coefficient estimate in
Lemma~\ref{lem:matrix-U} applies with coefficient scales
\[
K_0:=\frac{D_0}{\sqrt n},
\qquad
K_1:=\frac{D_1}{\sqrt n}.
\]
Therefore, with failure probability at most \(\delta/6\),
\[
\sup_{q\in\mathcal I}
\left\|
\sum_{i<j}\zeta_{ij}\mathsf M_{ij}(q)
\right\|_{\op}
\leq
C_U(D_0+D_1)
\frac{n\log^{5/2}(6n/\delta)}{D}.
\]
By \eqref{eq:path-log-nongrid}, there is an event
\(\mathcal E_{\nabla q}\) satisfying
\(\mathbb P(\mathcal E_{\nabla q}^c)\leq\delta/6\) on which
\[
\sup_{q\in\mathcal I}
\left\|
\sum_{i<j}\zeta_{ij}\mathsf M_{ij}(q)
\right\|_{\op}
\leq
K_{\nabla q}\frac{nL^{5/2}}{D},
\qquad
K_{\nabla q}
:=
3^{5/2}C_U(D_0+D_1).
\]
By the definition of the dilation,
\[
\left\|
\sum_{i<j}\zeta_{ij}\mathsf M_{ij}(q)
\right\|_{\op}
=
\left\|
\sum_{i<j}\zeta_{ij}\boldsymbol d_{ij}(q)
\right\|_2.
\]
Since \(P_{\T}\) is an orthogonal projection, we conclude that
\begin{equation}\label{eq:path-mixed-final}
\sup_{q\in\mathcal I}
\left\|
P_{\T}
\left.
\nabla_\beta\partial_q
\left[
\cQ(\overline{\boldsymbol\alpha}(q)+\boldsymbol\beta)
\right]
\right|_{\boldsymbol\beta=0}
\right\|_2
\leq
K_{\nabla q}\frac{nL^{5/2}}{D}.
\end{equation}

Define
\[
\mathcal E_0
:=
\mathcal E_{\nabla,\mathrm{grid}}
\cap
\mathcal E_G
\cap
\mathcal E_q
\cap
\mathcal E_{H,\op}
\cap
\mathcal E_{\mathrm{diag},\mathrm{grid}}
\cap
\mathcal E_{\nabla q}.
\]
Each event has complement probability at most \(\delta/6\), so
\[
\mathbb P(\mathcal E_0^c)\leq\delta.
\]
On \(\mathcal E_0\),
\eqref{eq:path-gram-final} and \eqref{eq:path-zeta-max} give the two
geometric bounds in Definition~\ref{def:basic-event}, while
\eqref{eq:path-q-derivative-final},
\eqref{eq:path-gradient-final},
\eqref{eq:path-hessian-op-final},
\eqref{eq:path-hessian-2inf-final}, and
\eqref{eq:path-mixed-final}
give its path bounds. Hence
\[
\mathcal E(
\boldsymbol b_{n,D,\delta},
\varepsilon_{n,D,\delta},
\nu_{n,D,\delta}
)
\]
holds on \(\mathcal E_0\), which proves the lemma.
\end{proof}

\subsection{Uniform strong concavity}

Throughout the rest of the proof, we condition on a realization of the
corruption signs satisfying
\begin{equation}
\label{eq:conditional-label-counts}
\kappa n\leq n_-<n_+\leq(1-\kappa)n
\end{equation}
for some fixed \(0<\kappa<1/2\).

For deterministic event parameters
\(\boldsymbol b,\varepsilon,\nu\), we further work on the basic event
\(\mathcal E(\boldsymbol b,\varepsilon,\nu)\) from
Definition~\ref{def:basic-event}. From this point until the final
probability assembly, all arguments are deterministic under these
conditions.

Define
\begin{equation}
\label{eq:deterministic-scales}
\mathfrak D(r)
:=
\frac{D}{n^2}b_{H,\op}
+nr
+Dn\nu r
+\frac{D}{n}\varepsilon^3.
\end{equation}
Let \(\varrho_0>0\) and \(K_R<\infty\) be the constants from
Lemma~\ref{lem:analytic-remainder} corresponding to the coordinate range
\([c_\alpha/(2n),2C_\alpha/n]\). Throughout, \(r_0\in(0,1/4)\) and the
local window \(\mathcal I\) are fixed as in \eqref{eq:q-window}.

\begin{lemma}[Uniformity over the local box]
\label{lem:local-uniformity}
Let \(r_\infty>0\) and
\[
\mathcal B^\infty(r_\infty)
:=
\{(q,\boldsymbol\beta):q\in\mathcal I,\ \boldsymbol\beta\in\T,
\ \|\boldsymbol\beta\|_\infty\leq r_\infty\}.
\]
There exists \(C_{\mathrm{box}}\geq 1\), depending only on \(\kappa,r_0\), such that, if
\(\mathcal E(\boldsymbol b,\varepsilon,\nu)\) holds,
\(\varepsilon\leq\varrho_0\), and \(nr_\infty\leq c_\alpha/2\), then
\begin{align}
\sup_{(q,\boldsymbol\beta)\in\mathcal B^\infty(r_\infty)}
\|P_{\T}[\nabla_\alpha^2\overline{\cG}(\overline{\boldsymbol\alpha}(q)+\boldsymbol\beta)
-\nabla_\alpha^2\overline{\cG}(\overline{\boldsymbol\alpha}(q))]P_{\T}\|_{\op}
&\leq C_{\mathrm{box}}\frac{n^3}{D}r_\infty,
\label{eq:mean-hessian-variation-op}\\
\sup_{(q,\boldsymbol\beta)\in\mathcal B^\infty(r_\infty)}
\|P_{\T}[\nabla_\alpha^2\overline{\cG}(\overline{\boldsymbol\alpha}(q)+\boldsymbol\beta)
-\nabla_\alpha^2\overline{\cG}(\overline{\boldsymbol\alpha}(q))]P_{\T}\|_{\infty\to\infty}
&\leq C_{\mathrm{box}}\frac{n^3}{D}r_\infty,
\label{eq:mean-hessian-variation-inf}
\end{align}
\begin{align}
\sup_{(q,\boldsymbol\beta)\in\mathcal B^\infty(r_\infty)}
\|\nabla_\alpha^2\cQ(\overline{\boldsymbol\alpha}(q)+\boldsymbol\beta)
-\nabla_\alpha^2\cQ(\overline{\boldsymbol\alpha}(q))\|_{\op}
&\leq C_{\mathrm{box}}n^3\nu r_\infty,
\label{eq:local-Q-op}\\
\sup_{(q,\boldsymbol\beta)\in\mathcal B^\infty(r_\infty)}
\|\nabla_\alpha^2\cQ(\overline{\boldsymbol\alpha}(q)+\boldsymbol\beta)
-\nabla_\alpha^2\cQ(\overline{\boldsymbol\alpha}(q))\|_{\infty\to\infty}
&\leq C_{\mathrm{box}}n^3\nu r_\infty,
\label{eq:local-Q-inf}
\end{align}
and
\begin{align}
\sup_{(q,\boldsymbol\beta)\in\mathcal B^\infty(r_\infty)}
\|\nabla_\alpha^2R(\overline{\boldsymbol\alpha}(q)+\boldsymbol\beta)\|_{\op}
&\leq C_{\mathrm{box}}n\varepsilon^3,
\label{eq:local-R-op}\\
\sup_{(q,\boldsymbol\beta)\in\mathcal B^\infty(r_\infty)}
\|\nabla_\alpha^2R(\overline{\boldsymbol\alpha}(q)+\boldsymbol\beta)\|_{\infty\to\infty}
&\leq C_{\mathrm{box}}n\varepsilon^3.
\label{eq:local-R-inf}
\end{align}
\end{lemma}

\begin{proof}
By Lemma~\ref{lem:path-geometry} and \(nr_\infty\leq c_\alpha/2\), every
coordinate in the local box satisfies
\begin{equation}
\label{eq:local-coordinate-window}
\frac{c_\alpha}{2n}
\leq
\overline\alpha_i(q)+\beta_i
\leq
\frac{2C_\alpha}{n}.
\end{equation}
Hence every third derivative of \(h\) on this region is bounded in absolute
value by \(C_hn^2\), where \(C_h\) depends only on \(\kappa,r_0\).

First we control the local variation of \(\cQ\).  Let
\[
A_Q(q,\boldsymbol\beta)
:=
\nabla_\alpha^2\cQ(\overline{\boldsymbol\alpha}(q)+\boldsymbol\beta)
-
\nabla_\alpha^2\cQ(\overline{\boldsymbol\alpha}(q)).
\]
For \(i<j\), let \(C_{ij}^{(2)}(\boldsymbol\alpha)\) denote the pairwise
Hessian block, supported on \(\{i,j\}\times\{i,j\}\).  By the mean-value
theorem and the preceding third-derivative bound, for
\((k,\ell)\in\{i,j\}\times\{i,j\}\),
\begin{equation}
\label{eq:local-Cij-difference}
\left|
[C_{ij}^{(2)}(\overline{\boldsymbol\alpha}(q)+\boldsymbol\beta)
-C_{ij}^{(2)}(\overline{\boldsymbol\alpha}(q))]_{k\ell}
\right|
\leq C_hn^2r_\infty.
\end{equation}
By Definition~\ref{def:basic-event},
\(\max_{i<j}|\zeta_{ij}|\leq\nu\).  For each row, only pairs containing that
coordinate contribute, and each such pair contributes to at most two entries.
Thus \eqref{eq:local-Cij-difference} gives, uniformly on the local box,
\begin{equation}
\label{eq:local-Q-row}
\|A_Q(q,\boldsymbol\beta)\|_{\infty\to\infty}
\leq2C_hn^3\nu r_\infty.
\end{equation}
Since \(A_Q(q,\boldsymbol\beta)\) is symmetric, its operator norm is bounded
by the same row-sum bound.  This proves \eqref{eq:local-Q-op} and
\eqref{eq:local-Q-inf} after enlarging \(C_{\mathrm{box}}\).

The corresponding bounds \eqref{eq:mean-hessian-variation-op}--\eqref{eq:mean-hessian-variation-inf} follow from the same argument, replacing $\zeta_{ij}$ with deterministic $2/D$ and noting \(P_{\T}\) is an orthogonal projection and
\(\|P_{\T}x\|_\infty\leq2\|x\|_\infty\), whence
\(\|P_{\T}AP_{\T}\|_{\infty\to\infty}\leq4\|A\|_{\infty\to\infty}\).

It remains to control the analytic remainder.  By Definition~\ref{def:basic-event},
\(\|\Delta_U\|_{\op}\vee\|\Delta_V\|_{\op}\leq\varepsilon\leq\varrho_0\).
Therefore Lemma~\ref{lem:analytic-remainder} applies uniformly on the
coordinate window \eqref{eq:local-coordinate-window}, and therefore
\[
\sup_{\mathcal B^\infty(r_\infty)}
\|\nabla_\alpha^2R\|_{\op}
\vee
\sup_{\mathcal B^\infty(r_\infty)}
\|\nabla_\alpha^2R\|_{\infty\to\infty}
\leq K_Rn\varepsilon^3.
\]
Choosing \(C_{\mathrm{box}}\geq\max\{1, 16C_h,K_R\}\) proves
\eqref{eq:local-R-op} and \eqref{eq:local-R-inf}.
\end{proof}

\begin{proposition}[Uniform strong concavity]
\label{prop:strong-concavity}
Let \(c_H>0\) be the constant from Lemma~\ref{lem:mean-curvature} such that
\(-P_{\T}\nabla_\alpha^2\overline{\cG}(\overline{\boldsymbol\alpha}(q))P_{\T}
\succeq c_H(n^2/D)I_{\T}\) uniformly over \(q\in\mathcal I\).
There exists \(c_{\mathrm{sc}}>0\), depending only on \(\kappa,r_0\), such
that the following holds.  Suppose
\(\mathcal E(\boldsymbol b,\varepsilon,\nu)\) holds,
\(\varepsilon\leq\varrho_0\), and
\begin{equation}
\label{eq:strong-concavity-smallness}
\mathfrak D(r_\infty)\leq c_{\mathrm{sc}}.
\end{equation}
Then, uniformly over \((q,\boldsymbol\beta)\in\mathcal B^\infty(r_\infty)\),
\[
-P_{\T}\nabla_\alpha^2\cG(\overline{\boldsymbol\alpha}(q)+\boldsymbol\beta)P_{\T}
\succeq
\frac{c_H}{2}\frac{n^2}{D}I_{\T}.
\]
\end{proposition}

\begin{proof}
Let \(C_{\mathrm{box}}\geq 1\) be the constant from Lemma~\ref{lem:local-uniformity}.
Choose \(c_{\mathrm{sc}}>0\) so small that
\[
c_{\mathrm{sc}}
\leq
\min\left\{
\frac{c_\alpha}{2},
\frac{c_H}{8C_{\mathrm{box}}}
\right\}.
\]
Since \(nr_\infty\leq\mathfrak D(r_\infty)\leq c_{\mathrm{sc}}\), the smallness assumption
also guarantees the coordinate condition needed below.

First, Lemma~\ref{lem:mean-curvature} and
Lemma~\ref{lem:local-uniformity} give
\[
-P_{\T}\nabla_\alpha^2\overline{\cG}
(\overline{\boldsymbol\alpha}(q)+\boldsymbol\beta)P_{\T}
\succeq
\left(c_H - C_{\mathrm{box}}nr_\infty\right)\frac{n^2}{D}I_{\T}
\succeq
\frac{7c_H}{8}\frac{n^2}{D}I_{\T}.
\]
Next, by Definition~\ref{def:basic-event},
\[
\sup_{q\in\mathcal I}
\|P_{\T}\nabla_\alpha^2\cQ(\overline{\boldsymbol\alpha}(q))P_{\T}\|_{\op}
\leq b_{H,\op}
\leq
\frac{c_H}{8}\frac{n^2}{D},
\]
where the last inequality follows from
\((D/n^2)b_{H,\op}\leq\mathfrak D(r_\infty)\leq c_{\mathrm{sc}}\) and $C_{\mathrm{box}}\geq 1$.
Lemma~\ref{lem:local-uniformity} similarly gives
\[
\sup_{\mathcal B^\infty(r_\infty)}
\|P_{\T}[\nabla_\alpha^2\cQ(\overline{\boldsymbol\alpha}(q)+\boldsymbol\beta)
-\nabla_\alpha^2\cQ(\overline{\boldsymbol\alpha}(q))]P_{\T}\|_{\op}
\leq
\frac{c_H}{8}\frac{n^2}{D},
\]
because \(Dn\nu r_\infty\leq\mathfrak D(r_\infty)\), and
\[
\sup_{\mathcal B^\infty(r_\infty)}
\|P_{\T}\nabla_\alpha^2R(\overline{\boldsymbol\alpha}(q)+\boldsymbol\beta)P_{\T}\|_{\op}
\leq
\frac{c_H}{8}\frac{n^2}{D},
\]
because \((D/n)\varepsilon^3\leq\mathfrak D(r_\infty)\).

Finally, by \eqref{eq:noise-gap-decomposition},
\[
-P_{\T}\nabla_\alpha^2\cG P_{\T}
=
-P_{\T}\nabla_\alpha^2\overline{\cG}P_{\T}
-P_{\T}\nabla_\alpha^2\cQ P_{\T}
+P_{\T}\nabla_\alpha^2R P_{\T}.
\]
Combining the preceding four bounds yields
\(-P_{\T}\nabla_\alpha^2\cG P_{\T}
\succeq(c_H/2)(n^2/D)I_{\T}\), as claimed.
\end{proof}

\subsection{The unique inner optimizer}

For fixed \(q\in\mathcal I\), we optimize
\(\boldsymbol\beta\mapsto
\cG(\overline{\boldsymbol\alpha}(q)+\boldsymbol\beta)\)
over the tangent space \(\T\).  Its first-order condition is
\[
F_q(\boldsymbol\beta)
:=
P_{\T}\nabla\cG(
\overline{\boldsymbol\alpha}(q)+\boldsymbol\beta)
=0.
\]
We linearize this equation around \(\boldsymbol\beta=0\), using as the
reference Hessian the deterministic mean curvature
\[
L_q
:=
-P_{\T}\nabla_\alpha^2\overline{\cG}
(\overline{\boldsymbol\alpha}(q))P_{\T}|_{\T}.
\]
By Lemma~\ref{lem:mean-curvature}, \(L_q\) is uniformly positive definite
and therefore invertible on \(\T\).  We collect the discrepancy between
\(F_q\) and this linear approximation into
\[
\mathfrak r_q(\boldsymbol\beta)
:=
F_q(\boldsymbol\beta)-F_q(0)+L_q\boldsymbol\beta.
\]
Thus
\[
F_q(\boldsymbol\beta)
=
F_q(0)-L_q\boldsymbol\beta+\mathfrak r_q(\boldsymbol\beta),
\]
where \(F_q(0)\) is the forcing term at the reference path,
\(L_q\boldsymbol\beta\) is the leading linear response, and
\(\mathfrak r_q(\boldsymbol\beta)\) contains the remaining local variation.
Thus solving \(F_q(\boldsymbol\beta)=0\) is equivalent to finding a fixed point of
\[
T_q(\boldsymbol\beta)
:=
L_q^{-1}\bigl[F_q(0)+\mathfrak r_q(\boldsymbol\beta)\bigr].
\]
We will prove existence and uniqueness of the inner optimizer by showing that
\(T_q\) is a contraction on a suitable neighborhood of the origin.

\begin{lemma}[Mixed-norm control of the fixed-point remainder]
\label{lem:fp-infty-remainder}
There exists \(C_{\mathrm{rem}}>0\), depending only on \(\kappa,r_0\), such
that the following holds.  Suppose
\(\mathcal E(\boldsymbol b,\varepsilon,\nu)\) holds,
\(\varepsilon\leq\varrho_0\), and
$\mathfrak D(r_\infty)\leq c_{\mathrm{sc}}$.
Then, uniformly over \(q\in\mathcal I\) and
\(\boldsymbol\beta,\widetilde{\boldsymbol\beta}\in\mathcal B^\infty(r_\infty)\),
\begin{align}
\|T_q(\boldsymbol\beta) - T_q(\widetilde{\boldsymbol\beta})\|_2
&\leq C_{\mathrm{rem}}\mathfrak D(r_\infty)
\|\boldsymbol\beta-\widetilde{\boldsymbol\beta}\|_2,
\label{eq:rem-2}\\
\|T_q(\boldsymbol\beta) - T_q(\widetilde{\boldsymbol\beta})\|_\infty
&\leq C_{\mathrm{rem}}\frac{D}{n^2}b_{H,2\to\infty}
\|\boldsymbol\beta-\widetilde{\boldsymbol\beta}\|_2
+C_{\mathrm{rem}}\mathfrak D(r_\infty)
\|\boldsymbol\beta-\widetilde{\boldsymbol\beta}\|_\infty.
\label{eq:rem-inf}
\end{align}
\end{lemma}

\begin{proof}
First note that $T_q(\boldsymbol\beta) - T_q(\widetilde{\boldsymbol\beta}) = L_q^{-1}\bigl[\mathfrak r_q(\boldsymbol\beta) - \mathfrak r_q(\widetilde{\boldsymbol\beta})\bigr]$.

By Lemma~\ref{lem:mean-curvature}, uniformly over \(q\in\mathcal I\),
\begin{equation}
\label{eq:rem-inverse-bounds}
\|L_q^{-1}\|_{2\to2}\leq C_L\frac{D}{n^2},
\qquad
\|L_q^{-1}\boldsymbol x\|_\infty
\leq C_L\frac{D}{n^2}\|\boldsymbol x\|_\infty.
\end{equation}
Fix \(q\in\mathcal I\) and
\(\boldsymbol\beta,\widetilde{\boldsymbol\beta}\in\mathcal B^\infty(r_\infty)\).
Set \(\boldsymbol h:=\boldsymbol\beta-\widetilde{\boldsymbol\beta}\) and
\(\boldsymbol\gamma_t:=\widetilde{\boldsymbol\beta}+t\boldsymbol h\).
Since \(\mathcal BB^\infty(r_\infty)\) is convex, the fundamental theorem of
calculus gives
\begin{equation}
\label{eq:rem-ftc}
\mathfrak r_q(\boldsymbol\beta)-\mathfrak r_q(\widetilde{\boldsymbol\beta})
=
\int_0^1
P_{\T}
[\nabla_\alpha^2\cG(\overline{\boldsymbol\alpha}(q)+\boldsymbol\gamma_t)
-\nabla_\alpha^2\overline{\cG}(\overline{\boldsymbol\alpha}(q))]
P_{\T}\boldsymbol h\,dt.
\end{equation}
Using \eqref{eq:noise-gap-decomposition}, we estimate the four contributions
in \eqref{eq:rem-ftc} separately.

For the deterministic mean variation, Lemma~\ref{lem:local-uniformity}
and \eqref{eq:rem-inverse-bounds} give, uniformly in \(t\),
\[
\|L_q^{-1}P_{\T}[\nabla^2\overline{\cG}
(\overline{\boldsymbol\alpha}(q)+\boldsymbol\gamma_t)
-\nabla^2\overline{\cG}(\overline{\boldsymbol\alpha}(q))]P_{\T}\boldsymbol h\|_2
\leq C_LC_{\mathrm{box}}nr_\infty\|\boldsymbol h\|_2,
\]
and the same coefficient multiplies \(\|\boldsymbol h\|_\infty\) in the
\(\ell_\infty\)-bound.

For the frozen centered Hessian, the path bounds in
Definition~\ref{def:basic-event} and \eqref{eq:rem-inverse-bounds} yield
\[
\|L_q^{-1}P_{\T}\nabla^2\cQ(\overline{\boldsymbol\alpha}(q))P_{\T}\boldsymbol h\|_2
\leq
C_L\frac{D}{n^2}b_{H,\op}\|\boldsymbol h\|_2,
\]
and
\[
\|L_q^{-1}P_{\T}\nabla^2\cQ(\overline{\boldsymbol\alpha}(q))P_{\T}\boldsymbol h\|_\infty
\leq
C_L\frac{D}{n^2}b_{H,2\to\infty}\|\boldsymbol h\|_2.
\]
For the local centered-Hessian variation,
Lemma~\ref{lem:local-uniformity}, \eqref{eq:rem-inverse-bounds}, and
\(\|P_{\T}AP_{\T}\|_{\infty\to\infty}\leq4\|A\|_{\infty\to\infty}\)
give coefficients bounded by
\(C_LC_{\mathrm{box}}Dn\nu r_\infty\) in \(\ell_2\) and by
\(4C_LC_{\mathrm{box}}Dn\nu r_\infty\) in \(\ell_\infty\).
Finally, the remainder-Hessian terms are bounded in the same way by
\(C_LC_{\mathrm{box}}(D/n)\varepsilon^3\) and
\(4C_LC_{\mathrm{box}}(D/n)\varepsilon^3\), respectively.

Thus, with
\[
C_{\mathrm{rem}}
:=
4C_LC_{\mathrm{box}},
\]
combining the four contributions in \eqref{eq:rem-ftc} gives
\eqref{eq:rem-2}--\eqref{eq:rem-inf} by the definition of
\(\mathfrak D\). 
\end{proof}

For \(r_2,r_\infty>0\), write
\[
\mathcal B(r_2,r_\infty)
:=
\{\boldsymbol\beta\in\T:\|\boldsymbol\beta\|_2\leq r_2,
\ \|\boldsymbol\beta\|_\infty\leq r_\infty\}.
\]

For deterministic event parameters \(\boldsymbol b,\varepsilon,\nu\), define
\begin{equation}
\label{eq:deterministic-R-scale}
\mathfrak R
:=
\frac{D}{n^{3/2}}\bigl(b_\nabla+\varepsilon^3\bigr).
\end{equation}

\begin{proposition}[Slice optimizer, regularity, and envelope formula]
\label{prop:inner}
There exist constants \(c_{\mathrm{fp}},C_{\mathrm{in}}>0\), depending only on
\(\kappa,r_0\), such that the following holds. Suppose
\(\mathcal E(\boldsymbol b,\varepsilon,\nu)\) holds,
\(\varepsilon\leq\varrho_0\), and
\begin{equation}
\label{eq:inner-deterministic-smallness}
\mathfrak D\!\left(
C_{\mathrm{in}}
\left(\frac{1}{\sqrt n}+\frac{D}{n^2}b_{H,2\to\infty}\right)
\mathfrak R
\right)
\leq c_{\mathrm{fp}}.
\end{equation}
Then, uniformly over \(q\in\mathcal I\), Proposition~\ref{prop:strong-concavity}
applies on
\[
\left\{
\boldsymbol\beta\in\T:
\|\boldsymbol\beta\|_\infty
\leq
C_{\mathrm{in}}
\left(\frac{1}{\sqrt n}+\frac{D}{n^2}b_{H,2\to\infty}\right)
\mathfrak R
\right\}.
\]
Moreover, the slice
$
\mathcal S_q
:=
\left\{
\boldsymbol\beta\in\T:
\overline{\boldsymbol\alpha}(q)+\boldsymbol\beta\in\D_n
\right\}
$
admits a unique maximizer
\[
\boldsymbol\beta^\star(q)
:=
\arg\max_{\boldsymbol\beta\in\mathcal S_q}
\cG(\overline{\boldsymbol\alpha}(q)+\boldsymbol\beta).
\]
It is interior and satisfies
\begin{equation}
\label{eq:inner-size}
\|\boldsymbol\beta^\star(q)\|_2
\leq
\frac{C_{\mathrm{in}}}{2}\mathfrak R,
\qquad
\|\boldsymbol\beta^\star(q)\|_\infty
\leq
\frac{C_{\mathrm{in}}}{2}
\left(\frac{1}{\sqrt n}+\frac{D}{n^2}b_{H,2\to\infty}\right)
\mathfrak R.
\end{equation}
Moreover, \(q\mapsto\boldsymbol\beta^\star(q)\) is continuously differentiable
on \(\mathcal I\), and, for
\(V(q):=\cG(\boldsymbol\alpha^\star(q))\) with
\[
\boldsymbol\alpha^\star(q)
:=
\overline{\boldsymbol\alpha}(q)+\boldsymbol\beta^\star(q),
\]
we have
\[
V'(q)
=
\left\langle
\nabla\cG(\boldsymbol\alpha^\star(q)),
\overline{\boldsymbol\alpha}'(q)
\right\rangle.
\]
\end{proposition}

\begin{proof}
Let \(C_L\) be as in \eqref{eq:rem-inverse-bounds}, let
\(C_{\mathrm{rem}}\) be the constant from
Lemma~\ref{lem:fp-infty-remainder}, and let \(c_{\mathrm{sc}}\) be the
constant from Proposition~\ref{prop:strong-concavity}. Set
\[
C_0:=C_L\max\{1,2K_R\},
\]
choose
\[
C_{\mathrm{in}}
\geq
8C_0(1+8C_{\mathrm{rem}}),
\qquad
c_{\mathrm{fp}}
\leq
\min\left\{c_{\mathrm{sc}},\frac{1}{4C_{\mathrm{rem}}}\right\},
\]
and, for the remainder of the proof, write
\[
r_2:=8C_0\mathfrak R,
\qquad
r_\infty
:=
C_{\mathrm{in}}
\left(
\frac{1}{\sqrt n}
+\frac{D}{n^2}b_{H,2\to\infty}
\right)
\mathfrak R.
\]
Then \eqref{eq:inner-deterministic-smallness} gives
\(\mathfrak D(r_\infty)\leq c_{\mathrm{fp}}\leq c_{\mathrm{sc}}\).
Since \(nr_\infty\leq\mathfrak D(r_\infty)\) and
\(c_{\mathrm{sc}}\leq c_\alpha/2\), we also have
\(nr_\infty\leq c_\alpha/2\).

We first control \(T_q(0)=L_q^{-1}F_q(0)\). Since Lemma~\ref{lem:mean-curvature} implies
\(P_{\T}\nabla\overline{\cG}(\overline{\boldsymbol\alpha}(q))=0\),
\eqref{eq:noise-gap-decomposition} gives
\begin{equation}
\label{eq:inner-F0-decomposition}
F_q(0)
=
P_{\T}\nabla\cQ(\overline{\boldsymbol\alpha}(q))
-
P_{\T}\nabla R(\overline{\boldsymbol\alpha}(q)).
\end{equation}
By Definition~\ref{def:basic-event},
\[
\|P_{\T}\nabla\cQ(\overline{\boldsymbol\alpha}(q))\|_\infty
\leq b_\nabla,
\qquad
\|P_{\T}\nabla\cQ(\overline{\boldsymbol\alpha}(q))\|_2
\leq\sqrt n\,b_\nabla.
\]
Also, Lemma~\ref{lem:analytic-remainder} gives
\(\|\nabla R\|_\infty\leq K_R\varepsilon^3\). Since
\(\|P_{\T}x\|_\infty\leq2\|x\|_\infty\),
\(\|P_{\T}x\|_2\leq\sqrt n\|P_{\T}x\|_\infty\), and
\eqref{eq:rem-inverse-bounds} holds, we obtain
\begin{equation}
\label{eq:inner-linear-term}
\sup_{q\in\mathcal I}\|T_q(0)\|_2
\leq C_0\mathfrak R,
\qquad
\sup_{q\in\mathcal I}\|T_q(0)\|_\infty
\leq \frac{C_0}{\sqrt n}\mathfrak R.
\end{equation}

By the choice of \(c_{\mathrm{fp}}\),
Proposition~\ref{prop:strong-concavity} applies on the
\(\ell_\infty\)-ball of radius \(r_\infty\), and
Lemma~\ref{lem:fp-infty-remainder} applies to any
\(\boldsymbol\beta,\widetilde{\boldsymbol\beta}\in\T\) with
\(\|\boldsymbol\beta\|_\infty\vee
\|\widetilde{\boldsymbol\beta}\|_\infty\leq r_\infty\).
Define
\[
\mathcal B_{\mathrm{in}}
:=
\mathcal B(r_2,r_\infty).
\]
Since \(nr_\infty\leq c_\alpha/2\),
Lemma~\ref{lem:path-geometry} implies
\(\mathcal B_{\mathrm{in}}\subset\mathcal S_q\).

Let \(\boldsymbol\beta\in\mathcal B_{\mathrm{in}}\). By
\eqref{eq:inner-linear-term}, Lemma~\ref{lem:fp-infty-remainder}, and
\(\mathfrak D(r_\infty)\leq c_{\mathrm{fp}}\),
\[
\begin{aligned}
\|T_q(\boldsymbol\beta)\|_2
&\leq
\|T_q(0)\|_2
+
\|T_q(\boldsymbol\beta)-T_q(0)\|_2\\
&\leq
C_0\mathfrak R
+
C_{\mathrm{rem}}c_{\mathrm{fp}}r_2
\leq
\frac18r_2+\frac14r_2
=
\frac38r_2.
\end{aligned}
\]
Similarly,
\[
\begin{aligned}
\|T_q(\boldsymbol\beta)\|_\infty
&\leq
\|T_q(0)\|_\infty
+
\|T_q(\boldsymbol\beta)-T_q(0)\|_\infty\\
&\leq
\frac{C_0}{\sqrt n}\mathfrak R
+
C_{\mathrm{rem}}\frac{D}{n^2}b_{H,2\to\infty}r_2
+
C_{\mathrm{rem}}c_{\mathrm{fp}}r_\infty\\
&\leq
C_0(1+8C_{\mathrm{rem}})
\left(
\frac{1}{\sqrt n}
+\frac{D}{n^2}b_{H,2\to\infty}
\right)\mathfrak R
+
\frac14r_\infty\\
&\leq
\frac18r_\infty+\frac14r_\infty
=
\frac38r_\infty.
\end{aligned}
\]
Thus \(T_q\) maps \(\mathcal B_{\mathrm{in}}\) into itself. Moreover, for
\(\boldsymbol\beta,\widetilde{\boldsymbol\beta}\in\mathcal B_{\mathrm{in}}\),
Lemma~\ref{lem:fp-infty-remainder} gives
\[
\|T_q(\boldsymbol\beta)-T_q(\widetilde{\boldsymbol\beta})\|_2
\leq
C_{\mathrm{rem}}\mathfrak D(r_\infty)
\|\boldsymbol\beta-\widetilde{\boldsymbol\beta}\|_2
\leq
\frac14
\|\boldsymbol\beta-\widetilde{\boldsymbol\beta}\|_2.
\]
Banach's fixed-point theorem therefore gives a unique fixed point
\(\boldsymbol\beta^\star(q)\in\mathcal B_{\mathrm{in}}\), and hence
\begin{equation}
\label{eq:inner-first-order}
P_{\T}\nabla\cG(
\overline{\boldsymbol\alpha}(q)+\boldsymbol\beta^\star(q))
=0.
\end{equation}
Since \(r_2=8C_0\mathfrak R\leq C_{\mathrm{in}}\mathfrak R\), the preceding
bounds imply \eqref{eq:inner-size}.

The set \(\mathcal S_q\) is convex, and
\(\cG(\overline{\boldsymbol\alpha}(q)+\cdot)\) is concave on it. Hence
\eqref{eq:inner-first-order} shows that \(\boldsymbol\beta^\star(q)\) is a
maximizer on the full slice. To prove uniqueness, suppose
\(\boldsymbol\beta'\in\mathcal S_q\) is another maximizer. Concavity implies
that
\[
t\mapsto
\cG\bigl(
\overline{\boldsymbol\alpha}(q)
+(1-t)\boldsymbol\beta^\star(q)
+t\boldsymbol\beta'
\bigr)
\]
is constant on \([0,1]\). Since the bounds above place
\(\boldsymbol\beta^\star(q)\) strictly inside \(\mathcal B_{\mathrm{in}}\), for
all sufficiently small \(t>0\) this segment remains in the region where
Proposition~\ref{prop:strong-concavity} gives strict concavity, a contradiction
unless \(\boldsymbol\beta'=\boldsymbol\beta^\star(q)\). Moreover,
\(\overline\alpha_i(q)+\beta_i^\star(q)\geq c_\alpha/(2n)>0\), so the maximizer
is interior.

Finally, define
\(\Psi(q,\boldsymbol\beta):=
P_{\T}\nabla\cG(\overline{\boldsymbol\alpha}(q)+\boldsymbol\beta)\).
Then \(\Psi(q,\boldsymbol\beta^\star(q))=0\), and
\(D_\beta\Psi=P_{\T}\nabla_\alpha^2\cG P_{\T}|_{\T}\) is invertible at the
maximizer by Proposition~\ref{prop:strong-concavity}. The implicit-function
theorem gives a local \(C^1\) solution; uniqueness of the slice optimizer
patches these local solutions into a \(C^1\) map on \(\mathcal I\).
Differentiating
\(V(q)=\cG(\overline{\boldsymbol\alpha}(q)+\boldsymbol\beta^\star(q))\)
gives
\[
V'(q)
=
\left\langle
\nabla\cG(\boldsymbol\alpha^\star(q)),
\overline{\boldsymbol\alpha}'(q)
+(\boldsymbol\beta^\star)'(q)
\right\rangle.
\]
Since \((\boldsymbol\beta^\star)'(q)\in\T\), the second term vanishes by
\eqref{eq:inner-first-order}, yielding the stated envelope formula.
\end{proof}

\subsection{Derivative of the optimized value}

For a differentiable function \(F\) and vector \(\boldsymbol\xi\), write
\(D_{\boldsymbol\xi}F(\boldsymbol\alpha):=
\langle\nabla F(\boldsymbol\alpha),\boldsymbol\xi\rangle\).

\begin{lemma}[Local mixed derivative for the centered term]
\label{lem:local-mixed-derivative}
There exists \(C_{\mathrm{mix}}>0\), depending only on \(\kappa,r_0\), such
that, if
\(\mathcal E(\boldsymbol b,\varepsilon,\nu)\) holds,
\(\varepsilon\leq\varrho_0\), and \(nr_\infty\leq c_\alpha/2\), then
\begin{equation}
\label{eq:local-mixed-derivative}
\sup_{\substack{q\in\mathcal I,\ \boldsymbol\beta\in\T\\
\|\boldsymbol\beta\|_\infty\leq r_\infty}}
\|P_{\T}\nabla_\beta D_{\overline{\boldsymbol\alpha}'(q)}
\cQ(\overline{\boldsymbol\alpha}(q)+\boldsymbol\beta)\|_2
\leq
C_{\mathrm{mix}}\bigl(b_{\nabla q}+n^{5/2}\nu r_\infty\bigr).
\end{equation}
\end{lemma}

\begin{proof}
At \(\boldsymbol\beta=0\), Definition~\ref{def:basic-event} gives
\[
\sup_{q\in\mathcal I}
\left\|P_{\T}
\left.\nabla_\beta D_{\overline{\boldsymbol\alpha}'(q)}
\cQ(\overline{\boldsymbol\alpha}(q)+\boldsymbol\beta)
\right|_{\boldsymbol\beta=0}\right\|_2
\leq b_{\nabla q}.
\]
It remains to control the variation in \(\boldsymbol\beta\).  For fixed
\(q\), the chain rule gives
\begin{equation}
\label{eq:local-mixed-chain-rule}
\nabla_\beta D_{\overline{\boldsymbol\alpha}'(q)}
\cQ(\overline{\boldsymbol\alpha}(q)+\boldsymbol\beta)
=
\nabla_\alpha^2\cQ(\overline{\boldsymbol\alpha}(q)+\boldsymbol\beta)
\overline{\boldsymbol\alpha}'(q).
\end{equation}
Writing \(A_Q(q,\boldsymbol\beta)\) as in the proof of
Lemma~\ref{lem:local-uniformity}, the difference from \(\boldsymbol\beta=0\)
is bounded by
\[
\|A_Q(q,\boldsymbol\beta)\|_{\op}
\|\overline{\boldsymbol\alpha}'(q)\|_2.
\]
By Lemma~\ref{lem:path-geometry},
\(\|\overline{\boldsymbol\alpha}'(q)\|_2\leq(\kappa\sqrt n)^{-1}\), while
\eqref{eq:local-Q-op} gives
\(\|A_Q(q,\boldsymbol\beta)\|_{\op}\leq
C_{\mathrm{box}}n^3\nu r_\infty\).  Therefore the variation is at most
\((C_{\mathrm{box}}/\kappa)n^{5/2}\nu r_\infty\).  Taking
\(C_{\mathrm{mix}}\geq\max\{1,C_{\mathrm{box}}/\kappa\}\) proves
\eqref{eq:local-mixed-derivative}.
\end{proof}

Given the deterministic event parameters, define
\begin{equation}
\label{eq:deterministic-derivative-scale}
\mathfrak E_V
:=
\frac{n^2}{D}\mathfrak R^2
+b_q
+b_{\nabla q}\mathfrak R
+n^{5/2}\nu
\left(\frac{1}{\sqrt n}+\frac{D}{n^2}b_{H,2\to\infty}\right)
\mathfrak R^2
+\varepsilon^3.
\end{equation}

\begin{proposition}[Derivative approximation]
\label{prop:Vprime}
There exists \(C_V>0\), depending only on \(\kappa,r_0\), such that, under
the hypotheses of Proposition~\ref{prop:inner},
\begin{equation}
\label{eq:Vprime-high-probability}
\left|V'(1/2)-\frac{1}{D}\theta_{n_+,n_-}\right|
\leq C_V\mathfrak E_V.
\end{equation}
\end{proposition}

\begin{proof}
Set
\[
\boldsymbol\xi_0:=\overline{\boldsymbol\alpha}'(1/2),
\qquad
\overline{\boldsymbol\alpha}:=\overline{\boldsymbol\alpha}(1/2),
\qquad
\boldsymbol\beta^\star:=\boldsymbol\beta^\star(1/2).
\]
Also set
\[
C_\xi
:=
\sup_{q\in\mathcal I}
n\|\overline{\boldsymbol\alpha}'(q)\|_\infty
\leq \kappa^{-1},
\]
where the last inequality follows from Lemma~\ref{lem:path-geometry}. By
Proposition~\ref{prop:inner},
\begin{equation}
\label{eq:Vprime-beta-size}
\|\boldsymbol\beta^\star\|_2
\leq C_{\mathrm{in}}\mathfrak R,
\qquad
\|\boldsymbol\beta^\star\|_\infty
\leq
C_{\mathrm{in}}
\left(
\frac{1}{\sqrt n}
+\frac{D}{n^2}b_{H,2\to\infty}
\right)
\mathfrak R.
\end{equation}
For brevity, write
\[
r_\infty
:=
C_{\mathrm{in}}
\left(
\frac{1}{\sqrt n}
+\frac{D}{n^2}b_{H,2\to\infty}
\right)
\mathfrak R.
\]

By the envelope formula in Proposition~\ref{prop:inner},
\[
V'(1/2)
=
D_{\boldsymbol\xi_0}
\cG(\overline{\boldsymbol\alpha}+\boldsymbol\beta^\star).
\]
Using \eqref{eq:noise-gap-decomposition}, we therefore have
\begin{equation}
\label{eq:Vprime-decomposition}
V'(1/2)
=
D_{\boldsymbol\xi_0}\overline{\cG}
(\overline{\boldsymbol\alpha}+\boldsymbol\beta^\star)
+
D_{\boldsymbol\xi_0}\cQ
(\overline{\boldsymbol\alpha}+\boldsymbol\beta^\star)
-
D_{\boldsymbol\xi_0}R
(\overline{\boldsymbol\alpha}+\boldsymbol\beta^\star).
\end{equation}
We estimate the three terms separately.

First consider the deterministic mean term. Since
\(\overline{\cG}=D^{-1}\cH\), \eqref{eq:theta-definition} gives
\begin{equation}
\label{eq:Vprime-mean-main}
D_{\boldsymbol\xi_0}\overline{\cG}
(\overline{\boldsymbol\alpha})
=
\left.
\frac{d}{dq}
\overline{\cG}(\overline{\boldsymbol\alpha}(q))
\right|_{q=1/2}
=
\frac{1}{D}\theta_{n_+,n_-}.
\end{equation}
Moreover, the first-order variation of
\(D_{\boldsymbol\xi_0}\overline{\cG}\) in tangent directions vanishes. To see
this, observe first that \(\overline{\cG}\) is invariant under every
permutation of the coordinates that preserves the two groups \(I_+\) and
\(I_-\). Both
\(\overline{\boldsymbol\alpha}
=\overline{\boldsymbol\alpha}(1/2)\) and
\(\boldsymbol\xi_0=\overline{\boldsymbol\alpha}'(1/2)\) are also constant
on each of these two groups. Therefore the scalar function
\[
\boldsymbol\alpha
\longmapsto
D_{\boldsymbol\xi_0}\overline{\cG}(\boldsymbol\alpha)
\]
has the same within-group permutation invariance. Consequently, at the
within-group symmetric point \(\overline{\boldsymbol\alpha}\), its gradient
cannot distinguish two coordinates belonging to the same group. Hence there
exist scalars \(a_+\) and \(a_-\) such that
\[
\nabla_\alpha
D_{\boldsymbol\xi_0}\overline{\cG}
(\overline{\boldsymbol\alpha})
=
a_+\boldsymbol 1_{I_+}
+
a_-\boldsymbol 1_{I_-}.
\]
Since every \(\boldsymbol\beta\in\T\) has zero sum on each group,
\[
\left\langle
\nabla_\alpha
D_{\boldsymbol\xi_0}\overline{\cG}
(\overline{\boldsymbol\alpha}),
\boldsymbol\beta
\right\rangle
=
a_+\sum_{i\in I_+}\beta_i
+
a_-\sum_{i\in I_-}\beta_i
=
0.
\]
Thus
\(\nabla_\alpha
D_{\boldsymbol\xi_0}\overline{\cG}
(\overline{\boldsymbol\alpha})\in\T^\perp\), and therefore
\begin{equation}
\label{eq:Vprime-mean-first-variation-zero}
P_{\T}
\nabla_\alpha
D_{\boldsymbol\xi_0}\overline{\cG}
(\overline{\boldsymbol\alpha})
=
0.
\end{equation}

We next bound the second tangent variation. On the coordinate window
\eqref{eq:local-coordinate-window}, the third derivatives of \(h\) are
bounded by \(C_hn^2\). Since
\(\|\boldsymbol\xi_0\|_\infty\leq C_\xi/n\), each nonzero entry of a
pairwise block contributing to
\(\nabla_\alpha^2D_{\boldsymbol\xi_0}\overline{\cG}\) is of order
\(n/D\). For a fixed row \(k\), only the \(n-1\) pairwise blocks indexed by
pairs \((i,j)\) with \(k\in\{i,j\}\) can contribute.
Consequently, the row-sum bound and symmetry give
\[
\left\|
P_{\T}
\nabla_\alpha^2
D_{\boldsymbol\xi_0}\overline{\cG}
(\overline{\boldsymbol\alpha}+\boldsymbol\beta)
P_{\T}
\right\|_{\op}
\leq
A_V\frac{n^2}{D}
\]
uniformly over
\(\boldsymbol\beta\in\T\) with
\(\|\boldsymbol\beta\|_\infty\leq r_\infty\), for a constant
\(A_V>0\) depending only on \(\kappa,r_0\).
Since \(\boldsymbol\beta^\star\in\T\), Taylor's theorem,
\eqref{eq:Vprime-beta-size}, and \eqref{eq:Vprime-mean-first-variation-zero}
therefore imply
\begin{equation}\label{eq:Vprime-mean-error}
\left|
D_{\boldsymbol\xi_0}\overline{\cG}
(\overline{\boldsymbol\alpha}+\boldsymbol\beta^\star)
-
D_{\boldsymbol\xi_0}\overline{\cG}
(\overline{\boldsymbol\alpha})
\right|
\leq
A_V\frac{n^2}{D}\|\boldsymbol\beta^\star\|_2^2
\leq
A_VC_{\mathrm{in}}^2
\frac{n^2}{D}\mathfrak R^2.
\end{equation}

Next consider the centered term. At the base point,
Definition~\ref{def:basic-event} gives
\begin{equation}
\label{eq:Vprime-Q-path}
\left|
D_{\boldsymbol\xi_0}\cQ
(\overline{\boldsymbol\alpha})
\right|
\leq b_q.
\end{equation}
To control the displacement from
\(\overline{\boldsymbol\alpha}\) to
\(\overline{\boldsymbol\alpha}+\boldsymbol\beta^\star\), apply the
fundamental theorem of calculus along the segment
\(\overline{\boldsymbol\alpha}+t\boldsymbol\beta^\star\). Since
\(\boldsymbol\beta^\star\in\T\),
\begin{align}
&D_{\boldsymbol\xi_0}\cQ
(\overline{\boldsymbol\alpha}+\boldsymbol\beta^\star)
-
D_{\boldsymbol\xi_0}\cQ
(\overline{\boldsymbol\alpha})
\notag\\
&\qquad=
\int_0^1
\left\langle
P_{\T}
\nabla_\alpha
D_{\boldsymbol\xi_0}\cQ
(\overline{\boldsymbol\alpha}
+t\boldsymbol\beta^\star),
\boldsymbol\beta^\star
\right\rangle
\,dt .
\label{eq:Vprime-Q-mvt}
\end{align}
Moreover,
\(\|t\boldsymbol\beta^\star\|_\infty\leq r_\infty\) for
\(t\in[0,1]\). Hence Lemma~\ref{lem:local-mixed-derivative} and
\eqref{eq:Vprime-beta-size} give
\begin{align}
&\left|
D_{\boldsymbol\xi_0}\cQ
(\overline{\boldsymbol\alpha}+\boldsymbol\beta^\star)
-
D_{\boldsymbol\xi_0}\cQ
(\overline{\boldsymbol\alpha})
\right|
\notag\\
&\qquad\leq
C_{\mathrm{mix}}
\bigl(
b_{\nabla q}
+n^{5/2}\nu r_\infty
\bigr)
\|\boldsymbol\beta^\star\|_2
\notag\\
&\qquad\leq
C_{\mathrm{mix}}C_{\mathrm{in}}
\left[
b_{\nabla q}
+
C_{\mathrm{in}}n^{5/2}\nu
\left(
\frac{1}{\sqrt n}
+\frac{D}{n^2}b_{H,2\to\infty}
\right)
\mathfrak R
\right]
\mathfrak R.
\label{eq:Vprime-Q-local-error}
\end{align}
Together with \eqref{eq:Vprime-Q-path}, this controls the centered
contribution.

Finally consider the analytic remainder. 
Consider the affine path
\[
\boldsymbol\alpha(t)
:=
\overline{\boldsymbol\alpha}
+\boldsymbol\beta^\star
+t\boldsymbol\xi_0.
\]
Then
\(\boldsymbol\alpha'(t)=\boldsymbol\xi_0\) and
\[
\left.
\frac{d}{dt}
R(\boldsymbol\alpha(t))
\right|_{t=0}
=
D_{\boldsymbol\xi_0}R
(\overline{\boldsymbol\alpha}+\boldsymbol\beta^\star).
\]
Moreover,
\[
\|n\diag(\boldsymbol\xi_0)\|_{\op}
=
n\|\boldsymbol\xi_0\|_\infty
\leq C_\xi.
\]
Since the basic event gives
\(\varepsilon\leq\varrho_0\), the affine-path derivative bound \eqref{eq:analytic-remainder-path-bound}
therefore yields
\begin{equation}
\label{eq:Vprime-R-error}
\left|
D_{\boldsymbol\xi_0}R
(\overline{\boldsymbol\alpha}+\boldsymbol\beta^\star)
\right|
\leq
C_\xi K_R\varepsilon^3.
\end{equation}

Combining \eqref{eq:Vprime-decomposition},
\eqref{eq:Vprime-mean-main}, \eqref{eq:Vprime-mean-error},
\eqref{eq:Vprime-Q-path}, \eqref{eq:Vprime-Q-local-error}, and
\eqref{eq:Vprime-R-error}, we obtain
\[
\begin{aligned}
\left|
V'(1/2)-\frac{1}{D}\theta_{n_+,n_-}
\right|
\leq C\Bigg[
&\frac{n^2}{D}\mathfrak R^2
+b_q
+b_{\nabla q}\mathfrak R\\
&+
n^{5/2}\nu
\left(
\frac{1}{\sqrt n}
+\frac{D}{n^2}b_{H,2\to\infty}
\right)
\mathfrak R^2
+\varepsilon^3
\Bigg],
\end{aligned}
\]
where \(C>0\) depends only on \(\kappa,r_0\).
Choosing \(C_V\) sufficiently large and using
\eqref{eq:deterministic-derivative-scale} gives
\eqref{eq:Vprime-high-probability}.
\end{proof}

\subsection{Balanced minimizer, local growth, and positivity}

Throughout this subsection, suppose that the hypotheses of
Propositions~\ref{prop:inner} are satisfied.

We also assume, whenever the balanced minimizer and positivity conclusions are
invoked, that
\begin{equation}\label{eq:section-balanced-assumptions}
C_V\mathfrak E_V
\leq
\frac{1}{2D}\theta_{n_+,n_-},
\qquad
\lambda\left(
\frac{1}{D}\theta_{n_+,n_-}+C_V\mathfrak E_V
\right)
<
2s.
\end{equation}

Recall \(\psi(q)=1-V(q)\)
and \(\Phi(q) := 2s\left|q-\frac12\right|+\lambda\psi(q)\). 

We say that a dual vector \(\boldsymbol \alpha\in\Delta_n\) is balanced if $q(\boldsymbol \alpha)=\frac12$.

\begin{proposition}[Unique balanced minimizer]
\label{prop:balanced}
Under the standing assumptions of this subsection, suppose that
\eqref{eq:section-balanced-assumptions} holds.  Then \(q_\star=\frac12\)
is the unique minimizer of \(\Phi\), and the full dual minimizer \(\boldsymbol \alpha^\star:=\boldsymbol \alpha^\star(1/2)\)
is unique and interior.
\end{proposition}

\begin{proof}
We first note that \(\Phi\) is convex on \([0,1]\) since \(\psi\) is convex by Proposition~\ref{prop:dual}.

Since \(n_+>n_-\), we have \(\theta_{n_+,n_-}>0\).
By Proposition~\ref{prop:Vprime}, \(V'(1/2) \geq \frac{\theta_{n_+,n_-}}{D} - C_V\mathfrak E_V \geq \frac{1}{2D}\theta_{n_+,n_-} >0\).
The same proposition gives \(V'(1/2) \leq \frac{\theta_{n_+,n_-}}{D} + C_V\mathfrak E_V\).
Thus \(\lambda V'(1/2)<2s\). The one-sided derivatives of the convex function
\(\Phi\) are \(\Phi'_-(1/2) = -2s-\lambda V'(1/2)<0\)
and \(\Phi'_+(1/2) = 2s-\lambda V'(1/2)>0\).
A convex function with strictly negative left derivative and strictly positive
right derivative has a unique minimizer at the kink. Proposition~\ref{prop:inner}
gives a unique interior inner optimizer there.
\end{proof}

\begin{proposition}[Local relative growth]
\label{prop:local-growth}
Under the standing assumptions of this subsection, suppose that
\eqref{eq:section-balanced-assumptions} holds.
Then \(F\) has local relative growth at
\(\boldsymbol\alpha^\star\).
\end{proposition}

\begin{proof}
Set
\[
r_{\mathrm{sc}}
:=
C_{\mathrm{in}}
\left(
\frac{1}{\sqrt n}
+\frac{D}{n^2}b_{H,2\to\infty}
\right)
\mathfrak R.
\]
By Proposition~\ref{prop:inner}, Proposition~\ref{prop:strong-concavity}
applies uniformly over
\[
q\in\mathcal I,
\qquad
\boldsymbol\beta\in\T,
\qquad
\|\boldsymbol\beta\|_\infty\leq r_{\mathrm{sc}},
\]
whereas
\begin{equation}
\label{eq:local-growth-inner-margin}
\|\boldsymbol\beta^\star(q)\|_\infty
\leq
\frac12 r_{\mathrm{sc}}
\end{equation}
uniformly over \(q\in\mathcal I\).

Let \(\boldsymbol\alpha\) be a feasible vector sufficiently close to
\(\boldsymbol\alpha^\star\). Then
\(q:=q(\boldsymbol\alpha)\) is sufficiently close to \(1/2\), and we can
write uniquely
\[
\boldsymbol\alpha
=
\overline{\boldsymbol\alpha}(q)+\boldsymbol\beta,
\qquad
\boldsymbol\beta\in\T.
\]
Since
\(\boldsymbol\beta^\star(q)\) is continuous in \(q\), by taking the
neighborhood of \(\boldsymbol\alpha^\star\) sufficiently small we may assume
that
\[
\|\boldsymbol\beta-\boldsymbol\beta^\star(q)\|_\infty
\leq
\frac12 r_{\mathrm{sc}}.
\]
Together with \eqref{eq:local-growth-inner-margin}, this implies that the
line segment joining \(\boldsymbol\beta^\star(q)\) and
\(\boldsymbol\beta\) lies in the region on which
Proposition~\ref{prop:strong-concavity} applies.

By Proposition~\ref{prop:balanced}, \(q=1/2\) is the unique minimizer of
\(\Phi\), and the one-sided derivatives at \(1/2\) are strictly negative
on the left and strictly positive on the right. Hence there exist
\(c_1,\varepsilon_1>0\) such that
\begin{equation}
\label{eq:local-growth-reduced}
\Phi(q)-\Phi(1/2)
\geq
c_1|q-1/2|
\end{equation}
whenever \(|q-1/2|\leq\varepsilon_1\).

Next, since 
$F(\boldsymbol\alpha)
=
2s\left|q-\frac12\right|
+\lambda\|Z(\boldsymbol\alpha)\|_1
$
and
$
\Phi(q)
=
2s\left|q-\frac12\right|
+\lambda\|Z(\boldsymbol\alpha^\star(q))\|_1,
$
we have
\begin{align}
F(\boldsymbol\alpha)-F(\boldsymbol\alpha^\star)
&=
\Phi(q)-\Phi(1/2)
\notag\\
&\qquad
+
\lambda
\left(
\|Z(\overline{\boldsymbol\alpha}(q)+\boldsymbol\beta)\|_1
-
\|Z(\boldsymbol\alpha^\star(q))\|_1
\right).
\label{eq:local-growth-decomposition}
\end{align}

On the line segment described above,
Proposition~\ref{prop:strong-concavity} implies that
\(\|Z(\overline{\boldsymbol\alpha}(q)+\cdot)\|_1=1-\cG\) is strongly
convex on \(\T\), with Hessian bounded below by
\((c_H/2)(n^2/D)I_{\T}\).
Since $\boldsymbol\alpha^\star(q) = \overline{\boldsymbol\alpha}(q)+\boldsymbol\beta^\star(q)$ is the tangent minimizer on the slice $q$, we have
\begin{equation}
\label{eq:local-growth-tangent}
\|Z(\overline{\boldsymbol\alpha}(q)+\boldsymbol\beta)\|_1
-
\|Z(\boldsymbol\alpha^\star(q))\|_1
\geq
\frac{c_H}{4}\frac{n^2}{D}
\|\boldsymbol\beta-\boldsymbol\beta^\star(q)\|_2^2.
\end{equation}
Combining \eqref{eq:local-growth-reduced},
\eqref{eq:local-growth-decomposition}, and
\eqref{eq:local-growth-tangent} gives
\begin{equation}
\label{eq:local-growth-F-lower}
F(\boldsymbol\alpha)-F(\boldsymbol\alpha^\star)
\geq
c_1|q-1/2|
+
\frac{\lambda c_H}{4}\frac{n^2}{D}
\|\boldsymbol\beta-\boldsymbol\beta^\star(q)\|_2^2.
\end{equation}

On the other hand,
\[
\boldsymbol\alpha-\boldsymbol\alpha^\star
=
\bigl(
\overline{\boldsymbol\alpha}(q)
-\overline{\boldsymbol\alpha}(1/2)
\bigr)
+
\bigl(
\boldsymbol\beta-\boldsymbol\beta^\star(q)
\bigr)
+
\bigl(
\boldsymbol\beta^\star(q)
-\boldsymbol\beta^\star(1/2)
\bigr).
\]
The path \(q\mapsto\overline{\boldsymbol\alpha}(q)\) is smooth, while
Proposition~\ref{prop:inner} gives that
\(q\mapsto\boldsymbol\beta^\star(q)\) is continuously differentiable.
Hence there exists \(C>0\) such that, for \(q\) sufficiently close to
\(1/2\),
\[
\|\boldsymbol\alpha-\boldsymbol\alpha^\star\|_2^2
\leq
C\left(
|q-1/2|^2
+
\|\boldsymbol\beta-\boldsymbol\beta^\star(q)\|_2^2
\right).
\]
Since \(|q-1/2|^2\leq|q-1/2|\) locally,
\eqref{eq:local-growth-F-lower} implies that, in a sufficiently small
neighborhood of \(\boldsymbol\alpha^\star\),
$
F(\boldsymbol\alpha)-F(\boldsymbol\alpha^\star)
\geq
\widetilde\mu
\|\boldsymbol\alpha-\boldsymbol\alpha^\star\|_2^2
$
for some \(\widetilde\mu>0\).

Finally, by Lemma~\ref{lem:local-kl-quadratic}, after possibly shrinking
the neighborhood,
\[
\KL(\boldsymbol\alpha^\star\|\boldsymbol\alpha)
\leq
C_{\mathrm{KL}}
\|\boldsymbol\alpha-\boldsymbol\alpha^\star\|_2^2.
\]
Thus
\[
F(\boldsymbol\alpha)-F(\boldsymbol\alpha^\star)
\geq
\frac{\widetilde\mu}{C_{\mathrm{KL}}}
\KL(\boldsymbol\alpha^\star\|\boldsymbol\alpha)
\]
throughout this neighborhood. Choosing \(r>0\) sufficiently small so that
\(\KL(\boldsymbol\alpha^\star\|\boldsymbol\alpha)\leq r\) implies that
\(\boldsymbol\alpha\) belongs to this neighborhood proves local relative
growth.
\end{proof}

\begin{lemma}[Rank of positive noise combinations]
\label{lem:positive-noise-rank}
Suppose \(\mathcal E(\boldsymbol b,\varepsilon,\nu)\) holds with
\(\varepsilon<1\).  Then \(U\) and \(\widehat V\) have full column rank, and,
for every \(\boldsymbol\alpha\in\operatorname{int}(\D_n)\),
\[
Z(\boldsymbol\alpha)
=U\Lambda_{\boldsymbol\alpha}\widehat V^\top
\]
has rank \(n\).  In particular,
\(\operatorname{rank}Z(\boldsymbol\alpha^\star)=n\) whenever
\(\boldsymbol\alpha^\star\in\operatorname{int}(\D_n)\).
\end{lemma}

\begin{proof}
By the geometric bounds in Definition~\ref{def:basic-event},
\(\|G_U-I\|_{\op}\vee\|G_V-I\|_{\op}<1\), so both Gram matrices are positive
definite.  Hence \(U\) and \(\widehat V\) have full column rank.  If
\(\boldsymbol\alpha\in\operatorname{int}(\D_n)\), then
\(\Lambda_{\boldsymbol\alpha}\) is invertible, and therefore
\(U\Lambda_{\boldsymbol\alpha}\widehat V^\top\) has rank \(n\).
\end{proof}

\begin{proposition}[KKT and envelope identification]
\label{prop:KKT-envelope}
Under the standing assumptions of this subsection, suppose that
\eqref{eq:section-balanced-assumptions} holds.  Let \(Z^\star:=Z(\boldsymbol \alpha^\star), \quad H^\star:=\pol(Z^\star)\).
Then \(\operatorname{rank}(Z^\star) = \operatorname{rank}(H^\star) = n\),
and \(V'(1/2) = h_-^\star-h_+^\star\),
where
\[
h_+^\star
=
\frac{1}{n_+}\sum_{i\in I_+}\ip{H^\star}{N_i},
\qquad
h_-^\star
=
\frac{1}{n_-}\sum_{i\in I_-}\ip{H^\star}{N_i}.
\]
The signal coefficient selected by the iterates is \(a^\star = \frac{\lambda}{2s}V'(1/2)\).
Moreover, \(0<a^\star<1\).
\end{proposition}

\begin{proof}
By Proposition~\ref{prop:balanced}, the full dual minimizer is
\(\boldsymbol \alpha^\star=\boldsymbol \alpha^\star(1/2)\), and it is interior.
By Lemma~\ref{lem:positive-noise-rank},
\(Z^\star=Z(\boldsymbol \alpha^\star)=U\Lambda_{\alpha^\star}\widehat V^\top\) has rank \(n\).
Since the polar factor has the same rank as the matrix to which it is applied,
\(\operatorname{rank}(H^\star)=n\). 
With the core notation from \eqref{eq:noise-core-notation},
Lemma~\ref{lem:polar-orthonormalization} gives \(\|Z(\boldsymbol \alpha)\|_1=\|Y(\alpha)\|_1\).
Since \(G_U^{1/2}\) and \(G_V^{1/2}\) are invertible and
\(\boldsymbol \alpha^\star\in\operatorname{int}(\D_n)\), the matrix \(Y^\star:=Y(\boldsymbol \alpha^\star) = G_U^{1/2}\Lambda_{\alpha^\star}G_V^{1/2}\)
is invertible.  Therefore Lemma~\ref{lem:gradient-schatten-one}, applied to
the full-rank \(n\times n\) matrix \(Y^\star\), implies that the coefficient
map \(\boldsymbol \alpha\mapsto \|Z(\boldsymbol \alpha)\|_1\)
is differentiable at \(\boldsymbol \alpha^\star\).
For the \(i\)-th coordinate direction, \(\frac{\partial}{\partial\alpha_i}Y(\boldsymbol \alpha^\star) = G_U^{1/2}E_{ii}G_V^{1/2}\),
where \(E_{ii}\) is the diagonal matrix with a single \(1\) in the \(i\)-th
diagonal entry.  Hence
\[
\frac{\partial}{\partial\alpha_i}\|Z(\boldsymbol \alpha^\star)\|_1
=
\left\langle
\pol(Y^\star),
G_U^{1/2}E_{ii}G_V^{1/2}
\right\rangle_F.
\]
On the other hand, Lemma~\ref{lem:polar-orthonormalization} gives \(H^\star = \pol(Z^\star) = \overline U\,\pol(Y^\star)\,\overline V^\top\),
where \(\overline U:=UG_U^{-1/2}, \overline V:=\widehat V G_V^{-1/2}\).
Since \(N_i=\boldsymbol u_i\widehat{\boldsymbol v}_i^\top\), we have
\[
\ip{H^\star}{N_i}
=
\boldsymbol u_i^\top \overline U
\pol(Y^\star)
\overline V^\top\widehat{\boldsymbol v}_i.
\]
Using \(U^\top\overline U=G_U^{1/2}, \quad \overline V^\top\widehat V=G_V^{1/2}\),
this becomes
\[
\ip{H^\star}{N_i}
=
e_i^\top
G_U^{1/2}\pol(Y^\star)G_V^{1/2}
e_i
=
\left\langle
\pol(Y^\star),
G_U^{1/2}E_{ii}G_V^{1/2}
\right\rangle_F.
\]
Therefore \(\frac{\partial}{\partial\alpha_i}\|Z(\boldsymbol \alpha^\star)\|_1 = \ip{H^\star}{N_i}\).
By Proposition~\ref{prop:inner}, \(V'(1/2) = \left\langle \nabla\cG(\boldsymbol \alpha^\star), \overline{\boldsymbol \alpha}'(1/2) \right\rangle\).
Since \(\cG(\boldsymbol{\alpha})=1-\|Z(\boldsymbol \alpha)\|_1\), we have \(\frac{\partial}{\partial\alpha_i}\cG(\alpha^\star) = -\ip{H^\star}{N_i}\).
Using
\[
\overline\alpha_i'(1/2)
=
\begin{cases}
1/n_+, & i\in I_+,\\
-1/n_-, & i\in I_-,
\end{cases}
\]
we get
\[
V'(1/2)
=
-\frac{1}{n_+}\sum_{i\in I_+}\ip{H^\star}{N_i}
+
\frac{1}{n_-}\sum_{i\in I_-}\ip{H^\star}{N_i}
=
h_-^\star-h_+^\star.
\]
Proposition~\ref{prop:selection} now gives \(a^\star = -\frac{\lambda}{2s}(h_+^\star-h_-^\star) = \frac{\lambda}{2s}V'(1/2)\).
The inequalities \(0<a^\star<1\) follow from Proposition~\ref{prop:balanced}.
Indeed, \(V'(1/2)>0\) gives \(a^\star>0\), and the strict right-derivative
inequality \(2s-\lambda V'(1/2)>0\)
is equivalent to \(a^\star<1\).
\end{proof}

\subsection{Completion of the proof}

We now return to the explicit high-probability scales.  Let
\(0<\delta<e^{-1}\), set \(L:=\log(n/\delta)\), and write
\[
\boldsymbol b:=\boldsymbol b_{n,D,\delta},
\qquad
\varepsilon:=\varepsilon_{n,D,\delta},
\qquad
\nu:=\nu_{n,D,\delta}.
\]
Throughout this subsection, the deterministic quantities
\(\mathfrak R\), \(\mathfrak D\), and \(\mathfrak E_V\) are evaluated at
these event thresholds.  Define
\begin{equation}
\label{eq:fixed-point-scales}
R_{n,\delta}
:=
\frac{\sqrt L}{n}
+\frac{L}{n^{3/2}}
+\frac{(n+L)^{3/2}}{n^{3/2}\sqrt D}.
\end{equation}
By \eqref{eq:path-event-parameters}, \eqref{eq:gram-event-parameters}, and
\eqref{eq:deterministic-R-scale}, there exists \(B_R\geq1\), depending only
on \(\kappa,r_0\), such that
\begin{equation}
\label{eq:explicit-R-comparison}
\mathfrak R
\leq
B_RR_{n,\delta},
\qquad
\left(\frac{1}{\sqrt n}+\frac{D}{n^2}b_{H,2\to\infty}\right)\mathfrak R
\leq
B_R\frac{1+L}{\sqrt n}R_{n,\delta}.
\end{equation}
Thus \(R_{n,\delta}\) is an explicit upper envelope, up to a constant, for
the deterministic localization scale \(\mathfrak R\); the second inequality
provides the corresponding \(\ell_\infty\)-scale.

\begin{theorem}[Conditional high-probability form]
\label{thm:main-conditional}
Fix \(0<\kappa<1/2\), \(s>0\), and \(\lambda>0\).  There exist constants
\(c_{\mathrm{dim}},c_{\mathrm{loc}},C_{\mathrm{sec}}>0\), depending only on
\(\kappa,r_0\), such that the following holds.  Condition on signs satisfying
\eqref{eq:conditional-label-counts}.  Let \(0<\delta<e^{-1}\), set
\(L=\log(n/\delta)\), and define \(R_{n,\delta}\) by
\eqref{eq:fixed-point-scales}.  Assume
\begin{equation}
\label{eq:conditional-basic-assumptions}
D\geq c_{\mathrm{dim}}nL,
\qquad
\mathfrak D\!\left(
C_{\mathrm{sec}}\frac{1+L}{\sqrt n}R_{n,\delta}
\right)
\leq c_{\mathrm{loc}},
\end{equation}
and
\begin{equation}
\label{eq:conditional-balanced-assumptions}
C_{\mathrm{sec}}\mathfrak E_V
\leq
\frac{1}{2D}\theta_{n_+,n_-},
\qquad
\lambda\left(
\frac{1}{D}\theta_{n_+,n_-}+C_{\mathrm{sec}}\mathfrak E_V
\right)<2s.
\end{equation}
Then, with conditional probability at least \(1-\delta\) over the spherical
noise variables, the dual spectral max-margin problem has a unique minimizer
\(\boldsymbol\alpha^\star\in\operatorname{int}(\D_n)\) with
\(q(\boldsymbol\alpha^\star)=1/2\).  Setting
\(Z^\star:=\sum_i\alpha_i^\star N_i\) and \(H^\star:=\pol(Z^\star)\), we have
\(\operatorname{rank}(Z^\star)=\operatorname{rank}(H^\star)=n\),
\(\boldsymbol\alpha_k\to\boldsymbol\alpha^\star\),
\[
\frac{W_k^{\mathrm{SpecGD}}}{\tau_k}
\longrightarrow
a^\star P_S+H^\star,
\qquad
\left|
a^\star-\frac{\lambda}{2sD}\theta_{n_+,n_-}
\right|
\leq
\frac{\lambda}{2s}C_{\mathrm{sec}}\mathfrak E_V.
\]
Moreover \(0<a^\star<1\), and
\(W_k^{\mathrm{SpecGD}}/\|W_k^{\mathrm{SpecGD}}\|_{\op}
\to a^\star P_S+H^\star\).
\end{theorem}

\begin{proof}
By increasing \(c_{\mathrm{dim}}\) if necessary,
\eqref{eq:conditional-basic-assumptions} implies the dimensional hypothesis
of Lemma~\ref{lem:path-derivatives} and
\(\varepsilon\leq\min\{\varrho_0,1/2\}\).  Hence, with conditional
probability at least \(1-\delta\),
$\mathcal E(\boldsymbol b,\varepsilon,\nu)$ holds. 
We work on this event.

Let \(c_{\mathrm{fp}},C_{\mathrm{in}}\) be the constants from
Proposition~\ref{prop:inner} and \(C_V\) from
Proposition~\ref{prop:Vprime}.  Enlarge \(C_{\mathrm{sec}}\) so that
\[
C_{\mathrm{sec}}
\geq
\max\{B_RC_{\mathrm{in}},C_V\},
\]
and choose
\[
c_{\mathrm{loc}}\leq c_{\mathrm{fp}}.
\]
By \eqref{eq:explicit-R-comparison},
\[
C_{\mathrm{in}}
\left(\frac{1}{\sqrt n}+\frac{D}{n^2}b_{H,2\to\infty}\right)
\mathfrak R
\leq
C_{\mathrm{sec}}\frac{1+L}{\sqrt n}R_{n,\delta}.
\]
Hence \eqref{eq:conditional-basic-assumptions} and monotonicity of
\(\mathfrak D(r)\) imply the deterministic smallness hypothesis of
Proposition~\ref{prop:inner}.

Proposition~\ref{prop:Vprime} and \(C_{\mathrm{sec}}\geq C_V\) give
\begin{equation}
\label{eq:conditional-Vprime-with-section-constant}
\left|V'(1/2)-\frac1D\theta_{n_+,n_-}\right|
\leq
C_{\mathrm{sec}}\mathfrak E_V.
\end{equation}
Therefore \eqref{eq:conditional-balanced-assumptions} implies the deterministic
assumptions \eqref{eq:section-balanced-assumptions}.  Proposition~\ref{prop:balanced}
then yields \(q_\star=1/2\) and a unique interior full dual minimizer
\(\boldsymbol\alpha^\star=\boldsymbol\alpha^\star(1/2)\).
Lemma~\ref{lem:positive-noise-rank} gives
\(\operatorname{rank}(Z^\star)=n\), hence also
\(\operatorname{rank}(H^\star)=n\).

By Proposition~\ref{prop:local-growth}, \(F\) has local relative growth at
\(\boldsymbol\alpha^\star\), so Theorem~\ref{thm:dual-convergence} gives
\(\boldsymbol\alpha_k\to\boldsymbol\alpha^\star\).  Proposition~\ref{prop:KKT-envelope}
gives \(a^\star=\lambda V'(1/2)/(2s)\), while
Proposition~\ref{prop:selection} yields
\(W_k^{\mathrm{SpecGD}}/\tau_k\to a^\star P_S+H^\star\).
Using \eqref{eq:conditional-Vprime-with-section-constant} gives the claimed
bound on \(a^\star\), and Proposition~\ref{prop:KKT-envelope} gives
\(0<a^\star<1\).  Finally,
\(\|a^\star P_S+H^\star\|_{\op}=1\) by signal--noise block orthogonality,
which yields the normalized directional convergence exactly as before.
\end{proof}

\begin{proof}[Proof of Theorem~\ref{thm:main_dispersed}]
The constants \(c>0\) and \(C>0\) in the theorem will be chosen successively
below.  Each time we shrink \(c\) or enlarge \(C\), we do so only in terms of
\(\rho\) and the constants from Theorem~\ref{thm:main-conditional};
hence the final constants still have the allowed dependence.

Let
\[
\mathcal E_{\mathrm{lab}}(\delta)
:=
\left\{
\left|\frac{n_-}{n}-\rho\right|
\leq
\sqrt{\frac{\log(4/\delta)}{2n}}
\right\}.
\]
By Hoeffding's inequality, \(\prob(\mathcal E_{\mathrm{lab}}(\delta))\geq 1-\frac{\delta}{2}\).
By taking \(c>0\) sufficiently small, depending only on \(\rho\), the assumption
\(L^5\leq cn\) implies that on \(\mathcal E_{\mathrm{lab}}(\delta)\),
\(\frac{\rho}{2}n\leq n_-\leq n_+\leq \left(1-\frac{\rho}{2}\right)n, \quad n_+-n_-\geq \frac{1-2\rho}{2}n\).
Thus \eqref{eq:conditional-label-counts} holds with \(\kappa=\frac{\rho}{2}\).

Condition on a realization of the corruption signs in
\(\mathcal E_{\mathrm{lab}}(\delta)\).  We apply
Theorem~\ref{thm:main-conditional} with failure probability \(\delta/2\).  Let
\(L_2:=\log\frac{2n}{\delta}\).  Then \(L_2\leq 2L\).  In the remainder of
this proof, \(\boldsymbol b,\varepsilon,\nu,\mathfrak R,\mathfrak D\), and
\(\mathfrak E_V\) are evaluated with failure probability \(\delta/2\).
Let \(c_{\mathrm{dim}},c_{\mathrm{loc}},C_{\mathrm{sec}}\) be the constants in
Theorem~\ref{thm:main-conditional} corresponding to
\(\kappa=\rho/2\).  It remains to verify
\eqref{eq:conditional-basic-assumptions} and
\eqref{eq:conditional-balanced-assumptions}, with \(\delta\) replaced by
\(\delta/2\).

First, by choosing \(c\leq1\), \(L^5\leq cn\) implies \(L\leq n\), and hence
\(n+L_2\leq3n\).  From \eqref{eq:fixed-point-scales}, there is a universal
constant \(C_R>0\) such that
\begin{equation}\label{eq:main-proof-R2-bound}
R_{n,\delta/2}
\leq
C_R\left(
\frac{\sqrt L}{n}
+
\frac{1}{\sqrt D}
\right).
\end{equation}
Consequently,
\begin{equation}\label{eq:main-proof-Rinf-bound}
\frac{1+L_2}{\sqrt n}R_{n,\delta/2}
\leq
C_R\left(
\frac{L^{3/2}}{n^{3/2}}
+
\frac{L}{\sqrt{nD}}
\right).
\end{equation}

By the definitions of \(\mathfrak D\), the event thresholds, and
\eqref{eq:main-proof-Rinf-bound}, there exists \(C_D>0\), depending only on
\(\rho,r_0\), such that
\begin{equation}\label{eq:main-proof-D2-bound}
\mathfrak D\!\left(
C_{\mathrm{sec}}\frac{1+L_2}{\sqrt n}R_{n,\delta/2}
\right)
\leq
C_D\left(
\frac{L^{5/2}}{\sqrt n}
+
L^2\sqrt{\frac{n}{D}}
+
\sqrt{\frac{n}{D}}
\right).
\end{equation}
Again, \(L^5\leq cn\) and \(D\geq CnL^4\) make the right-hand side at most
\(c_{\mathrm{loc}}\) after decreasing \(c\) and increasing \(C\).  Also,
increasing \(C\) if necessary gives \(D\geq c_{\mathrm{dim}}nL_2\).
Thus \eqref{eq:conditional-basic-assumptions} holds with failure probability
\(\delta/2\).

Next, under the same assumptions, the deterministic derivative error satisfies
\begin{equation}\label{eq:EV-readable-bound}
\mathfrak E_V
\leq
C_E\frac{n}{D}
\left(
\frac{L^3}{n}
+
\sqrt{\frac{n}{D}}
\right),
\end{equation}
where \(C_E>0\) depends only on \(\rho,r_0\).  Indeed, this follows term by
term from \eqref{eq:deterministic-derivative-scale},
\eqref{eq:explicit-R-comparison}, \eqref{eq:main-proof-R2-bound}, and
\eqref{eq:main-proof-Rinf-bound}.  The mixed terms are controlled using
\(L^5\leq cn\) and \(D\geq CnL^4\).

On \(\mathcal E_{\mathrm{lab}}(\delta)\), the conditional label-count bounds
imply
\begin{equation}\label{eq:theta-upper-lower-main}
c_\rho n
\leq
\theta_{n_+,n_-}
\leq
C_\rho n
\end{equation}
for constants \(c_\rho,C_\rho>0\) depending only on \(\rho\).  Combining the
lower bound in \eqref{eq:theta-upper-lower-main} with
\eqref{eq:EV-readable-bound}, and then taking \(c>0\) sufficiently small and
\(C>0\) sufficiently large, gives
\(C_{\mathrm{sec}}\mathfrak E_V\leq\frac{1}{2D}\theta_{n_+,n_-}\).
For the second inequality in \eqref{eq:conditional-balanced-assumptions}, use
the first inequality just proved and the upper bound in
\eqref{eq:theta-upper-lower-main}.  We obtain
\[
\frac{1}{D}\theta_{n_+,n_-}
+
C_{\mathrm{sec}}\mathfrak E_V
\leq
\frac{3}{2D}\theta_{n_+,n_-}
\leq
\frac{3}{2}C_\rho\frac nD,
\]
where \(C_\rho>0\) depends only on \(\rho\).  Hence, by the assumption
\(\frac{\lambda n}{sD}\leq c\), and by choosing \(c>0\) sufficiently small
depending only on \(\rho\), we get
\[
\lambda\left(
\frac{1}{D}\theta_{n_+,n_-}
+
C_{\mathrm{sec}}\mathfrak E_V
\right)
<
2s.
\]

Thus all assumptions of Theorem~\ref{thm:main-conditional} hold with
\(\kappa=\rho/2\) and failure probability \(\delta/2\).  Conditional on
\(\mathcal E_{\mathrm{lab}}(\delta)\), the conclusions of
Theorem~\ref{thm:main-conditional} hold with probability at least
\(1-\delta/2\).  A union bound gives probability at least \(1-\delta\) over the
joint randomness.

It remains only to pass from the realized-count expression to the displayed
main-theorem expression.  Write \(p_n:=\frac{n_+}{n}\).  Then
\begin{equation}\label{eq:theta-polynomial-main}
\frac{\theta_{n_+,n_-}}{n}
=
\frac{1}{4}(2p_n-1)
\left[
1+4p_n(1-p_n)
\right].
\end{equation}
On \(\mathcal E_{\mathrm{lab}}(\delta)\),
\(|p_n-(1-\rho)|=\left|\frac{n_-}{n}-\rho\right|
\leq\sqrt{\frac{\log(4/\delta)}{2n}}\).
Since the function
\[
p\longmapsto
\frac{1}{4}(2p-1)
\left[
1+4p(1-p)
\right]
\]
has bounded derivative on \([0,1]\), there exists a universal constant
\(C_\theta\) such that, by \eqref{eq:theta-polynomial-main},
\begin{equation}\label{eq:theta-realized-to-population-main}
\left|
\frac{\theta_{n_+,n_-}}{n}
-
\frac14(1-2\rho)
\left[
1+4\rho(1-\rho)
\right]
\right|
\leq
C_\theta\sqrt{\frac{\log(4/\delta)}{n}}.
\end{equation}

By Theorem~\ref{thm:main-conditional},
\[
\left|
a^\star
-
\frac{\lambda}{2sD}\theta_{n_+,n_-}
\right|
\leq
\frac{\lambda}{2s}C_{\mathrm{sec}}\mathfrak E_V.
\]
Combining this with \eqref{eq:EV-readable-bound} and
\eqref{eq:theta-realized-to-population-main} and increasing \(C\) if necessary
yields
\[
\left|
a^\star
-
\frac{\lambda n}{8sD}
(1-2\rho)
\left[
1+4\rho(1-\rho)
\right]
\right|
\leq
C\frac{\lambda n}{sD}
\left(
\sqrt{\frac{\log(1/\delta)}{n}}
+
\frac{L^3}{n}
+
\sqrt{\frac{n}{D}}
\right).
\]
This is \eqref{eq:a-main-high-probability}. The remaining claims are exactly
the conclusions of Theorem~\ref{thm:main-conditional}.
\end{proof}

\subsection{Completion of the proof of Corollary~\ref{cor:generalization-dispersed}}

\begin{proof}[Proof of 
Corollary~\ref{cor:generalization-dispersed}]

First choose $c, C>0$ so that Theorem~\ref{thm:main_dispersed} holds. On the event from Theorem~\ref{thm:main_dispersed}, $W^\star=a^\star P_S+H^\star$, $\operatorname{rank}(H^\star)=n$, $0<a^\star<1$.
For a fresh test point,
\[
y\langle W^\star,X\rangle_F
=
y\left\langle a^\star P_S+H^\star,
yS+\lambda \boldsymbol u\boldsymbol v^\top
\right\rangle_F .
\]
Using the signal-noise block orthogonality, \(\langle P_S,\boldsymbol u\boldsymbol v^\top\rangle_F=0, \quad \langle H^\star,S\rangle_F=0\).
Therefore
\[
y\langle W^\star,X\rangle_F
=
a^\star \langle P_S,S\rangle_F
+
\lambda y\langle H^\star,\boldsymbol u\boldsymbol v^\top\rangle_F
=
a^\star s
+
\lambda y\langle H^\star,\boldsymbol u\boldsymbol v^\top\rangle_F.
\]
Let
\[
T:=\langle H^\star,\boldsymbol u\boldsymbol v^\top\rangle_F,
\qquad
r:=\frac{a^\star s}{\lambda}.
\]
Conditioning further on \(y\), the error event is \(\{T\leq-r\}\) if \(y=1\)
and \(\{T\geq r\}\) if \(y=-1\).  Since, conditionally on the training data,
\(T\) is symmetric about zero, these two probabilities are equal.  Hence
\[
\mathcal R(W^\star)
=
\frac12
\prob\left(
\left|
\langle H^\star,\boldsymbol u\boldsymbol v^\top\rangle_F
\right|
\geq
\frac{a^\star s}{\lambda}
\,\middle|\,
\mathcal D_n
\right).
\]
Applying Lemma~\ref{lem:fresh-bilinear-tail} with \(H=H^\star\) and
\(r_H=n\), we obtain
\begin{equation}
\label{eq:generalization-a-star-bound}
\mathcal R(W^\star)
\leq
2\exp\left[
-C^\prime \min\left\{
\frac{D^2}{n}
\left(\frac{a^\star s}{\lambda}\right)^2,
D\frac{a^\star s}{\lambda}
\right\}
\right].
\end{equation}

It remains to derive the simplified bound.  By Theorem~\ref{thm:main_dispersed},
\[
\left|
a^\star
-
\frac{\lambda n}{sD}\Gamma_\rho
\right|
\leq
C\frac{\lambda n}{sD}
\left(
\sqrt{\frac{\log(1/\delta)}{n}}
+
\frac{L^3}{n}
+
\sqrt{\frac nD}
\right).
\]
Thus, by replacing $c$ with a sufficiently small constant and  $C$ with a sufficiently large constant, we have \(\frac{\lambda n}{sD}\frac{\Gamma_\rho}{2} \leq a^\star \leq \frac{\lambda n}{sD}\frac{3\Gamma_\rho}{2}\).
Equivalently, \(\frac{\Gamma_\rho}{2}\frac{n}{D} \leq \frac{a^\star s}{\lambda}  \leq \frac{3\Gamma_\rho}{2}\frac{n}{D} \).
Substituting this lower bound into
\eqref{eq:generalization-a-star-bound} gives
\[
\mathcal R(W^\star)
\leq
2\exp\left[
-C^\prime\min\left\{
\Gamma_\rho^2 n,
\Gamma_\rho n
\right\}
\right] = 2\exp(-C^\prime\Gamma_\rho^2 n)
\]
after renaming $C^\prime$, where the equality is due to $\Gamma_\rho <1/4<1$ on $\rho \in (0, 1/2)$. Setting $C_0 := C^\prime \Gamma_\rho^2$, we obtain the desired upper bound.

For the lower bound, let \(K_0\) be the universal constant from
Lemma~\ref{lem:fresh-bilinear-tail}. By replacing $C$ with a larger constant if necessary so that $C\geq K_0$, we have $D \geq Cn \geq K_0n$.

Since 
$\frac{a^\star s}{\lambda}
\leq
\frac{3\Gamma_\rho}{2}\frac nD$
and
\(\Gamma_\rho<1/4\), we have
\[
\frac{a^\star s}{\lambda}
<
\frac nD .
\]
Therefore the lower-tail estimate \eqref{eq:fresh-bilinear-lower-tail} applies
with \(H=H^\star\), \(r_H=n\), and \(t=a^\star s/\lambda\).  Combining this
with the tail identity gives
\[
\mathcal R(W^\star)
\geq
\frac{c_0}{2}
\exp\left[
-C_0^\prime\frac{D^2}{n}
\left(\frac{a^\star s}{\lambda}\right)^2
\right].
\]
Using again
\[
\frac{a^\star s}{\lambda}
\leq
\frac{3\Gamma_\rho}{2}\frac nD ,
\]
we obtain
\[
\mathcal R(W^\star)
\geq
\frac{c_0}{2}
\exp\left[
-\frac94 C_0^\prime\Gamma_\rho^2 n
\right].
\]
Renaming the universal constant $c_0$ and setting $C_0 = \frac{9}{4}C_0^\prime \Gamma_\rho^2$ proves
$\mathcal R(W^\star) \geq c_0e^{-C_0 n}$.

Finally, Theorem~\ref{thm:main_dispersed} gives \(\frac{W_k^{\mathrm{SpecGD}}}{\|W_k^{\mathrm{SpecGD}}\|_{\op}}\to W^\star\).
The classifier is invariant under positive rescaling of \(W_k^{\mathrm{SpecGD}}\).  Moreover, for
fixed training data, the fresh-noise distribution is continuous, so \(\prob\left( y\langle W^\star,X\rangle_F=0 \,\middle|\, \mathcal D_n \right)=0\).
Therefore the indicators \(\mathbf 1\{y\langle W_k^{\mathrm{SpecGD}},X\rangle_F\leq0\}\)
converge almost surely to \(\mathbf 1\{y\langle W^\star,X\rangle_F\leq0\}\),
and dominated convergence gives \(\mathcal R(W_k^{\mathrm{SpecGD}})\to \mathcal R(W^\star)\).
\end{proof}

\subsection{Proof of Proposition~\ref{prop:implicitbias-orthogonal}}\label{sec:proof-prop-orthogonal}

\begin{proof}
Set \(a_{\star\star}:=(2\|S\|_1)^{-1}\log(n_+/n_-)\).  For
\(W\in\mathbb R^{d\times d}\), let
\(m_i(W):=\langle W,\widetilde y_iX_i\rangle_F
=\langle W,\xi_iS+\lambda N_i\rangle_F\) and \(q_i(W):=e^{-m_i(W)}\).
At iteration \(k\), write \(q_{i,k}:=q_i(W_k)\) and
\(s_k:=\sum_i\xi_iq_{i,k}\).  Then
\[
-\nabla L(W_k)
=
\frac1n
\left(
s_kS+\lambda\sum_{i=1}^nq_{i,k}N_i
\right).
\]
The factor \(1/n\) is irrelevant for the polar factor.  Since the signal block
is orthogonal to the noise block, and the \(N_i\)'s have mutually orthogonal
left and right singular directions, \(q_{i,k}>0\) implies
\[
\pol(-\nabla L(W_k))
=
\operatorname{sgn}(s_k)P_S+\sum_{i=1}^nN_i ,
\]
where \(\operatorname{sgn}(0)=0\).

Thus, starting from \(W_0=0\), the iterates have the exact form
\[
W_k=a_kP_S+\tau_k\sum_{i=1}^nN_i,
\qquad
\tau_k:=\sum_{\ell<k}\eta_\ell,
\]
where \(a_0=0\) and \(a_{k+1}=a_k+\eta_k\operatorname{sgn}(s_k)\).  Using this
representation and orthogonality,
\[
m_i(W_k)=\xi_i\|S\|_1a_k+\lambda\tau_k,
\]
so
\[
s_k
=
e^{-\lambda\tau_k}
\left(
n_+e^{-\|S\|_1a_k}-n_-e^{\|S\|_1a_k}
\right).
\]
Hence
\[
\operatorname{sgn}(s_k)=\operatorname{sgn}(a_{\star\star}-a_k),
\]
and therefore the signal coefficient obeys the scalar recursion
\[
a_{k+1}=a_k+\eta_k\operatorname{sgn}(a_{\star\star}-a_k).
\]

Let \(K:=\sup\{k:\tau_k<a_{\star\star}\}\).  Since
\(\sum_k\eta_k=\infty\), this \(K\) is finite.  Before crossing
\(a_{\star\star}\), the recursion gives \(a_k=\tau_k\), and hence
\(|a_K-a_{\star\star}|\leq\eta_K\).  Since the stepsizes are non-increasing,
the recursion moves toward \(a_{\star\star}\) by steps of size at most
\(\eta_K\) for all \(k\geq K\).  Consequently
\[
|a_k-a_{\star\star}|\leq\eta_K,
\qquad k\geq K.
\]
Thus \(|a_k|\leq a_{\star\star}+\eta_K\) for \(k\geq K\).  Dividing the exact
representation of \(W_k\) by \(\tau_k\) gives
\[
\left\|
\frac{W_k^{\mathrm{SpecGD}}}{\tau_k}
-
\sum_{i=1}^nN_i
\right\|_F
=
\frac{|a_k|}{\tau_k}\|P_S\|_F
\leq
\frac{(a_{\star\star}+\eta_K)\sqrt r}{\tau_k},
\qquad k\geq K.
\]
Since \(\tau_k\to\infty\), the claimed directional convergence follows.
\end{proof}

%% file: proof_main_collapsed.tex
We will use the strategy developed in Section~\ref{app:main_dispersed} for the DSM.

\begin{proof}[Proof of Theorem~\ref{thm:main_collapsed}]
Absorbing $\tilde y_i$ o $\boldsymbol{u}_i$, let
$U:=[\boldsymbol u_1,\ldots,\boldsymbol u_n]\in\mathbb R^{D\times n}$,
$G:=U^\top U$, and
$\boldsymbol\xi:=(\xi_1,\ldots,\xi_n)^\top$.

\paragraph{Step 1: exact dual objective.}
By Proposition~\ref{prop:dual}, the dual objective is
\[
F(\boldsymbol\alpha)
=
\left\|
\sum_{i=1}^n
\alpha_i\widetilde y_iX_i
\right\|_1,
\qquad
\boldsymbol\alpha\in\Delta_n.
\]
Under the CSM,
\[
\sum_{i=1}^n
\alpha_i\widetilde y_iX_i
=
(\boldsymbol\xi^\top\boldsymbol\alpha)S
+
\lambda(U\boldsymbol\alpha)\boldsymbol v_0^\top.
\]
The two terms have mutually orthogonal row and column spaces, so their
nuclear norms add. Moreover,
$(U\boldsymbol\alpha)\boldsymbol v_0^\top$ has rank one and, since
$\|\boldsymbol v_0\|_2=1$,
\[
\|(U\boldsymbol\alpha)\boldsymbol v_0^\top\|_1
=
\|U\boldsymbol\alpha\|_2
=
\sqrt{\boldsymbol\alpha^\top G\boldsymbol\alpha}.
\]
By Lemma~\ref{lem:block} and $\|S\|_1=s$, 
\begin{equation}
\label{eq:collapsed-dual-objective}
F(\boldsymbol\alpha)
=
s|\boldsymbol\xi^\top\boldsymbol\alpha|
+
\lambda\sqrt{\boldsymbol\alpha^\top G\boldsymbol\alpha}.
\end{equation}

\paragraph{Step 2: high-probability geometry.}
Let $\widehat\rho:=n_-/n$. By Hoeffding's inequality, with probability
at least $1-\delta/3$,
\begin{equation}
\label{eq:collapsed-label-balance}
|\widehat\rho-\rho|
\leq
\sqrt{\frac{\log(6/\delta)}{2n}}.
\end{equation}

By Lemma~\ref{lem:gram-concentration}, applied with failure probability
$\delta/3$ and enlarging the universal constant $C_G$, there exists a universal constant $K>0$ such that
\begin{equation}
\label{eq:collapsed-gram-control}
\|G-I_n\|_{\op}
\leq
K
\sqrt{
\frac{n+\log(n/\delta)}{D}
}.
\end{equation}
Increasing the dimensional constant if necessary, the right-hand side
is at most $1/2$, and therefore
\begin{equation}
\label{eq:collapsed-inverse-op-control}
\|G^{-1}-I_n\|_{\op}
\leq
K
\sqrt{
\frac{n+\log(n/\delta)}{D}
},
\end{equation}
after enlarging the same universal constant $K$ if necessary.

By Lemma~\ref{lem:inverse-gram-concentration}, applied with failure
probability $\delta/3$, and enlarging $K$ once more if necessary,
\begin{equation}
\label{eq:collapsed-inverse-control}
\max_{\boldsymbol x\in\{\boldsymbol 1,\boldsymbol\xi\}}
\|(G^{-1}-I_n)\boldsymbol x\|_\infty
\leq
K
\sqrt{
\frac{n\log(n/\delta)}{D}
}.
\end{equation}

Let $\mathcal E$ be the intersection of these three events. Then $\mathbb P(\mathcal E)\geq1-\delta$, and we work on $\mathcal E$ throughout. 
Define
\[
\varepsilon_{\op}
:=
K
\sqrt{
\frac{n+\log(n/\delta)}{D}
},
\qquad
\varepsilon_\infty
:=
K
\sqrt{
\frac{n\log(n/\delta)}{D}
}.
\]
By increasing the constants in the assumptions if necessary,
$\widehat\rho$ lies in a fixed compact subinterval of $(0,1/2)$
depending only on $\rho$.

\paragraph{Step 3: solution of the dual problem.}
For $\boldsymbol\alpha\in\Delta_n$, let
$q(\boldsymbol\alpha):=\sum_{i\in I_+}\alpha_i$, so that
$\boldsymbol\xi^\top\boldsymbol\alpha=2q(\boldsymbol\alpha)-1$, and define
\[
\psi(q)
:=
\min_{\substack{
\boldsymbol\alpha\in\Delta_n\\
q(\boldsymbol\alpha)=q
}}
\sqrt{\boldsymbol\alpha^\top G\boldsymbol\alpha}.
\]
Set
\[
A:=\boldsymbol 1^\top G^{-1}\boldsymbol 1,
\qquad
B:=\boldsymbol 1^\top G^{-1}\boldsymbol\xi,
\qquad
C:=\boldsymbol\xi^\top G^{-1}\boldsymbol\xi,
\qquad
\Delta:=AC-B^2.
\]
The operator-norm bound in Step~2 immediately gives
\begin{equation}
\label{eq:collapsed-ABC-control}
\left|\frac{A}{n}-1\right|
\vee
\left|
\frac{B}{n}-(1-2\widehat\rho)
\right|
\vee
\left|\frac{C}{n}-1\right|
\leq
\varepsilon_{\op}.
\end{equation}

We first solve the inner problem.

\begin{lemma}
\label{lem:collapsed-inner-problem}
On $\mathcal E$, for every
$q\in[1/2,\,1/2+B/(2A)]$, the minimizer defining $\psi(q)$ is unique,
belongs to $\operatorname{int}(\Delta_n)$, and is given by
\begin{equation}
\label{eq:collapsed-alpha-q}
\boldsymbol\alpha^\star(q)
=
\frac{
G^{-1}
\left[
\{C-(2q-1)B\}\boldsymbol 1
+
\{(2q-1)A-B\}\boldsymbol\xi
\right]
}{\Delta}.
\end{equation}
Moreover,
\begin{equation}
\label{eq:collapsed-psi}
\psi(q)^2
=
\frac{
C-2(2q-1)B+(2q-1)^2A
}{\Delta}.
\end{equation}
\end{lemma}

\begin{proof}
Since $G\succ0$, consider first the equality-constrained problem
\[
\min_{\boldsymbol\alpha\in\mathbb R^n}
\boldsymbol\alpha^\top G\boldsymbol\alpha
\quad\text{subject to}\quad
\boldsymbol 1^\top\boldsymbol\alpha=1,
\qquad
\boldsymbol\xi^\top\boldsymbol\alpha=2q-1.
\]
Its objective is strictly convex. Introducing Lagrange multipliers
$a,b\in\mathbb R$, the stationarity condition gives
\[
G\boldsymbol\alpha
=
a\boldsymbol 1+b\boldsymbol\xi,
\]
and hence
$\boldsymbol\alpha
=
G^{-1}(a\boldsymbol 1+b\boldsymbol\xi)$.
The two constraints therefore imply
\[
\begin{pmatrix}
A&B\\
B&C
\end{pmatrix}
\begin{pmatrix}
a\\
b
\end{pmatrix}
=
\begin{pmatrix}
1\\
2q-1
\end{pmatrix}.
\]
On $\mathcal E$, both $I_+$ and $I_-$ are nonempty, so
$\boldsymbol 1$ and $\boldsymbol\xi$ are linearly independent.
Since $G^{-1}\succ0$, the matrix above is their Gram matrix in the
$G^{-1}$-inner product and hence
$\Delta=AC-B^2>0$. Thus
\[
a
=
\frac{C-(2q-1)B}{\Delta},
\qquad
b
=
\frac{(2q-1)A-B}{\Delta},
\]
which gives \eqref{eq:collapsed-alpha-q}. Moreover,
using
$G\boldsymbol\alpha^\star
=
a\boldsymbol 1+b\boldsymbol\xi$ and the two constraints,
\[
\boldsymbol\alpha^{\star\top}
G
\boldsymbol\alpha^\star
=
a+b(2q-1)
=
\frac{
C-2(2q-1)B+(2q-1)^2A
}{\Delta},
\]
proving \eqref{eq:collapsed-psi}.

It remains to verify that
$\boldsymbol\alpha^\star(q)$ is strictly positive. By
\eqref{eq:collapsed-ABC-control},
\[
\frac12+\frac{B}{2A}
=
1-\widehat\rho+O_\rho(\varepsilon_{\op}).
\]
Hence, after increasing the dimensional constant if necessary,
$1-q\geq\rho/2$ throughout the stated interval.

When $G=I_n$, the corresponding equality-constrained minimizer is
\[
\overline\alpha_{0,i}(q)
=
\begin{cases}
q/n_+, & i\in I_+,\\[1mm]
(1-q)/n_-, & i\in I_-,
\end{cases}
\]
and therefore
\[
\min_{i\in[n]}\overline\alpha_{0,i}(q)
\geq
\frac{\rho}{2n}.
\]

Let $a_0(q),b_0(q)$ denote the corresponding Lagrange multipliers when
$G=I_n$, so that
$\overline{\boldsymbol\alpha}_0(q)
=a_0(q)\boldsymbol 1+b_0(q)\boldsymbol\xi$.
Since $\widehat\rho$ stays in a fixed compact subset of $(0,1/2)$,
\eqref{eq:collapsed-ABC-control} and the explicit formulas for $a,b$
give, uniformly over the stated interval,
\[
|a-a_0(q)|+|b-b_0(q)|
\leq
K_\rho\frac{\varepsilon_{\op}}{n},
\qquad
|a|+|b|
\leq
\frac{K_\rho}{n},
\]
for a constant $K_\rho>0$ depending only on $\rho$. Therefore
\[
\begin{aligned}
\boldsymbol\alpha^\star(q)
-
\overline{\boldsymbol\alpha}_0(q)
&=
(G^{-1}-I_n)
(a\boldsymbol 1+b\boldsymbol\xi)
\\
&\quad
+
(a-a_0(q))\boldsymbol 1
+
(b-b_0(q))\boldsymbol\xi.
\end{aligned}
\]
Using the inverse-Gram coordinate bound from Step~2,
\[
\|
\boldsymbol\alpha^\star(q)
-
\overline{\boldsymbol\alpha}_0(q)
\|_\infty
\leq
K_\rho
\frac{
\varepsilon_\infty+\varepsilon_{\op}
}{n}
\leq
2K_\rho\frac{\varepsilon_\infty}{n}.
\]
Increasing the dimensional constant once more if necessary, the
right-hand side is at most $\rho/(4n)$. Hence
\[
\min_{i\in[n]}\alpha_i^\star(q)
\geq
\frac{\rho}{4n}>0.
\]
Thus the equality-constrained minimizer belongs to
$\operatorname{int}(\Delta_n)$ and is therefore the unique minimizer
defining $\psi(q)$.
\end{proof}

By \eqref{eq:collapsed-dual-objective}, the dual problem reduces to
\[
\min_{q\in[0,1]}
\left\{
2s\left|q-\frac12\right|
+
\lambda\psi(q)
\right\}.
\]
From \eqref{eq:collapsed-psi},
\[
\psi'(q)
=
\frac{
2\{A(2q-1)-B\}
}{
\Delta\psi(q)
},
\]
and hence $\psi$ is uniquely minimized at $q_{\rm sh}
:= \dfrac12+\dfrac{B}{2A}$.
By \eqref{eq:collapsed-ABC-control}, $B>0$, so $q_{\rm sh}>1/2$.

We next localize the minimizer of the reduced dual problem. If
$q<1/2$, increasing $q$ toward $1/2$ decreases both
$|q-1/2|$ and $\psi(q)$, since $q<1/2<q_{\rm sh}$. Similarly, if
$q>q_{\rm sh}$, decreasing $q$ toward $q_{\rm sh}$ decreases both
terms. Therefore every minimizer belongs to
$[1/2,q_{\rm sh}]$.

On this interval, differentiating once more gives
\[
\psi''(q)
=
\frac{4}{\Delta\psi(q)^3}
>0.
\]
Thus the reduced objective is strictly convex on
$[1/2,q_{\rm sh}]$ and admits a unique minimizer, denoted by
$q^\star$. Set
$\boldsymbol\alpha^\star
:=
\boldsymbol\alpha^\star(q^\star)$.
By Lemma~\ref{lem:collapsed-inner-problem},
$\boldsymbol\alpha^\star\in\operatorname{int}(\Delta_n)$ and is the
unique minimizer of $F$.

Finally,
\[
\psi'\left(\frac12\right)
=
-\frac{2B}{\sqrt{C\Delta}}.
\]
Since the right derivative of the reduced objective at $q=1/2$ is $2s - \frac{2\lambda B}{\sqrt{C\Delta}}$,
we obtain
\begin{equation}
\label{eq:collapsed-balanced-condition}
q^\star=\frac12
\quad\Longleftrightarrow\quad
\frac{\lambda B}{s\sqrt{C\Delta}}
\leq1;
\end{equation}
otherwise $q^\star\in(1/2,q_{\rm sh})$.

\paragraph{Step 4: convergence of the dual weights.}
By Step~3, $F$ has a unique minimizer
$\boldsymbol\alpha^\star\in\operatorname{int}(\Delta_n)$.
By Theorem~\ref{thm:dual-convergence}, it remains only to verify local relative growth at $\boldsymbol\alpha^\star$.

Set
$r(\boldsymbol\alpha)
:=
\sqrt{\boldsymbol\alpha^\top G\boldsymbol\alpha}$.
For every $\boldsymbol\alpha\in\Delta_n$ and every
$\boldsymbol v$ satisfying $\boldsymbol 1^\top\boldsymbol v=0$,
\[
\boldsymbol v^\top
\nabla^2r(\boldsymbol\alpha)
\boldsymbol v
=
\frac{
(\boldsymbol v^\top G\boldsymbol v)
(\boldsymbol\alpha^\top G\boldsymbol\alpha)
-
(\boldsymbol v^\top G\boldsymbol\alpha)^2
}{
(\boldsymbol\alpha^\top G\boldsymbol\alpha)^{3/2}
}.
\]
By Cauchy--Schwarz in the $G$-inner product, this quantity is
nonnegative, with equality only if
$\boldsymbol v$ and $\boldsymbol\alpha$ are collinear. Since
$\boldsymbol 1^\top\boldsymbol v=0$ while
$\boldsymbol 1^\top\boldsymbol\alpha=1$, equality is impossible for
$\boldsymbol v\neq0$. Hence, by continuity, there exist a neighborhood
$\mathcal U$ of $\boldsymbol\alpha^\star$ and $\mu_0>0$ such that
\[
\boldsymbol v^\top
\nabla^2r(\boldsymbol\alpha)
\boldsymbol v
\geq
\mu_0\|\boldsymbol v\|_2^2
\]
for every $\boldsymbol\alpha\in\mathcal U$ and every tangent vector
$\boldsymbol v$. In other words, $r(\boldsymbol \alpha)$ is $\mu_0$-strongly convex on $\mathcal U$ relative to $\Delta_n$.

Since
$\boldsymbol\alpha\mapsto
s|\boldsymbol\xi^\top\boldsymbol\alpha|$
is convex, it follows that
\[
F(\boldsymbol\alpha)-F(\boldsymbol\alpha^\star)
\geq
\frac{\lambda\mu_0}{2}
\|
\boldsymbol\alpha-\boldsymbol\alpha^\star
\|_2^2,
\qquad
\boldsymbol\alpha\in\mathcal U.
\]
Because
$\boldsymbol\alpha^\star\in\operatorname{int}(\Delta_n)$,
Lemma~\ref{lem:local-kl-quadratic} gives, after shrinking
$\mathcal U$ if necessary,
\[
\KL(
\boldsymbol\alpha^\star
\|
\boldsymbol\alpha
)
\leq
C_{\rm KL}
\|
\boldsymbol\alpha-\boldsymbol\alpha^\star
\|_2^2.
\]
Thus $F$ has local relative growth at $\boldsymbol\alpha^\star$. Let $\boldsymbol\alpha_k$ be the normalized loss weights as in \eqref{eq:alpha-def}.
Theorem~\ref{thm:dual-convergence} yields
\begin{equation}
\label{eq:collapsed-dual-convergence}
\boldsymbol\alpha_k
\longrightarrow
\boldsymbol\alpha^\star.
\end{equation}

\paragraph{Step 5: characterization of the primal direction.}
Define
\[
H(\boldsymbol\alpha)
:=
\frac{U\boldsymbol\alpha}
{\sqrt{\boldsymbol\alpha^\top G\boldsymbol\alpha}}
\boldsymbol v_0^\top.
\]
By Step~4,
$\boldsymbol\alpha_k\to\boldsymbol\alpha^\star
=\boldsymbol\alpha^\star(q^\star)$, and hence
\begin{equation}\label{eq:h_star_specgd}
H_k
\longrightarrow
H^\star
:=
\frac{U\boldsymbol\alpha^\star}
{\sqrt{\boldsymbol\alpha^{\star\top}
G\boldsymbol\alpha^\star}}
\boldsymbol v_0^\top.    
\end{equation}
In particular, $H^\star$ has rank one and its unique nonzero singular
value is $1$.

Since the dual minimizer is unique and interior and $F$ has local
relative growth, Proposition~\ref{prop:selection} gives
\[
\frac{W_k^{\mathrm{SpecGD}}}{\tau_k}
\longrightarrow
a^\star P_S+H^\star,
\qquad
a^\star
=
-\frac{\lambda}{2s}
(h_+^\star-h_-^\star).
\]

For $t:=2q-1$, \eqref{eq:collapsed-alpha-q} gives
\[
G\boldsymbol\alpha^\star(q)
=
\frac{
\{C-tB\}\boldsymbol 1
+
\{tA-B\}\boldsymbol\xi
}{\Delta}.
\]
Since
$\boldsymbol\alpha^\star(q)^\top
G\boldsymbol\alpha^\star(q)=\psi(q)^2$,
we obtain
\[
h_+(q)-h_-(q)
=
\frac{2\{(2q-1)A-B\}}{\Delta\psi(q)}
=
\psi'(q).
\]
Consequently,
\[
a^\star
=
-\frac{\lambda}{2s}\psi'(q^\star).
\]

Since $q^\star$ minimizes
$2s(q-1/2)+\lambda\psi(q)$ over $[1/2,q_{\rm sh}]$ and
$\psi'$ is strictly increasing,
\[
-\frac{\lambda}{2s}\psi'(q^\star)
=
\min\left\{
1,\,
-\frac{\lambda}{2s}\psi'(1/2)
\right\}.
\]
Using
\[
\psi'(1/2)
=
-\frac{2B}{\sqrt{C\Delta}},
\]
we conclude that
\begin{equation}
\label{eq:csm-a-star}
a^\star
=
\min\left\{
1,\,
\frac{\lambda B}{s\sqrt{C\Delta}}
\right\}.
\end{equation}

\paragraph{Step 6: approximation of the signal coefficient.}
Define
\[
\Phi(a,b,c)
:=
\frac{b}{\sqrt{c(ac-b^2)}}.
\]
Since $\widehat\rho$ remains in a fixed compact subset of $(0,1/2)$,
$\Phi$ is Lipschitz on a fixed neighborhood of
$(1,1-2\widehat\rho,1)$. Noting that
\[
\sqrt n\,
\frac{B}{\sqrt{C\Delta}}
=
\Phi
\left(
\frac{A}{n},
\frac{B}{n},
\frac{C}{n}
\right),
\]
we obtain, by \eqref{eq:collapsed-ABC-control},
\[
\left|
\sqrt n\,
\frac{B}{\sqrt{C\Delta}}
-
\frac{1-2\widehat\rho}
{2\sqrt{\widehat\rho(1-\widehat\rho)}}
\right|
\leq
K_\rho\varepsilon_{\op}
\leq
K_\rho K\sqrt{\frac{n + \log(n/\delta)}{D}}.
\]
By \eqref{eq:collapsed-label-balance} and the Lipschitz continuity of
$q\mapsto(1-2q)/(2\sqrt{q(1-q)})$ on a fixed neighborhood of $\rho$,
\[
\left|
\sqrt n\,
\frac{B}{\sqrt{C\Delta}}
-
\Gamma_\rho^\prime
\right|
\leq
K_\rho
\left(
\sqrt{\frac{\log(1/\delta)}{n}}
+
\sqrt{\frac{n + \log(n/\delta)}{D}}
\right).
\]
Finally, since $x\mapsto\min\{1,x\}$ is $1$-Lipschitz,
\eqref{eq:csm-a-star} yields
\[
\left|
a^\star
-
\min
\left\{
1,\,
\Gamma_\rho^\prime
\frac{\lambda}{s\sqrt n}
\right\}
\right|
\leq
K_\rho
\frac{\lambda}{s\sqrt n}
\left(
\sqrt{\frac{\log(1/\delta)}{n}}
+
\sqrt{\frac{n + \log(n/\delta)}{D}}
\right).
\]
Replacing $K_\rho$ with a larger constant, this completes the proof.
\end{proof}

\begin{proof}[Proof of Corollary~\ref{cor:generalization-collapsed}]
Let $(X,y)$ be an independent clean sample. We may write
\[
yX
=
S+\lambda \boldsymbol u\boldsymbol v_0^\top,
\qquad
\boldsymbol u
\sim
\operatorname{Unif}(\mathbb S^{D-1}).
\]
Since the signal and shortcut blocks are orthogonal,
\[
y\langle W^\star,X\rangle_F
=
a^\star s
+
\lambda
\langle \boldsymbol h^\star,\boldsymbol u\rangle,
\]
where $\|\boldsymbol h^\star\|_2=1$. Hence, conditionally on the
training sample,
\[
\mathcal R(W^\star)
=
\mathbb P\left(
\langle \boldsymbol h^\star,\boldsymbol u\rangle
\leq
-\frac{a^\star s}{\lambda}
\,\middle|\,
\mathcal D_n
\right).
\]

Set $t := \dfrac{a^\star s}{\lambda}$. By Theorem~\ref{thm:main_collapsed},
\[
t
=
\min\left\{
\frac{s}{\lambda},
\frac{\Gamma_\rho'+\varepsilon_n}{\sqrt n}
\right\},
\qquad
|\varepsilon_n|
\leq
C\left(
\sqrt{\frac{\log(1/\delta)}{n}}
+
\sqrt{\frac nD}
\right).
\]
By choosing the constants in the assumptions of
Theorem~\ref{thm:main_collapsed} appropriately, we may ensure that
\[
c_\rho
\min\left\{
\frac{s}{\lambda},
\frac1{\sqrt n}
\right\}
\leq
t
\leq
C_\rho
\min\left\{
\frac{s}{\lambda},
\frac1{\sqrt n}
\right\},
\]
for constants $c_\rho,C_\rho>0$ depending only on $\rho$.
Moreover, since $\log(n/\delta)\geq1$, by decreasing
$c$ further if necessary, we may also ensure that $t\leq1/2$.

Applying Lemma~\ref{lem:spherical-linear} conditionally on
$\mathcal D_n$ gives
\[
c_0\exp(-C_0Dt^2)
\leq
\mathcal R(W^\star)
\leq
2\exp(-cDt^2).
\]
Combining this with the preceding bounds on $t$ yields
\[
c_0
\exp\left(
-c_1D
\min\left\{
\frac1n,
\frac{s^2}{\lambda^2}
\right\}
\right)
\leq
\mathcal R(W^\star)
\leq
2
\exp\left(
-c_2D
\min\left\{
\frac1n,
\frac{s^2}{\lambda^2}
\right\}
\right),
\]
where $c_0,c_1,c_2>0$ depend only on $\rho$.

Finally, the convergence of the risks follows by the same argument as in Corollary~\ref{cor:generalization-dispersed}.
\end{proof}

%% file: Proof_GD.tex
In this section, we study the late-stage implicit bias and generalization of GD under both DSM and CSM.

By Lemma~\ref{lem:linear-separability}, the training sample is linearly
separable when $D\geq n$. Therefore, the standard implicit bias result
\citep{gunasekar18a} implies that $W_k^\mathrm{GD}$ converges in direction
toward the unique Euclidean max-margin classifier $W_\mathrm{GD}^\star$,
where uniqueness follows from strict convexity.

The Euclidean hard-margin problem for both DSM and CSM has the form
\begin{equation}
    \min_W\frac12\|W\|_F^2
    \qquad\text{subject to}\qquad
    \langle W,\widetilde y_iX_i\rangle_F\geq1,
    \quad i\in[n].
\label{eq:gd-margin-primal}
\end{equation}
Its Lagrangian is
\[
    \mathcal L(W,\boldsymbol\alpha)
    =
    \frac12\|W\|_F^2
    +
    \sum_{i=1}^n
    \alpha_i
    \bigl(1-\langle W,\widetilde y_iX_i\rangle_F\bigr),
    \qquad \boldsymbol\alpha\geq0.
\]
Minimization over $W$ gives
$W=\sum_{i=1}^n\alpha_i\widetilde y_iX_i$, and hence the dual problem is
\begin{equation}
    \max_{\boldsymbol\alpha\geq0}
    \left\{
        \boldsymbol1^\top\boldsymbol\alpha
        -\frac12\boldsymbol\alpha^\top K\boldsymbol\alpha
    \right\},
\label{eq:gd-dual}
\end{equation}
where $K=(K_{ij})_{i,j=1}^n$ is the Gram matrix of the signed observations,
with $K_{ij}:=\langle\widetilde y_iX_i,\widetilde y_jX_j\rangle_F$.
Set $s_F:=\|S\|_F$ and
$G:=(\langle N_i,N_j\rangle_F)_{i,j=1}^n
=(U^\top U)\circ(\widehat V^\top\widehat V)$.
Since $\widetilde y_iX_i=\xi_iS+\lambda N_i$, signal--shortcut
orthogonality gives
\begin{equation}
    K
    =
    \lambda^2G+s_F^2\boldsymbol\xi\boldsymbol\xi^\top.
\label{eq:K-gd}
\end{equation}
The following lemma isolates the common GD dual algebra; the DSM and CSM
analyses differ only in the control of $G^{-1}$ and the positivity of the
resulting optimizer.

\begin{lemma}[Unconstrained GD dual solution]
\label{lem:gd-dual-solution}
Suppose $G\succ0$. Dropping the constraint
$\boldsymbol\alpha\geq0$ in \eqref{eq:gd-dual}, the unique maximizer is
\[
\boldsymbol\alpha^\star
=
\frac1{\lambda^2}
G^{-1}
\left(
\boldsymbol1-
a_{\rm GD}^\star s_F^2\boldsymbol\xi
\right),
\qquad
a_{\rm GD}^\star
=
\frac{
\boldsymbol\xi^\top G^{-1}\boldsymbol1
}{
\lambda^2+s_F^2\boldsymbol\xi^\top G^{-1}\boldsymbol\xi
}
=
\boldsymbol\xi^\top\boldsymbol\alpha^\star.
\]
If $\boldsymbol\alpha^\star>0$, then it is also the optimizer of
\eqref{eq:gd-dual}, and
\[
W_{\rm GD}^\star
=
a_{\rm GD}^\star S
+
\lambda\sum_{i=1}^n\alpha_i^\star N_i.
\]
\end{lemma}

\begin{proof}
The unconstrained first-order condition gives
$\boldsymbol\alpha^\star=K^{-1}\boldsymbol1$.
Applying the Sherman--Morrison formula to \eqref{eq:K-gd} gives the
displayed formula and
$a_{\rm GD}^\star=\boldsymbol\xi^\top\boldsymbol\alpha^\star$.
If $\boldsymbol\alpha^\star>0$, the inequality constraints are inactive,
and the final identity follows from
$W_{\rm GD}^\star=\sum_i\alpha_i^\star\widetilde y_iX_i$.
\end{proof}

We now study DSM and CSM separately.

\begin{proposition}[Directional convergence of GD under DSM]
\label{prop:gd-dispersed}
Assume the sample $\mathcal D_n$ is generated from the DSM and that
assumption~\ref{ass:stepsize} holds. Let $0<\delta<e^{-1}$.
There exists a constant $C_{\rm dim}>0$, depending only on $\rho$, such
that, if
\[
    n\ge C_{\rm dim}\log(4/\delta),
    \qquad
    D\ge C_{\rm dim}n\log(8n^2/\delta),
\]
then, with probability at least $1-\delta$ over the training sample, GD
converges in direction toward the unique Euclidean max-margin classifier
\[
    W_{\rm GD}^\star
    =
    a_{\rm GD}^\star S+H_{\rm GD}^\star.
\]
Moreover, $\operatorname{rank}(H_{\rm GD}^\star)=n$, and there exists a
universal constant $C>0$ such that
\[
\left|
a_{\rm GD}^\star
-
\frac{(1-2\rho)n}
{\lambda^2+n\|S\|_F^2}
\right|
\leq
C\frac{n}{\lambda^2+n\|S\|_F^2}
\left(
\sqrt{\frac{\log(4/\delta)}{n}}
+
\frac{n\log(8n^2/\delta)}{D}
\right).
\]
There also exist constants $c_H,C_H>0$, depending only on $\rho$, such
that
\[
    \frac{c_H}{\lambda}
    \leq
    \sigma_n(H_{\rm GD}^\star)
    \leq
    \sigma_1(H_{\rm GD}^\star)
    \leq
    \frac{C_H}{\lambda}.
\]
\end{proposition}

\begin{proof}
Since multiplication by $\widetilde y_i\in\{\pm1\}$ preserves the
uniform distribution on the sphere, the columns of $\widehat V$ satisfy
the assumptions of Lemma~\ref{lem:gram-concentration-2}.

\paragraph{Concentration of the shortcut Gram matrix.}
Apply Lemma~\ref{lem:gram-concentration-2} with failure probability
$\delta/2$. On the resulting event,
\[
    \max_{i\neq j}
    \left\{
        |\langle\boldsymbol u_i,\boldsymbol u_j\rangle|,
        |\langle\widehat{\boldsymbol v}_i,
                  \widehat{\boldsymbol v}_j\rangle|
    \right\}
    \leq
    C_{\scriptscriptstyle G}
    \sqrt{\frac{\log(8n^2/\delta)}{D}}.
\]
Since $G_{ii}=1$ and
$|G_{ij}|=
|\langle\boldsymbol u_i,\boldsymbol u_j\rangle|
|\langle\widehat{\boldsymbol v}_i,\widehat{\boldsymbol v}_j\rangle|$
for $i\neq j$, this gives
$\|G-I_n\|_{\infty\to\infty}
\leq C_{\scriptscriptstyle G}^2
n\log(8n^2/\delta)/D=:\varepsilon_G$.
Choosing $C_{\rm dim}$ sufficiently large, depending only on $\rho$,
we may assume $\varepsilon_G\leq1/2$. Thus $G$ is invertible and
\begin{equation}
    \|G^{-1}-I_n\|_{\infty\to\infty}
    \leq
    \frac{\varepsilon_G}{1-\varepsilon_G}
    \leq
    2\varepsilon_G.
\label{eq:GN-inverse}
\end{equation}
After increasing $C_{\rm dim}$ if necessary, the same application of
Lemma~\ref{lem:gram-concentration-2} also gives
\begin{equation}
    \frac12 I_n
    \preceq U^\top U\preceq\frac32 I_n,
    \qquad
    \frac12 I_n
    \preceq\widehat V^\top\widehat V\preceq\frac32 I_n.
\label{eq:UV-gram-gd}
\end{equation}

By Lemma~\ref{lem:gd-dual-solution}, the unconstrained dual optimizer
is explicit in terms of $G^{-1}$. We next approximate its signal
coefficient and then verify positivity.

\paragraph{Approximation of the signal coefficient.}
Set $m:=1-2\rho$. By Hoeffding's inequality, with probability at least
$1-\delta/2$,
\begin{equation}
    \left|
        \boldsymbol1^\top\boldsymbol\xi-mn
    \right|
    \leq
    \sqrt{2n\log(4/\delta)}.
\label{eq:label-balance-gd}
\end{equation}
We henceforth work on the intersection of this event with the Gram
concentration event, which has probability at least $1-\delta$.

Set $b:=\boldsymbol\xi^\top G^{-1}\boldsymbol1$ and
$c:=\boldsymbol\xi^\top G^{-1}\boldsymbol\xi$.
By \eqref{eq:GN-inverse},
$|b-\boldsymbol\xi^\top\boldsymbol1|\leq2n\varepsilon_G$ and
$|c-n|\leq2n\varepsilon_G$. Hence, by
\eqref{eq:label-balance-gd},
\begin{equation}
|b-mn|
\leq
\sqrt{2n\log(4/\delta)}+2n\varepsilon_G,
\qquad
|c-n|\leq2n\varepsilon_G.
\label{eq:quadratic-concentration-gd}
\end{equation}

By Lemma~\ref{lem:gd-dual-solution},
$a_{\rm GD}^\star=b/(\lambda^2+s_F^2c)$.
Increasing $C_{\rm dim}$ if necessary, we may assume
$\varepsilon_G\leq1/4$, so $n/2\leq c\leq3n/2$ and hence
$\lambda^2+s_F^2c\asymp\lambda^2+ns_F^2$ with universal constants.
Moreover,
\[
a_{\rm GD}^\star
-
\frac{mn}{\lambda^2+ns_F^2}
=
\frac{b-mn}{\lambda^2+s_F^2c}
+
\frac{mn s_F^2(n-c)}
{
(\lambda^2+s_F^2c)(\lambda^2+ns_F^2)
}.
\]
Using \eqref{eq:quadratic-concentration-gd} and the definition of
$\varepsilon_G$ therefore gives
\begin{equation}
\left|
a_{\rm GD}^\star
-
\frac{mn}{\lambda^2+ns_F^2}
\right|
\leq
C\frac{n}{\lambda^2+ns_F^2}
\left(
    \sqrt{\frac{\log(4/\delta)}{n}}
    +
    \frac{n\log(8n^2/\delta)}{D}
\right).
\label{eq:gd-a-approx}
\end{equation}

\paragraph{Positivity of the dual weights.}
By choosing $C_{\rm dim}$ sufficiently large depending only on $\rho$,
the quantity in parentheses in \eqref{eq:gd-a-approx}, multiplied by
$C$, is at most $\min\{m/2,\rho\}$. Hence
$a_{\rm GD}^\star
\geq \frac m2\,n/(\lambda^2+ns_F^2)>0$ and
$0<a_{\rm GD}^\star s_F^2\leq m+\rho=1-\rho$.

Set
$\boldsymbol z:=\boldsymbol1-a_{\rm GD}^\star s_F^2\boldsymbol\xi$.
Then $\rho\leq z_i\leq2$ for every $i$, and by
\eqref{eq:GN-inverse},
$\|G^{-1}\boldsymbol z-\boldsymbol z\|_\infty
\leq\|G^{-1}-I_n\|_{\infty\to\infty}\|\boldsymbol z\|_\infty
\leq4\varepsilon_G$.
Increasing $C_{\rm dim}$ once more so that
$4\varepsilon_G\leq\rho/2$, we obtain
$\rho/2\leq(G^{-1}\boldsymbol z)_i\leq3$. Thus
Lemma~\ref{lem:gd-dual-solution} yields
\begin{equation}
    \frac{\rho}{2\lambda^2}
    \leq
    \alpha_i^\star
    \leq
    \frac{3}{\lambda^2},
    \qquad i\in[n].
\label{eq:dual-weight-bounds}
\end{equation}
In particular, $\boldsymbol\alpha^\star>0$, so the same lemma gives
$W_{\rm GD}^\star=a_{\rm GD}^\star S+H_{\rm GD}^\star$, where
$H_{\rm GD}^\star:=\lambda\sum_{i=1}^n\alpha_i^\star N_i$.

\paragraph{Singular values of the shortcut component.}
Let
$\Lambda^\star:=\operatorname{diag}
(\alpha_1^\star,\ldots,\alpha_n^\star)$, so that
$H_{\rm GD}^\star=\lambda U\Lambda^\star\widehat V^\top$.
By Lemma~\ref{lem:polar-orthonormalization}, its nonzero singular values
are $\lambda$ times those of
$(U^\top U)^{1/2}\Lambda^\star
(\widehat V^\top\widehat V)^{1/2}$.
Thus, by \eqref{eq:UV-gram-gd} and
\eqref{eq:dual-weight-bounds},
\[
\begin{aligned}
    \sigma_n(H_{\rm GD}^\star)
    &\geq
    \lambda
    \sqrt{\lambda_{\min}(U^\top U)}
    \min_i\alpha_i^\star
    \sqrt{\lambda_{\min}(\widehat V^\top\widehat V)}
    \geq
    \frac{\rho}{4\lambda},\\
    \sigma_1(H_{\rm GD}^\star)
    &\leq
    \lambda
    \sqrt{\lambda_{\max}(U^\top U)}
    \max_i\alpha_i^\star
    \sqrt{\lambda_{\max}(\widehat V^\top\widehat V)}
    \leq
    \frac{9}{2\lambda}.
\end{aligned}
\]
Hence the claimed singular-value bounds hold with
$c_H=\rho/4$ and $C_H=9/2$, and
$\operatorname{rank}(H_{\rm GD}^\star)=n$.
The directional convergence of $W_k^{\rm GD}$ toward
$W_{\rm GD}^\star$ follows from the common implicit-bias argument at
the beginning of this section.
\end{proof}

\begin{proposition}[Directional convergence of GD under CSM]
\label{prop:gd-collapsed}
Assume the sample $\mathcal D_n$ is generated from the CSM and that
Assumption~\ref{ass:stepsize} holds. Let $0<\delta<e^{-1}$.
There exists a constant $C_{\rm dim}>0$, depending only on $\rho$, such
that, if
\[
    n\ge C_{\rm dim}\log(4/\delta),
    \qquad
    D\ge C_{\rm dim}n\log(8n/\delta),
\]
then, with probability at least $1-\delta$ over the training sample, GD
converges in direction toward the unique Euclidean max-margin classifier
\[
    W_{\rm GD}^\star
    =
    a_{\rm GD}^\star S+H_{\rm GD}^\star,
\]
where $\operatorname{rank}(H_{\rm GD}^\star)=1$.

Moreover, there exists a universal constant $C>0$ such that
\[
\left|
a_{\rm GD}^\star
-
\frac{(1-2\rho)n}
{\lambda^2+n\|S\|_F^2}
\right|
\leq
C\frac{n}{\lambda^2+n\|S\|_F^2}
\left(
    \sqrt{\frac{\log(4/\delta)}{n}}
    +
    \sqrt{\frac{n\log(8n/\delta)}{D}}
\right).
\]
There also exists a constant $C_H>0$, depending only on $\rho$, such that
\[
\left|
\frac{\lambda}{\sqrt n}\|H_{\rm GD}^\star\|_F
-
\sqrt{
    4\rho(1-\rho)
    +(1-2\rho)^2
    \frac{\lambda^4}
    {(\lambda^2+n\|S\|_F^2)^2}
}
\right|
\leq
C_H
\left(
    \sqrt{\frac{\log(4/\delta)}{n}}
    +
    \sqrt{\frac{n\log(8n/\delta)}{D}}
\right).
\]
\end{proposition}

\begin{proof}
Set $m:=1-2\rho$ and
$\varepsilon_G:=C\sqrt{n\log(8n/\delta)/D}$.

\paragraph{Concentration.}
Under the CSM, after absorbing $\widetilde y_i$ into $\boldsymbol u_i$,
$G=U^\top U$. Applying Lemma~\ref{lem:inverse-gram-concentration}
with failure probability $\delta/2$ gives
\begin{equation}
    \frac12I_n\preceq G\preceq\frac32I_n,
    \qquad
    \max_{\boldsymbol x\in\{\boldsymbol1,\boldsymbol\xi\}}
    \|(G^{-1}-I_n)\boldsymbol x\|_\infty
    \leq\varepsilon_G.
\label{eq:G-concentration-collapsed}
\end{equation}
Here and below, universal constants may change from line to line.
By Hoeffding's inequality, with failure probability at most $\delta/2$,
$|\boldsymbol1^\top\boldsymbol\xi-mn|
\leq\sqrt{2n\log(4/\delta)}$.
We work on the intersection of these events, which has probability at
least $1-\delta$.

\paragraph{Approximation of $a_{\rm GD}^\star$.}
Set $b:=\boldsymbol\xi^\top G^{-1}\boldsymbol1$ and
$c:=\boldsymbol\xi^\top G^{-1}\boldsymbol\xi$.
By \eqref{eq:G-concentration-collapsed},
$|b-\boldsymbol\xi^\top\boldsymbol1|\leq n\varepsilon_G$ and
$|c-n|\leq n\varepsilon_G$. Hence
$|b-mn|\leq\sqrt{2n\log(4/\delta)}+n\varepsilon_G$ and
$|c-n|\leq n\varepsilon_G$.
Moreover, $2n/3\leq c\leq2n$, so
$\lambda^2+s_F^2c\asymp\lambda^2+ns_F^2$.

By Lemma~\ref{lem:gd-dual-solution},
$a_{\rm GD}^\star=b/(\lambda^2+s_F^2c)$, and
\[
a_{\rm GD}^\star-\frac{mn}{\lambda^2+ns_F^2}
=
\frac{b-mn}{\lambda^2+s_F^2c}
+
\frac{mn s_F^2(n-c)}
{(\lambda^2+s_F^2c)(\lambda^2+ns_F^2)}.
\]
Therefore,
\begin{equation}
\left|
a_{\rm GD}^\star
-
\frac{mn}{\lambda^2+ns_F^2}
\right|
\leq
C\frac{n}{\lambda^2+ns_F^2}
\left(
    \sqrt{\frac{\log(4/\delta)}{n}}
    +
    \sqrt{\frac{n\log(8n/\delta)}{D}}
\right).
\label{eq:gd-a-approx-collapsed}
\end{equation}

\paragraph{Positivity.}
Set
$\varepsilon:=
\sqrt{\log(4/\delta)/n}
+\sqrt{n\log(8n/\delta)/D}$.
By choosing $C_{\rm dim}$ sufficiently large depending only on $\rho$,
the normalized right-hand side of
\eqref{eq:gd-a-approx-collapsed} is at most
$\min\{m/2,\rho\}$. Consequently,
$a_{\rm GD}^\star
\geq \frac m2\,n/(\lambda^2+ns_F^2)>0$ and
$0\leq a_{\rm GD}^\star s_F^2\leq m+\rho=1-\rho$.

Set
$\boldsymbol z:=\boldsymbol1-a_{\rm GD}^\star s_F^2\boldsymbol\xi$.
Then $\rho\leq z_i\leq2$ for every $i$, while
\[
\|(G^{-1}-I_n)\boldsymbol z\|_\infty
\leq
\|(G^{-1}-I_n)\boldsymbol1\|_\infty
+
a_{\rm GD}^\star s_F^2
\|(G^{-1}-I_n)\boldsymbol\xi\|_\infty
\leq2\varepsilon_G.
\]
Increasing $C_{\rm dim}$ if necessary so that
$2\varepsilon_G\leq\rho/2$, we obtain
$(G^{-1}\boldsymbol z)_i\geq\rho/2$. Hence
Lemma~\ref{lem:gd-dual-solution} gives
$\alpha_i^\star\geq\rho/(2\lambda^2)>0$ for every $i$.

Since $N_i=\boldsymbol u_i\boldsymbol v_0^\top$ under CSM, the same
lemma gives
$W_{\rm GD}^\star=a_{\rm GD}^\star S+H_{\rm GD}^\star$, where
$H_{\rm GD}^\star
:=\lambda U\boldsymbol\alpha^\star\boldsymbol v_0^\top$.
Since $G=U^\top U\succ0$ and $\boldsymbol\alpha^\star\neq0$,
$\operatorname{rank}(H_{\rm GD}^\star)=1$.

\paragraph{Shortcut norm.}
By Lemma~\ref{lem:gd-dual-solution},
$\boldsymbol\alpha^\star=\lambda^{-2}G^{-1}\boldsymbol z$, and hence
\begin{equation}
    \frac{\lambda^2}{n}\|H_{\rm GD}^\star\|_F^2
    =
    \frac{\lambda^4}{n}
    (\boldsymbol\alpha^\star)^\top G\boldsymbol\alpha^\star
    =
    \frac1n\boldsymbol z^\top G^{-1}\boldsymbol z.
\label{eq:gd-H-norm-collapsed}
\end{equation}
Since $\|\boldsymbol z\|_1\leq2n$ and
$\|(G^{-1}-I_n)\boldsymbol z\|_\infty\leq2\varepsilon_G$,
\[
\left|
\frac1n\boldsymbol z^\top G^{-1}\boldsymbol z
-
\frac1n\|\boldsymbol z\|_2^2
\right|
\leq
C\varepsilon_G.
\]

Set $\overline\xi:=n^{-1}\boldsymbol1^\top\boldsymbol\xi$,
$\theta:=a_{\rm GD}^\star s_F^2$, and
$\theta_0:=mn s_F^2/(\lambda^2+ns_F^2)$.
Then $n^{-1}\|\boldsymbol z\|_2^2
=1-2\theta\overline\xi+\theta^2$, while
\eqref{eq:gd-a-approx-collapsed} and Hoeffding's bound give
$|\theta-\theta_0|\leq C\varepsilon$ and
$|\overline\xi-m|
\leq C\sqrt{\log(4/\delta)/n}$.
Since $\theta$ and $\theta_0$ are uniformly bounded,
\[
\left|
\frac1n\|\boldsymbol z\|_2^2
-
(1-2m\theta_0+\theta_0^2)
\right|
\leq C\varepsilon.
\]
Moreover,
\[
1-2m\theta_0+\theta_0^2
=
4\rho(1-\rho)
+
m^2\frac{\lambda^4}{(\lambda^2+ns_F^2)^2}.
\]
Combining these estimates with
\eqref{eq:gd-H-norm-collapsed} yields
\[
\left|
\frac{\lambda^2}{n}\|H_{\rm GD}^\star\|_F^2
-
\left[
4\rho(1-\rho)
+
m^2\frac{\lambda^4}{(\lambda^2+ns_F^2)^2}
\right]
\right|
\leq C\varepsilon.
\]
The quantity in brackets is at least $4\rho(1-\rho)>0$; hence
$|\sqrt{x}-\sqrt{y}|=|x-y|/(\sqrt{x}+\sqrt{y})$ gives
\[
\left|
\frac{\lambda}{\sqrt n}\|H_{\rm GD}^\star\|_F
-
\sqrt{
4\rho(1-\rho)
+
(1-2\rho)^2
\frac{\lambda^4}
{(\lambda^2+n\|S\|_F^2)^2}
}
\right|
\leq
C_H\varepsilon,
\]
where $C_H>0$ depends only on $\rho$.
\end{proof}

\begin{corollary}
\label{cor:gd-dispersed}
Assume the sample $\mathcal D_n$ is generated from the DSM and that
Assumption~\ref{ass:stepsize} holds. Let $0<\delta<e^{-1}$.
There exist constants $C_{\rm dim},c,C>0$, depending only on $\rho$,
such that, if
\[
    n\geq C_{\rm dim}\log(4/\delta),
    \qquad
    D\geq C_{\rm dim}n\log(8n^2/\delta),
\]
then, with probability at least $1-\delta$ over the training sample,
\begin{align}
\mathcal R(W_{\rm GD}^\star)
&\leq
4\exp\left[
-c\min\left\{
\frac{D^2n\|S\|_F^4}
{(\lambda^2+n\|S\|_F^2)^2},
\frac{Dn\|S\|_F^2}
{\lambda^2+n\|S\|_F^2}
\right\}
\right],
\label{eq:gd-dispersed-risk-upper}
\\
\mathcal R(W_{\rm GD}^\star)
&\geq
\frac12
-
Ce^{-cn}
-
C
\frac{D\sqrt n\,\|S\|_F^2}
{\lambda^2+n\|S\|_F^2}.
\label{eq:gd-dispersed-risk-lower}
\end{align}
Consequently,
\[
    \lambda^2\ll D\sqrt n\,\|S\|_F^2
    \quad\Longrightarrow\quad
    \mathcal R(W_{\rm GD}^\star)\to0,
\]
whereas
\[
    \lambda^2\gg D\sqrt n\,\|S\|_F^2
    \quad\Longrightarrow\quad
    \mathcal R(W_{\rm GD}^\star)\to\frac12.
\]
\end{corollary}

\begin{proof}
Increase $C_{\rm dim}$, if necessary, so that the conclusions of
Proposition~\ref{prop:gd-dispersed} hold and its approximation of
$a_{\rm GD}^\star$ implies
\begin{equation}\label{eq:gd-dispersed-signal-scale}
    a_{\rm GD}^\star\|S\|_F^2
    \asymp
    \frac{n\|S\|_F^2}
    {\lambda^2+n\|S\|_F^2}.
\end{equation}
We work on the corresponding event, which has probability at least
$1-\delta$.

Set $B_{\rm GD}^\star:=\lambda H_{\rm GD}^\star$.
Proposition~\ref{prop:gd-dispersed} gives
$\operatorname{rank}(B_{\rm GD}^\star)=n$ and
$c\leq\sigma_n(B_{\rm GD}^\star)
\leq\sigma_1(B_{\rm GD}^\star)\leq C$, and therefore
\begin{equation}
    \|B_{\rm GD}^\star\|_F\asymp\sqrt n,
    \qquad
    \|B_{\rm GD}^\star\|_{\op}\asymp1.
\label{eq:gd-dispersed-B-scales}
\end{equation}

For an independent clean test sample,
$yX=S+\lambda\boldsymbol u\boldsymbol v^\top$, where
$\boldsymbol u,\boldsymbol v$ are independent and uniformly distributed
on $\mathbb S^{D-1}$. Hence
\[
    y\langle W_{\rm GD}^\star,X\rangle_F
    =
    a_{\rm GD}^\star\|S\|_F^2
    +
    \boldsymbol u^\top B_{\rm GD}^\star\boldsymbol v.
\]

\paragraph{Upper bound.}
Set $t:=a_{\rm GD}^\star\|S\|_F^2$.
By Lemma~\ref{lem:fresh-bilinear-tail},
\[
\mathcal R(W_{\rm GD}^\star)
\leq
4\exp\left[
-c\min\left\{
\frac{D^2t^2}{\|B_{\rm GD}^\star\|_F^2},
\frac{Dt}{\|B_{\rm GD}^\star\|_{\op}}
\right\}
\right].
\]
Using \eqref{eq:gd-dispersed-signal-scale} and
\eqref{eq:gd-dispersed-B-scales} gives
\eqref{eq:gd-dispersed-risk-upper}.

\paragraph{Lower bound.}
Let $B_{\rm GD}^\star=U_B\Sigma_BV_B^\top$ be its compact SVD.
By Proposition~\ref{prop:gd-dispersed},
$\operatorname{rank}(B_{\rm GD}^\star)=n$ and
$\sigma_n(B_{\rm GD}^\star)\geq c_H$, where $c_H>0$ depends only on
$\rho$.

Let $P:=U_BU_B^\top$ be the orthogonal projection onto the left singular
space of $B_{\rm GD}^\star$. Then
\begin{equation}\label{eq:gd-dispersed-Bu-lower}
\begin{aligned}
    \|(B_{\rm GD}^\star)^\top\boldsymbol u\|_2
    &=
    \|V_B\Sigma_BU_B^\top\boldsymbol u\|_2
    =
    \|\Sigma_BU_B^\top\boldsymbol u\|_2 \\
    &\geq
    \sigma_n(B_{\rm GD}^\star)
    \|U_B^\top\boldsymbol u\|_2
    =
    \sigma_n(B_{\rm GD}^\star)
    \|P\boldsymbol u\|_2
    \geq
    c_H\|P\boldsymbol u\|_2.
\end{aligned}
\end{equation}
Here we used $V_B^\top V_B=I_n$.

Since $P$ has rank $n$, Lemma~\ref{lem:spherical-projection-norm} gives
\begin{equation}\label{eq:gd-dispersed-projection-lower}
    \prob\left(
        \|P\boldsymbol u\|_2
        \leq
        \frac12\sqrt{\frac nD}
        \,\middle|\,
        \mathcal D_n
    \right)
    \leq
    2e^{-n/16}.
\end{equation}
Together with \eqref{eq:gd-dispersed-Bu-lower}, this implies
\begin{equation}\label{eq:gd-dispersed-Bu-scale}
    \|(B_{\rm GD}^\star)^\top\boldsymbol u\|_2
    \geq
    \frac{c_H}{2}\sqrt{\frac nD}
\end{equation}
except on an event of probability at most $2e^{-n/16}$.

Now condition on such a realization of $\boldsymbol u$. Since
$\boldsymbol v\sim\Unif(\mathbb S^{D-1})$ is independent of
$\boldsymbol u$, Lemma~\ref{lem:spherical-linear} gives, for every $t>0$,
\[
\prob\left(
\left|
\boldsymbol u^\top B_{\rm GD}^\star\boldsymbol v
\right|
\leq t
\,\middle|\,
\boldsymbol u,\mathcal D_n
\right)
\leq
C\sqrt D\,
\frac{t}
{\|(B_{\rm GD}^\star)^\top\boldsymbol u\|_2}
\leq
C_\rho\frac{Dt}{\sqrt n},
\]
where the last inequality follows from
\eqref{eq:gd-dispersed-Bu-scale}. Integrating over $\boldsymbol u$ gives
\[
\prob\left(
\left|
\boldsymbol u^\top B_{\rm GD}^\star\boldsymbol v
\right|
\leq t
\,\middle|\,
\mathcal D_n
\right)
\leq
Ce^{-cn}
+
C_\rho\frac{Dt}{\sqrt n}.
\]

Finally, $Z:=\boldsymbol u^\top B_{\rm GD}^\star\boldsymbol v$ has a
continuous distribution symmetric about zero. Since
\[
\mathcal R(W_{\rm GD}^\star)
=
\prob(Z\leq-t\mid\mathcal D_n)
=
\frac12-\frac12\prob(|Z|<t\mid\mathcal D_n),
\]
the preceding bound together with
\eqref{eq:gd-dispersed-signal-scale} yields
\[
\mathcal R(W_{\rm GD}^\star)
\geq
\frac12
-
Ce^{-cn}
-
C
\frac{
D\sqrt n\,\|S\|_F^2
}{
\lambda^2+n\|S\|_F^2
},
\]
which proves \eqref{eq:gd-dispersed-risk-lower}.
\end{proof}

\begin{corollary}
\label{cor:gd-collapsed}
Assume the sample $\mathcal D_n$ is generated from the CSM and that
Assumption~\ref{ass:stepsize} holds. Let $0<\delta<e^{-1}$.
There exist constants $C_{\rm dim},c,C>0$, depending only on $\rho$,
such that, if
\[
    n\geq C_{\rm dim}\log(4/\delta),
    \qquad
    D\geq C_{\rm dim}n\log(8n/\delta),
\]
then, with probability at least $1-\delta$ over the training sample,
\begin{align}
\mathcal R(W_{\rm GD}^\star)
&\leq
2\exp\left[
    -c
    \frac{Dn\|S\|_F^4}
    {(\lambda^2+n\|S\|_F^2)^2}
\right],
\label{eq:gd-collapsed-risk-upper}
\\
\mathcal R(W_{\rm GD}^\star)
&\geq
\frac12
-
C
\frac{\sqrt{nD}\,\|S\|_F^2}
{\lambda^2+n\|S\|_F^2}.
\label{eq:gd-collapsed-risk-lower}
\end{align}
Consequently,
\[
    \lambda^2\ll \|S\|_F^2\sqrt{nD}
    \quad\Longrightarrow\quad
    \mathcal R(W_{\rm GD}^\star)\to0,
\]
whereas
\[
    \lambda^2\gg \|S\|_F^2\sqrt{nD}
    \quad\Longrightarrow\quad
    \mathcal R(W_{\rm GD}^\star)\to\frac12.
\]
\end{corollary}

\begin{proof}
Increase $C_{\rm dim}$, if necessary, so that the conclusions of
Proposition~\ref{prop:gd-collapsed} hold and its estimates imply
\begin{equation}
    a_{\rm GD}^\star\|S\|_F^2
    \asymp_\rho
    \frac{n\|S\|_F^2}
    {\lambda^2+n\|S\|_F^2},
    \qquad
    \lambda\|H_{\rm GD}^\star\|_F
    \asymp_\rho
    \sqrt n.
\label{eq:gd-collapsed-scales}
\end{equation}
We work on the corresponding event, which has probability at least
$1-\delta$.

Since $\operatorname{rank}(H_{\rm GD}^\star)=1$, write
$H_{\rm GD}^\star
=\boldsymbol h_{\rm GD}^\star\boldsymbol v_0^\top$ with
$\|\boldsymbol v_0\|_2=1$, so that
$\|H_{\rm GD}^\star\|_F=\|\boldsymbol h_{\rm GD}^\star\|_2$.
For an independent clean test sample, by spherical symmetry we may write
$yX=S+\lambda\boldsymbol u\boldsymbol v_0^\top$, where
$\boldsymbol u\sim\Unif(\mathbb S^{D-1})$ is independent of the training
sample. Hence
\[
    y\langle W_{\rm GD}^\star,X\rangle_F
    =
    a_{\rm GD}^\star\|S\|_F^2
    +
    \lambda
    \langle
        \boldsymbol h_{\rm GD}^\star,\boldsymbol u
    \rangle.
\]

Therefore, conditionally on $\mathcal D_n$,
\[
\mathcal R(W_{\rm GD}^\star)
=
\prob(Z\leq-t\mid\mathcal D_n),
\qquad
Z:=
\left\langle
\frac{\boldsymbol h_{\rm GD}^\star}
{\|\boldsymbol h_{\rm GD}^\star\|_2},
\boldsymbol u
\right\rangle,
\qquad
t:=
\frac{
a_{\rm GD}^\star\|S\|_F^2
}{
\lambda\|\boldsymbol h_{\rm GD}^\star\|_2
}.
\]
By \eqref{eq:gd-collapsed-scales},
\begin{equation}
t
\asymp_\rho
\frac{
\sqrt n\,\|S\|_F^2
}{
\lambda^2+n\|S\|_F^2
}.
\label{eq:gd-collapsed-snr}
\end{equation}
By symmetry and Lemma~\ref{lem:spherical-linear},
\[
\frac12-C\sqrt D\,t
\leq
\mathcal R(W_{\rm GD}^\star)
\leq
2e^{-cDt^2}.
\]
Together with \eqref{eq:gd-collapsed-snr}, this gives
\eqref{eq:gd-collapsed-risk-upper} and
\eqref{eq:gd-collapsed-risk-lower}.

Finally, if $\lambda^2\ll\|S\|_F^2\sqrt{nD}$, then, since
$D\gtrsim n\log n$,
$\sqrt{nD}\,\|S\|_F^2/(\lambda^2+n\|S\|_F^2)\to\infty$.
Thus \eqref{eq:gd-collapsed-risk-upper} gives
$\mathcal R(W_{\rm GD}^\star)\to0$.

Conversely, if $\lambda^2\gg\|S\|_F^2\sqrt{nD}$, then
$\sqrt{nD}\,\|S\|_F^2/(\lambda^2+n\|S\|_F^2)\to0$.
Hence \eqref{eq:gd-collapsed-risk-lower}, together with
$\mathcal R(W_{\rm GD}^\star)\leq1/2$ by symmetry and the positive signal
coefficient, gives
$\mathcal R(W_{\rm GD}^\star)\to\frac12$.
\end{proof}

%% file: Proof_onestep_dispersed.tex
\begin{lemma}[One-step spherical polar alignment]
\label{lem:one-step-spherical-polar-alignment}
Let $U=[\boldsymbol u_1,\ldots,\boldsymbol u_n]$ and
$\widehat V=[\widehat{\boldsymbol v}_1,\ldots,\widehat{\boldsymbol v}_n]$,
where all columns are independent and uniformly distributed on
$\mathbb S^{D-1}$. Set
$N:=U\widehat V^\top
=\sum_{i=1}^n\boldsymbol u_i\widehat{\boldsymbol v}_i^\top$
and $H:=\pol(N)$.
There exist universal constants $c,C>0$ such that, if
\[
D\geq c\max\{n,\log(1/\delta)\},
\]
then, with probability at least $1-\delta$,
\[
\max_{1\leq i\leq n}
\left|
\left\langle
H,\boldsymbol u_i\widehat{\boldsymbol v}_i^\top
\right\rangle_F-1
\right|
\leq
C\sqrt{
\frac{\max\{n,\log(1/\delta)\}}{D}
}.
\]
\end{lemma}

\begin{proof}
Let $c_{\scriptscriptstyle G},C_{\scriptscriptstyle G}>0$ be the
universal constants in Lemma~\ref{lem:gram-concentration-2}, and set
$G_U:=U^\top U$ and $G_V:=\widehat V^\top\widehat V$.
Since $\log n\leq n$, we have
$n+\log(n/\delta)\leq3\max\{n,\log(1/\delta)\}$.
Thus, by taking $c>0$ sufficiently large,
Lemma~\ref{lem:gram-concentration-2} implies that, with probability at
least $1-\delta$,
\[
\|G_U-I\|_{\op}\vee\|G_V-I\|_{\op}
\leq
C_{\scriptscriptstyle G}
\sqrt{\frac{n+\log(n/\delta)}{D}}
\leq
C_0\sqrt{\frac{\max\{n,\log(1/\delta)\}}{D}}
=:\eta,
\]
where $C_0>0$ is a universal constant. Increasing $c$, we have
$\eta\leq1/4$. We work on this event.

Set $\overline U:=UG_U^{-1/2}$,
$\overline V:=\widehat V G_V^{-1/2}$, and
$Y:=G_U^{1/2}G_V^{1/2}$. Then
$\overline U^\top\overline U
=\overline V^\top\overline V=I_n$ and
$N=\overline UY\overline V^\top$, so
Lemma~\ref{lem:polar-orthonormalization} gives
$H=\overline U\pol(Y)\overline V^\top$.

For $N_i:=\boldsymbol u_i\widehat{\boldsymbol v}_i^\top$,
\[
\langle H,N_i\rangle_F
=
e_i^\top
G_U^{1/2}\pol(Y)G_V^{1/2}e_i,
\]
where we used $U^\top\overline U=G_U^{1/2}$ and
$\overline V^\top\widehat V=G_V^{1/2}$.

Since $\eta\leq1/4$,
$\|G_U-I\|_{\op}\vee\|G_V-I\|_{\op}\leq\eta<1$.
The square-root map is uniformly Lipschitz on a fixed compact subset of
the positive definite cone, so, for a universal constant $C_1>0$,
\[
\|G_U^{1/2}-I\|_{\op}
\vee
\|G_V^{1/2}-I\|_{\op}
\leq C_1\eta.
\]
Since
$Y-I=(G_U^{1/2}-I)G_V^{1/2}+(G_V^{1/2}-I)$, it follows that
$\|Y-I\|_{\op}\leq C_2\eta$ for a universal constant $C_2>0$.

Moreover,
$Y^\top Y-I
=G_V^{1/2}(G_U-I)G_V^{1/2}+(G_V-I)$, and hence
\[
\|Y^\top Y-I\|_{\op}
\leq
(2+\eta)\eta
\leq
\frac{9}{16}.
\]
Thus $Y^\top Y$ lies in a fixed compact subset of the positive definite
cone. The inverse square-root map is therefore uniformly Lipschitz there,
so, for a universal constant $C_3>0$,
\[
\|(Y^\top Y)^{-1/2}-I\|_{\op}\leq C_3\eta.
\]
Since
$\pol(Y)-I
=(Y-I)(Y^\top Y)^{-1/2}
+\{(Y^\top Y)^{-1/2}-I\}$,
and $\|(Y^\top Y)^{-1/2}\|_{\op}$ is uniformly bounded, the triangle
inequality gives
\[
\|\pol(Y)-I\|_{\op}\leq C_4\eta
\]
for a universal constant $C_4>0$.
Finally,
$G_U^{1/2}\pol(Y)G_V^{1/2}-I
=(G_U^{1/2}-I)\pol(Y)G_V^{1/2}
+(\pol(Y)-I)G_V^{1/2}
+(G_V^{1/2}-I)$,
so another application of the triangle inequality yields
\[
\left\|
G_U^{1/2}\pol(Y)G_V^{1/2}-I
\right\|_{\op}
\leq C_5\eta.
\]
Thus
\[
\max_{i\in[n]}
\left|
\langle H,N_i\rangle_F-1
\right|
\leq
C_0C_5\sqrt{
\frac{\max\{n,\log(1/\delta)\}}{D}
}.
\]
Renaming $C_0C_5$ as $C$ proves the claim.
\end{proof}

We are now ready to show Theorem~\ref{thm:onestep_dispersed}. Actually, we show slightly sharper result:

\begin{theorem}[One-step interpolation and generalization under dispersed shortcuts]
\label{thm:onestep_dispersed_sharp}
Assume the sample $\calD_n$ is generated from the DSM. There exist universal constants \(c,C>0\) such that the following holds.
Let \(\delta\in(0,1)\) and suppose that $n\geq
\frac{2}{(1-2\rho)^2}\log\frac{2}{\delta}.$
Then, with probability at least \(1-\delta\) over the training sample
\(\mathcal D_n\), the following statements hold.
\begin{enumerate}
    \item \textbf{Generalization.} For an independent test sample \((X,y)\),
    \[
    \prob\left(
    y\langle W_1^{\mathrm{SpecGD}},X\rangle_F\leq 0
    \,\middle|\,
    \mathcal D_n
    \right)
    \leq
    2\exp\left(
    -c\min\left\{
    \frac{D^2s^2}{n\lambda^2},
    \frac{Ds}{\lambda}
    \right\}
    \right).
    \]

    \item \textbf{Interpolation.} If, in addition, \(\lambda>s\) and $D\geq
    C\left(1-\frac{s}{\lambda}\right)^{-2}
    \max\left\{
    n,\log\frac{2}{\delta}
    \right\},$
    then \[\widehat{\mathcal R}_n
    \left(W_1^{\mathrm{SpecGD}}\right)=0.\]
\end{enumerate}
\end{theorem}

\begin{proof}
Let $\mathcal E_{\rm sign}:=\{n_+>n_-\}$. Since
$n_-=\sum_{i=1}^n\boldsymbol 1_{\{\xi_i=-1\}}\sim\operatorname{Bin}(n,\rho)$,
Hoeffding's inequality gives
\[
\mathbb P(\mathcal E_{\rm sign}^c)
\leq
\exp\left\{-\frac12n(1-2\rho)^2\right\}
\leq
\frac{\delta}{2},
\]
where the last inequality follows from the lower bound on $n$. We work on
$\mathcal E_{\rm sign}$.

Set $s:=\|S\|_1$ and $N:=\sum_{i=1}^nN_i$. Since
$W_0^{\rm SpecGD}=0$,
\[
-\nabla\mathcal L(W_0^{\rm SpecGD})
=
\frac{1}{n}\sum_{i=1}^n A_i
=
\frac{1}{n}\left\{(n_+-n_-)S+\lambda N\right\}.
\]
Because $n_+>n_-$ and the signal and shortcut blocks are orthogonal,
\[
\pol\left(\sum_{i=1}^nA_i\right)
=
P_S+\pol(N).
\]
Writing $H:=\pol(N)$, we therefore have
$W_1^{\rm SpecGD}=\eta_0(P_S+H)$.

We first prove the conditional generalization bound. Conditional on the
training data, $H$ is deterministic, with
$\operatorname{rank}(H)=\operatorname{rank}(N)\leq n$,
$\|H\|_{\op}=1$, and
$\|H\|_F^2=\operatorname{rank}(H)\leq n$ almost surely.
For a fresh test point
$X=yS+\lambda\boldsymbol u\boldsymbol v^\top$, signal--noise block
orthogonality gives
\[
y\langle W_1^{\rm SpecGD},X\rangle_F
=
\eta_0\left[
s+\lambda y
\langle H,\boldsymbol u\boldsymbol v^\top\rangle_F
\right].
\]
Set $T:=\langle H,\boldsymbol u\boldsymbol v^\top\rangle_F$.
Conditional on the training data, $T$ is symmetric, and multiplication by
the independent sign $y$ preserves its distribution. Hence
\[
\mathbb P\left(
y\langle W_1^{\rm SpecGD},X\rangle_F\leq0
\,\middle|\,\mathcal D_n
\right)
=
\frac12
\mathbb P\left(
|T|\geq\frac{s}{\lambda}
\,\middle|\,\mathcal D_n
\right).
\]
Applying Lemma~\ref{lem:fresh-bilinear-tail} and using
$\|H\|_F^2\leq n$ and $\|H\|_{\op}=1$ gives
\[
\mathbb P\left(
y\langle W_1^{\rm SpecGD},X\rangle_F\leq0
\,\middle|\,\mathcal D_n
\right)
\leq
2\exp\left[
-c\min\left\{
\frac{D^2s^2}{n\lambda^2},
\frac{Ds}{\lambda}
\right\}
\right].
\]

It remains to prove interpolation. Assume $\lambda>s$. Applying
Lemma~\ref{lem:one-step-spherical-polar-alignment} with failure probability
$\delta/2$, the dimensional assumption and a sufficiently large universal
constant $C$ imply that, with probability at least $1-\delta/2$,
\[
\max_{i\in[n]}
\left|
\langle H,N_i\rangle_F-1
\right|
\leq
C_0\sqrt{
\frac{\max\{n,\log(2/\delta)\}}{D}
}
\leq
\frac12\left(1-\frac{s}{\lambda}\right),
\]
where $C_0>0$ is a universal constant. Therefore
\[
\min_{i\in[n]}
\langle H,N_i\rangle_F
\geq
\frac12\left(1+\frac{s}{\lambda}\right)
>
\frac{s}{\lambda}.
\]
On the intersection of this event with $\mathcal E_{\rm sign}$, for every
$i\in[n]$,
\[
\eta_0^{-1}
\langle W_1^{\rm SpecGD},\widetilde y_iX_i\rangle_F
=
\langle P_S+H,\xi_iS+\lambda N_i\rangle_F
=
\xi_i s+\lambda\langle H,N_i\rangle_F.
\]
If $i\in I_+$, this is positive, while if $i\in I_-$, it is larger than
$-s+\lambda(s/\lambda)=0$. Thus
\[
\min_{i\in[n]}
\langle W_1^{\rm SpecGD},\widetilde y_iX_i\rangle_F>0,
\]
and hence $\widehat R_n(W_1^{\rm SpecGD})=0$.

Finally, $\mathcal E_{\rm sign}$ fails with probability at most $\delta/2$
and the alignment event with probability at most $\delta/2$. Thus their
intersection has probability at least $1-\delta$, and both conclusions hold
there.
\end{proof}

%% file: one_step_collapsed.tex
In this section, we prove Theorem~\ref{thm:onestep_collapsed}.
Actually, We show the following slightly sharper result.

\begin{theorem}[One-step interpolation and generalization under collapsed shortcuts]
\label{thm:onestep_collapsed_sharp}
Assume the sample $\calD_n$ is generated from the CSM. There exist universal
constants $c,C>0$ such that the following holds. Let $\delta\in(0,1)$ and
suppose that
\[
n\geq
\frac{2}{(1-2\rho)^2}\log\frac{2}{\delta}.
\]
Then, with probability at least $1-\delta$ over the training sample
$\mathcal D_n$, the following statements hold.
\begin{enumerate}
    \item \textbf{Generalization.} For an independent clean test sample $(X,y)$,
    \[
    \prob\left(
    y\langle W_1^{\mathrm{SpecGD}},X\rangle_F\leq0
    \,\middle|\,
    \mathcal D_n
    \right)
    \leq
    2\exp\left(
    -cD\min\left\{
    \frac{s^2}{\lambda^2},1
    \right\}
    \right).
    \]

    \item \textbf{Interpolation.} If, in addition, $\lambda>s\sqrt n$ and
    $D\geq
    C\left(1-\frac{s\sqrt n}{\lambda}\right)^{-2}
    n\log\frac{4n}{\delta},
    $
    then
    \[
    \widehat{\mathcal R}_n
    \left(W_1^{\mathrm{SpecGD}}\right)=0.
    \]
\end{enumerate}
\end{theorem}

\begin{proof}
Let $\mathcal E_{\rm sign}:=\{n_+>n_-\}$, where
$n_\pm:=|\{i:\xi_i=\pm1\}|$. By Hoeffding's inequality and the assumption
on $n$,
\[
\prob(\mathcal E_{\rm sign}^c)
\leq
\exp\left\{-\frac12n(1-2\rho)^2\right\}
\leq
\frac{\delta}{2}.
\]
We work on $\mathcal E_{\rm sign}$. 
By signal--shortcut block orthogonality,
\[
W_1^{\rm SpecGD}
=
\eta_0\left(
P_S+\boldsymbol h\boldsymbol v_0^\top
\right),
\qquad
\boldsymbol h
:=
\frac{\sum_{i=1}^n\tilde y_i\boldsymbol u_i}
{\left\|\sum_{i=1}^n\tilde y_i\boldsymbol u_i\right\|_2}.
\]

We first prove the conditional generalization bound. Let $(X,y)$ be an
independent clean test sample, so that
$yX=S+\lambda y\boldsymbol u\boldsymbol v_0^\top$. Then
$\eta_0^{-1}y\langle W_1^{\rm SpecGD},X\rangle_F
=s+\lambda\langle\boldsymbol h,y\boldsymbol u\rangle$. By spherical symmetry, $y\boldsymbol{u}$ is uniformly distributed on $\mathbb{S}^{D-1}$. 
Conditional on $\mathcal D_n$, $\boldsymbol h$ is a fixed unit vector, so
Lemma~\ref{lem:spherical-linear} gives
\[
\prob\left(
y\langle W_1^{\rm SpecGD},X\rangle_F\leq0
\,\middle|\,
\mathcal D_n
\right)
\leq
2\exp\left(
-cD\min\left\{\frac{s^2}{\lambda^2},1\right\}
\right).
\]

It remains to prove interpolation. Assume $\lambda>s\sqrt n$ and set
$\kappa:=1-s\sqrt n/\lambda\in(0,1)$. We show that, with probability at
least $1-\delta/2$,
$\min_{i\in[n]}\langle\boldsymbol h,\tilde y_i\boldsymbol u_i\rangle>s/\lambda$.
For each $i\in[n]$, set
$r_i:=\sum_{j\neq i}\langle\tilde y_i\boldsymbol u_i,\tilde y_j\boldsymbol u_j\rangle$.
Conditionally on $\tilde y_i\boldsymbol u_i$, the summands are independent, centered,
and sub-Gaussian, so, for a universal constant $c>0$ and $t\in(0,1)$,
$\prob(|r_i|>t\mid\tilde y_i\boldsymbol u_i)\leq
2\exp(-cDt^2/n)$. A union bound gives
\[
\prob\left(\max_{i\in[n]}|r_i|>t\right)
\leq
2n\exp\left(-c\frac{Dt^2}{n}\right).
\]
On the event $\max_i|r_i|\leq t$, since
$\|\sum_j\tilde y_j\boldsymbol u_j\|_2^2=n+\sum_jr_j\leq n(1+t)$, for $t\leq1/2$,
\[
\langle\boldsymbol h,\tilde y_i\boldsymbol u_i\rangle
=
\frac{1+r_i}{\|\sum_j\tilde y_j\boldsymbol u_j\|_2}
\geq
\frac{1-t}{\sqrt{n(1+t)}}
\geq
\frac{1-2t}{\sqrt n}.
\]
Taking $t=\kappa/4$, the dimensional assumption implies, for $C$
sufficiently large, that with probability at least $1-\delta/2$,
\[
\min_{i\in[n]}
\langle\boldsymbol h,\tilde y_i\boldsymbol u_i\rangle
\geq
\frac{1-\kappa/2}{\sqrt n}
>
\frac{1-\kappa}{\sqrt n}
=
\frac{s}{\lambda}.
\]

On the intersection of this event with $\mathcal E_{\rm sign}$, for every
$i\in[n]$,
\[
\eta_0^{-1}
\left\langle
W_1^{\rm SpecGD},
\widetilde y_iX_i
\right\rangle_F
=
\xi_i s+\lambda
\langle\boldsymbol h,\tilde y_i\boldsymbol u_i\rangle.
\]
If $i\in I_+$, this is positive, while if $i\in I_-$, it is larger than
$-s+\lambda(s/\lambda)=0$. Hence
$\widehat{\mathcal R}_n(W_1^{\rm SpecGD})=0$.

Finally, $\mathcal E_{\rm sign}$ and the interpolation event fail with
probability at most $\delta/2$ each, so their intersection has probability
at least $1-\delta$, and both conclusions hold there.
\end{proof}

%% file: auxiliary.tex
\subsection{Linear separability}

\begin{lemma}[Linear separability of the training sample]
\label{lem:linear-separability}
Suppose that $D\ge n$. Then, under the $K$-dimensional shortcut model
for any $1\le K\le D$, the training sample
$\{(X_i,\widetilde y_i)\}_{i=1}^n$ is linearly separable with respect
to the observed labels almost surely. More precisely, almost surely
there exists $W_{\rm sep}\in\mathbb R^{d\times d}$ such that
\[
    \widetilde y_i\langle W_{\rm sep},X_i\rangle_F=1,
    \qquad i\in[n].
\]
Consequently, both the spectral and Frobenius max-margin values are
strictly positive.
\end{lemma}

\begin{proof}
Recall that $X_i=y_iS+\lambda\boldsymbol u_i\boldsymbol v_i^\top$,
where $\boldsymbol u_i$ are drawn independently and uniformly from
the unit sphere of the $D$-dimensional shortcut subspace $\mathcal U$.
Since $D\ge n$, the vectors
$\boldsymbol u_1,\ldots,\boldsymbol u_n$ are linearly independent
almost surely.

On this event, let
$\boldsymbol u_1^\sharp,\ldots,\boldsymbol u_n^\sharp
\in\operatorname{span}\{\boldsymbol u_1,\ldots,\boldsymbol u_n\}$
denote the dual family satisfying
$\langle\boldsymbol u_i^\sharp,\boldsymbol u_j\rangle=\delta_{ij}$, and define
\[
    W_{\rm sep}
    :=
    \frac1\lambda
    \sum_{i=1}^n
    \widetilde y_i
    \boldsymbol u_i^\sharp\boldsymbol v_i^\top.
\]
Since $\boldsymbol u_i^\sharp$ and $\boldsymbol v_i$ belong to the left and
right shortcut subspaces, respectively, which are orthogonal to the
corresponding signal subspaces of $S$, we have
$\langle W_{\rm sep},S\rangle_F=0$. Moreover, for every $j\in[n]$,
\[
\langle W_{\rm sep},Z_j\rangle_F
=
\sum_{i=1}^n
\widetilde y_i
\langle\boldsymbol u_i^\sharp,\boldsymbol u_j\rangle
\langle\boldsymbol v_i,\boldsymbol v_j\rangle
=
\widetilde y_j\|\boldsymbol v_j\|_2^2
=
\widetilde y_j.
\]
Hence
$\widetilde y_j\langle W_{\rm sep},X_j\rangle_F
=\widetilde y_j\langle W_{\rm sep},y_jS+Z_j\rangle_F
=\widetilde y_j^2=1$, proving strict linear separability.

In particular, if
$\gamma_{\op}:=
\max_{\|W\|_{\op}\le1}\min_{i\in[n]}
\widetilde y_i\langle W,X_i\rangle_F$,
then normalization of $W_{\rm sep}$ gives
$\gamma_{\op}\ge\|W_{\rm sep}\|_{\op}^{-1}>0$.
The same argument with Frobenius normalization gives a strictly positive
Frobenius max-margin value.
\end{proof}

\subsection{Concentration inequalities}

\begin{lemma}[Spherical linear tails and small balls]
\label{lem:spherical-linear}
Let $\boldsymbol u\sim\Unif(\mathbb S^{D-1})$ and let
$\boldsymbol w\in\mathbb R^D\setminus\{0\}$ be fixed. There exist
universal constants $c,C,c_0,C_0>0$ such that
\[
\mathbb P\!\left(
|\langle\boldsymbol u,\boldsymbol w\rangle|
\geq t\|\boldsymbol w\|_2
\right)
\leq
2e^{-cDt^2},
\qquad t>0,
\]
and
\[
\mathbb P\!\left(
|\langle\boldsymbol u,\boldsymbol w\rangle|
\leq t\|\boldsymbol w\|_2
\right)
\leq
C\sqrt D\,t,
\qquad t>0.
\]
Moreover, for $0\leq t\leq1/2$,
\[
\mathbb P\!\left(
\langle\boldsymbol u,\boldsymbol w\rangle
\leq
-t\|\boldsymbol w\|_2
\right)
\geq
c_0e^{-C_0Dt^2}.
\]
\end{lemma}

\begin{proof}
By rotational invariance,
$Z:=\langle\boldsymbol u,\boldsymbol w\rangle/\|\boldsymbol w\|_2$
has the same distribution as the first coordinate of a uniform vector on
$\mathbb S^{D-1}$. The upper-tail bound follows directly from the standard
spherical concentration inequality
\citep[Theorem~3.4.5]{vershynin2026high}.

For the remaining claims, $Z$ has density
\[
f_D(z)
=
b_D(1-z^2)^{(D-3)/2},
\qquad
b_D
=
\frac{\Gamma(D/2)}
{\sqrt{\pi}\Gamma((D-1)/2)}
\asymp\sqrt D,
\qquad |z|<1.
\]
Hence
$\mathbb P(|Z|\le t)\le2t\|f_D\|_\infty\le C\sqrt D\,t$.
For $0\le t\le1/2$, using
$\log(1-z^2)\ge-Cz^2$ for $0\le z\le3/4$ gives
$f_D(z)\ge c\sqrt D\,e^{-CDz^2}$. Therefore,
\[
\mathbb P(Z\ge t)
\ge
\int_t^{t+(4\sqrt D)^{-1}} f_D(z)\,dz
\ge
c_0e^{-C_0Dt^2}.
\]
By symmetry, the same bound holds for $\mathbb P(Z\le-t)$.
\end{proof}

\begin{lemma}[Spherical projection norm]
\label{lem:spherical-projection-norm}
Let $\boldsymbol u\sim\Unif(\mathbb S^{D-1})$, and let $P$ be an
orthogonal projection of rank $r\leq D$. Then
\[
    \prob\left(
        \|P\boldsymbol u\|_2
        \leq
        \frac12\sqrt{\frac rD}
    \right)
    \leq
    2e^{-r/16}.
\]
\end{lemma}

\begin{proof}
Write $\boldsymbol u=\boldsymbol g/\|\boldsymbol g\|_2$ with
$\boldsymbol g\sim N(0,I_D)$. Then
$\|P\boldsymbol g\|_2^2\sim\chi_r^2$ and
$\|\boldsymbol g\|_2^2\sim\chi_D^2$. By the Laurent--Massart inequalities,
\[
    \prob\left(
        \|P\boldsymbol g\|_2^2<\frac r2
        \ \text{or}\
        \|\boldsymbol g\|_2^2>2D
    \right)
    \leq
    2e^{-r/16}.
\]
On the complementary event,
$\|P\boldsymbol u\|_2
=\|P\boldsymbol g\|_2/\|\boldsymbol g\|_2
\geq\frac12\sqrt{r/D}$.
\end{proof}

\begin{lemma}[Spherical Gram concentration]
\label{lem:gram-concentration}
Let $U=[\boldsymbol u_1,\ldots,\boldsymbol u_n]$,
where $\boldsymbol u_1,\ldots,\boldsymbol u_n$ are independent and uniformly
distributed on $\mathbb S^{D-1}$. Then there exist
universal constants $c_{\scriptscriptstyle G},C_{\scriptscriptstyle G}>0$
such that, if
$D\geq c_{\scriptscriptstyle G}\{n+\log(n/\delta)\}$, then, with probability
at least $1-\delta$,
\[
    \|U^\top U-I_n\|_{\op}
    \leq
    C_{\scriptscriptstyle G}
    \sqrt{\frac{n+\log(n/\delta)}{D}},
\]
and
\[
    \max_{i\neq j}
    |\boldsymbol u_i^\top\boldsymbol u_j|
    \leq
    C_{\scriptscriptstyle G}
    \sqrt{\frac{\log(2n^2/\delta)}{D}}.
\]
\end{lemma}

\begin{proof}
Write
$\boldsymbol u_i=\boldsymbol g_i/\|\boldsymbol g_i\|_2$,
where $\boldsymbol g_i\sim N(0,I_D)$ independently, and let
$G=[\boldsymbol g_1,\ldots,\boldsymbol g_n]\in\R^{D\times n}$.
The rows of $G$ are independent, mean-zero, isotropic, sub-Gaussian random
vectors in $\R^n$. Hence, by the standard singular-value bound
\cite[Theorem~4.6.1]{vershynin2026high}, for every $t\geq0$,
\[
\prob\left(
\left\|
\frac1DG^\top G-I
\right\|_{\op}
>
C\left[
\sqrt{\frac{n+t^2}{D}}
+
\frac{n+t^2}{D}
\right]
\right)
\leq
2e^{-t^2}.
\]
Taking $t^2=\log(8/\delta)$ gives, with probability at least $1-\delta/4$,
\[
\left\|
\frac1DG^\top G-I
\right\|_{\op}
\leq
C\left[
\sqrt{\frac{n+\log(8/\delta)}{D}}
+
\frac{n+\log(8/\delta)}{D}
\right].
\]
By increasing $c_{\scriptscriptstyle G}$ if necessary, on the same event
\[
\left\|
\frac1DG^\top G-I
\right\|_{\op}
\leq
2C\sqrt{\frac{n+\log(8n/\delta)}{D}}.
\]

We next control the normalization factors. By the Laurent--Massart
chi-square concentration inequality
\cite[Lemma~1]{LaurentMassart2000}, if $X\sim\chi_D^2$, then for every
$s>0$,
$\prob(X-D\geq2\sqrt{Ds}+2s)\leq e^{-s}$ and
$\prob(D-X\geq2\sqrt{Ds})\leq e^{-s}$. Consequently,
\[
\prob\left(
\left|
\frac XD-1
\right|
>
2\left[
\sqrt{\frac{s}{D}}+\frac{s}{D}
\right]
\right)
\leq
2e^{-s}.
\]
Applying this with $X=\|\boldsymbol g_i\|_2^2$ and
$s=\log(8n/\delta)$, and taking a union bound over $i\in[n]$, gives, with
probability at least $1-\delta/4$,
\[
\max_{i\leq n}
\left|
\frac{\|\boldsymbol g_i\|_2^2}{D}-1
\right|
\leq
2\left[
\sqrt{\frac{\log(8n/\delta)}{D}}
+
\frac{\log(8n/\delta)}{D}
\right].
\]
Increasing $c_{\scriptscriptstyle G}$ if necessary, we may further assume
\[
\max_{i\leq n}
\left|
\frac{\|\boldsymbol g_i\|_2^2}{D}-1
\right|
\leq
4\sqrt{\frac{\log(8n/\delta)}{D}}
\leq
\frac12.
\]

Define
\[
\Lambda
:=
\diag\left(
\frac{\sqrt D}{\|\boldsymbol g_1\|_2},
\ldots,
\frac{\sqrt D}{\|\boldsymbol g_n\|_2}
\right),
\qquad
X:=\frac1DG^\top G.
\]
Then $U=D^{-1/2}G\Lambda$, so
$U^\top U=\Lambda X\Lambda$ and
$U^\top U-I_n=\Lambda(X-I_n)\Lambda+(\Lambda^2-I_n)$.
The preceding normalization bound gives
$\|\Lambda\|_{\op}\leq\sqrt2$. Moreover,
\[
\left|
\frac{D}{\|\boldsymbol g_i\|_2^2}-1
\right|
=
\frac{
\left|\frac{\|\boldsymbol g_i\|_2^2}{D}-1\right|
}{
\frac{\|\boldsymbol g_i\|_2^2}{D}
}
\leq
2\left|
\frac{\|\boldsymbol g_i\|_2^2}{D}-1
\right|,
\]
and therefore
$\|\Lambda^2-I_n\|_{\op}
\leq2\max_{i\leq n}
|\|\boldsymbol g_i\|_2^2/D-1|$.
Combining these bounds yields
\begin{equation}\label{eq:grambound}
\begin{aligned}
\|U^\top U-I_n\|_{\op}
&\leq
\|\Lambda\|_{\op}^2\|X-I_n\|_{\op}
+
\|\Lambda^2-I_n\|_{\op} \\
&\leq
(4C+8)
\sqrt{\frac{n+\log(8n/\delta)}{D}}.
\end{aligned}
\end{equation}
Thus the desired Gram bound holds with probability at least $1-\delta/2$.

It remains to prove the coherence bound. By rotational invariance and
standard spherical concentration
\citep[Theorem~3.4.5]{vershynin2026high}, for every $i\neq j$ and $x>0$,
$\prob(|\boldsymbol u_i^\top\boldsymbol u_j|>x)\leq2e^{-cDx^2}$.
Setting $x:=\sqrt{\log(2n^2/\delta)/(cD)}$ and taking a union bound over
the at most $n(n-1)/2$ relevant inner products gives
\begin{equation}\label{eq:coherentbound}
\prob\left(
\max_{i\neq j}
|\boldsymbol u_i^\top\boldsymbol u_j|
>
\sqrt{\frac{\log(2n^2/\delta)}{cD}}
\right)
\leq
n^2
\exp\left(
-cD\frac{\log(2n^2/\delta)}{cD}
\right)
=
\frac\delta2.
\end{equation}
Combining \eqref{eq:grambound} and \eqref{eq:coherentbound} and choosing
$C_{\scriptscriptstyle G}$ sufficiently large proves the result.
\end{proof}

\begin{lemma}[Simultaneous spherical Gram concentration]
\label{lem:gram-concentration-2}
Let
\[
    U=[\boldsymbol u_1,\ldots,\boldsymbol u_n],
    \qquad
    \widehat V=
    [\widehat{\boldsymbol v}_1,\ldots,\widehat{\boldsymbol v}_n],
\]
where the columns of both matrices are independent and uniformly
distributed on $\mathbb S^{D-1}$. 
There exist universal constants
$c_{\scriptscriptstyle G},C_{\scriptscriptstyle G}>0$ such that, if
$D\geq c_{\scriptscriptstyle G}(n+\log(n/\delta))$, then, with probability
at least $1-\delta$,
\[
    \max\left\{
        \|U^\top U-I_n\|_{\op},
        \|\widehat V^\top\widehat V-I_n\|_{\op}
    \right\}
    \leq
    C_{\scriptscriptstyle G}
    \sqrt{\frac{n+\log(n/\delta)}{D}},
\]
and
\[
    \max_{i\neq j}
    \left\{
        |\boldsymbol u_i^\top\boldsymbol u_j|,
        |\widehat{\boldsymbol v}_i^\top
        \widehat{\boldsymbol v}_j|
    \right\}
    \leq
    C_{\scriptscriptstyle G}
    \sqrt{\frac{\log(4n^2/\delta)}{D}}.
\]
\end{lemma}

\begin{proof}
Apply Lemma~\ref{lem:gram-concentration} separately to $U$ and $\widehat V$,
each with failure probability $\delta/2$, take a union bound, and enlarging $C_G$.
\end{proof}

\begin{lemma}[Inverse spherical Gram concentration]
\label{lem:inverse-gram-concentration}
Let
\[
    U=[\boldsymbol u_1,\ldots,\boldsymbol u_n],
    \qquad
    G_U:=U^\top U,
\]
where $\boldsymbol u_1,\ldots,\boldsymbol u_n$ are independent and
uniformly distributed on $\mathbb S^{D-1}$. Let
$\boldsymbol\xi\in\{-1,1\}^n$ be independent of $U$.
There exist universal constants
$c_{\scriptscriptstyle I},C_{\scriptscriptstyle I}>0$ such that, if
$D\geq c_{\scriptscriptstyle I}n\log(8n/\delta)$, then, with probability
at least $1-\delta$,
\begin{equation}
    \frac12I_n
    \preceq
    G_U
    \preceq
    \frac32I_n,
\label{eq:inverse-gram-op}
\end{equation}
and
\begin{equation}
    \max_{\boldsymbol x\in\{\boldsymbol1,\boldsymbol\xi\}}
    \left\|
        G_U^{-1}\boldsymbol x-\boldsymbol x
    \right\|_\infty
    \leq
    C_{\scriptscriptstyle I}
    \sqrt{\frac{n\log(8n/\delta)}{D}}.
\label{eq:inverse-gram-entrywise}
\end{equation}
\end{lemma}

\begin{proof}
Apply Lemma~\ref{lem:gram-concentration} with failure probability $\delta/2$.
There exists a universal constant $C_{\scriptscriptstyle G}>0$ such that,
on an event $\mathcal E$ with $\prob(\mathcal E)\geq1-\delta/2$,
\begin{equation}
    \|G_U-I_n\|_{\op}
    \leq
    C_{\scriptscriptstyle G}
    \sqrt{\frac{n+\log(2n/\delta)}{D}}
    =:\varepsilon.
\label{eq:GU-op-event}
\end{equation}
By increasing $c_{\scriptscriptstyle I}$ if necessary, we may assume
$\varepsilon\leq1/4$. Hence, on $\mathcal E$,
$\frac34I_n\preceq G_U\preceq\frac54I_n$, which in particular implies
\eqref{eq:inverse-gram-op}.

It remains to establish \eqref{eq:inverse-gram-entrywise}. Fix
$\boldsymbol x\in\{\boldsymbol1,\boldsymbol\xi\}$ and $i\in[n]$.
After moving the $i$th coordinate to the first position, write
\[
    G_U
    =
    \begin{pmatrix}
        1 & \boldsymbol g_i^\top\\
        \boldsymbol g_i & G_{U,-i}
    \end{pmatrix},
    \qquad
    \boldsymbol x
    =
    \begin{pmatrix}
        x_i\\
        \boldsymbol x_{-i}
    \end{pmatrix},
\]
where
$G_{U,-i}:=U_{-i}^\top U_{-i}$ and
$\boldsymbol g_i:=U_{-i}^\top\boldsymbol u_i$.
The Schur-complement formula gives
\[
    (G_U^{-1}\boldsymbol x)_i
    =
    \frac{
        x_i-\boldsymbol g_i^\top
        G_{U,-i}^{-1}\boldsymbol x_{-i}
    }{
        1-\boldsymbol g_i^\top
        G_{U,-i}^{-1}\boldsymbol g_i
    },
\]
and therefore
\begin{equation}
\begin{aligned}
    (G_U^{-1}\boldsymbol x)_i-x_i
    &=
    \frac{
        x_i\boldsymbol g_i^\top
        G_{U,-i}^{-1}\boldsymbol g_i
        -
        \boldsymbol g_i^\top
        G_{U,-i}^{-1}\boldsymbol x_{-i}
    }{
        1-\boldsymbol g_i^\top
        G_{U,-i}^{-1}\boldsymbol g_i
    }.
\end{aligned}
\label{eq:inverse-gram-schur}
\end{equation}

We first control the denominator. Since $G_{U,-i}$ is a principal submatrix
of $G_U$, \eqref{eq:GU-op-event} implies, on $\mathcal E$,
$(1-\varepsilon)I_{n-1}\preceq G_{U,-i}\preceq
(1+\varepsilon)I_{n-1}$, and hence
$\|G_{U,-i}^{-1}\|_{\op}\leq(1-\varepsilon)^{-1}\leq2$.
Moreover,
$\|\boldsymbol g_i\|_2\leq\|G_U-I_n\|_{\op}\leq\varepsilon$.
Thus
$0\leq\boldsymbol g_i^\top G_{U,-i}^{-1}\boldsymbol g_i
\leq2\varepsilon^2\leq1/8$, and therefore
\begin{equation}
    1-\boldsymbol g_i^\top
    G_{U,-i}^{-1}\boldsymbol g_i
    \geq
    \frac12.
\label{eq:inverse-gram-denominator}
\end{equation}

It remains to control
$\boldsymbol g_i^\top G_{U,-i}^{-1}\boldsymbol x_{-i}$.
For each $i\in[n]$, define
$\mathcal E_{-i}:=
\{\|G_{U,-i}-I_{n-1}\|_{\op}\leq\varepsilon\}$.
Since $G_{U,-i}-I_{n-1}$ is a principal submatrix of $G_U-I_n$, we have
$\mathcal E\subseteq\mathcal E_{-i}$.

On $\mathcal E_{-i}$,
\[
\begin{aligned}
    \left\|
        U_{-i}G_{U,-i}^{-1}\boldsymbol x_{-i}
    \right\|_2^2
    &=
    \boldsymbol x_{-i}^\top
    G_{U,-i}^{-1}
    U_{-i}^\top U_{-i}
    G_{U,-i}^{-1}
    \boldsymbol x_{-i}\\
    &=
    \boldsymbol x_{-i}^\top
    G_{U,-i}^{-1}\boldsymbol x_{-i}
    \leq
    \frac{\|\boldsymbol x_{-i}\|_2^2}{1-\varepsilon}
    \leq
    2n.
\end{aligned}
\label{eq:leave-one-out-vector-norm}
\]
Furthermore,
\[
    \boldsymbol g_i^\top
    G_{U,-i}^{-1}\boldsymbol x_{-i}
    =
    \left\langle
        \boldsymbol u_i,\,
        U_{-i}G_{U,-i}^{-1}\boldsymbol x_{-i}
    \right\rangle.
\]
Conditionally on $U_{-i}$ and $\boldsymbol\xi$, the vector
$U_{-i}G_{U,-i}^{-1}\boldsymbol x_{-i}$ is fixed, whereas
$\boldsymbol u_i$ remains uniformly distributed on $\mathbb S^{D-1}$.
Standard spherical concentration therefore gives, for every $t>0$,
\[
\mathbb P\left(
    \mathcal E_{-i}
    \cap
    \left\{
        \left|
        \boldsymbol g_i^\top
        G_{U,-i}^{-1}\boldsymbol x_{-i}
        \right|>t
    \right\}
    \,\middle|\,
    U_{-i},\boldsymbol\xi
\right)
\leq
2\exp\left(
    -c\frac{Dt^2}{n}
\right),
\]
where we used \eqref{eq:leave-one-out-vector-norm}.

Set $L:=\log(8n/\delta)$ and
$t:=C_0\sqrt{nL/D}$, where $C_0>0$ is a sufficiently large universal
constant. Since $\mathcal E\subseteq\mathcal E_{-i}$, a union bound over
$i\in[n]$ and $\boldsymbol x\in\{\boldsymbol1,\boldsymbol\xi\}$ gives
\[
\mathbb P\left(
    \mathcal E
    \cap
    \left\{
    \max_{\substack{i\in[n]\\
    \boldsymbol x\in\{\boldsymbol1,\boldsymbol\xi\}}}
    \left|
        \boldsymbol g_i^\top
        G_{U,-i}^{-1}\boldsymbol x_{-i}
    \right|
    >t
    \right\}
\right)
\leq
4n\exp(-cC_0^2L)
\leq
\frac{\delta}{2}.
\label{eq:inverse-gram-linear-term}
\]
Thus, outside an event of probability at most $\delta/2$, simultaneously for
every $i\in[n]$ and
$\boldsymbol x\in\{\boldsymbol1,\boldsymbol\xi\}$,
\begin{equation}
    \left|
        \boldsymbol g_i^\top
        G_{U,-i}^{-1}\boldsymbol x_{-i}
    \right|
    \leq
    C_0\sqrt{\frac{n\log(8n/\delta)}{D}}.
\label{eq:inverse-gram-linear-bound}
\end{equation}

Finally, by \eqref{eq:GU-op-event},
$\boldsymbol g_i^\top G_{U,-i}^{-1}\boldsymbol g_i
\leq2\varepsilon^2
\leq C\{n+\log(2n/\delta)\}/D$.
Under the dimensional assumption
$D\geq c_{\scriptscriptstyle I}n\log(8n/\delta)$,
\[
\frac{n+\log(2n/\delta)}{D}
\leq
C\sqrt{\frac{n\log(8n/\delta)}{D}}.
\]
Combining this estimate with
\eqref{eq:inverse-gram-denominator},
\eqref{eq:inverse-gram-linear-bound}, and
\eqref{eq:inverse-gram-schur}, we obtain
\[
    \left|
        (G_U^{-1}\boldsymbol x)_i-x_i
    \right|
    \leq
    C_{\scriptscriptstyle I}
    \sqrt{\frac{n\log(8n/\delta)}{D}}
\]
simultaneously for all $i\in[n]$ and
$\boldsymbol x\in\{\boldsymbol1,\boldsymbol\xi\}$.
Together with $\prob(\mathcal E)\geq1-\delta/2$ and
\eqref{eq:inverse-gram-linear-term}, this event has probability at least
$1-\delta$, proving \eqref{eq:inverse-gram-entrywise}.
\end{proof}

\begin{lemma}[Spherical bilinear tails]
\label{lem:fresh-bilinear-tail}
Let \(H\in\R^{D\times D}\) be deterministic with
\(\operatorname{rank}(H)=r_H\geq1\). Let
\(\boldsymbol u,\boldsymbol v\sim\operatorname{Unif}(\mathbb S^{D-1})\) be
independent. Then there is a universal constant \(c>0\) such that, for all
\(t>0\),
\begin{equation}
\label{eq:fresh-bilinear-upper-tail}
\prob\left(
\left|
\left\langle H,\boldsymbol u\boldsymbol v^\top\right\rangle_F
\right|
\geq t
\right)
\leq
4\exp\left[
-c\min\left\{
\frac{D^2t^2}{\|H\|_F^2},
\frac{Dt}{\|H\|_{\op}}
\right\}
\right].
\end{equation}
Moreover, there exist universal constants \(K_0,c_0,C_0>0\) such that, if
\(D\geq K_0\|H\|_F^2/\|H\|_{\op}^2\), then, for all
\(0<t\leq\|H\|_F^2/(D\|H\|_{\op})\),
\begin{equation}
\label{eq:fresh-bilinear-lower-tail}
\prob\left(
\left|
\left\langle H,\boldsymbol u\boldsymbol v^\top\right\rangle_F
\right|
\geq t
\right)
\geq
c_0
\exp\left[
-C_0\frac{D^2t^2}{\|H\|_F^2}
\right].
\end{equation}
\end{lemma}

\begin{proof}
Let $H=U\Sigma V^\top$ be the compact SVD of $H$, where
$\Sigma=\diag(\sigma_1,\ldots,\sigma_{r_H})$ and
$\sigma_1\geq\cdots\geq\sigma_{r_H}>0$. By rotational invariance, we may
assume without loss of generality that
$U=V=(I_{r_H},0)^\top$. Thus, with
$T:=\langle H,\boldsymbol u\boldsymbol v^\top\rangle_F$,
we have $T=\boldsymbol u^\top H\boldsymbol v
=\sum_{j=1}^{r_H}\sigma_ju_jv_j$.

Let $\boldsymbol g,\boldsymbol h\sim N(0,I_D)$ be independent. Then
$\boldsymbol u=\boldsymbol g/\|\boldsymbol g\|_2$ and
$\boldsymbol v=\boldsymbol h/\|\boldsymbol h\|_2$, and hence
\begin{equation}
\label{eq:fresh-bilinear-gaussian-representation-general}
T
=
\frac{S_H}{\|\boldsymbol g\|_2\|\boldsymbol h\|_2},
\qquad
S_H:=\sum_{j=1}^{r_H}\sigma_jg_jh_j.
\end{equation}

We first prove \eqref{eq:fresh-bilinear-upper-tail}. Let
$\mathcal E_{\mathrm{den}}^-:=
\{\|\boldsymbol g\|_2^2\geq D/2,\,
\|\boldsymbol h\|_2^2\geq D/2\}$.
By the standard lower-tail bound for chi-square random variables,
$\prob((\mathcal E_{\mathrm{den}}^-)^c)\leq2e^{-c_1D}$ for a universal
constant $c_1>0$. On $\mathcal E_{\mathrm{den}}^-$,
\eqref{eq:fresh-bilinear-gaussian-representation-general} implies
$|T|\geq t\Rightarrow |S_H|\geq Dt/2$. Therefore
\begin{equation}
\label{eq:fresh-bilinear-upper-union-general}
\prob(|T|\geq t)
\leq
2e^{-c_1D}
+
\prob\left(
|S_H|\geq\frac{Dt}{2}
\right).
\end{equation}
Since $S_H=\sum_{j=1}^{r_H}\sigma_jg_jh_j$ is a weighted sum of independent
centered sub-exponential random variables, Bernstein's inequality gives a
universal constant $c_2>0$ such that, for all $x>0$,
\[
\prob(|S_H|\geq x)
\leq
2\exp\left[
-c_2\min\left\{
\frac{x^2}{\sum_{j=1}^{r_H}\sigma_j^2},
\frac{x}{\sigma_1}
\right\}
\right].
\]
Since $\sum_j\sigma_j^2=\|H\|_F^2$ and
$\sigma_1=\|H\|_{\op}$, taking $x=Dt/2$ gives, after changing the universal
constant,
\[
\prob\left(
|S_H|\geq\frac{Dt}{2}
\right)
\leq
2\exp\left[
-c_3\min\left\{
\frac{D^2t^2}{\|H\|_F^2},
\frac{Dt}{\|H\|_{\op}}
\right\}
\right].
\]
If $0<t\leq\|H\|_{\op}$, then
$\min\{D^2t^2/\|H\|_F^2,Dt/\|H\|_{\op}\}\leq D$, so the denominator error
in \eqref{eq:fresh-bilinear-upper-union-general} is absorbed into the same
exponential. This proves \eqref{eq:fresh-bilinear-upper-tail} for
$0<t\leq\|H\|_{\op}$. If $t>\|H\|_{\op}$, its left-hand side is zero because
$|T|\leq\|H\|_{\op}$. Thus \eqref{eq:fresh-bilinear-upper-tail} holds for
all $t>0$.

We now prove \eqref{eq:fresh-bilinear-lower-tail}. Assume
$D\geq K_0\|H\|_F^2/\|H\|_{\op}^2$, where $K_0\geq1$ will be chosen
sufficiently large, and define
$\sigma_g^2:=\sum_{j=1}^{r_H}\sigma_j^2g_j^2$ and
\[
\mathcal E_g
:=
\left\{
\sigma_g^2\geq\frac{\|H\|_F^2}{2},
\quad
\|\boldsymbol g\|_2^2\leq2D
\right\}.
\]
We first show that $\mathcal E_g$ has probability bounded below by a universal
constant. We have $\mathbb E\sigma_g^2=\|H\|_F^2$, while
\[
\mathbb E(\sigma_g^2)^2
=
\sum_{j=1}^{r_H}\sigma_j^4\mathbb E g_j^4
+
\sum_{i\neq j}\sigma_i^2\sigma_j^2
\mathbb E g_i^2\mathbb E g_j^2
=
\|H\|_F^4+2\sum_{j=1}^{r_H}\sigma_j^4
\leq
3\|H\|_F^4.
\]
Hence, by Paley--Zygmund,
$\prob(\sigma_g^2\geq\|H\|_F^2/2)\geq1/12$.
On the other hand, the standard chi-square upper-tail bound gives
$\prob(\|\boldsymbol g\|_2^2>2D)\leq e^{-cD}$. Since
$D\geq K_0\|H\|_F^2/\|H\|_{\op}^2\geq K_0$, choosing $K_0$ sufficiently
large gives $e^{-cD}\leq1/24$. Therefore
\begin{equation}\label{eq:fresh-bilinear-Eg-probability-general}
\prob(\mathcal E_g)
\geq
\frac1{12}-\frac1{24}
=
\frac1{24}.
\end{equation}

Fix $\boldsymbol g\in\mathcal E_g$, and let
$\boldsymbol e_g\in\R^D$ be the unit vector whose first $r_H$ coordinates are
$(\sigma_1g_1,\ldots,\sigma_{r_H}g_{r_H})/\sigma_g$ and whose remaining
coordinates are zero. Decompose
$\boldsymbol h=Z\boldsymbol e_g+\boldsymbol h_\perp$, where
$Z\sim N(0,1)$, $\boldsymbol h_\perp\sim N(0,I_{D-1})$, and $Z$ is
independent of $\boldsymbol h_\perp$. Then
$S_H=\sigma_gZ$ and
$\|\boldsymbol h\|_2^2=Z^2+\|\boldsymbol h_\perp\|_2^2$.

We use the following elementary Gaussian consequence: there exist universal
constants $b_1,c_1',C_1'>0$ such that, for all $D\geq1$ and
$0\leq a\leq b_1\sqrt D$,
\[
\prob\left(
|Z|\geq a,\quad Z^2\leq\frac D2
\right)
\geq
c_1'e^{-C_1'a^2}.
\]
Also, by a chi-square bound, there is a universal constant $p>0$ such that
$\prob(\|\boldsymbol h_\perp\|_2^2\leq3D/2)\geq p$.
Since $Z$ and $\boldsymbol h_\perp$ are independent, for
$a:=x/\sigma_g\leq b_1\sqrt D$,
\[
\prob\left(
|S_H|\geq x,\quad
\|\boldsymbol h\|_2^2\leq2D
\,\middle|\,
\boldsymbol g
\right)
\geq
pc_1'e^{-C_1'a^2}.
\]
On $\mathcal E_g$, $\sigma_g^2\geq\|H\|_F^2/2$, and hence
$a^2\leq2x^2/\|H\|_F^2$. Thus, after changing universal constants,
\begin{equation}
\label{eq:fresh-bilinear-conditional-lower-general}
\prob\left(
|S_H|\geq x,\quad
\|\boldsymbol h\|_2^2\leq2D
\,\middle|\,
\boldsymbol g
\right)
\geq
c_2'e^{-C_2'x^2/\|H\|_F^2},
\end{equation}
whenever $0<x\leq2\|H\|_F^2/\|H\|_{\op}$, provided $K_0$ is sufficiently
large. Indeed, in this range, on $\mathcal E_g$,
$a=x/\sigma_g
\leq\sqrt2\,x/\|H\|_F
\leq2\sqrt2\,\|H\|_F/\|H\|_{\op}$.
Since
$D\geq K_0\|H\|_F^2/\|H\|_{\op}^2$, this gives
$a\leq2\sqrt2\,\sqrt D/\sqrt{K_0}$. Choosing
$K_0\geq8/b_1^2$ ensures $a\leq b_1\sqrt D$.

Integrating \eqref{eq:fresh-bilinear-conditional-lower-general} over
$\mathcal E_g$ and using
\eqref{eq:fresh-bilinear-Eg-probability-general}, we obtain
\begin{equation}
\label{eq:fresh-bilinear-joint-lower-general}
\prob\left(
|S_H|\geq x,\quad
\|\boldsymbol g\|_2^2\leq2D,\quad
\|\boldsymbol h\|_2^2\leq2D
\right)
\geq
\frac{c_2'}{24}e^{-C_2'x^2/\|H\|_F^2},
\qquad
0<x\leq\frac{2\|H\|_F^2}{\|H\|_{\op}}.
\end{equation}
Absorbing $c_2'/24$ into a universal constant $c_0'>0$ gives
\[
\prob\left(
|S_H|\geq x,\quad
\|\boldsymbol g\|_2^2\leq2D,\quad
\|\boldsymbol h\|_2^2\leq2D
\right)
\geq
c_0'e^{-C_2'x^2/\|H\|_F^2},
\qquad
0<x\leq\frac{2\|H\|_F^2}{\|H\|_{\op}}.
\]

Finally, take $x:=2Dt$. If
$0<t\leq\|H\|_F^2/(D\|H\|_{\op})$, then
$0<x\leq2\|H\|_F^2/\|H\|_{\op}$. On the event in the preceding display,
$\|\boldsymbol g\|_2\|\boldsymbol h\|_2\leq2D$, so
\eqref{eq:fresh-bilinear-gaussian-representation-general} gives
$|T|\geq|S_H|/(2D)\geq x/(2D)=t$. Therefore
\[
\prob(|T|\geq t)
\geq
c_0'\exp\left[
-4C_2'\frac{D^2t^2}{\|H\|_F^2}
\right].
\]
Renaming constants proves \eqref{eq:fresh-bilinear-lower-tail}.
\end{proof}

\subsection{Polar factor}

Throughout this section, for a matrix \(X\), write a compact singular value
decomposition as
\[
X=U\Sigma V^\top,\qquad
\Sigma=\operatorname{diag}(\sigma_1,\ldots,\sigma_r),\qquad
\sigma_i>0.
\]

\begin{lemma}[Basic properties of the compact polar factor]
\label{lem:polar-basic}
Let \(X\neq0\). Then
\[
\|\pol(X)\|_{\op}=1,\qquad
\langle\pol(X),X\rangle_F=\|X\|_1,\qquad
\pol(cX)=\operatorname{sgn}(c)\pol(X)
\]
for every \(c\in\mathbb R\setminus\{0\}\).
\end{lemma}

\begin{proof}
Let $X=U\Sigma V^\top$ be a compact singular value decomposition. Then
$\pol(X)=UV^\top$, so
$\|\pol(X)\|_{\op}=\|UV^\top\|_{\op}=1$. Moreover,
\[
\langle\pol(X),X\rangle_F
=
\operatorname{tr}\bigl((UV^\top)^\top U\Sigma V^\top\bigr)
=
\operatorname{tr}(\Sigma)
=
\|X\|_1.
\]
If $c>0$, then $cX=U(c\Sigma)V^\top$, so $\pol(cX)=\pol(X)$.
If $c<0$, then $cX=U((-c)\Sigma)(-V)^\top$, so
$\pol(cX)=-\pol(X)$.
\end{proof}

\begin{lemma}[Spectral--Schatten-\(1\) duality and subgradient]
\label{lem:spectral-schatten-duality}
For every matrix \(X\),
\[
\|X\|_1
=
\sup_{\|G\|_{\op}\leq1}
\langle G,X\rangle_F.
\]
The supremum is attained at \(G=\pol(X)\). Moreover,
\(\pol(X)\in\partial\|X\|_1\); that is, for every matrix \(Y\),
\[
\|Y\|_1
\geq
\|X\|_1
+
\langle\pol(X),Y-X\rangle_F.
\]
\end{lemma}

\begin{proof}
If $X=0$, the claims are immediate. Suppose $X\neq0$, and write
$X=U\Sigma V^\top$. For any $G$ satisfying $\|G\|_{\op}\leq1$,
\[
\langle G,X\rangle_F
=
\operatorname{tr}\bigl((U^\top GV)^\top\Sigma\bigr)
\leq
\sum_{i=1}^r\sigma_i
=
\|X\|_1,
\]
since $\|U^\top GV\|_{\op}\leq1$ implies that each diagonal entry of
$U^\top GV$ has absolute value at most one. Taking
$G=\pol(X)=UV^\top$ gives equality by Lemma~\ref{lem:polar-basic},
proving the duality.

Consequently, for every matrix $Y$,
$\|Y\|_1\geq\langle\pol(X),Y\rangle_F
=\|X\|_1+\langle\pol(X),Y-X\rangle_F$, where the equality follows from
Lemma~\ref{lem:polar-basic}. This is the subgradient inequality.
\end{proof}

\begin{lemma}[Gradient of the Schatten \(1\)-norm at a full-rank matrix]
\label{lem:gradient-schatten-one}
Let \(X\in\mathbb R^{m\times n}\) have full column rank. Then
\(\nabla\|X\|_1=\pol(X)\). Equivalently, for every perturbation
\(E\in\mathbb R^{m\times n}\),
\[
\dd\|\cdot\|_1(X)[E]
=
\left.\frac{d}{dt}\right|_{t=0}\|X+tE\|_1
=
\langle\pol(X),E\rangle_F.
\]
The analogous statement holds for full row rank matrices by applying the
result to \(X^\top\).
\end{lemma}

\begin{proof}
Assume $X$ has full column rank and set $H:=(X^\top X)^{1/2}$.
Then $H\succ0$, $\|X\|_1=\operatorname{tr}(H)$, and
$\pol(X)=XH^{-1}$.

For an arbitrary perturbation $E$, let
$H(t):=((X+tE)^\top(X+tE))^{1/2}$. Since $X$ has full column rank, so does
$X+tE$ for all $t$ sufficiently close to $0$, and hence
$\|X+tE\|_1=\operatorname{tr}(H(t))$ for such $t$.
The square-root map is smooth on the positive definite cone, for instance
by the holomorphic functional calculus representation
\[
A^{1/2}
=
\frac{1}{2\pi i}
\int_\Gamma z^{1/2}(zI-A)^{-1}\,dz,
\]
with $\Gamma$ enclosing the spectrum of $A$ in the right half-plane.
Thus $H(t)$ is differentiable near $t=0$.

Writing $\dot H:=H'(0)$ and differentiating
$H(t)^2=(X+tE)^\top(X+tE)$ gives
$H\dot H+\dot H H=X^\top E+E^\top X$.
Multiplying by $H^{-1}$, taking traces, and using cyclicity yields
\[
2\operatorname{tr}(\dot H)
=
2\operatorname{tr}\bigl((XH^{-1})^\top E\bigr)
=
2\langle\pol(X),E\rangle_F.
\]
Therefore
$\left.\frac{d}{dt}\right|_{t=0}\|X+tE\|_1
=\operatorname{tr}(\dot H)=\langle\pol(X),E\rangle_F$.
Hence $\nabla\|X\|_1=\pol(X)$. The full row rank case follows by applying
the result to $X^\top$.
\end{proof}

\begin{lemma}[Local regularity of the polar factor]
\label{lem:polar-continuity}
Let \(X\in\mathbb R^{m\times n}\) have full column rank. Then the map
\(Y\mapsto\pol(Y)\) is locally Lipschitz in a neighborhood of \(X\).
In particular, if \(X_j\to X\), then
\(\pol(X_j)\to\pol(X)\). The same statements hold for full row rank
matrices.
\end{lemma}

\begin{proof}
The set of full-column-rank matrices is open, and on this set
$\pol(Y)=Y(Y^\top Y)^{-1/2}$. Since
$A\mapsto A^{-1/2}$ is smooth on the positive definite cone
(for instance by the holomorphic functional calculus), the polar map is
smooth, and hence locally Lipschitz, in a neighborhood of $X$.
The continuity assertion follows immediately. The full row rank case follows
by applying the result to $X^\top$.
\end{proof}

\begin{lemma}[Orthonormalization identity]
\label{lem:polar-orthonormalization}
Let \(U\in\mathbb R^{m_1\times n}\) and \(V\in\mathbb R^{m_2\times n}\) have
full column rank. Let \(G_U:=U^\top U\) and \(G_V:=V^\top V\).
Then, for every \(M\in\mathbb R^{n\times n}\), the nonzero singular values of
\(UMV^\top\) coincide with those of \(G_U^{1/2}MG_V^{1/2}\):
\[
    \sigma_j(UMV^\top)
    =
    \sigma_j(G_U^{1/2}MG_V^{1/2}),
    \qquad j\in[n].
\]
Consequently,
\[
    \operatorname{rank}(UMV^\top)
    =
    \operatorname{rank}(M),
    \qquad
    \|UMV^\top\|_1
    =
    \|G_U^{1/2}MG_V^{1/2}\|_1.
\]
Moreover,
\[
\pol(UMV^\top)
=
\overline U\,
\pol(G_U^{1/2}MG_V^{1/2})\,
\overline V^\top,
\qquad
\overline U:=UG_U^{-1/2},\quad
\overline V:=VG_V^{-1/2}.
\]
\end{lemma}

\begin{proof}
By definition,
$\overline U^\top\overline U
=G_U^{-1/2}U^\top UG_U^{-1/2}=I$, and similarly
$\overline V^\top\overline V=I$. Also
$U=\overline U G_U^{1/2}$ and $V=\overline V G_V^{1/2}$, so
\[
UMV^\top
=
\overline U
\left(
G_U^{1/2}MG_V^{1/2}
\right)
\overline V^\top.
\]
Multiplication on the left and right by matrices with orthonormal columns
preserves all nonzero singular values. Hence
$\|UMV^\top\|_1
=\|G_U^{1/2}MG_V^{1/2}\|_1$, and the same singular value decomposition
shows that the compact polar factor is
$\overline U\,\pol(G_U^{1/2}MG_V^{1/2})\,\overline V^\top$.
\end{proof}

\begin{lemma}[Differential of the polar factor at a positive diagonal matrix]
\label{lem:polar-differential-diagonal}
Let \(\Lambda=\operatorname{diag}(d_1,\ldots,d_n)\), where \(d_i>0\). For a
perturbation \(E\in\mathbb R^{n\times n}\), define \(X(t):=\Lambda+tE\). Then
\[
\pol(X(t))=I+tK+o(t),
\]
where \(K\) is skew-symmetric and is determined by
\(K\Lambda+\Lambda K=E-E^\top\). Equivalently,
\[
K_{ij}=\frac{E_{ij}-E_{ji}}{d_i+d_j}\quad(i\neq j),
\qquad
K_{ii}=0.
\]
\end{lemma}

\begin{proof}
For sufficiently small $t$, the matrix $X(t)$ is invertible. Write its polar
decomposition as $X(t)=Q(t)H(t)$, where $Q(t)$ is orthogonal and $H(t)$ is
symmetric positive definite. Since $\Lambda$ is positive diagonal,
$Q(0)=I$ and $H(0)=\Lambda$. Let $K:=Q'(0)$. Differentiating
$Q(t)^\top Q(t)=I$ at $t=0$ gives $K^\top+K=0$, so $K$ is skew-symmetric.

Differentiating $\Lambda+tE=Q(t)H(t)$ at $t=0$ gives
$E=K\Lambda+H'(0)$. Since $H'(0)$ is symmetric, taking transposes gives
$E^\top=-\Lambda K+H'(0)$. Subtracting yields
$E-E^\top=K\Lambda+\Lambda K$. Since $\Lambda$ is diagonal,
$(K\Lambda+\Lambda K)_{ij}=(d_i+d_j)K_{ij}$, and hence
$K_{ij}=(E_{ij}-E_{ji})/(d_i+d_j)$ for $i\neq j$.
For $i=j$, skew-symmetry gives $K_{ii}=0$.
\end{proof}

\begin{lemma}[Hessian of the Schatten \(1\)-norm at a positive diagonal matrix]
\label{lem:schatten-one-hessian-diagonal}
Let \(\Lambda=\operatorname{diag}(d_1,\ldots,d_n)\), where \(d_i>0\). Then,
for every perturbation \(E\in\mathbb R^{n\times n}\),
\[
\nabla^2\|\cdot\|_1(\Lambda)[E,E]
=
\sum_{i<j}
\frac{(E_{ij}-E_{ji})^2}{d_i+d_j}.
\]
Consequently,
\[
\|\Lambda+tE\|_1
=
\operatorname{tr}(\Lambda)
+
t\operatorname{tr}(E)
+
\frac{t^2}{2}
\sum_{i<j}
\frac{(E_{ij}-E_{ji})^2}{d_i+d_j}
+
o(t^2).
\]
\end{lemma}

\begin{proof}
Since $\Lambda$ is invertible, the Schatten $1$-norm is differentiable in a
neighborhood of $\Lambda$, and by
Lemma~\ref{lem:gradient-schatten-one},
$\nabla\|X\|_1=\pol(X)$ for $X$ near $\Lambda$. Thus, for
$X(t)=\Lambda+tE$,
\[
\frac{d}{dt}\|X(t)\|_1
=
\langle\pol(X(t)),E\rangle_F.
\]
At $t=0$, $\pol(\Lambda)=I$, so the first derivative is
$\langle I,E\rangle_F=\operatorname{tr}(E)$.
By Lemma~\ref{lem:polar-differential-diagonal},
$\pol(X(t))=I+tK+o(t)$, where
$K_{ij}=(E_{ij}-E_{ji})/(d_i+d_j)$ for $i\neq j$ and $K_{ii}=0$.
Therefore,
\[
\nabla^2\|\cdot\|_1(\Lambda)[E,E]
=
\left.\frac{d}{dt}\right|_{t=0}
\langle\pol(\Lambda+tE),E\rangle_F
=
\langle K,E\rangle_F.
\]
Since $K$ is skew-symmetric,
\[
\langle K,E\rangle_F
=
\sum_{i<j}
K_{ij}(E_{ij}-E_{ji})
=
\sum_{i<j}
\frac{(E_{ij}-E_{ji})^2}{d_i+d_j}.
\]
The Taylor expansion follows from differentiability of the gradient in a
neighborhood of $\Lambda$.
\end{proof}

%% file: app_xp.tex
We provide additional numerical experiments.

\paragraph{Early stopping mitigates overfitting for SpecGD, but not GD in the dispersed shortcut model.}

Figure~\ref{fig:ktau} shows how generalization evolves along training as the shortcut dimension $K$ and the cumulative optimization time
    $\tau := \sum_{\ell=0}^{T}\eta_\ell$ vary. For SpecGD, increasing the training time can substantially degrade test performance, especially as the shortcuts become more dispersed, while an intermediate stopping time preserves good generalization over a much broader range of $K$. In contrast, GD exhibits little comparable benefit from early stopping: once interpolation is reached, its generalization behavior is largely determined by the shortcut geometry. Thus, training time provides an additional control over the implicit bias of SpecGD, whereas it has a much weaker effect on GD.
\begin{figure}
    \centering
    \includegraphics[width=\textwidth]{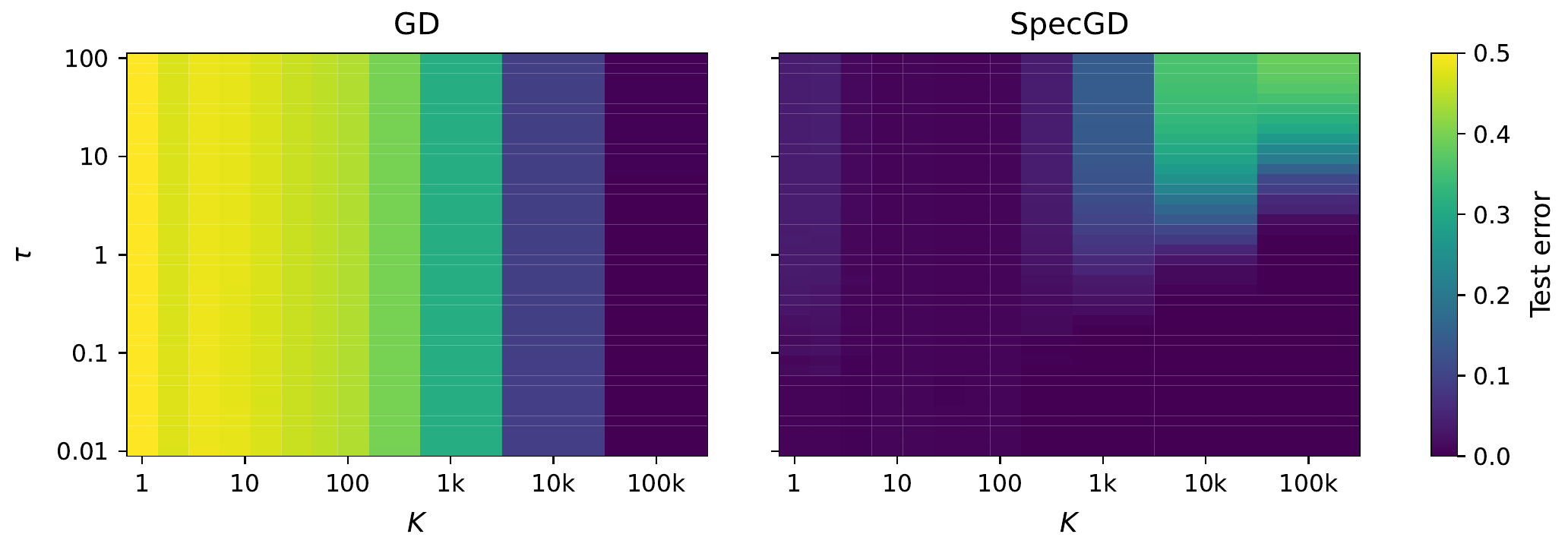}
    \caption{ Test error as a function of shortcut dimension $K$ and cumulative optimization time
    $\tau := \sum_{\ell=0}^{T}\eta_\ell$
    for GD and SpecGD.
    Parameters: $n=100$, $r=1$, $D=d-r=10^5$, $\rho=0.45$, $\lambda=125$,
    $\eta=10^{-2}$, and $T=10^4$}
    \label{fig:ktau}
\end{figure}

\paragraph{Influence of the sample size $n$.} As illustrated by Figure~\ref{fig:nd}, under the dispersed model, increasing the sample size improves the performance of SpecGD.

\begin{figure}
    \centering
    \includegraphics[width=\textwidth]{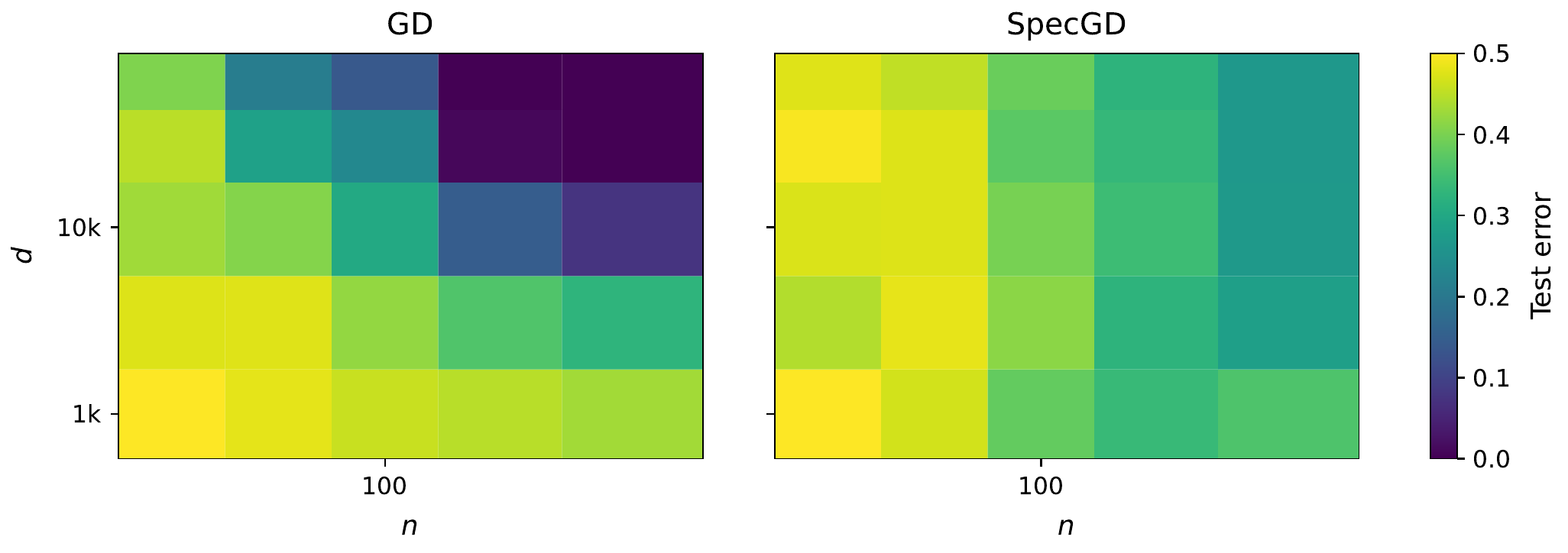}
    \caption{
Test error as a function of the sample size $n$ and ambient dimension $d$ for GD and SpecGD in the random dispersed model, with $r=1$ and $K=D=d-r$. Parameters: $\rho=0.45$, $\lambda=125$, and $\eta=10^{-2}$. We compare all configurations at fixed normalized optimization time $n\tau/D=1.5$. Results are averaged over $10$ independent runs.
}
    \label{fig:nd}
\end{figure}